\documentclass[letterpaper,journal]{IEEEtran}

\usepackage{amsmath,amsfonts}
\usepackage{array}
\usepackage[caption=false,font=normalsize,labelfont=sf,textfont=sf]{subfig}
\usepackage{textcomp}
\usepackage{stfloats}
\usepackage{url}
\usepackage{verbatim}
\usepackage{graphicx}
\usepackage[utf8]{inputenc} 
\usepackage[T1]{fontenc}    
\usepackage{booktabs}       
\usepackage{nicefrac}       
\usepackage{microtype}      
\usepackage{xcolor}         

\usepackage{algpseudocode}
\usepackage[ruled,vlined]{algorithm2e}
\usepackage{booktabs}   
\usepackage{mathabx}
\usepackage{bm}
\usepackage{mathtools}
\usepackage{multirow}
\usepackage{cuted}
\usepackage{amsthm}
\usepackage{scalerel,stackengine}
\usepackage{enumitem}
\usepackage{wrapfig}
\usepackage{etoc}
\usepackage{tikz}
\usepackage{pifont}
\usepackage{colortbl}
\usepackage{calc}
\usepackage{fp}
\usepackage{float}
\usepackage{placeins}
\usepackage[citecolor=blue,colorlinks]{hyperref}

\stackMath
\newcommand\reallywidehat[1]{%
\savestack{\tmpbox}{\stretchto{%
  \scaleto{%
    \scalerel*[\widthof{\ensuremath{#1}}]{\kern-.6pt\bigwedge\kern-.6pt}%
    {\rule[-\textheight/2]{1ex}{\textheight}}
  }{\textheight}%
}{0.5ex}}%
\stackon[1pt]{#1}{\tmpbox}%
}
\let\origref\ref
\renewcommand{\eqref}[1]{\textup{Eq.~{(\origref{#1})}}}
\newtheorem{definition}{Definition}

\newtheorem{proposition}{Proposition}

\newtheorem{remark}{Remark}

\definecolor{color1}{rgb}{0,0,0}  
\definecolor{color2}{rgb}{0.45,0.45,0.45}  
\definecolor{commentcolor}{RGB}{47, 157, 215
} 

\newcommand{\progressbar}[2][0.7cm]{%
    \textcolor{color1}{\rule{#1 * \real{#2} / 128}{1.5ex}}%
    \textcolor{color2!15}{\rule{#1 - #1 * \real{#2} / 128}{1.5ex}}}

\renewcommand{\vec}[1]{\mathbf{#1}}
\definecolor{paralinrow}{RGB}{245,235,255}

\newcommand{\norm}[1]{\left\|#1\right\|}

\def\BibTeX{{\rm B\kern-.05em{\sc i\kern-.025em b}\kern-.08em
    T\kern-.1667em\lower.7ex\hbox{E}\kern-.125emX}}

\begin{document}

\title{ParaLin++: Accelerating Parallel Diffusion Integrator via Intrinsic Partially Linear Structure}

\author{Jianrong Lu, Zhiyu Zhu, Hui Liu, and Junhui Hou
\thanks{This project was supported in part by the NSFC Excellent Young Scientists Fund 62422118, in part by the Hong Kong RGC under Grants RFS2627-1S02,  R1018-25F, N\_CityU1114/25, and 11219324, and in part by the Hong Kong UGC under Grant UGC/FDS11/E03/24 and Grant UGC/FDS11/E03/25.}
\thanks{J. Lu and J. Hou are with the Department of Computer Science, City University of Hong Kong,  Hong Kong SAR, China. (email:jrong.alvin@gmail.com; jh.hou@cityu.edu.hk)}
\thanks{Z. Zhu is with City University of Hong Kong (Dongguan), Dongguan, China, and also with the Department of Computer Science, City University of Hong Kong, Hong Kong SAR, China. (e-mail: zhiyuzhu2@my.cityu.edu.hk; )}%
\thanks{H. Liu is with the School of Computing and Information Sciences, Saint Francis University, Hong Kong SAR, China (e-mail: h2liu@sfu.edu.hk). }%

}

\markboth{Manuscript Under Review}%
{ParaLin++: Accelerating Parallel Diffusion Integrator via Intrinsic Partially Linear Structure}

\maketitle

\begin{abstract}
Training-free few-step samplers are often described as nearly quality-lossless under conventional FID/CLIP-style evaluation. It remains unclear whether that conclusion transfers to more realistic and complex generation-quality benchmarks. We first systematize the quality metrics used to evaluate training-free sampling acceleration for diffusion and flow-matching models.

We find that representative training-free step-shortening methods often remain nearly lossless under FID/CLIP-class metrics, yet can suffer substantial quality degradation when evaluated with semantic-composition, dense-instruction, and long-text-rendering benchmarks such as GenEval, DPG-Bench, and LongTextBench. Parallel sampling is a notable exception: by retaining the full-step discretization and removing only the solver's temporal dependency, it can preserve near-lossless quality while accelerating sampling. Building on this observation, we propose \textit{ParaLin++}, a step-parity parallel sampler that preserves the full-step discretization while accelerating the solver iterations. Experiments across Stable Diffusion-v2, LSUN Church, FLUX.1-dev, and Ming-Flash-Omni 2.0 show that ParaLin++ preserves the full-step Euler quality profile at $2.6\times$--$4.8\times$, $1.68\times$ on FLUX, and $1.21\times$ on Ming.  Project page: \url{https://paralinplus.github.io/}.
\end{abstract}

\begin{IEEEkeywords}
Diffusion Probabilistic Models, Parallel Sampling, Generative Models, Model Acceleration, Ordinary Differential Equations, System of Linear Equations.
\end{IEEEkeywords}

\section{Introduction}
\IEEEPARstart{D}{iffusion} probabilistic models (DPMs)~\cite{DDPM,ScoreBase} are now widely used in image and video generation, including large flow-matching models such as FLUX~\cite{bfl2024flux} and Ming-Flash-Omni 2.0~\cite{ai2025ming}, with applications from high-fidelity synthesis~\cite{SD,SVD} to molecular generation~\cite{Medsegdiff,medicalapp} and a broad range of generative tasks~\cite{graphapplication2,liu2023instaflow,lu2024depriving,chung2023parallel,lu2023preserving,Zhiyu2024nvs,liu2025prefpaint}. In practical use, sampling cost is a dominant inference bottleneck: each high-resolution sample requires a long sequence of neural function evaluations (NFEs), and the latency of this \emph{sequential} trajectory constrains serving throughput, inference cost, and on-device feasibility.

The reference discretization of a DPM is a fixed $N$-step ODE/SDE integrator, typically Euler with $N{=}30$ steps for flow-matching models. Accelerating this pipeline has followed two structurally different routes, alongside a broader landscape of training-free speedups~\cite{gong2024semi,zheng2023fast,gonzalez2024seeds,equilibriummodels,liu2023instaflow,luo2023latent}. The first route~\emph{shortens the discretization}: high-order or few-step training-free solvers such as DDIM~\cite{DDIM}, DPMSolver~\cite{dpmsolver,dpmsolverpp}, and UniPC~\cite{zhao2023unipc}, together with distillation/consistency approaches~\cite{Progressivedistillation,consistencymodel,lipman2022flow,liu2022flow}, cut the step count at the cost of changing the discretized trajectory. The second route~\emph{removes solver latency instead}: parallel sampling refines many time points concurrently, beginning with Picard iteration in ParaDiGMS~\cite{shih2024parallel} and progressing through triangular~\cite{tang2024accelerating} and banded~\cite{parasolver} nonlinear solvers, recently made NFE-efficient by the partial-linear formulation of ParaLin~\cite{paralin2026}. These two routes buy speed from different places: the first spends \emph{fewer integration steps}, the second spends \emph{parallel compute to erase the solver's temporal dependency}.

The decisive axis of a training-free accelerator is the \emph{metric} that certifies it ``quality-lossless,'' not the speed lever it uses. Much of the literature uses distributional and coarse text-alignment proxies---FID measures closeness of the generated \emph{distribution} to a reference, while CLIP measures coarse prompt--image embedding similarity. Post-trained text-to-image models, however, are often judged on a broader generation-quality profile: semantic composition (GenEval~\cite{ghosh2024geneval}), dense instruction following (DPG-Bench~\cite{hu2024ella}), and long in-image text rendering (LongTextBench~\cite{xomni2025}). A step-shortening strategy that looks lossless under FID can still alter the prompt-conditioned trajectory and produce a failed generation, because it changes the discrete trajectory on which prompt and text fidelity depend even when the marginal image distribution is preserved. We therefore focus on \emph{step-parity} parallelization: keep the complete fixed-step system, but solve its coupled integration equations in parallel so that latency is removed without changing the discretization on which prompt-level and in-image text fidelity depend.

\textbf{Our contributions.} We make the following contributions.

\textbf{1}. \emph{A metric systemization and motivating finding.} We systematize the quality metrics used to evaluate training-free sampling acceleration for diffusion and flow-matching models. Table~\ref{tab:systemization} enumerates the evaluation protocols of a broad set of samplers---DDPM, DDIM, DPM-Solver, DPMSolver++, DPM-Solver-v3, UniPC, distillation and consistency methods, ParaDiGMS, ParaSolver, ParaLin, and ours---and shows that acceleration claims are usually certified by FID/CLIP-class metrics. These metrics measure statistical alignment or coarse prompt--image similarity; they do not directly measure semantic composition, dense instruction following, or in-image text rendering. Throughout, ``complex generation quality'' refers specifically to these three measurable axes rather than to an exhaustive definition of generation quality. Prior FID-oriented evaluations of DDIM, DPMSolver, DPMSolver++, and UniPC often report that few-step configurations remain close to the reference distribution~\cite{DDIM,dpmsolver,dpmsolverpp,zhao2023unipc}. Together with Table~\ref{Table:Industrial_Motivation_FLUX}, this literature motivates a stricter audit: a 16-step DPMSolver++ configuration loses $21\%$--$24\%$ GenEval, $6.5$--$8.2$ DPG-Bench points, and $67\%$--$70\%$ LongTextBench. Thus, the usual quality-lossless claim does not automatically transfer from FID/CLIP-style evidence to complex generation-quality benchmarks.

\begin{table}[!t]
\centering
\caption{Quality metrics used to evaluate training-free diffusion/flow-matching acceleration in the literature and in this work. ``Dist.'' = distributional, reference-based metrics of the FID class (FID, sFID, clean-FID, IS, Precision--Recall); ``T2I-align.'' = text--image alignment metrics (CLIP score, CLIP-FID); ``Complex'' = generation-quality benchmarks that assess semantic composition, dense instruction following, or in-image text rendering. The acceleration literature usually stops at distributional/text-alignment proxies, whereas this work adds a quality-focused audit. \ding{51}/\ding{55}: metric reported / not reported by the method's own evaluation.}
\label{tab:systemization}
\setlength{\tabcolsep}{4pt}
\resizebox{0.9\columnwidth}{!}{%
\begin{tabular}{lccc}
\toprule
Method & Dist. & T2I-align. & Complex\\
\midrule
DDPM~\cite{DDPM} & \ding{51} & \ding{55} & \ding{55}\\
DDIM~\cite{DDIM} & \ding{51} & \ding{55} & \ding{55}\\
DPM-Solver~\cite{dpmsolver} & \ding{51} & \ding{55} & \ding{55}\\
DPMSolver++~\cite{dpmsolverpp} & \ding{51} & \ding{51} & \ding{55}\\
DPM-Solver-v3~\cite{zheng2023dpm} & \ding{51} & \ding{55} & \ding{55}\\
UniPC~\cite{zhao2023unipc} & \ding{51} & \ding{51} & \ding{55}\\
Euler (flow reference) & \ding{51} & \ding{55} & \ding{55}\\
Prog.~distillation~\cite{Progressivedistillation} & \ding{51} & \ding{51} & \ding{55}\\
Consistency model~\cite{consistencymodel} & \ding{51} & \ding{55} & \ding{55}\\
InstaFlow~\cite{liu2023instaflow} & \ding{51} & \ding{51} & \ding{55}\\
ParaDiGMS~\cite{shih2024parallel} & \ding{51} & \ding{55} & \ding{55}\\
ParaTAA~\cite{tang2024accelerating} & \ding{51} & \ding{55} & \ding{55}\\
ParaSolver~\cite{parasolver} & \ding{51} & \ding{51} & \ding{55}\\
ParaLin~\cite{paralin2026} & \ding{51} & \ding{51} & \ding{55}\\
\rowcolor{paralinrow} ParaLin++ (ours) & \ding{51} & \ding{51} & \ding{51}\\
\bottomrule
\end{tabular}}
\end{table}

\textbf{2}. \emph{A stronger step-parity parallel solver, ParaLin++.} We build the strongest available realization of the step-parity route to test how far this route reaches. Starting from ParaLin's $\varrho$-nonlinear partial-linear system, we add an adaptive nonlinearization iteration, a successive-over-relaxation (SOR) update for both nonlinear and linear blocks, a sliding window, and an adaptive stopping rule based on relative historical errors and an exponential moving average. We prove that the SOR update is stable for $0<\omega<2/(1+L)$ (contraction factor $L<1$) and that the system admits the sequential trajectory as its unique convergent root.

\textbf{3}. \emph{A quality--speed audit on realistic generation benchmarks.} We evaluate Stable Diffusion-v2, LSUN Church, FLUX.1-dev, Ming-Flash-Omni 2.0, five CFG values, three complex-generation benchmarks, visual comparisons, and ablations. Our key lessons are:
\begin{enumerate}
\item \emph{FID/CLIP agreement does not imply generation-quality preservation.} On FLUX, a $1.8\times$ few-step claim that is valid under FID  coincides with a $21\%$--$24\%$ GenEval drop, a $6.5$--$8.2$ DPG-Bench point drop, and a $67\%$--$70\%$ LongTextBench drop---an operational failure the standard evaluation hides.
\item \emph{Step-parity keeps the reference profile.} Keeping all 30 steps and parallel-solving gives $1.68\times$ on FLUX and $1.21\times$ on Ming-Flash-Omni 2.0 with an Euler-level generation-quality profile ($0.64$ vs.~$0.66$ GenEval, $84.24$ vs.~$83.97$ DPG, $0.26$ vs.~$0.27$ LTB), and $2.6\times$--$4.8\times$ on Stable Diffusion-v2/LSUN Church.
\item \emph{Guidance sharpens the quality gap.} Across CFG $0$--$14$, the completed DPG-Bench matrix shows ParaLin++ tracking Euler closely ($82.18$--$84.24$ vs.~$82.25$--$83.97$) while 20-step UniPC/DPMSolver++ drop to $68.85$--$79.67$.
\item \emph{Report iterations and generation-quality metrics, not only FID.} A speedup claim should separate discrete steps from solver iterations/NFE and be measured against the same-CFG 30-step Euler reference on a generation-quality benchmark; FID-only evidence is not sufficient to establish quality preservation.
\end{enumerate}

\noindent\textbf{Implications of our study.} Representative training-free step-shortening methods can appear nearly lossless under FID/CLIP yet show substantial degradation on more realistic and complex quality benchmarks. This argues for reporting generation-quality results alongside FID and for treating step-parity parallel sampling as an exception that can preserve near-lossless quality while accelerating sampling.

The remainder of this paper is organized as follows. Section~\ref{sec:background} provides the necessary background and reviews existing training-free accelerators. Section~\ref{sec:systemization} systematizes the quality metrics used to evaluate training-free acceleration and shows that FID-matched step shortening can incur large quality degradation on complex-generation benchmarks. Section~\ref{sec:proposed_method} first reviews ParaLin's partial-linear reformulation and then designs the stronger parallel solver ParaLin++ (Section~\ref{sec:solving_ES}). Section~\ref{sec:experiments} reports the quality--speed audit on complex-generation benchmarks and ablations. Section~\ref{sec:conclusion} concludes this paper. Complete derivations and proofs of the propositions in Section~\ref{sec:proposed_method} are provided in Appendices A--D of the separate supplementary material.
\vspace{-2mm}

\section{Background}\label{sec:background}\label{sec:Preliminary}
\textbf{Sequential Integral Solvers.} The forward diffusion process degrades a clean image $\mathbf{X}_T$ into random noise $\mathbf{X}_0$ according to the stochastic differential equation (SDE)
\begin{equation}
\label{eq:SDE}
   d\mathbf{X}_t=f(t)\mathbf{X}_t\,dt+g(t)\,d\mathbf{W},
\end{equation}
where $d\mathbf{W}$ is the standard Wiener process. This induces a marginal distribution $\mathbf{X}_t\sim\mathcal{N}(\alpha_t\mathbf{X}_T,\sigma_t\mathbf{I})$ with noise schedules $\alpha_t$ and $\sigma_t$. Generation reverses the process through the reverse-time SDE
\begin{equation}
\label{eq:RevSDE}
d\mathbf{X}_t=\underbrace{\left[f(t)\mathbf{X}_t-g^2(t)\nabla_{\mathbf{X}_t}\log p(\mathbf{X}_t)\right]}_{\varphi(\mathbf{X}_t,t)}dt+\underbrace{g(t)}_{\sigma_t}d\mathbf{W},
\end{equation}
where a neural network $\mathbf{S}_\theta(\cdot)$ approximates the score. Let $\Phi(t,s,\mathbf{X}_s)$ denote the integral solution of~\eqref{eq:RevSDE} over $[s,t]$:
\begin{equation}
\label{eq:Integral}
\Phi(t,s,\mathbf{X}_s)=\mathbf{X}_s+\int_s^t\varphi(\mathbf{X}_{\tau},\tau)\,d\tau+\int_s^t\sigma_\tau\,d\mathbf{W}.
\end{equation}

\textbf{Parallel Integral Solvers.} ParaSolver~\cite{parasolver} provides a unified framework for sequential and parallel sampling. Given a trajectory divided into $N$ sub-intervals $0=t_0<t_1<\cdots<t_N=T$, it forms the banded nonlinear system
\begin{align}
\label{eq:BNEs}
\begin{cases}
\hat{\mathbf{X}}_{t_0}-\mathbf{X}_{t_0}=\mathbf{0},\\
\hat{\mathbf{X}}_{t_{n+1}}-\Phi(t_{n+1},t_n,\hat{\mathbf{X}}_{t_n})=\mathbf{0}, & n\in\{0,1,\ldots,N-1\}.
\end{cases}
\end{align}
Starting from an initial estimate, the concatenated state is updated by
\begin{equation}
\label{eq:paraiters}
\hat{\mathbf{X}}_{t_0:t_N}^{(k+1)}
=\hat{\mathbf{X}}_{t_0:t_N}^{(k)}
-\mathcal{G}^{(k)}\mathcal{R}_{t_0:t_N}^{(k)},
\end{equation}
where $\mathcal{R}_{t_0}^{(k)}=\mathbf{0}$ and
$\mathcal{R}_{t_{n+1}}^{(k)}=\hat{\mathbf{X}}_{t_{n+1}}^{(k)}-\Phi(t_{n+1},t_n,\hat{\mathbf{X}}_{t_n}^{(k)})$.
The solver matrix $\mathcal{G}^{(k)}$ selects the algorithm: $\mathcal{G}^{(k)}=\mathbf{I}$ gives fixed-point iteration~\cite{Fixed_point_iteration}, while $\mathcal{G}^{(k)}=(\mathcal{J}^{(k)})^{-1}$ gives Newton--Raphson iteration~\cite{Newtonmethod}. ParaLin~\cite{paralin2026} reduces the NFE burden of these fully nonlinear alternatives by exploiting the partial linearity of the integral solver.
\subsection{Existing Training-Free Accelerators}\label{sec:related_work}
\textbf{Step-parity temporal parallelization.} Parallel sampling removes the temporal dependency of a fixed-step diffusion integrator by refining multiple time states concurrently. ParaDiGMS~\cite{shih2024parallel} introduced this direction through Picard--Lindel\"of fixed-point iteration; Self-Refining Diffusion Samplers~\cite{selvam2024selfrefining} use Parareal iterations; ParaTAA~\cite{tang2024accelerating} formulates a triangular nonlinear system; and ParaSolver~\cite{parasolver} improves efficiency through banded and hierarchical parallelization. CHORDS~\cite{han2025chords} further treats multi-core sampling as a hierarchy of increasingly accurate solvers, whereas ParaAnya~\cite{yu2026paraanya} uses a plug-and-play output cache to avoid repeated evaluations at overlapping timesteps. These methods reinforce the same trade-off as ParaDiGMS: additional parallel compute replaces sequential time, rather than coarsening the integration grid. ParaLin~\cite{paralin2026} reduces the NFE overhead of fully nonlinear iterations by exploiting partial linearity; ParaLin++ strengthens this solver with adaptive nonlinearization, over-relaxation, and adaptive stopping.

A closely related but distinct solver is EPD~\cite{zhu2025parallelgradients}, which evaluates multiple gradient directions concurrently inside each ODE step and learns a small set of solver parameters through distillation. Unlike step-parity methods, it changes the local integration rule and requires a lightweight training stage. Recent theory also cautions against an unconstrained parallelization story: parallel-round lower bounds~\cite{kou2026parallelbounds} show that some distributions necessarily retain polynomial sequential dependence. ParaLin++ is therefore formulated as a practical fixed-step parallel integrator, not as a claim that all diffusion sampling can be made depth-free.

\textbf{Speculative and asynchronous parallelism.} A second parallel family drafts approximate future states and verifies or refines them in parallel. Diffusion speculative sampling~\cite{debdortoli2025speculative} preserves the target distribution by accepting or correcting draft proposals; FREE~\cite{wen2025free} extends this idea to feature-level autoregression in DiTs; SpecDiff~\cite{pan2026specdiff} combines self-speculation with feature caching; and DRiffusion~\cite{bai2026driffusion} uses a draft-and-refine process. AsyncDiff~\cite{chen2024asyncdiff} instead breaks component dependencies through asynchronous denoising. These methods exploit drafts, verifiers, or stale hidden states, and thus change the execution graph; ParaLin++ solves the coupled fixed-step integration system directly and does not introduce an auxiliary drafter.

\textbf{Step-shortening solvers and schedules.} Classical high-order ODE samplers shorten the timestep grid~\cite{DDIM,dpmsolver,dpmsolverpp,zhao2023unipc}. Recent training-free variants improve this route through adaptive flow paths~\cite{jin2026aflops}, multirate integration for the guidance branch~\cite{lee2025tortoise}, stiffness-aware integration~\cite{tan2025stork}, and hierarchical timestep-schedule optimization~\cite{zhu2026hso}. These methods are relevant baselines for the quality audit, but they still alter the discretization on which prompt-conditioned trajectories depend.

\textbf{Per-step compute reduction and model parallelism.} Orthogonal accelerators reduce the cost of each denoising evaluation. They include feature and block caching~\cite{ma2024deepcache,Learning2cache,CacheQuant}, DiT attention or token sparsification~\cite{zhao2025pab,jiang2025sada}, temporal feature forecasting~\cite{liu2025teacache,han2026spectrum}, path-planned cache scheduling~\cite{cui2026dpcache}, and distributed or asynchronous execution~\cite{li2024distrifusion,wang2024pipefusion,chen2024asyncdiff}. Guidance acceleration~\cite{PFDiff}, trajectory stitching~\cite{pan2024t}, nested diffusion~\cite{elata2024nested}, and minimum-step prediction~\cite{yu2024step} are likewise complementary. These systems change the executed network path or decide when a cheaper computation is acceptable; ParaLin++ changes neither the timestep grid nor the full-step denoiser evaluation.

\textbf{Quality-focused evaluation.} The claim that FID/CLIP evidence is insufficient is consistent with the evaluation literature: TIFA uses visual question answering to test fine-grained prompt faithfulness~\cite{hu2023tifa}; VQAScore and GenAI-Bench target compositional prompts~\cite{lin2024vqascore}; and T2I-CompBench++ evaluates attributes, relations, and numeracy~\cite{huang2025t2icompbench}. GenEval~2 additionally warns that static benchmarks can drift as generation models improve~\cite{kamath2025geneval2}. Our audit therefore interprets GenEval, DPG-Bench, and LongTextBench as a fixed, reproducible quality profile rather than as an absolute or exhaustive definition of generation quality.

\section{Systemization of Training-Free Diffusion Acceleration}\label{sec:systemization}
We systematize training-free diffusion acceleration through the quality metrics on which its ``quality-lossless'' claims rest, rather than through the speed lever it uses, because FID/CLIP-style evidence does not by itself establish preservation on complex prompts. We first enumerate the evaluation protocols of the academic acceleration literature (Table~\ref{tab:systemization}) and identify what these methods actually measure; we then test representative few-step quality claims under a generation-quality audit and find substantial degradation (Section~\ref{sec:experiments}, summarized in Table~\ref{Table:Industrial_Motivation_FLUX}).

\subsection{Quality Metrics in the Training-Free Acceleration Literature}
Table~\ref{tab:systemization} surveys the quality metrics used to evaluate a broad set of training-free samplers and their close relatives, from DDPM and DDIM through high-order solvers (DPM-Solver, DPMSolver++, DPM-Solver-v3, UniPC) and distillation/consistency methods, to parallel samplers (ParaDiGMS, ParaTAA, ParaSolver, ParaLin) and our ParaLin++. Two patterns are immediate. \emph{First}, nearly every method is evaluated with distributional and text-alignment metrics only: FID (and its variants sFID, clean-FID) measures closeness of the generated \emph{distribution} to a reference distribution, while CLIP score/CLIP-FID measures coarse prompt--image embedding similarity, and a handful additionally report Inception-based diversity (IS, Precision--Recall). These quantify \emph{statistical alignment}---whether the marginal image distribution or the prompt--image embedding stays close to a reference---not whether a specific generated sample succeeds at semantic composition, dense instruction following, or long-text rendering. \emph{Second}, few methods explicitly audit these generation-quality dimensions. Recent large-scale models such as FLUX~\cite{bfl2024flux} and Ming-Flash-Omni 2.0~\cite{ai2025ming}, in contrast, report GenEval-, DPG-Bench-, and LongTextBench-style numbers, whereas FID/CLIP are scarce or absent in their evaluation reports. A sampler certified only by FID/CLIP-style metrics can therefore appear lossless while remaining fragile on complex prompts.
\subsection{A Realistic Quality Audit of Training-Free Accelerators}
The metrics gap is not cosmetic. We further test the same training-free accelerators under the generation-quality profile---GenEval, DPG-Bench, and LongTextBench---on the high-capacity FLUX.1-dev~\cite{bfl2024flux}, against the same-CFG 30-step Euler reference; the result is summarized in Table~\ref{Table:Industrial_Motivation_FLUX} and elaborated in Section~\ref{sec:experiments}. Methods advertised as nearly quality-lossless under FID show substantial degradation on the generation-quality metrics. Prior FID-oriented evaluations support the general expectation that reduced-step configurations can remain close to the reference distribution~\cite{DDIM,dpmsolver,dpmsolverpp,zhao2023unipc}; in our audit, however, 16-step UniPC and DPMSolver++ cause GenEval to fall from $0.66$ to $0.50$--$0.52$, DPG-Bench from $83.97$ to $75.75$--$77.46$, and LongTextBench from $0.27$ to $0.08$--$0.09$. An image distribution can stay FID-plausible while the model omits requested objects, breaks spatial relations, or renders unreadable text. \emph{Hence, representative training-free few-step solvers can appear nearly lossless under FID/CLIP yet lose substantially under realistic, complex quality metrics. Parallel sampling is an important exception: by preserving the full-step trajectory, it can retain near-lossless quality while removing solver latency.}

This failure is structural rather than incidental. Step-shortening solvers (DDIM, DPMSolver++, UniPC) buy fewer discrete steps by altering the prompt-conditioned trajectory, and the altered trajectory is what breaks dense-prompt and text fidelity; FID merely fails to reflect this break because it overlooks sample-level semantics. Parallel sampling instead keeps all $N$ discrete steps and removes only the solver's temporal dependency, preserving the prompt-conditioned trajectory and the quality that depends on it. We call this the \emph{step-parity} route, and ParaLin++ (Section~\ref{sec:solving_ES}) is designed as its strongest realization. The rest of the paper extends the generation-quality audit (Section~\ref{sec:experiments}) and strengthens the step-parity solver (Sections~\ref{sec:proposed_method}--\ref{sec:solving_ES}).

\section{Exploring the Space: From ParaLin to ParaLin++}\label{sec:proposed_method}
\subsection{The ParaLin Framework}
Parallel solvers reformulate sequential diffusion integration into the nonlinear system in~\eqref{eq:BNEs} and~\eqref{eq:paraiters}, but denoising many states in every iteration causes large NFE overhead. The accepted ParaLin framework~\cite{paralin2026} addresses this by reducing fully nonlinear score computation to partially linear operations. This section briefly recalls that formulation, which serves as the foundation of ParaLin++.
\subsubsection{Reformulation of the Sequential Integral Solver}
Existing approaches typically develop parallel solvers based on the sequential integral formulation in~\eqref{eq:Integral}, which results in a fully nonlinear system of equations in ~\eqref{eq:BNEs} when employing the nonlinear  score function $\mathbf{S}_{\theta}(\cdot)$. We reveal that the sequential integral solver in~\eqref{eq:Integral} can be expressed in a more elementary form and reduced to a linear solver, as formally stated in the following Proposition~\ref{proposition:Linear_SQ_Solver}.
\begin{proposition}[Linear Approximation (Proof in Supplementary Material, Appendix A)]\label{proposition:Linear_SQ_Solver}
The sequential integral solver $\Phi(t,s,\mathbf{X}_s)$ from~\eqref{eq:Integral} can be decomposed into a linear solver $\mathcal{H}$ and a negligible error term $\mathcal{Z}_s$:
\begin{equation}
\Phi(t,s,\mathbf{X}_s) = \mathcal{H}\left(t,s,\mathbf{X}_s,\mathbf{X}_T(\mathbf{X}_s)\right) + \mathcal{Z}_s,
\end{equation}
where the linear sequential integral solver $\mathcal{H}(\cdot)$ is defined as:
\begin{footnotesize}
\begin{equation}
\begin{aligned}
\label{eq:linear_SQ_solver}
\mathcal{H}(\cdot) = &\frac{\alpha_{t}}{\alpha_{s}}\mathbf{X}_s - \alpha_{t}\int_{\lambda_s}^{\lambda_t}e^{-\lambda} \frac{\mathbf{X}_\lambda - \alpha_\lambda\mathbf{X}_T(\mathbf{X}_s)}{\sigma_\lambda} d\lambda  + \int_{\lambda_s}^{\lambda_t} g(\lambda) d\mathbf{W},
\end{aligned}
\end{equation}
\end{footnotesize}
Here, $h = s-t$ denotes the step size and $\lambda_t = \log(\alpha_t/\sigma_t)$ . Under the assumption that the model prediction error is bounded by $\zeta$, the decomposition error satisfies $\|\mathcal{Z}_s\| \leq \zeta \alpha_t (\frac{\sigma_s}{\alpha_s} - \frac{\sigma_t}{\alpha_t})$, which is negligible for $h\leq1$.
\end{proposition}

\begin{remark}
The small step size $h$ is well-founded in practice. Theoretically, convergence guarantees for diffusion solvers strictly require bounded steps ($h \leq 1$~\cite{chen2023sampling}), while practically high-fidelity generation involves more fine discretization (e.g., $N \ge 5$ steps), ensuring $h = 1/N\leq 0.2$ remains small to render the error negligible. As detailed in Tab.~\ref{tab:error_comparison_edm_main_text}, this error term is small for practical discretization. Specifically, for the EDM framework, the bound simplifies to $\|\mathcal{Z}_s\| \leq \zeta h$; for the VE and VP schedules, the corresponding bounds shrink to first order as $h=s-t$ becomes small. This ensures that the approximation error remains negligible under high-fidelity sampling.
\end{remark}

\begin{table}[!t]
\centering
\caption{Practical error bounds for diffusion frameworks in the EDM design space. Parameters follow Table~1 of \cite{EDM}, with $\beta_{\min}=0.1$, $\beta_d=19.9$, and $\norm{\vec{S}_{\theta^*}(\vec{X}_t,t)-\vec{\epsilon}_t}<\zeta$ for a small step size $h=s-t<1$. Rows give the schedule settings, the $\sigma_t/\alpha_t$ substitution, and the resulting bound; columns are the diffusion frameworks.}
\label{tab:error_comparison_edm_main_text}
\renewcommand{\arraystretch}{1.15}
\resizebox{\columnwidth}{!}{%
\begin{tabular}{lccc}
\toprule
 & \textbf{EDM}~\cite{EDM} & \textbf{VE}~\cite{ScoreBase} & \textbf{VP}~\cite{ScoreBase}\\
\midrule
\textbf{Schedule settings}\rule{0pt}{2.4ex} & $\begin{aligned}\alpha_t&=1\\\sigma_t&=t\end{aligned}$ & $\begin{aligned}\alpha_t&=1\\\sigma_t&=\sqrt{t}\end{aligned}$ & $\begin{aligned}\alpha_t&=1/\sqrt{e^{\frac12\beta_dt^2+\beta_{\min}t}}\\\sigma_t&=\alpha_t\sqrt{e^{\frac12\beta_dt^2+\beta_{\min}t}-1}\end{aligned}$\\
\textbf{Substitution\,}$\sigma_t/\alpha_t=$ & $t$ & $\sqrt{t}$ & $\sqrt{e^{\frac12\beta_dt^2+\beta_{\min}t}-1}$\\
\textbf{Final error bound} & $\norm{\mathcal{Z}_s}\le\zeta(s-t)=\zeta h$ & $\norm{\mathcal{Z}_s}\le\zeta(\sqrt{s}-\sqrt{t})$ & $\norm{\mathcal{Z}_s}\le\zeta\alpha_t(\sigma_s/\alpha_s-\sigma_t/\alpha_t)$\\
\bottomrule
\end{tabular}}
\end{table}

\noindent Now, following ParaSolver~\cite{parasolver}, the linear sequential integral solver in~\eqref{eq:linear_SQ_solver} can be decomposed into many integral functions on different time subintervals, written as
\begin{equation}
\label{eq:N_Solutions}
\mathbf{X}_{t_{n+1}} = \mathcal{H}(t_{n+1},t_{n},\mathbf{X}_{t_n},\mathbf{X}_{t_N}), n \in \{0,1,\cdots,N-1\}.
\end{equation}
Motivated by recent parallel algorithms~\cite{parasolver,song2021accelerating}, this inherently sequential problem can be addressed in a parallel framework. Specifically, such a sequence of cascaded linear functions can be interpreted as a system of fully linear equations. Denote by $\{\mathbf{X}_{t_{n}}, n=0,~\cdots,~N \}$ a set of points exactly on the diffusion trajectory, and $\left\{\hat{\mathbf{X}}_{t_{n}}, n=0,~\cdots,~N \right\}$ the unknown variables to be optimized towards the trajectory. The cascade of integral formulations as \eqref{eq:N_Solutions} can then be reformulated as the solutions to the following \textit{fully linear equations}.
\begin{definition}[Fully Linear Equation System]
\label{def:FLE}
We define the system of  fully linear equations for the sequential sampling process in~\eqref{eq:N_Solutions} as
\begin{small}
\begin{equation}
\label{eq:FLE}
\left\{
\begin{aligned}
&\hat{\mathbf{X}}_{t_0} - \mathbf{X}_{t_0} = \mathbf{0},\\
&\hat{\mathbf{X}}_{t_{n+1}} - \mathcal{H}(t_{n+1},t_{n},\hat{\mathbf{X}}_{t_n},\hat{\mathbf{X}}_{t_N}) = \mathbf{0},
\quad n \in \{0,1,\cdots,N-1\}.
\end{aligned}
\right.
\end{equation}
\end{small}
\end{definition}
\noindent The complete linearity of the above system permits exceptionally efficient solving with negligible computational cost. However, this very full linearity prevents them from achieving the sample quality of the sequential integral solvers in~\eqref{eq:Integral}, as they lack information regarding the target data distribution. To bridge this quality gap, we introduce the target distribution information into the linear system by supplying the estimated clean sample $\hat{\mathbf{X}}_{t_N}$. Recall that in the forward process of DPMs, we can compute a clean sample  predicted  $\hat{\mathbf{X}}_{t_N}$ as
\begin{equation}
    \label{eq:predicted_clean_sample}
    \hat{\mathbf{X}}_{t_N}(\hat{\mathbf{X}}_{t_{n}},\theta) = \frac{\hat{\mathbf{X}}_{t_{n}} - \sigma_{t_{n}}\mathbf{S}_{\theta}(\hat{\mathbf{X}}_{t_{n}},t_{n})}{\alpha_{t_{n}}},
\end{equation}
where  $\mathbf{X}_{t_n} \sim \mathcal{N} (\alpha_{t_{n}}\mathbf{X}_{t_N},\sigma_{t_{n}}\mathbf{I})$ where $\alpha$ and $\sigma$ are the noise schedule. Note that the diffusion model consumes most computational resources on the score-matching process, i.e., estimating corresponding noise or clean samples by a neural network. In this paper, to lessen such burden of score matching, we explore the consistency of reverse diffusion process and propose to divide the NEs to a partial-linear equation system, namely $\varrho$-nonlinear equation system. Specifically, due to the diffusion models make progressive transition from pure noise to clean data distribution, the score function is highly unstable on the high noise region, which is also evidenced by recent research~\textcolor{blue}{\cite{pan2024t,ma2024deepcache}}. For compactness, let \(\mathcal{P}_{i}^{(j)}=\hat{\mathbf{X}}_{t_N}(\hat{\mathbf{X}}_{t_i}^{(j)},\theta)\) denote the clean-sample prediction formed from state \(i\) at iteration \(j\), and write \(\mathcal{H}_i(x,y)=\mathcal{H}(t_{i+1},t_i,x,y)\). Based on such an insight, to further balance the score estimation accuracy and computational consumption, we divide the original NEs shown as Eq.~\ref{eq:BNEs} into $\varrho$-Nonlinear Equation System in Definition~\ref{def:varrho_NES}, where for the unstable and stable score region, we apply the nonlinear and linear solvers, respectively.
\begin{definition}[$\varrho$-Nonlinear Equation System]
\label{def:varrho_NES}
We define a  $\varrho$-nonlinear equation system   for the sequential sampling process in~\eqref{eq:N_Solutions} as
\begin{equation}
\begin{small} 
\label{eq:varpho_NES}
    \begin{cases}
    \hat{\mathbf{X}}_{t_0} - \mathbf{X}_{t_0} = \mathbf{0},\\
     \hat{\mathbf{X}}_{t_{n+1}} - \mathcal{H}\Big(t_{n+1},t_{n},\hat{\mathbf{X}}_{t_n},\hat{\mathbf{X}}_{t_N}(\hat{\mathbf{X}}_{t_{n}},\theta)\Big) = \mathbf{0} , n \in [0,\varrho-1],\\
     \hat{\mathbf{X}}_{t_{n+1}} - \mathcal{H}\Big(t_{n+1},t_{n},\hat{\mathbf{X}}_{t_n},\hat{\mathbf{X}}_{t_N}(\hat{\mathbf{X}}_{t_{\varrho-1}},\theta)\Big) = \mathbf{0} , n \in [\varrho,N-1].\\
\end{cases}
\end{small}
\end{equation}
\end{definition}
\begin{remark}
The scalar $\varrho$ serves as a critical transition point between nonlinear (score) and linear (fixed score) regimes in the system. The first $\varrho$ equations evolve into computationally costly nonlinear equations, reducing into exiting nonlinear parallel system in~\eqref{eq:BNEs}. Subsequent equations (for $n \geq \varrho$) retain computational efficiency as linear equations, as they depend only on the fixed prediction from the $(\varrho-1)^{th}$ time step.
\end{remark}
\begin{remark}
The above framework establishes a unified paradigm, subsuming all existing parallel methods such as ParaDiGMS, ParaTAA, and ParaSolver  as special cases. In particular, existing methods are fully nonlinear equation systems, corresponding to the limiting case $\varrho = N-1$.
\end{remark}


\begin{proposition}[Accumulated Error Bound (Proof in Supplementary Material, Appendix B)]
\label{proposition:Unique_Solution}
For any $\varrho$, let $\{\hat{\mathbf{X}}_{t_n}\}_{n=0}^N$ be the solution to the $\varrho$-nonlinear surrogate system in~\eqref{eq:varpho_NES} and $\{\mathbf{X}_{t_n}\}_{n=0}^N$ the exact sequential integral trajectory, with deviation $\boldsymbol{\Delta}_{t_n}=\hat{\mathbf{X}}_{t_n}-\mathbf{X}_{t_n}$ and clean-sample estimation error $\boldsymbol{\xi}_n=\mathbf{E}_{t_n}-\mathbf{X}_T$. Under the high-quality sampling regime $h\le 1$, the deviation obeys the per-step recurrence
\begin{equation}
    \|\boldsymbol{\Delta}_{t_{n+1}}\| \leq |\mathcal{A}_n| \|\boldsymbol{\Delta}_{t_n}\| + |\mathcal{B}_n| \|\boldsymbol{\xi}_n\| + \|\mathcal{Z}_n\|,
\end{equation}
with $\mathcal{A}_n = \frac{\alpha_{t_{n+1}}}{\alpha_{t_n}} (1 - \Delta \lambda_n)$ and $\mathcal{B}_n = \alpha_{t_{n+1}} \Delta \lambda_n \left( 1 - \frac{\Delta \lambda_n}{2} \right)$. Starting from $\boldsymbol{\Delta}_{t_0}=\mathbf{0}$ and using $|\mathcal{A}_n|\le 1$, $|\mathcal{B}_n|=\mathcal{O}(h)$, and $\|\mathcal{Z}_n\|=\zeta\,\mathcal{O}(h)$ (Proposition~\ref{proposition:Linear_SQ_Solver}), unrolling over the $N$ steps yields the \emph{accumulated error bound}
\begin{equation}
\begin{aligned}
\|\boldsymbol{\Delta}_{t_N}\|
&\leq \sum_{n=0}^{N-1}\!\Bigl(\prod_{j=n+1}^{N-1}|\mathcal{A}_j|\Bigr)\bigl(|\mathcal{B}_n|\|\boldsymbol{\xi}_n\|+\|\mathcal{Z}_n\|\bigr)\\
&\leq \sum_{n=0}^{N-1}\!\bigl(|\mathcal{B}_n|\|\boldsymbol{\xi}_n\|+\|\mathcal{Z}_n\|\bigr) \leq N\,\mathcal{O}(h)\bigl(\|\boldsymbol{\xi}\|_{\max}+\zeta\bigr).
\end{aligned}
\end{equation}
The bound is structurally independent of $\varrho$: $\varrho$ only changes the source of the estimator $\mathbf{E}_{t_n}$ (and hence $\|\boldsymbol{\xi}_n\|$), not the coefficients $\mathcal{A}_n$ and $\mathcal{B}_n$.
\end{proposition}

\begin{remark}
Prop.~\ref{proposition:Unique_Solution} guarantees that different $\varrho$ does not impact the quality of the generated samples. We note that $\varrho$ is less of a hyperparameter and more of a hardware configuration parameter. For good efficiency, $\varrho$ can simply be set to available processor counts.
\end{remark}

\vspace{-5mm}
\subsection{Designing ParaLin++: A Stronger Step-Parity Solver}\label{sec:solving_ES}
ParaLin++ accelerates the ParaLin solver from two complementary perspectives. First, it retains ParaLin's adaptive nonlinearization and immediate update strategy, which exploit the latest refined states within an iteration. Second, it introduces SOR and a relative adaptive stopping rule to reduce the remaining iterations. The sliding-window and initialization strategies are then integrated into the complete algorithm.
\subsubsection{Adaptive Nonlinearization Iteration}
\label{Sec:ANIter}
We investigate the update rule in~\eqref{eq:paraiters} to solve the prescribed $\varrho$-nonlinear system. Specializing~\eqref{eq:paraiters} to our system yields the following iterative scheme:
\begin{equation}
\label{eq:paraiters_for_varpho_NES}
\hat{\mathbf{X}}_{t_0:t_N}^{(k+1)} = \hat{\mathbf{X}}_{t_0:t_N}^{(k)} - \mathcal{G}^{(k)} \mathcal{R}_{t_0:t_N}^{(k)},
\end{equation}
where
the residual term $\mathcal{R}_{t_{n+1}}^{(k)}$ with $\mathcal{R}_{t_{0}}^{(k)} =\hat{\mathbf{X}}_{t_{0}}^{(k)} - \mathbf{X}_{t_{0}} =0 $ is defined as:
\begin{small}
\begin{equation}
\begin{aligned}
\label{eq:iteration_residual}
\mathcal{R}_{t_{n+1}}^{(k)} =
\begin{cases}
\hat{\mathbf{X}}_{t_{n+1}}^{(k)} - \mathcal{H}_n\bigl(\hat{\mathbf{X}}_{t_n}^{(k)},\mathcal{P}_{n}^{(k)}\bigr), & n\in [0,\varrho-1], \\
\hat{\mathbf{X}}_{t_{n+1}}^{(k)} - \mathcal{H}_n\bigl(\hat{\mathbf{X}}_{t_n}^{(k)},\mathcal{P}_{\varrho-1}^{(k)}\bigr), & n\in [\varrho, N-1].
\end{cases}
\end{aligned}
\end{equation}
\end{small}

Observe that this update rule applies a Jacobi-type iteration, where all components $\mathbf{X}_{t_1:t_{N}}^{(k+1)}$ are updated exclusively from the old components $\mathbf{X}_{t_1:t_{N}}^{(k)}$, resulting in delayed information propagation across time steps. Inspired by the classical Gauss-Seidel iteration, we propose a more efficient immediate update strategy wherein newly updated components are incorporated immediately into subsequent components' updating within the same iteration. Specifically, we first update the nonlinear-region components $\mathbf{X}_{t_n}^{(k)}$ for $n \in \{1,\cdots,\varrho\}$. Upon obtaining the refined nonlinear component $\mathbf{X}_{t_\varrho}^{(k+1)}$ , we instantly propagate this refinement forward to compute the remaining linear-region components (i.e., $\mathbf{X}_{t_{\varrho+1}}^{(k+1)}, \dots, \mathbf{X}_{t_N}^{(k+1)}$). This approach ensures faster convergence by propagating the latest information forward, leading to the following improved update rule:
\begin{equation}
\label{eq:improved_paraiters}
\hat{\mathbf{X}}_{t_0:t_N}^{(k+1)}  = \hat{\mathbf{X}}_{t_0:t_N} ^{(k)} - \mathcal{G}^{(k)} \mathcal{R}_{t_0:t_N}^{(k)},
\end{equation}
where the residual term $\mathcal{R}_{t_{n+1}}^{(k)}$ now incorporates the latest refined components:
\begin{small}
\begin{equation}
\begin{aligned}
\label{eq:refined_iteration_residual}
\mathcal{R}_{t_{n+1}}^{(k)} =
\begin{cases}
\hat{\mathbf{X}}_{t_{n+1}}^{(k)} - \mathcal{H}_n\bigl(\hat{\mathbf{X}}_{t_n}^{(k)},\mathcal{P}_{n}^{(k)}\bigr), & n\in [0,\varrho-1],\\
\hat{\mathbf{X}}_{t_{n+1}}^{(k)} - \mathcal{H}_n\bigl(\hat{\mathbf{X}}_{t_n}^{(k+1)},\mathcal{P}_{\varrho-1}^{(k)}\bigr), & n\in [\varrho,N-1].
\end{cases}
\end{aligned}
\end{equation}
\end{small}

We observe that once the nonlinear components converge, the predicted clean sample $\hat{\mathbf{X}}_{t_N}(\hat{\mathbf{X}}_{t_{\varrho-1}}^{(k)},\theta)$ becomes unchanged, implying that no further precise information can be propagated for the linear equations in subsequential iterations. To address this, we progressively convert the linear equations into nonlinear ones, such that the number of nonlinear equations remains unchanged as the iteration proceeds. Formally, at iteration $k$, suppose the unconverged components are $\hat{\mathbf{X}}_{t_{n}}^{(k)},\cdots,\hat{\mathbf{X}}_{t_{N}}^{(k)}$, we maintain nonlinearity in the residuals $\{\mathcal{R}_{t_{i+1}}^{(k)},i =n,n+1,\cdots,n+\varrho-1\}$ of the first $\varrho$ components through the following formulation:
\begin{equation}
\label{eq:nonliear_aware_paraiters_for_varpho_NES}
\hat{\mathbf{X}}_{t_{n}:t_N}^{(k+1)} = \hat{\mathbf{X}}_{t_{n}:t_N}^{(k)} - \mathcal{G}^{(k)} \mathcal{R}_{t_{n}:t_N}^{(k)},
\end{equation}
where the residual term $\mathcal{R}_{t_{i+1}}^{(k)}$ can adaptively convert the linear equations into non-linear ones:
\begin{small}
\begin{equation}
\begin{aligned}
\label{eq:nonliear_aware_iteration_residual}
\mathcal{R}_{t_{i+1}}^{(k)} =
\begin{cases}
\hat{\mathbf{X}}_{t_{i+1}}^{(k)} - \mathcal{H}_i\bigl(\hat{\mathbf{X}}_{t_i}^{(k)},\mathcal{P}_{i}^{(k)}\bigr), & i\in [n,n+\varrho-1],\\
\hat{\mathbf{X}}_{t_{i+1}}^{(k)} - \mathcal{H}_i\bigl(\hat{\mathbf{X}}_{t_i}^{(k+1)},\mathcal{P}_{n+\varrho-1}^{(k)}\bigr), & i\in [n+\varrho,N-1].
\end{cases}
\end{aligned}
\end{equation}
\end{small}
where $\mathcal{R}_{t_{n}}^{(k)} = \hat{\mathbf{X}}_{t_{n}}^{(k)} - \mathbf{X}_{t_{n}} = \mathbf{0}$. This approach ensures continuous refinement of $\hat{\mathbf{X}}_{t_N}(\hat{\mathbf{X}}_{t_{n+\varrho-1}}^{(k)},\theta)$, yielding increasingly accurate solutions for the linear equations.

\begin{proposition}[Convergence Analysis]
\label{proposition:Convergence}
The update rule in~\eqref{eq:nonliear_aware_paraiters_for_varpho_NES} converges, within $K \leq T$ iterations, to the \emph{partially linear surrogate system}, i.e., the $\varrho$-nonlinear system defined in~\eqref{eq:varpho_NES}. The surrogate trajectory $\{\hat{\mathbf{X}}_{t_n}\}_{n=0}^N$ differs from the exact sequential integral trajectory $\{\mathbf{X}_{t_n}\}_{n=0}^N$ by an accumulated error bounded in Proposition~\ref{proposition:Unique_Solution}. In other words, the iteration closes the layer-3 to layer-2 gap (iterative solve $\to$ surrogate system), while the layer-2 to layer-1 gap (surrogate $\to$ exact sequential sampler) is the bounded deviation studied separately in Proposition~\ref{proposition:Unique_Solution}.
\end{proposition}
\begin{proof}
     See the supplementary material (Appendix C) for the complete derivation.
\end{proof}

\subsubsection{Convergence Acceleration via Successive Over-Relaxation}
\label{Sec:ConvergenceAcceleration}
\noindent While the immediate update strategy enhances information flow within each iteration, the solver may still require a considerable number of iterations ($K$) to converge, particularly for high-dimensional problems. To further reduce the iteration count, we incorporate the Successive Over-Relaxation (SOR) method, a well-established technique for accelerating the convergence of iterative solvers. The core idea is to introduce a relaxation factor $\omega$ to extrapolate the update step, pushing the solution more aggressively towards the root.

The standard Gauss-Seidel-like update can be viewed as a two-step process. First, an intermediate state $\hat{\mathbf{Z}}^{(k+1)}$ is computed based on the residual:
\begin{equation}
\hat{\mathbf{Z}}^{(k+1)} = \hat{\mathbf{X}}^{(k)} - \mathcal{G}^{(k)} \mathcal{R}^{(k)}.
\end{equation}
The SOR update then modifies this step by taking a weighted average of the current state $\hat{\mathbf{X}}^{(k)}$ and the intermediate state $\hat{\mathbf{Z}}^{(k+1)}$:
\begin{equation}
\label{eq:sor_update}
\hat{\mathbf{X}}^{(k+1)} = (1-\omega)\hat{\mathbf{X}}^{(k)} + \omega \hat{\mathbf{Z}}^{(k+1)},
\end{equation}
where $\omega $ is the relaxation factor. Setting $\omega=1$ recovers the original update rule, while $\omega > 1$ (over-relaxation) typically accelerates convergence for the systems encountered in our problem. We now integrate this SOR mechanism into our solver for both the nonlinear and linear components.

\noindent\textbf{Update for Nonlinear Components.} For the first $\varrho$ components within the current window, which are processed in parallel, we first compute their intermediate states based on the nonlinear residuals defined in \eqref{eq:nonliear_aware_iteration_residual}:
\begin{equation}
\hat{\mathbf{Z}}_{t_{n+1}:t_{n+\varrho}}^{(k+1)}  = \hat{\mathbf{X}}_{t_{n+1}:t_{n+\varrho}}^{(k)} - \mathcal{G}^{(k)} \mathcal{R}_{t_{n+1}:t_{n+\varrho}}^{(k)}.
\end{equation}
Subsequently, we apply the SOR update to obtain the refined states for the nonlinear block:
\begin{equation}
\hat{\mathbf{X}}_{t_{n+1}:t_{n+\varrho}}^{(k+1)}  = (1-\omega)\hat{\mathbf{X}}_{t_{n+1}:t_{n+\varrho}}^{(k)} +\omega \hat{\mathbf{Z}}_{t_{n+1}:t_{n+\varrho}}^{(k+1)}.
\end{equation}

\noindent\textbf{Update for Linear Components.} For the remaining components ($i \in [n+\varrho, N-1]$), we combine the SOR method with the immediate update strategy. These components are updated sequentially, allowing the most recent information to propagate immediately. For each linear component, we compute the intermediate state:
\begin{equation}
\hat{\mathbf{Z}}_{t_{i+1}}^{(k+1)} = \hat{\mathbf{X}}_{t_{i+1}}^{(k)} - \mathcal{G}^{(k)} \mathcal{R}_{t_{i+1}}^{(k)},
\end{equation}
where the residual $\mathcal{R}_{t_{i+1}}^{(k)}$ is computed using the already updated value $\hat{\mathbf{X}}_{t_i}^{(k+1)}$, as specified in \eqref{eq:nonliear_aware_iteration_residual}. The final updated state is then obtained via the SOR formulation:
\begin{equation}
\hat{\mathbf{X}}_{t_{i+1}}^{(k+1)}  = (1-\omega)\hat{\mathbf{X}}_{t_{i+1}}^{(k)} + \omega\hat{\mathbf{Z}}_{t_{i+1}}^{(k+1)}.
\end{equation}
By integrating SOR, our proposed solver, ParaLin++, significantly reduces the number of iterations $K$ required to reach the desired tolerance, thereby providing a substantial improvement in overall sampling efficiency without compromising the quality of the generated samples.

\begin{proposition}[Relaxation Factor Analysis]
\label{proposition:Optimal_Relaxation_Factor}
The successive over-relaxation update guarantees convergence if the relaxation factor $\omega$ is chosen within the range $0 < \omega < 2 / (1+L)$, where $L < 1$ is the contraction factor of the underlying Gauss-Seidel-like iterative solver.
\end{proposition}
\begin{proof}
See the supplementary material (Appendix D) for the complete derivation.
\end{proof}

\begin{remark}
 Empirically, selecting $\omega$ within the range of $(1.0, 1.08]$ effectively accelerates the process without compromising the generation quality.
\end{remark}

\subsubsection{Sliding Window Strategy}
\label{Sec:SlidingWindow}
To optimize computational efficiency, ParaSolver utilizes a sliding window of fixed size $p$, retaining only the most recent $p$ denoised points $\{\hat{\mathbf{X}}_{t_n}^{(k)} : n = 0, \dots, p-1\}$ for parallel processing. This approach enhances parallel performance while adhering to GPU memory limitations. Adapted to our framework, the update rule becomes:
\begin{equation}
\label{eq:nonliear_aware_paraiters_for_varpho_NES_window}
\hat{\mathbf{X}}_{t_{n}:t_{n+p}}^{(k+1)} = \hat{\mathbf{X}}_{t_{n}:t_{n+p}}^{(k)} - \mathcal{G}^{(k)} \mathcal{R}_{t_{n}:t_{n+p}}^{(k)},
\end{equation}
where the residual term $\mathcal{R}_{t_{i+1}}^{(k)}$ over a sliding window becomes:
\begin{small}
\begin{equation}
\begin{aligned}
\label{eq:nonliear_aware_iteration_residua_window}
\mathcal{R}_{t_{i+1}}^{(k)} =
\begin{cases}
\hat{\mathbf{X}}_{t_{i+1}}^{(k)} - \mathcal{H}_i\bigl(\hat{\mathbf{X}}_{t_i}^{(k)},\mathcal{P}_{i}^{(k)}\bigr), & i\in [n, n+\varrho-1], \\
\hat{\mathbf{X}}_{t_{i+1}}^{(k)} - \mathcal{H}_i\bigl(\hat{\mathbf{X}}_{t_i}^{(k+1)},\mathcal{P}_{n+\varrho-1}^{(k)}\bigr), & i\in [n+\varrho, p-1].
\end{cases}
\end{aligned}
\end{equation}
\end{small}

\subsubsection{Adaptive Stopping Criterion}
\label{Sec:AdaptiveStop}

Prior works~\cite{shih2024parallel,tang2024accelerating,parasolver} typically employ a single tolerance, denoted as $\delta$, to control the state difference across all timesteps. However, we find that this coarse-grained method cannot accurately identify states that have already converged, leading to superfluous iteration steps to update them. To address this inefficiency, we introduce a more fine-grained, adaptive stopping criterion.

Our approach is based on the principle of relative reduction to a historical average. Instead of comparing the current update error to a fixed threshold, we evaluate it against a fraction of the average historical error. This allows the criterion to dynamically adjust to the convergence progress.
\begin{figure*}
    \centering
    \includegraphics[width=\linewidth]{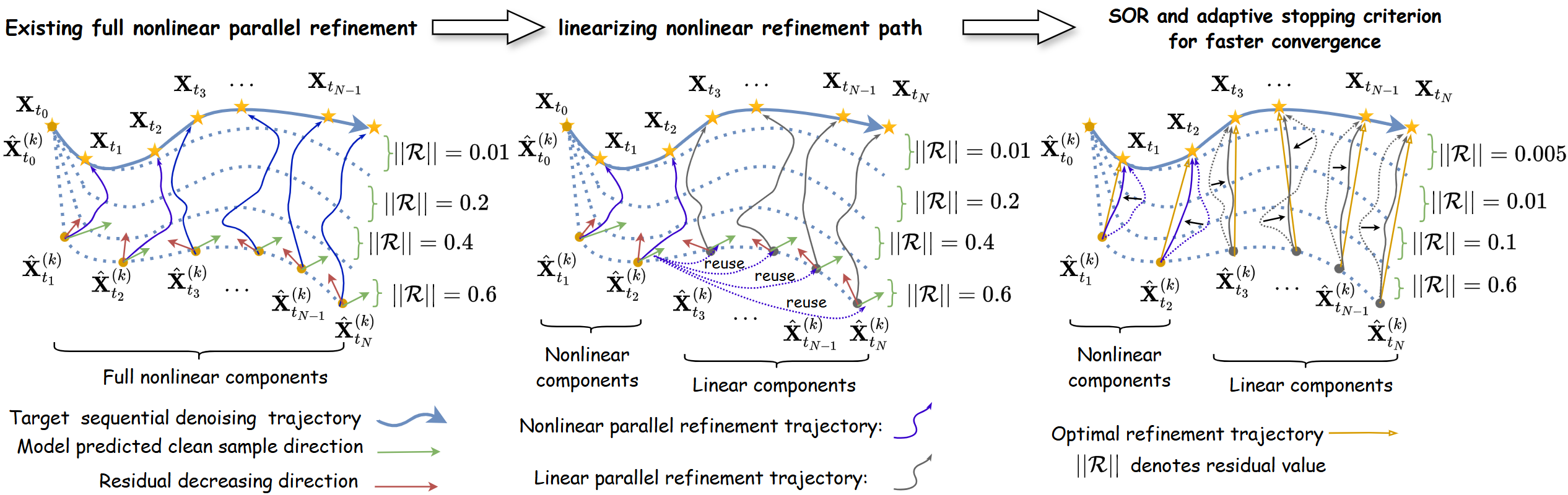}\vspace{-0.5cm}
 \caption{Illustration of  ParaLin++. During iteration $k$, ParaLin++ performs simultaneous denoising updates across multiple steps. From left to right: existing full nonlinear parallel methods, our partial linear method, and our method for faster convergence. }
    \label{fig:parasolver_Pipeline}
\end{figure*}

Specifically, following~\cite{shih2024parallel,tang2024accelerating,parasolver}, let the error of  state $\hat{\mathbf{X}}_{t_{i}}$ at  iteration $k \in[0,K]$ be defined as:
$$
\mathcal{E}^{(k)}_{t_i} = \frac{1}{D\sigma^2_{t_i}}\left\|\hat{\mathbf{X}}_{t_{i}}^{(k+1)} - \hat{\mathbf{X}}_{t_{i}}^{(k)}\right\|_F^2,
$$
where $D$ is the dimensionality of the state $\hat{\mathbf{X}}_{t_{i}}$ and $\|\cdot\|_F$ denotes the Frobenius norm. We maintain a running average of all historical errors, $\bar{\mathcal{E}}_{t_i}$, up to the current iteration $k$:
$$
\bar{\mathcal{E}}_{t_i} = \frac{1}{k+1} \sum_{j=0}^{k} \mathcal{E}^{(j)}_{t_i}.
$$
The stopping condition is then met when the current error, $\mathcal{E}^{(k)}_{t_i}$, is significantly smaller than this historical error average, scaled by a  tolerance $\delta$:
$$
\mathcal{E}^{(k)}_{t_i} \leq \delta \cdot \bar{\mathcal{E}}_{t_i}.
$$

Furthermore, the tolerance $\delta$ is not a fixed hyperparameter but is itself updated adaptively to refine the stopping condition. We introduce an update mechanism with an EDM (Exponential Decay Mean) decay rate, $\beta$, which allows the tolerance to respond to the recent convergence behavior. The tolerance for the next step, $\delta$, is updated as follows:
$$
\delta = \beta \delta + (1-\beta) \cdot \mathcal{M}\Big(\big\{\mathcal{E}^{(k)}_{t_i}  \mid i \in [n, n+p-1]\big\}\Big),
$$
where  \( \mathcal{M}(\cdot) \) computes the median value which provides a robust estimate of recent error.

This dual-adaptive strategy, comparing to the historical average and dynamically updating the tolerance threshold, enables a more efficient convergence. We empirically find it terminates iterations more precisely when a state has stabilized, thereby reducing the number of steps required while preserving sample quality.

\subsubsection{Initialization}
\label{Sec:Initialization}
To establish the diffusion trajectory, we adopt the ParaSolver methodology by performing $M$ steps of a sequential integral solver as preconditioning steps. The initialization process is formally expressed as:
\begin{small}
\begin{equation}
\label{eq:init}
\begin{cases}
\hat{\mathbf{X}}_{t_{n}}^{(0)} = \mathbf{X}_{t_{n}}, & \text{if } 0 \leq n < M, \\
\hat{\mathbf{X}}_{t_{n}}^{(0)} = \mathcal{H}\Big(t_{n}, t_{n-1}, \hat{\mathbf{X}}_{t_{n-1}}^{(0)}, \hat{\mathbf{X}}_{t_N}(\hat{\mathbf{X}}_{t_{M-1}}^{(0)}, \theta)\Big), & \text{if } M \leq n < N,
\end{cases}
\end{equation}
\end{small}
where the initial state $\mathbf{X}_{t_0}$ is sampled from a standard Gaussian distribution, $\mathbf{X}_{t_0} \sim \mathcal{N}(\mathbf{0}, \mathbf{I})$. When $M=0$, following \cite{shih2024parallel}, all samples are initialized with random noise.

\subsubsection{Complete Algorithm}
\label{Sec:Algorithm}
Algorithm~\ref{Alg:ParaSolver} and Figure~\ref{fig:parasolver_Pipeline} outlines the full procedure of the proposed method. Initially, an array of starting guesses $\{\hat{\mathbf{X}}_{t_n}^{(0)}: n = 0, \dots, p\}$ is prepared either through random noises or inexpensive preconditioning steps (Line~\ref{Line:1}). The parallel sampling loop then begins at Line~\ref{Line:while}.
In Lines~\ref{Line:nonlinear_R_start}--\ref{Line:nonlinear_R_end}, all nonlinear equations are solved simultaneously. Subsequently, Lines~\ref{Line:linear_R_start}--\ref{Line:linear_R_end} compute the solutions to the linear equations, leveraging the predicted clean samples obtained from the nonlinear solving phase. This step is computationally efficient since it avoids model evaluations. The discrepancy between the newly computed states and the current states is assessed at Line~\ref{Line:error_start}, followed by determining the feasible window sliding stride at Line~\ref{line:stride}. Line~\ref{Line:Initilize} initializes $s$ new points beyond the current window based on the computed stride. Finally, Lines~\ref{Line:adapt_strat}--\ref{Line:adapt_end} dynamically adjust the tolerance threshold $\delta$ based on the real-time residuals.
\begin{algorithm}[t]
\SetKwInOut{KIN}{Input}
\SetKwInOut{KOUT}{Output}
\SetKwInput{KwReturn}{Return}
\caption{ParaLin++: a partial nonlinear parallel integral solver for diffusion models}
\label{Alg:ParaSolver}
\footnotesize
\LinesNumbered
\KIN{Diffusion model $S_{\theta}$, noise schedule $\alpha_t$ and $\sigma_t$,  subinterval number $N$, preconditioning steps $M$,  tolerance $\delta$, batch window size $p$, sample dimension $D$, nonlinear degree $\varrho$, EMA decay rate $\beta$, relaxation factor  $\omega$.}
\KOUT{A sample.}

Initialize  $\{\mathbf{\hat{X}}_{t_{n}}^{(0)},n=0,\cdots,p\}$ by \eqref{eq:init} \label{Line:1} \hfill\hfill\textcolor{commentcolor}{// Initialize with a few  sampling  steps.}\\
Initialize $\{\mathcal{E}_{t_{n}}^{(0)} = 0,n=0,\cdots,N\}$ \label{Line:error_inital} \hfill\hfill\textcolor{commentcolor}{// Initialize error list.}\\
$n,k \leftarrow 0,0$, $k \in [0, K]$, and $n \in [0, N-1]$\label{line:two}\\
\label{Line:while}\While{$n < N$}{
$\hat{\mathbf{X}}_{t_N}(\hat{\mathbf{X}}_{t_{i}}^{(k)},\theta) = \frac{\hat{\mathbf{X}}_{t_{i}}^{(k)} - \sigma_{t_{i}}\mathbf{S}_{\theta}(\hat{\mathbf{X}}_{t_{i}}^{(k)},t_{i})}{\alpha_{t_{i}}}, \forall i \in \{n, \cdots, n+\varrho-1\}.$\hfill
\textcolor{commentcolor}{// Predict clean samples in parallel.} \label{Line:nonlinear_R_start} \\
$\mathcal{R}_{t_{i+1}}^{(k)} = \hat{\mathbf{X}}_{t_{i+1}}^{(k)} -\mathcal{H}\Big(t_{i+1},t_{i},\hat{\mathbf{X}}_{t_i}^{(k)},\hat{\mathbf{X}}_{t_N}(\hat{\mathbf{X}}_{t_{i}}^{(k)},\theta)\Big), \forall i \in \{n, \cdots, n+\varrho-1\}$ \hfill
\textcolor{commentcolor}{// Nonlinear residual.}  \\
$\hat{\mathbf{Z}}_{t_{n+1}:t_{n+\varrho}}^{(k+1)}  = \hat{\mathbf{X}}_{t_{n+1}:t_{n+\varrho}} ^{(k)} - \mathcal{G}^{(k)} \mathcal{R}_{t_{n+1}:t_{n+\varrho}}^{(k)}$ \label{Line:nonlinear_intermediate_states}\hfill
\textcolor{commentcolor}{//Obtain intermediate states.}  \\
$\hat{\mathbf{X}}_{t_{n+1}:t_{n+\varrho}}^{(k+1)}  = (1-\omega)\hat{\mathbf{X}}_{t_{n+1}:t_{n+\varrho}} ^{(k)} +\omega \hat{\mathbf{Z}}_{t_{n+1}:t_{n+\varrho}}^{(k+1)}$ \label{Line:nonlinear_R_end}\hfill
\textcolor{commentcolor}{// Update the nonlinear components.}  \\
\For{$i \in \{ n+\varrho,n+\varrho+1, \cdots,n+p-1\}$}{
\label{Line:linear_R_start}$\mathcal{R}_{t_{i+1}}^{(k)} = \hat{\mathbf{X}}_{t_{i+1}}^{(k)} -\mathcal{H}\Big(t_{i+1},t_{i},\hat{\mathbf{X}}_{t_i}^{(k+1)},\hat{\mathbf{X}}_{t_N}(\hat{\mathbf{X}}_{t_{n+\varrho-1}}^{(k)},\theta)\Big) $ in~\eqref{eq:refined_iteration_residual}\hfill
\textcolor{commentcolor}{// Linear residual.}  \\
$\hat{\mathbf{Z}}_{t_{i+1}}^{(k+1)}  = \hat{\mathbf{X}}_{t_{i+1}} ^{(k)} - \mathcal{G}^{(k)} \mathcal{R}_{t_{i+1}}^{(k)}$\label{Line:linear_intermediate_states}\hfill
\textcolor{commentcolor}{// Obtain intermediate states.}  \\
$\hat{\mathbf{X}}_{t_{i+1}}^{(k+1)}  = (1-\omega)\hat{\mathbf{X}}_{t_{i+1}} ^{(k)} + \omega\hat{\mathbf{Z}}_{t_{i+1}}^{(k+1)} $\label{Line:linear_R_end}\hfill
\textcolor{commentcolor}{// Update the linear components.}  \\
}
$\mathcal{E}_{t_{i+1}}^{(k)} = \frac{1}{D\sigma^2_{t_{i+1}}}||\mathbf{X}_{t_{i+1}}^{(k+1)}-\mathbf{X}_{t_{i+1}}^{(k)}||_F^2, \forall i \in \{n, \cdots, n+p-1\}$\label{Line:error_start}\hfill
\textcolor{commentcolor}{// Compute error for each component.}  \\
$\bar{\mathcal{E}}_{t_{i+1}} = \frac{1}{1+k}\sum_{j=0}^{k}\mathcal{E}_{t_{i+1}}^{(j)}, \forall i \in \{n, \cdots, n+p-1\}$\label{Line:error_mean}\hfill
\textcolor{commentcolor}{// Compute average error  for each component.}  \\
$s = \min\limits_{i+1}\big(\big\{i+1| \mathcal{E}_{t_{i+1}}^{(k)}>\delta^2\bar{\mathcal{E}}_{t_{i+1}} \forall i \in [n, n+p)\big\}\cup\big\{n+p\big\}\big)$\label{line:stride} \hfill \textcolor{commentcolor}{// The sliding stride.} \\
$\hat{\mathbf{X}}_{t_{i+1}}^{(k+1)}= \mathcal{H}\Big(t_{i+1},t_{i},\hat{\mathbf{X}}_{t_i}^{(k+1)},\hat{\mathbf{X}}_{t_N}(\hat{\mathbf{X}}_{t_{n+\varrho-1}}^{(k)},\theta)\Big), \forall i \in \{n+p, \cdots, n+p+s-1\}$\label{Line:Initilize} \hfill\textcolor{commentcolor}{// Initialize new components for  the next window.}\\
$e = \mathcal{M}\Big(\{\mathcal{E}_{t_{i+1}}^{(k)}|i = n,\cdots,n+p-1\}\Big)$\label{Line:adapt_strat}\hfill\textcolor{commentcolor}{// Compute the median value of the errors.}\\
$\delta = \beta\delta +  (1-\beta)e $\label{Line:adapt_end}\hfill\textcolor{commentcolor}{// Adaptive tolerance updating.}\\
$n \leftarrow n+ s, \quad k \leftarrow k+1,\quad p \leftarrow \min (p, N-n)$\hfill\textcolor{commentcolor}{// Update the iteration indexs.}\\
}
\KwReturn{$\hat{\mathbf{X}}_{t_{N}}^{(K)}$}
\end{algorithm}

\section{Quality--Speed Audit on Generation-Quality Benchmarks}\label{sec:experiments}
\subsection{Experiment Settings}
\noindent\textbf{Evaluation Metrics and Datasets.}
We report function evaluations (NFEs), iterations, wall-clock time, FID~\cite{fid}, and CLIP score~\cite{hessel2021clipscore}. For Stable Diffusion-v2 and LSUN Church, FID and CLIP are computed on 5000 random samples, using ImageNet-1k with classifier guidance $w=7.5$ and the LSUN model, respectively. For FLUX, FID is computed on 5000 FluxPrompting~\cite{FluxPrompting} samples for the 12B model~\cite{bfl2024flux}. The CIFAR-10 less-step study uses 50,000 samples, as reported in Table~\ref{tab:performance}. For generation-quality-focused evaluation, we additionally report GenEval~\cite{ghosh2024geneval}, DPG-Bench~\cite{hu2024ella}, and LongTextBench~\cite{xomni2025}, which assess semantic composition, dense instruction following, and long in-image text rendering, respectively.

\noindent\textbf{Models.} The evaluation covers latent-space Stable Diffusion-v2, CIFAR-10, the 12B FLUX flow-matching model, the pixel-space LSUN Church model, and VP/VE/EDM processes. The generation-quality study centers on FLUX and uses the Ming-Flash-Omni 2.0 generator~\cite{ai2025ming} as a second large-scale reference model.

\noindent\textbf{Algorithms.} ParaLin++ is applied to DDIM~\cite{DDIM} and DPMSolver~\cite{dpmsolver}. It is compared with ParaLin~\cite{paralin2026}, ParaDiGMS~\cite{shih2024parallel}, ParaSolver~\cite{parasolver}, and relevant few-step solvers. All compared baseline samplers are invoked through the official implementations integrated in the Diffusers library~\cite{diffusers}; we do not use custom re-implementations of the baseline solvers.

\noindent\textbf{Hyperparameters.} Following~\cite{parasolver,tang2024accelerating}, we use the best tolerances reported for ParaDiGMS and ParaSolver. For ParaLin++, we set $\varrho=8$ unless otherwise specified, use eight NVIDIA RTX~3090 GPUs for multi-GPU settings, and search $\delta$, $\beta$, and $\omega$ independently for each sampler. Additional details are given in the supplementary material.

\noindent\textbf{FLUX generation-quality benchmark protocol.}
The FLUX.1-dev study uses $1024{\times}1024$ generation with guidance scale $3.5$ on eight GPUs. The Euler sampler with 30 discrete diffusion steps is the reference. UniPC and DPM-Solver++ are evaluated with their standard sequential few-step samplers at $16,20,24,26$, and $30$ discrete diffusion steps. ParaLin++ keeps 30 discrete diffusion steps and uses the proposed parallel solver with tolerance $\delta=0.01$, nonlinear-transition parameter $\varrho=1$, parallel window size $p=100$, and a solver-step cap of $1000$. LongTextBench is the average of the English and Chinese text scores. Speedup is measured from the sampler/model forward-path time of the GenEval generation pass and excludes model loading and scoring; these timing measurements are used consistently across the detailed FLUX table.

\noindent\textbf{Ming-Flash-Omni 2.0 generation-quality benchmark protocol.}
The Ming-Flash-Omni 2.0 study uses the Ming-Flash-Omni 2.0 generator at $1024{\times}1024$ with guidance scale $5.0$ on eight GPUs. Euler, UniPC, and DPM-Solver++ are evaluated with their standard sequential samplers at the listed step settings. ParaLin++ keeps 30 discrete diffusion steps and uses the proposed parallel solver with tolerance $\delta=0.01$, nonlinear-transition parameter $\varrho=1$, parallel window size $p=100$, and a solver-step cap of $1000$. LongTextBench uses 320 prompts with 4 images each, seed 42, hybrid OCR, a bounding-box score threshold of $0.85$, and reports the average of the English and Chinese scores. DPG-Bench uses 1065 images with the mPLUG visual-question-answering model. GenEval uses 553 prompts with 4 images each, bf16 generation, and 8 evaluation workers. The table reports the complete $16$-, $20$-, $24$-, $26$-, and $30$-step quality matrix. All speedup values use the CFG-5 GenEval generation pass against the $30$-step Euler reference, excluding model loading and scoring.
\subsection{Experimental Results}
\noindent\textbf{Latent-space Speedup Comparison}.
Table~\ref{Table:SDv2_CMP} shows the experimental results on DDIM, DPMSolver, and their parallel variants when combined with ParaSolver, ParaDiGMS, ParaLin, and our ParaLin++, respectively. ParaLin++ demonstrates significant acceleration for DDIM and DPMSolver while maintaining comparable FID and CLIP scores. Notably, it achieves a remarkable wall-clock time speedup of up to $4.8\times$, setting a new benchmark for parallel diffusion methods. Although parallelization inherently trades computational efficiency for speed, leading to increased NFEs relative to sequential approaches, our method capitalizes on partial linearity to substantially reduce NFEs, thereby unleashing the parallel efficiency.
\begin{table}[!t]
  \caption{\label{Table:SDv2_CMP}Comparisons  on Stable Diffusion-v2, with classifier guidance w = 7.5.
  The visual comparisons are shown
  in the supplementary figures.
  }\vspace{-0.25cm}
  \centering
  \setlength{\tabcolsep}{2.5pt}
  \renewcommand{\arraystretch}{0.95}
  \resizebox{0.95\columnwidth}{!}{%
  \begin{tabular}{c|c|cccccc}
    \toprule[1.5pt]
     \multirow{2}{*}{Steps} &\multirow{2}{*}{Method} &\multicolumn{6}{c}{Stable Diffusion-v2}\\
    \cmidrule(lr){3-8}  
	&&~\begin{tabular}[c]{@{}c@{}}  Iters\end{tabular} $\downarrow$~&~\begin{tabular}[c]{@{}c@{}} NFE\end{tabular} $\downarrow$~   &~CLIP$\uparrow$~ &~FID$\downarrow$~ &~Time (s)$\downarrow$~ &~Speedup$\uparrow$~\\
    \midrule
    \multirow{5}{*}{50} &DDIM~\cite{DDIM}      &50 &50 &25.6 &57.2 &6.3\hfill\progressbar{63} & 1.0$\times$  \\
     &DDIM + ParaDiGMS~\cite{shih2024parallel}      &21 &132 &25.6 &56.9 &3.3 \hfill\progressbar{33}  & 1.9$\times$  \\
     &DDIM + ParaSolver~\cite{parasolver}     &13 &83 &25.6 &56.9 &1.9 \hfill\progressbar{19} & 3.3$\times$  \\
     &DDIM + ParaLin \cite{paralin2026}      &13&73&25.6 &57.1 &1.7\hfill\progressbar{17} & 3.7$\times$  \\
     \rowcolor{paralinrow} &DDIM + ParaLin++      &\textbf{10}&\textbf{69}&25.5 &57.0 &\textbf{1.3} \hfill\progressbar{13} & \textbf{4.8$\times$}  \\

    \midrule
    \multirow{5}{*}{25} &DDIM~\cite{DDIM}      &25 &25 &25.4 &62.9 &3.4 \hfill\progressbar{34} & 1.0$\times$  \\
     &DDIM + ParaDiGMS~\cite{shih2024parallel}     &18 &53  &25.4  &62.8 &3.2 \hfill\progressbar{32}  & 1.2$\times$ \\
     &DDIM + ParaSolver~\cite{parasolver}       &11 &49  &25.4  &61.9 &1.6 \hfill\progressbar{16}  & 2.1$\times$ \\
     &DDIM + ParaLin \cite{paralin2026}      &11&42  &25.4  &61.8 &1.55 \hfill\progressbar{15.5}  & 2.2$\times$ \\
     \rowcolor{paralinrow} &DDIM + ParaLin++      &\textbf{9}&\textbf{40}  &25.4  &61.9 &\textbf{1.3} \hfill\progressbar{13}  & \textbf{2.6}$\times$ \\
    \midrule
    \multirow{5}{*}{50} &DPMSolver~\cite{dpmsolver}      &50 &50 &25.6 &57.2 &6.3 \hfill\progressbar{63} & 1.0$\times$ \\
     &DPMSolver + ParaDiGMS~\cite{shih2024parallel}      &25 &132 &25.6 &57.2 &3.3 \hfill\progressbar{33} & 1.8$\times$  \\
     &DPMSolver + ParaSolver~\cite{parasolver}       &15 &96 &25.6 &57.1 &2.2 \hfill\progressbar{22} & 3.1$\times$  \\
     &DPMSolver + ParaLin \cite{paralin2026}      &14 &93 &25.6 &57.1 &1.9\hfill\progressbar{19} & 3.3$\times$  \\
     \rowcolor{paralinrow} &DPMSolver + ParaLin++      &\textbf{10} &\textbf{77} &25.6 &57.1 &\textbf{1.35} \hfill\progressbar{13.5} & \textbf{4.7$\times$}  \\
    \midrule
    \multirow{5}{*}{25} &DPMSolver~\cite{dpmsolver}      &25 &25 &25.4 &62.3 &3.2 \hfill\progressbar{32}& 1.0$\times$  \\
     &DPMSolver + ParaDiGMS~\cite{shih2024parallel}      &15 &81  &25.4 &62.3 &2.2 \hfill\progressbar{22} & 1.5$\times$  \\
     &DPMSolver + ParaSolver~\cite{parasolver}       &11 &58 &25.4 &62.2 &1.5\hfill\progressbar{15} & 2.1$\times$   \\
     &DPMSolver + ParaLin \cite{paralin2026}      &10&55 &25.4 &62.1 &1.4\hfill\progressbar{14} & 2.3$\times$   \\
     \rowcolor{paralinrow} &DPMSolver + ParaLin++      &\textbf{8}&\textbf{46} &25.4 &62.1 &\textbf{1.1}\hfill\progressbar{11} & \textbf{2.9}$\times$   \\
     \bottomrule[1.5pt]
  \end{tabular}}
  \vspace{-0.2cm}
\end{table}

\noindent\textbf{Pixel-space Speedup Comparison}.
Similarly, Table~\ref{Table:LSUNChurch_CMP} shows that ParaLin++ consistently surpasses existing methods, delivering wall-clock speedups of $3.1\times$--$4.5\times$ on the LSUN Church model across the tested step configurations. For the 50-step and 25-step DDIM settings, it achieves $4.5\times$ and $3.1\times$ speedups while preserving a nearly identical FID score ($15.7$ versus $15.5$), respectively. Notably, ParaLin++'s parallelism reduces iterations by $3.1\times$--$5.6\times$ relative to the sequential baselines and, among the parallel variants, keeps NFE low; in the 50-step DDIM case, it uses only 9 iterations and 64 NFEs versus the baseline's 50. This acceleration is important for real-time applications, such as interactive design tools, where latency directly impacts user experience.
\begin{table}[!t]
\caption{\label{Table:LSUNChurch_CMP}Quantitative comparisons of different methods on LSUN Church over $5000$  samples.  DPMSolver is not integrated with the LSUN model in the Diffusers library; thus we exclude it.}
\vspace{-0.25cm}
  \centering
  \setlength{\tabcolsep}{2.5pt}
  \renewcommand{\arraystretch}{1.0}
  \resizebox{0.95\columnwidth}{!}{%
  \begin{tabular}{c|c|ccccc}
    \toprule[1.5pt]
     \multirow{2}{*}{Steps} &\multirow{2}{*}{Method} &\multicolumn{5}{c}{LSUN Church}\\
    \cmidrule(lr){3-7}  
    &&~\begin{tabular}[c]{@{}c@{}} Iters\end{tabular} $\downarrow$~&~\begin{tabular}[c]{@{}c@{}}NFE\end{tabular} $\downarrow$~    &~FID$\downarrow$~ &~Time (s)$\downarrow$~ &~Speedup$\uparrow$~\\
    \midrule
    \multirow{5}{*}{50} &DDIM~\cite{DDIM}      &50 &50  &15.5 &1.8 \hfill\progressbar{6.8} &1.0$\times$ \\
     &DDIM + ParaDiGMS~\cite{shih2024parallel}      &23 &202  &15.8 &1.5 \hfill\progressbar{3.2} &1.2$\times$ \\
     &DDIM + ParaSolver~\cite{parasolver}      &14 &73  &15.7 &0.8 \hfill\progressbar{0.8} &2.2$\times$ \\
     &DDIM + ParaLin \cite{paralin2026}      &14 &67  &15.7 &0.6 \hfill\progressbar{0.6} &3.0$\times$ \\
     \rowcolor{paralinrow} &DDIM + ParaLin++      &9 &64  &15.7 &\textbf{0.4} \hfill\progressbar{0.4} &\textbf{4.5$\times$} \\
    \midrule
    \multirow{5}{*}{25} &DDIM~\cite{DDIM}      &25 &25  &15.6 &1.1 \hfill\progressbar{5.1} &1.0$\times$ \\
     &DDIM + ParaDiGMS~\cite{shih2024parallel}     &15 &96    &15.7 &0.8 \hfill\progressbar{3.0}  &1.4$\times$ \\
     &DDIM + ParaSolver~\cite{parasolver}      &10 &44    &15.7 &0.5 \hfill\progressbar{0.5}  &2.2$\times$ \\
     &DDIM + ParaLin \cite{paralin2026}      &10 &39    &15.7 &0.45 \hfill\progressbar{0.45}  &2.4$\times$ \\
     \rowcolor{paralinrow} &DDIM + ParaLin++      &8 &33    &15.7 &\textbf{0.36} \hfill\progressbar{0.36}  &\textbf{3.1$\times$} \\
    \bottomrule[1.5pt]
  \end{tabular}}
\end{table}

\noindent\textbf{Less-Step Sampling.} Table~\ref{tab:performance} evaluates 10-step sampling on CIFAR-10 with 50,000 samples. ParaLin++ combined with DPM-Solver-v3~\cite{zheng2023dpm} reaches an FID of $2.64$ and a $2.41\times$ speedup, matching the 25-step DPM-Solver-v3 quality while improving over the compared 10-step baselines.
\begin{table}[!t]
  \centering
  \caption{Less-step sampling on the CIFAR-10 model with 50k samples.}
  \label{tab:performance}
  \setlength{\tabcolsep}{3.5pt}
  \begin{tabular}{lccc}
    \toprule
    Method & Iters & FID$\downarrow$ & Speedup$\uparrow$\\
    \midrule
    DPM-Solver-v3~\cite{zheng2023dpm} & 25 & 2.64 & $1.00\times$\\
    DPM-Solver-v3~\cite{zheng2023dpm} & 10 & 3.40 & $2.43\times$\\
    UniPC~\cite{zhao2023unipc} & 10 & 3.40 & $2.41\times$\\
    PD~\cite{Progressivedistillation} & 10 & 3.10 & $2.48\times$\\
    CD~\cite{consistencymodel} & 10 & 2.90 & $2.46\times$\\
    \rowcolor{paralinrow} ParaLin++~+ DPM-Solver-v3 & 10 & \textbf{2.64} & $2.41\times$\\
    \bottomrule
  \end{tabular}
\end{table}


\noindent\textbf{Generation-Quality Benchmark Evaluation.}
\begin{table*}[!t]
\centering
\caption{Motivating comparison on FLUX. Few-step sequential solvers alter the 30-step discretization and sharply degrade semantic composition, complex prompt following, and long-text rendering. ParaLin++ keeps the same 30 discrete diffusion steps and parallelizes the solver iterations. For sequential solvers, ``Iters'' equals the number of discrete steps; for ParaLin++, it reports the average number of parallel solver iterations.}
\vspace{-0.25cm}
\label{Table:Industrial_Motivation_FLUX}
\resizebox{0.9\textwidth}{!}{%
\begin{tabular}{llrccccc}
\toprule
Method & Sampling regime & Discrete steps & Iters$\downarrow$ & GenEval$\uparrow$ & DPG-Bench$\uparrow$ & LongTextBench$\uparrow$ & Speedup$\uparrow$\\
\midrule
Euler~\cite{bfl2024flux} & Sequential reference & 30 & 30 & \textbf{0.66} & 83.97 & \textbf{0.27} & $1.00\times$\\
UniPC~\cite{zhao2023unipc} & Few-step & 16 & 16 & 0.52 & 77.46 & 0.09 & $1.81\times$\\
DPMSolver++~\cite{dpmsolverpp} & Few-step & 16 & 16 & 0.50 & 75.75 & 0.08 & $1.85\times$\\
\rowcolor{paralinrow} ParaLin++ & Step-parity parallel & 30 & \textbf{16.37} & 0.64 & \textbf{84.24} & 0.26 & $1.68\times$\\
\bottomrule
\end{tabular}}
\end{table*}
Prior FID-oriented evaluations of DDIM, DPMSolver, DPMSolver++, and UniPC support the common expectation that few-step solvers can preserve distributional statistics~\cite{DDIM,dpmsolver,dpmsolverpp,zhao2023unipc}. Table~\ref{Table:Industrial_Motivation_FLUX} therefore asks whether this expectation transfers to GenEval, DPG-Bench, and LongTextBench for the same FLUX solvers. UniPC and DPMSolver++ with 16 discrete steps lose $21\%$--$24\%$ relative GenEval score, $6.5$--$8.2$ DPG-Bench points, and $67\%$--$70\%$ of the LongTextBench score. ParaLin++ instead keeps all 30 discrete diffusion steps and parallelizes the iterative solve: it closely preserves the Euler-level generation-quality profile ($0.64$ versus $0.66$ GenEval, $84.24$ versus $83.97$ DPG-Bench, and $0.26$ versus $0.27$ LongTextBench) and achieves a $1.68\times$ GenEval-forward speedup. Hence, the acceleration is obtained by removing latency in the solver iteration rather than by shortening the diffusion discretization on which prompt-level and in-image text fidelity depends.

This evaluation protocol also reflects the objectives of recent large-scale generators. Ming-Flash-Omni 2.0~\cite{ai2025ming}, the sparse unified successor to Ming-Omni~\cite{Mingomni2025}, reports $0.94$ GenEval and $86.98$ DPG-Bench, and its post-training explicitly optimizes realism, instruction adherence, and aesthetic quality together with task-specific rewards. These are generation-quality axes that FID does not reliably capture. Accordingly, the FLUX study and Ming-Flash-Omni 2.0 reference establish a more realistic quality test bed for judging whether a fast diffusion sampler merely matches a distributional statistic or actually preserves generation quality.
In addition to the motivating 16-step comparison, Table~\ref{Table:Industrial_Benchmark_Detailed} records the complete 16-, 20-, 24-, 26-, and 30-step generation-quality benchmark matrices for FLUX and Ming-Flash-Omni 2.0. The table separates the number of discrete diffusion steps from the number of solver iterations: sequential samplers use one iteration per step, whereas ParaLin++ keeps 30 discrete steps and reports its average number of parallel solver iterations separately. For Ming-Flash-Omni 2.0, ParaLin++ reaches a $1.21\times$ GenEval generation-pass speedup over the 30-step Euler reference while retaining the Euler-level profile: GenEval is $0.87$ versus $0.88$, DPG-Bench is $86.99$ versus $86.86$, and LongTextBench is $0.95$ versus $0.95$.

\begin{table*}[!t]
\centering
\caption{Detailed generation-quality benchmark results on FLUX and Ming-Flash-Omni 2.0. Within each model block, speedup is measured against that model's 30-step Euler baseline on the same GenEval generation pass, excluding model loading and scoring. For sequential samplers, ``Iters'' equals the number of discrete steps; for ParaLin++, it denotes the number of parallel solver iterations.}
\vspace{-0.25cm}
\label{Table:Industrial_Benchmark_Detailed}
\resizebox{0.9\textwidth}{!}{%
\begin{tabular}{llcccccc}
\toprule
\multicolumn{8}{c}{\textbf{FLUX.1-dev~\cite{bfl2024flux}}}\\
\midrule
Method & Sampling regime & Steps & Iters & GenEval$\uparrow$ & DPG-Bench$\uparrow$ & LongTextBench$\uparrow$ & Speedup$\uparrow$\\
\midrule
Euler~\cite{bfl2024flux} & Sequential reference & 30 & 30 & 0.66 & 83.97 & 0.27 & $1.00\times$\\
\midrule
\multirow{5}{*}{UniPC~\cite{zhao2023unipc}} & \multirow{5}{*}{Few-step}
 & 16 & 16 & 0.52 & 77.46 & 0.09 & $1.81\times$\\
 & & 20 & 20 & 0.55 & 79.67 & 0.13 & $1.49\times$\\
 & & 24 & 24 & 0.57 & 80.98 & 0.15 & $1.26\times$\\
 & & 26 & 26 & 0.58 & 81.35 & 0.16 & $1.16\times$\\
 & & 30 & 30 & 0.60 & 82.14 & 0.19 & $1.01\times$\\
\midrule
\multirow{5}{*}{DPMSolver++~\cite{dpmsolverpp}} & \multirow{5}{*}{Few-step}
 & 16 & 16 & 0.50 & 75.75 & 0.08 & $1.85\times$\\
 & & 20 & 20 & 0.54 & 78.41 & 0.11 & $1.50\times$\\
 & & 24 & 24 & 0.58 & 80.06 & 0.14 & $1.27\times$\\
 & & 26 & 26 & 0.58 & 80.56 & 0.15 & $1.17\times$\\
 & & 30 & 30 & 0.62 & 81.43 & 0.17 & $1.02\times$\\
\midrule
\rowcolor{paralinrow} ParaLin++ & Step-parity parallel & 30 & 16.37 & 0.64 & 84.24 & 0.26 & $1.68\times$\\
\midrule
\multicolumn{8}{c}{\textbf{Ming-Flash-Omni 2.0~\cite{ai2025ming}}}\\
\midrule
Method & Sampling regime & Steps & Iters & GenEval$\uparrow$ & DPG-Bench$\uparrow$ & LongTextBench$\uparrow$ & Speedup$\uparrow$\\
\midrule
Euler~\cite{bfl2024flux} & Sequential reference & 30 & 30 & 0.88 & 86.86 & 0.95 & $1.00\times$\\
\midrule
\multirow{5}{*}{UniPC~\cite{zhao2023unipc}} & \multirow{5}{*}{Few-step}
 & 16 & 16 & 0.69 & 84.38 & 0.82 & $1.80\times$\\
 & & 20 & 20 & 0.75 & 85.29 & 0.91 & $1.47\times$\\
 & & 24 & 24 & 0.74 & 85.55 & 0.94 & $1.20\times$\\
 & & 26 & 26 & 0.81 & 85.07 & 0.94 & $1.13\times$\\
 & & 30 & 30 & 0.83 & 85.77 & 0.95 & $0.98\times$\\
\midrule
\multirow{5}{*}{DPMSolver++~\cite{dpmsolverpp}} & \multirow{5}{*}{Few-step}
 & 16 & 16 & 0.73 & 84.20 & 0.80 & $1.81\times$\\
 & & 20 & 20 & 0.70 & 85.63 & 0.87 & $1.48\times$\\
 & & 24 & 24 & 0.70 & 85.91 & 0.92 & $1.23\times$\\
 & & 26 & 26 & 0.86 & 85.93 & 0.93 & $1.13\times$\\
 & & 30 & 30 & 0.85 & 86.62 & 0.95 & $0.99\times$\\
\midrule
\rowcolor{paralinrow} ParaLin++ & Step-parity parallel & 30 & 24.69 & 0.87 & 86.99 & 0.95 & $1.21\times$\\
\bottomrule
\end{tabular}}
\end{table*}

Table~\ref{Table:FLUX_CFG_Sweep} extends the FLUX audit to different guidance strengths. Across the five CFG values, ParaLin++ obtains GenEval speedups of $1.55\times$--$1.86\times$ over the same-CFG 30-step Euler reference while preserving $97\%$--$99\%$ of the Euler GenEval score. Its GenEval scores are $0.59$--$0.66$ and its LongTextBench scores are $0.15$--$0.26$, close to Euler's respective ranges of $0.60$--$0.67$ and $0.17$--$0.27$. In contrast, the 20-step UniPC and DPMSolver++ configurations yield only $0.43$--$0.55$ GenEval and $0.06$--$0.13$ LongTextBench across the same CFG values. The DPG-Bench scores, now complete at all five CFG values, follow the same pattern: ParaLin++ tracks the Euler reference closely across CFG $0$--$14$ ($82.18$--$84.24$ versus Euler's $82.25$--$83.97$), whereas the 20-step UniPC and DPMSolver++ configurations drop to $68.85$--$79.67$. The reference scores also show that guidance affects benchmarks differently: Euler's GenEval score peaks at CFG $10.5$, whereas its LongTextBench score peaks at CFG $3.5$. ParaLin++ remains close to the reference at every guidance setting. \textbf{(Takeaway)} Across CFG $0$--$14$, the completed DPG-Bench matrix confirms that step-parity survives rising guidance while 20-step few-step solvers do not.

\begin{table}[!t]
\centering
\caption{FLUX.1-dev CFG-sweep results. All benchmark cells (GenEval, DPG-Bench, and LongTextBench) are complete for all five CFG values; LTB abbreviates LongTextBench and DPM++ abbreviates DPMSolver++. ``Iters'' equals the discrete-step count for sequential baselines and denotes the average measured parallel solver iterations for ParaLin++. Speedup is measured against the 30-step Euler reference at the same CFG on the GenEval generation pass, excluding model loading and scoring.} \vspace{-0.2cm}
\label{Table:FLUX_CFG_Sweep}
\setlength{\tabcolsep}{2pt}
\renewcommand{\arraystretch}{0.98}
\resizebox{0.85\columnwidth}{!}{%
\begin{tabular}{ccccccc}
\toprule
CFG & Method & Iters & GenEval$\uparrow$ & DPG$\uparrow$ & LTB$\uparrow$ & Speedup$\uparrow$\\
\midrule
\multirow{4}{*}{0}
 & Euler~\cite{bfl2024flux} & 30 & 0.60 & 82.25 & 0.17 & $1.00\times$\\
 & UniPC~\cite{zhao2023unipc} & 20 & 0.46 & 78.97 & 0.08 & $1.49\times$\\
 & DPM++~\cite{dpmsolverpp} & 20 & 0.43 & 77.71 & 0.06 & $1.49\times$\\
\rowcolor{paralinrow} & ParaLin++ & 14.79 & 0.59 & 82.18 & 0.15 & $1.86\times$\\
\midrule
\multirow{4}{*}{3.5}
 & Euler~\cite{bfl2024flux} & 30 & 0.66 & 83.97 & 0.27 & $1.00\times$\\
 & UniPC~\cite{zhao2023unipc} & 20 & 0.55 & 79.67 & 0.13 & $1.49\times$\\
 & DPM++~\cite{dpmsolverpp} & 20 & 0.54 & 78.41 & 0.11 & $1.50\times$\\
\rowcolor{paralinrow} & ParaLin++ & 16.92 & 0.64 & 84.24 & 0.26 & $1.68\times$\\
\midrule
\multirow{4}{*}{7}
 & Euler~\cite{bfl2024flux} & 30 & 0.66 & 83.29 & 0.25 & $1.00\times$\\
 & UniPC~\cite{zhao2023unipc} & 20 & 0.49 & 72.60 & 0.07 & $1.49\times$\\
 & DPM++~\cite{dpmsolverpp} & 20 & 0.48 & 68.85 & 0.06 & $1.49\times$\\
\rowcolor{paralinrow} & ParaLin++ & 17.85 & 0.64 & 83.22 & 0.23 & $1.55\times$\\
\midrule
\multirow{4}{*}{10.5}
 & Euler~\cite{bfl2024flux} & 30 & 0.67 & 83.21 & 0.22 & $1.00\times$\\
 & UniPC~\cite{zhao2023unipc} & 20 & 0.52 & 77.89 & 0.07 & $1.48\times$\\
 & DPM++~\cite{dpmsolverpp} & 20 & 0.52 & 76.54 & 0.06 & $1.49\times$\\
\rowcolor{paralinrow} & ParaLin++ & 17.03 & 0.66 & 83.23 & 0.21 & $1.62\times$\\
\midrule
\multirow{4}{*}{14}
 & Euler~\cite{bfl2024flux} & 30 & 0.65 & 82.52 & 0.25 & $1.00\times$\\
 & UniPC~\cite{zhao2023unipc} & 20 & 0.54 & 75.88 & 0.08 & $1.49\times$\\
 & DPM++~\cite{dpmsolverpp} & 20 & 0.54 & 74.17 & 0.07 & $1.49\times$\\
\rowcolor{paralinrow} & ParaLin++ & 17.13 & 0.64 & 82.75 & 0.23 & $1.61\times$\\
\bottomrule
\end{tabular}}
\end{table}

Figures~\ref{fig:industrial_visual_flux} provide benchmark-specific visual inspections of the same quality axes. The FLUX examples use the exact 30-step reference/few-step settings compared above: Euler and ParaLin++ retain 30 discrete steps, while UniPC and DPMSolver++ use 16 steps. The examples show typical failures hidden by an FID-only criterion, such as missing or malformed objects, weaker dense-instruction compliance, and unreadable in-image text. The step-parity ParaLin++ outputs remain visually close to the 30-step Euler reference. 

\begin{figure}[!t]
\centering
\includegraphics[width=\columnwidth]{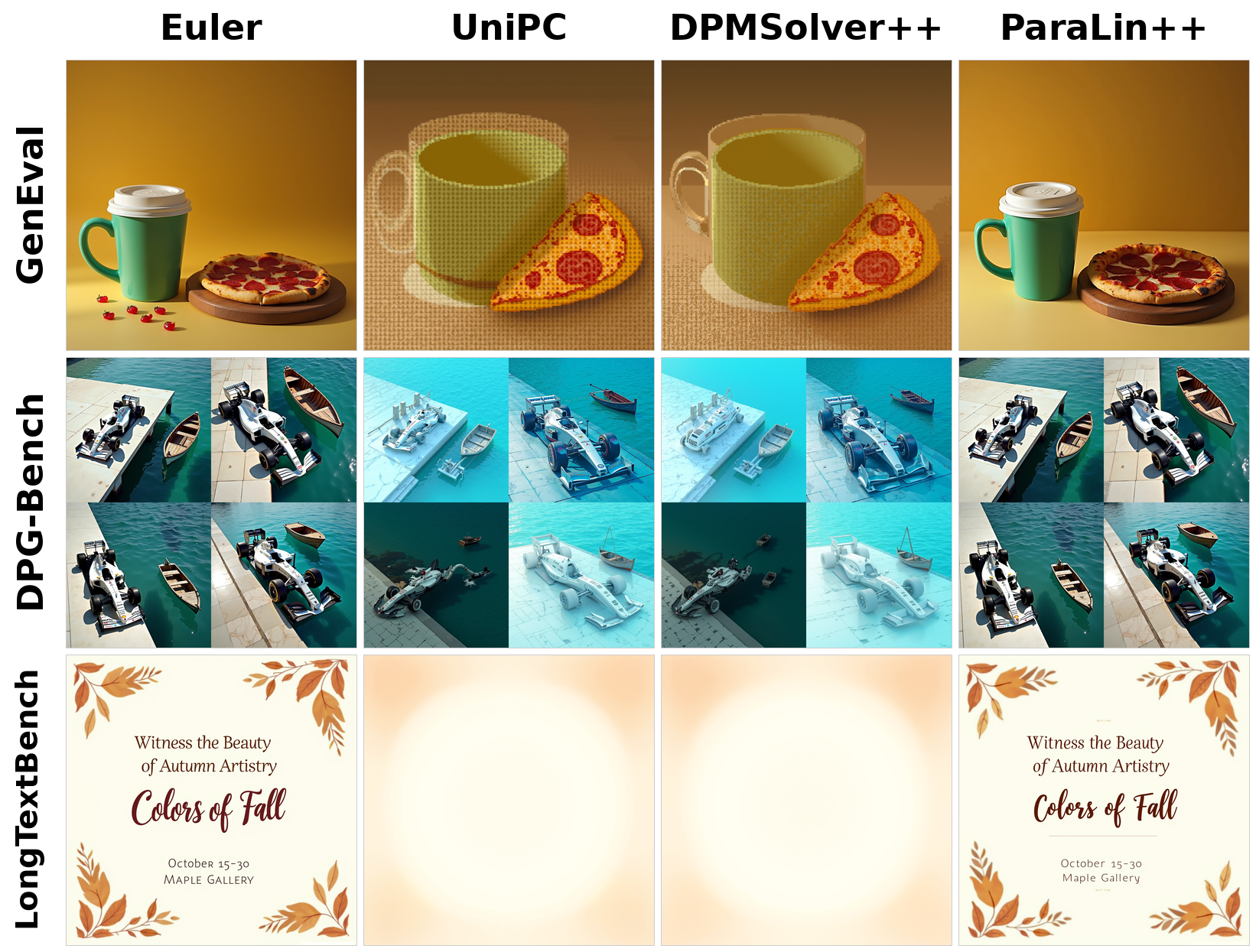}
\caption{Generation-quality benchmark visual comparison on FLUX.1-dev. Each row is one representative benchmark prompt, and the columns show 30-step Euler, 16-step UniPC, 16-step DPMSolver++, and 30-step ParaLin++ (17 parallel solver iterations). Few-step examples retain plausible low-level appearance but can violate semantic composition, dense instructions, or in-image text requirements; ParaLin++ remains visually aligned with the 30-step Euler reference.}
\label{fig:industrial_visual_flux}
\end{figure}


Figures~\ref{fig:flux_cfg_sweep_geneval} and~\ref{fig:flux_cfg_sweep_longtextbench} show fixed-prompt comparisons over all five CFG values. The same prompt and sample index are used in every column and row, so differences come from the method and CFG rather than from selecting a different example for each method. LongTextBench is represented by an English prompt because its reported metric averages the English and Chinese subsets. Figure~\ref{fig:flux_cfg_sweep_dpg} gives corresponding DPG-Bench samples across all five guidance settings for a representative prompt; the complete CFG-$0$--$14$ DPG-Bench matrix is reported in Table~\ref{Table:FLUX_CFG_Sweep}.

\begin{figure}[!t]
\centering
\includegraphics[width=0.96\columnwidth]{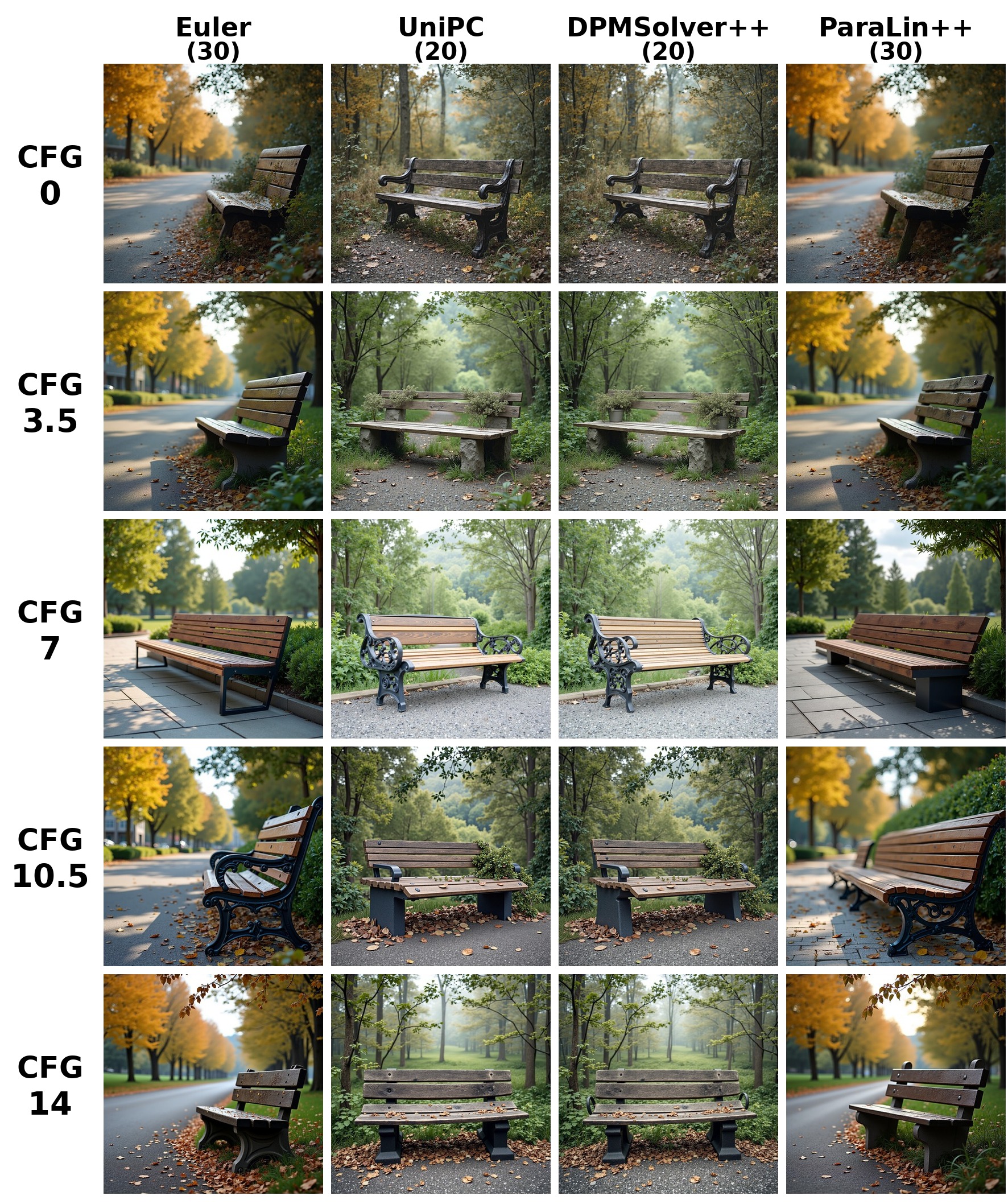}
\vspace{-0.25cm}
\caption{GenEval CFG-sweep visual comparison on FLUX.1-dev. Rows are CFG $0$, $3.5$, $7$, $10.5$, and $14$; columns are 30-step Euler, 20-step UniPC, 20-step DPMSolver++, and 30-step ParaLin++. All cells use prompt 00000 and sample 0000.}
\label{fig:flux_cfg_sweep_geneval}
\end{figure}

\begin{figure}[!t]
\centering
\includegraphics[width=0.96\columnwidth]{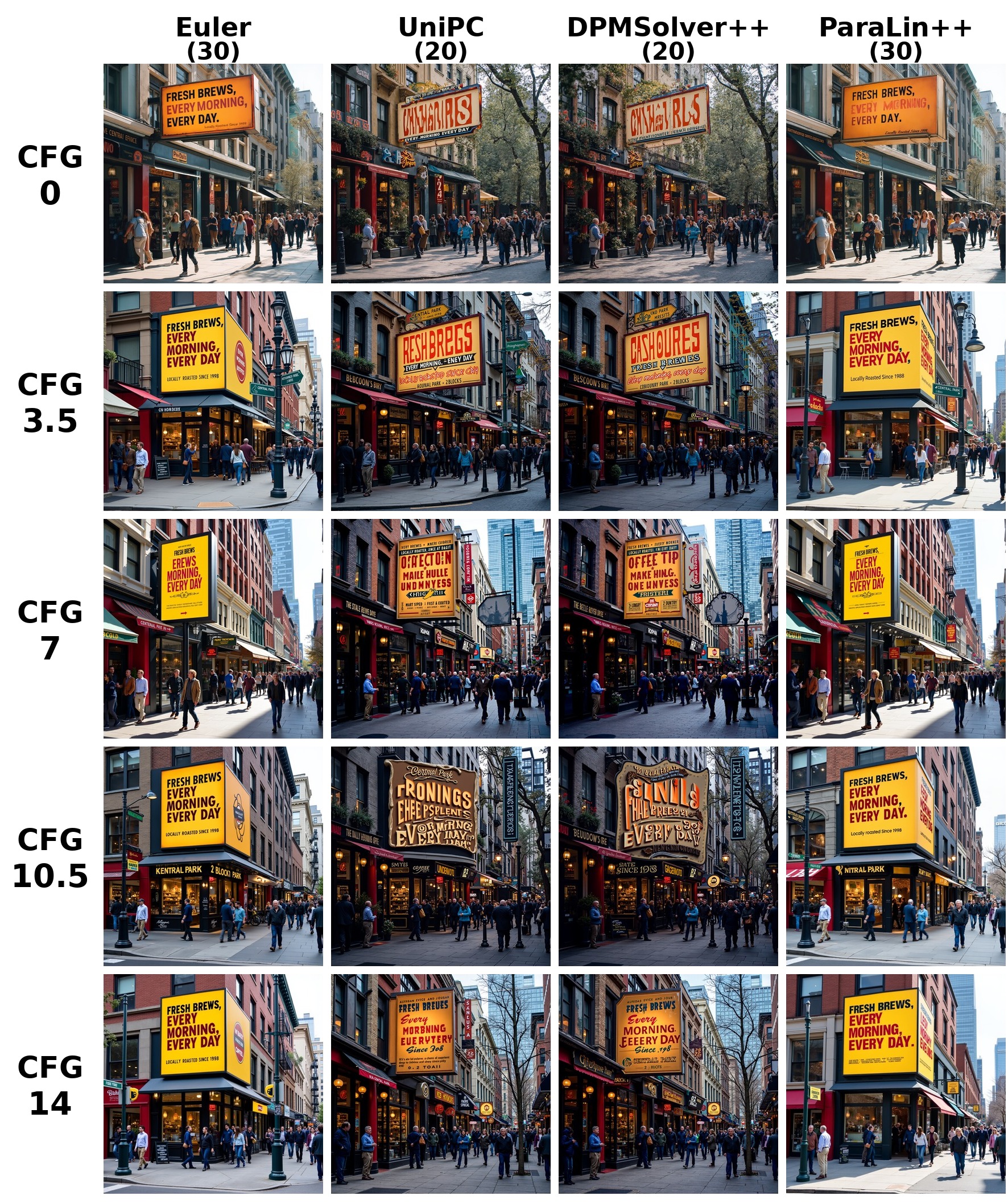}\vspace{-0.25cm}
\caption{LongTextBench CFG-sweep visual comparison on FLUX.1-dev, using the English prompt 0000 with variant 1. Rows are CFG $0$, $3.5$, $7$, $10.5$, and $14$; columns are 30-step Euler, 20-step UniPC, 20-step DPMSolver++, and 30-step ParaLin++. The reported LongTextBench metric averages the English and Chinese scores.}
\label{fig:flux_cfg_sweep_longtextbench}
\end{figure}

\begin{figure}[!t]
\centering
\includegraphics[width=0.96\columnwidth]{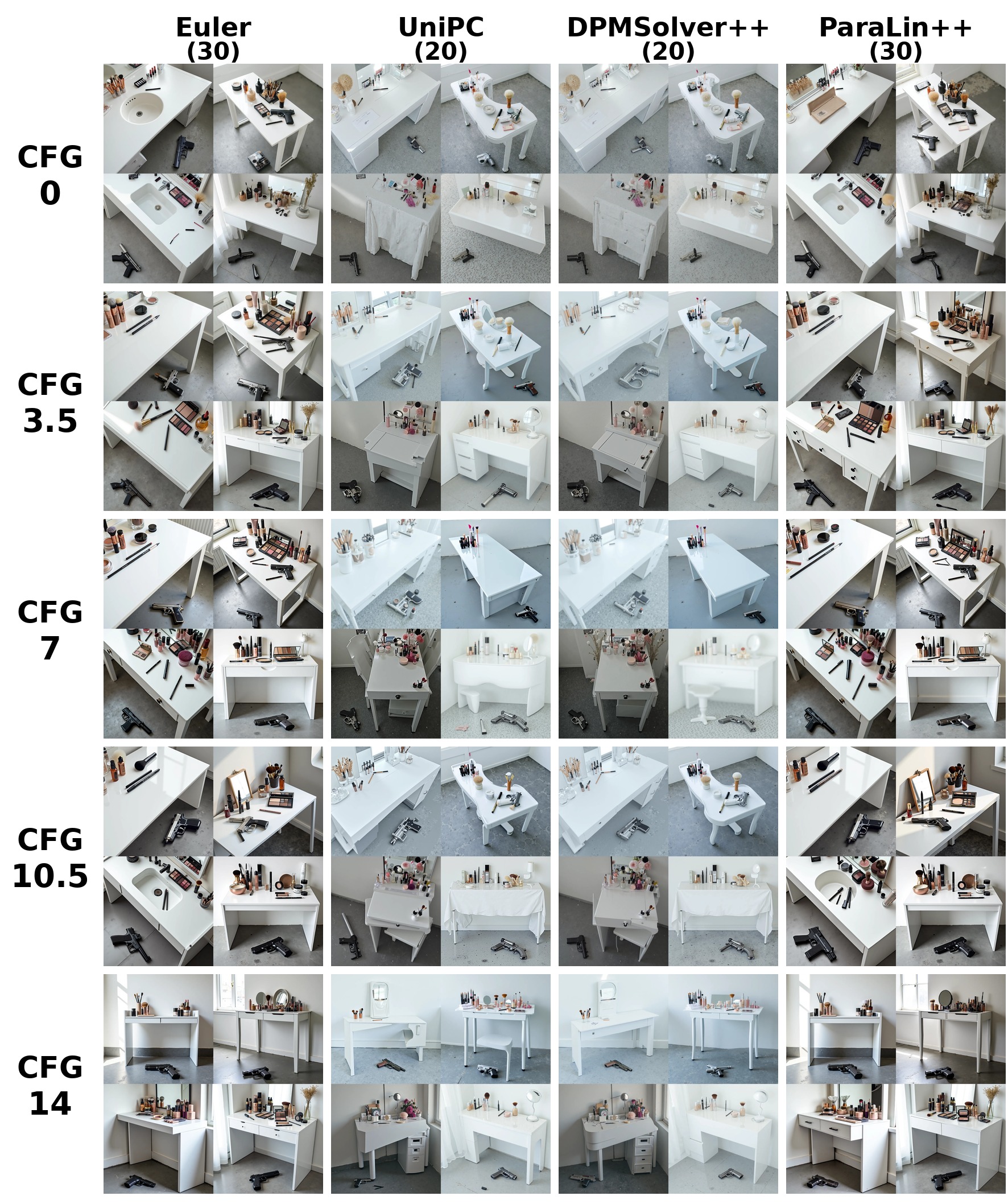}
\vspace{-0.25cm}
\caption{DPG-Bench CFG-sweep visual comparison on FLUX.1-dev. Rows are CFG $0$, $3.5$, $7$, $10.5$, and $14$; columns are 30-step Euler, 20-step UniPC, 20-step DPMSolver++, and 30-step ParaLin++. All cells use prompt 100. The complete CFG-$0$--$14$ DPG-Bench matrix is reported in Table~\ref{Table:FLUX_CFG_Sweep}.}
\label{fig:flux_cfg_sweep_dpg}
\end{figure}

\noindent\textbf{Single GPU for Acceleration}
Parallel sampling methods are often criticized for trading excessive computational cost for speed, which is impractical for users with limited resources. However, our experiments in Table \ref{Table:SDv2_CMP_Single_GPU} demonstrate the real-world viability of our approach on a single GPU. Unlike competing parallel methods like ParaDiGMS, which often fail to provide any speedup or even decelerate the process in this setting, our ParaLin++ achieves consistent and significant wall-clock acceleration across all baseline samplers. For instance, it speeds up the 100-step DDIM by $4.2\times$ and the 100-step DPMSolver by $3.7\times$. This is accomplished while maintaining virtually identical generation quality (FID and CLIP scores).  This validates ParaLin++ as a practical and potent accelerator for everyday users, effectively overcoming the typical computational barriers of parallel sampling even in resource-constrained, single-GPU environments.
\begin{table}[!t]
  \caption{Comparisons on Stable Diffusion-v2 when only using one GPU, with classifier guidance $w=7.5$. ``DPM'' denotes DPM-Solver; ParaDiGMS, ParaSolver, ParaLin, and ParaLin++ indicate the corresponding parallel wrappers of the sequential sampler.}
  \label{Table:SDv2_CMP_Single_GPU}
  \vspace{-0.25cm}
  \centering
  \scriptsize
  \setlength{\tabcolsep}{2.1pt}
  \renewcommand{\arraystretch}{1.0}
  \resizebox{0.98\columnwidth}{!}{%
  \begin{tabular}{cc|cccccc}
    \toprule[1.5pt]
    Sampler & Method & Iters$\downarrow$ & NFE$\downarrow$ & CLIP$\uparrow$ & FID$\downarrow$ & Time (s)$\downarrow$ & Speedup$\uparrow$\\
    \midrule
    \multirow{5}{*}{DDIM--50}
     & Sequential &50 &50 &25.6 &57.2 &6.3 & 1.0$\times$\\
     & ParaDiGMS &34 &67 &25.6 &56.9 &7.8 & 0.8$\times$\\
     & ParaSolver &18 &53 &25.6 &56.9 &6.0 & 1.05$\times$\\
     & ParaLin &14 &51 &25.6 &57.1 &5.7 & 1.1$\times$\\
     \rowcolor{paralinrow} & ParaLin++ &\textbf{8} &\textbf{32} &25.6 &57.1 &\textbf{3.2} & \textbf{2.0$\times$}\\
    \midrule
    \multirow{5}{*}{DDIM--100}
     & Sequential &100 &100 &25.6 &56.3 &10.1 & 1.0$\times$\\
     & ParaDiGMS &66 &131 &25.5 &56.2 &12.0 & 0.8$\times$\\
     & ParaSolver &35 &102 &25.5 &56.3 &8.5 & 1.2$\times$\\
     & ParaLin &18 &70 &25.6 &56.3 &6.1 & 1.7$\times$\\
     \rowcolor{paralinrow} & ParaLin++ &\textbf{11} &\textbf{36} &25.5 &56.4 &\textbf{2.4} & \textbf{4.2$\times$}\\
    \midrule
    \multirow{5}{*}{DPM--50}
     & Sequential &50 &50 &25.6 &57.2 &5.9 & 1.0$\times$\\
     & ParaDiGMS &27 &78 &25.6 &57.2 &6.1 & 0.9$\times$\\
     & ParaSolver &33 &65 &25.6 &57.2 &6.0 & 0.9$\times$\\
     & ParaLin &31 &56 &25.5 &56.9 &5.4 & 1.1$\times$\\
     \rowcolor{paralinrow} & ParaLin++ &\textbf{11} &\textbf{28} &25.6 &57.1 &\textbf{2.3} & \textbf{2.5$\times$}\\
    \midrule
    \multirow{5}{*}{DPM--100}
     & Sequential &100 &100 &25.5 &56.7 &9.9 & 1.0$\times$\\
     & ParaDiGMS &66 &131 &25.6 &56.6 &11.9 & 0.8$\times$\\
     & ParaSolver &64 &127 &25.5 &56.6 &10.7 & 0.9$\times$\\
     & ParaLin &64 &121 &25.6 &56.7 &10.0 & 1.0$\times$\\
     \rowcolor{paralinrow} & ParaLin++ &\textbf{14} &\textbf{40} &25.6 &56.5 &\textbf{2.7} & \textbf{3.7$\times$}\\
    \bottomrule[1.5pt]
  \end{tabular}}
  \vspace{-0.3cm}
\end{table}

\noindent\textbf{Ablation Study}. Table~\ref{Table:Ablation_Study} isolates the value of the ParaLin++ components on a 100-step DDIM task. Full nonlinearity reduces iterations but multiplies NFEs and becomes slower than the sequential baseline. Partial nonlinearity removes $93.3\%$ of these NFEs relative to the fully nonlinear variant and already restores wall-clock speedup. Adaptive stopping slightly reduces both iterations and NFEs, while SOR gives the strongest reduction, reaching 15 iterations, 117 NFEs, and a $6.5\times$ speedup. The trend confirms that the largest gain comes from ParaLin's partial-linear formulation, while ParaLin++ further improves convergence by lowering iterations with negligible extra computation.

\noindent\textbf{Generalization on VP, VE, and EDM  Process.}
We investigated the impact of the VP, VE, and EDM diffusion processes by converting the sequential sampler in Alg. 1 of work~\cite{EDM} into our parallel version. Adopting the Tab. 1 configurations specified in work~\cite{EDM}, we then analyzed the effect of the nonlinear degree $\varrho$ on FID performance across these three processes using the provided ImageNet dataset. As shown in Tab.~\ref{tab:combined_impact_merged}, our parallel sampling method demonstrates its applicability and effectiveness across all three diffusion processes. For each process, our parallel approach consistently reduces the number of sampling iterations while achieving comparable or improved FID scores compared to the sequential baseline. This highlights the robustness and efficiency of our method, making it a viable accelerator for various diffusion model formulations. We notice that our method yields a better FID score for every schedule and a substantially larger gain in the VP process. This improvement likely stems from the inherently stiff nature of the VP-ODE, combined with the step-parallel updates of our parallel solver. The step-parallel updates sever the temporal dependence chain across steps, thus effectively curtail the accumulation of truncation error across steps~\cite{EDM}. As a result, our solver acts as an implicit numerical stabilizer for the stiff VP-ODE, which typically exhibits significant truncation error that degrades the performance of conventional VP samplers.
\begin{table*}[t]
\caption{Ablation study on Stable Diffusion v2. The table isolates nonlinearity, adaptive stopping, and SOR in the 100-step DDIM task; speedup is measured relative to the sequential DDIM baseline.}
\label{Table:Ablation_Study}
\centering
\setlength{\tabcolsep}{5pt}
\renewcommand{\arraystretch}{1.18}
\resizebox{0.98\textwidth}{!}{%
\begin{tabular}{lccccccc}
\toprule[1.5pt]
\multirow{2}{*}{Configuration} &
\multicolumn{3}{c}{Solver components} &
\multicolumn{4}{c}{100-step DDIM results}\\
\cmidrule(lr){2-4}\cmidrule(lr){5-8}
 & Nonlinearity & Stopping criterion & SOR &
Iters$\downarrow$ & NFE$\downarrow$ & Time (s)$\downarrow$ & Speedup$\uparrow$\\
\midrule
Sequential DDIM & -- & -- & -- &
100 & 100 & 10.6 & $1.0\times$\\
\midrule
Full nonlinearity &
Full ($\varrho=100$) & Fixed ($\beta=1.0$) & No ($\omega=1.0$) &
72 & 2100 & 24.4 & $0.43\times$\\
Partial nonlinearity &
Partial ($\varrho=8$) & Fixed ($\beta=1.0$) & No ($\omega=1.0$) &
18 & 140 & 2.1 & $5.0\times$\\
Partial + adaptive stopping &
Partial ($\varrho=8$) & Adaptive ($\beta=0.999$) & No ($\omega=1.0$) &
17 & 139 & 2.0 & $5.3\times$\\
\rowcolor{paralinrow} ParaLin++ (partial + adaptive + SOR) &
Partial ($\varrho=8$) & Adaptive ($\beta=0.999$) & Yes ($\omega=1.04$) &
\textbf{15} & \textbf{117} & \textbf{1.6} & $\mathbf{6.5\times}$\\
\bottomrule[1.5pt]
\end{tabular}%
}
\end{table*}
\begin{table}[h!]
\centering
\caption{The impact of nonlinear degree $\varrho$ on FID across EDM, VE, and VP processes over  ImageNet dataset. We set $\delta=0.01, \beta=9e-4$ for the EDM  process, $\delta=0.01, \beta=9e-7$ for VP processes, and $\delta=0.1, \beta=9e-5$ for the VE process.}
\label{tab:combined_impact_merged}
\begin{tabular}{lcccccc}
\toprule
\multirow{2}{*}{\textbf{Method}} & \multicolumn{2}{c}{\textbf{EDM Process}} & \multicolumn{2}{c}{\textbf{VE Process}} & \multicolumn{2}{c}{\textbf{VP Process}} \\
\cmidrule(lr){2-3} \cmidrule(lr){4-5} \cmidrule(lr){6-7}
& \textbf{FID} & \textbf{Iters} & \textbf{FID} & \textbf{Iters} & \textbf{FID} & \textbf{Iters} \\
\midrule
Baseline & 11.1649 & 25 & 423.0628 & 25 & 104.14673   & 25 \\
$\varrho=4$  & 10.8811 & 22 & 422.9794 & 14 & 35.8963  & 12 \\
$\varrho=6$  & 10.8718 & 21 & 422.9293 & 14 & 41.1196  & 10 \\
$\varrho=8$  & 10.8621 & 20 & 422.8247 & 13 & 58.5588  & 8  \\
$\varrho=10$ & 10.8626 & 15 & 422.7672 & 10 & 58.5215  & 7  \\
$\varrho=12$ & 10.8618 & 12 & 422.7651 & 9  & 58.8340  & 7  \\
\bottomrule
\end{tabular}
\end{table}

\noindent\textbf{Nonlinear Degree $\varrho$ Analysis}: Table~\ref{tab:rho_ablation} conducted an ablation study by accelerating a 50-step DPMSolver on the Stable Diffusion v2 model using 8 GPUs.
As shown in the table, the FID and CLIP scores remain remarkably stable across a wide range of $\varrho$ values. This stability confirms our theoretical finding in Prop.~\ref{proposition:Unique_Solution} that $\varrho$ does not compromise the fidelity of the generated samples. Furthermore, the results reveal a clear trade-off between $\varrho$ and efficiency. While increasing $\varrho$ initially yields significant speedups, the gains diminish and eventually reverse. This is because a larger $\varrho$ requires more NFEs, as seen in the table. The increased NFE count results in a greater computational overhead for each GPU (proportional to NFEs / GPUs), which leads to a rapid increase in runtime after an optimal point (around $\varrho = 15$).
\begin{table}[t]
    \centering
    \caption{Impact of varying $\varrho$ on generation quality (FID, CLIP) and performance (NFEs, Speedup).}
    \label{tab:rho_ablation}
    \begin{tabular}{@{}lccccc@{}}
        \toprule
        \textbf{Method} & \textbf{FID}  & \textbf{CLIP}  & \textbf{Iters} & \textbf{NFEs} & \textbf{Speedup} \\
        \midrule
        Baseline & 57.23 & 25.64 & 50 & 50 & 1$\times$ \\
        \midrule
        $\varrho = 5$   & 57.18 & 25.55 & 18 & 82 & 2.9$\times$ \\
        $\varrho = 10$  & 57.25 & 25.59 & 12 & 101 & 3.0$\times$ \\
        $\varrho = 15$  & 57.14 & 25.61 & 11 & 123 & 3.1$\times$ \\
        $\varrho = 20$  & 57.22 & 25.60 & 12 & 160 & 2.6$\times$ \\
        $\varrho = 25$  & 57.14 & 25.53 & 12 & 202 & 2.1$\times$ \\
        $\varrho = 30$  & 57.06 & 25.58 & 12 & 223 & 2.1$\times$ \\
        $\varrho = 35$  & 57.18 & 25.62 & 12 & 245 & 1.8$\times$ \\
        $\varrho = 40$  & 57.15 & 25.61 & 12 & 264 & 1.8$\times$ \\
        $\varrho = 45$  & 57.16 & 25.59 & 12 & 275 & 1.6$\times$ \\
        $\varrho = 50$  & 57.14 & 25.57 & 12 & 284 & 1.6$\times$ \\
        \bottomrule
    \end{tabular}
\end{table}

\noindent\textbf{EMA Decay Rate $\beta$ and Tolerance $\delta$ Analysis.}
As shown in Table~\ref{Table:Effect_Of_Tor_on_FID}, the EMA decay rate ($\beta$) is critical for sample quality. A low value ($\beta=0.9$) causes noticeable FID degradation, whereas higher values ($\beta \geq 0.99$) consistently achieve FIDs comparable to or better than the baseline. Conversely, for the practical high-decay choices ($\beta\geq0.99$), the tolerance ($\delta$) has only a small impact on the final FID score. This reveals a key practical trade-off: $\delta$ can be relaxed (set higher) to significantly reduce the number of iterations and thus accelerate computation, for example, from 21 to 11 iterations when $\beta=1.0$, without a meaningful sacrifice in generation quality.
\begin{table}[h!]
\centering
  \caption{The effect of the tolerance $\delta$ and EMA decay rate $\beta$ on FID using Stable Diffusion-v2 over $5000$ random samples. Our ParaLin++ ($\varrho=8,p=40$) is employed for  DDIM.}
  \label{Table:Effect_Of_Tor_on_FID}
  \centering
\begin{tabular}{lcccccc}
\toprule
\multirow{2}{*}{\textbf{Method}} & \multicolumn{2}{c}{$\delta=0.01$} & \multicolumn{2}{c}{$\delta=0.05$} & \multicolumn{2}{c}{$\delta=0.1$} \\
\cmidrule(lr){2-3} \cmidrule(lr){4-5} \cmidrule(lr){6-7}
& \textbf{FID} & \textbf{Iters} & \textbf{FID} & \textbf{Iters} & \textbf{FID} & \textbf{Iters} \\
\midrule
DDIM & 57.2 & 50 & 57.2 & 50 & 57.2 & 50 \\
$\beta=0.9$  & 59.7 & 9 & 58.6 & 9 & 57.8  & 9 \\
$\beta=0.99$  & 57.8 & 13 & 57.1 & 13 & 56.9  & 11 \\
$\beta=0.999$  & 57.0 & 16 & 56.9 & 13 & 56.8  & 11  \\
$\beta=1.0$ & 57.0 & 21 & 57.1 & 13 & 57.2  &11  \\
\bottomrule
\end{tabular}
\end{table}

\noindent\textbf{Relaxation Factor $\omega$  Analysis.}
Table~\ref{tab:omega_ablation} demonstrates the trade-off between acceleration and stability controlled by the relaxation factor $\omega$. Increasing $\omega$ from 1.0 to 1.08 effectively reduces the number of iterations and, consequently, boosts the speedup from $3.81\times$ to a peak of $4.94\times$. However, this gain is bounded by stability. Sample quality, as measured by FID and CLIP, remains consistent with the baseline up to $\omega = 1.06$. At $\omega = 1.08$, a slight degradation in quality appears (FID 58.13), and by $\omega = 1.12$, the method becomes completely unstable, leading to a catastrophic collapse in generation quality (FID 289.54). This suggests an optimal value around $\omega = 1.06$, which achieves a substantial $4.76\times$ speedup while fully preserving the original sample fidelity.
\begin{table}[t]
    \centering
    \caption{Impact of varying $\omega$ on generation quality and performance. We set $\varrho = 8,p=40,\beta = 0.999$.}
    \label{tab:omega_ablation}
    \begin{tabular}{@{}lccccc@{}}
        \toprule
        \textbf{Method} & \textbf{FID}  & \textbf{CLIP}  & \textbf{Iters} & \textbf{NFEs} & \textbf{Speedup} \\
        \midrule
        DDIM & 57.23 & 25.64 & 50 & 50 & 1$\times$ \\
        \midrule
        $\omega = 1$   & 57.20 & 25.59 & 11 & 83 & 3.81$\times$ \\
        $\omega = 1.02$  & 57.38 & 25.57 & 11 & 79 & 3.89$\times$ \\
        $\omega = 1.04$  & 57.38 & 25.57 & 10 & 70 & 4.34$\times$ \\
        $\omega = 1.06$  & 57.24 & 25.56 & 9 & 67 & 4.76$\times$ \\
        $\omega = 1.08$  & 58.13 & 25.12 & 9 & 67 & 4.94$\times$ \\
        $\omega = 1.12$  & 289.54 & 15.03 & 8 & 64 & 4.3$\times$ \\
        \bottomrule
    \end{tabular}
\end{table}
\noindent Figure~\ref{fig:param_comparison} provides qualitative samples for the corresponding $\varrho$, $\beta$, and $\omega$ ablations.

\begin{figure}[!htbp]
    \centering

    \begin{minipage}{\linewidth}
        \centering
        \begin{minipage}{0.19\linewidth}
            \centering
            \includegraphics[width=\linewidth]{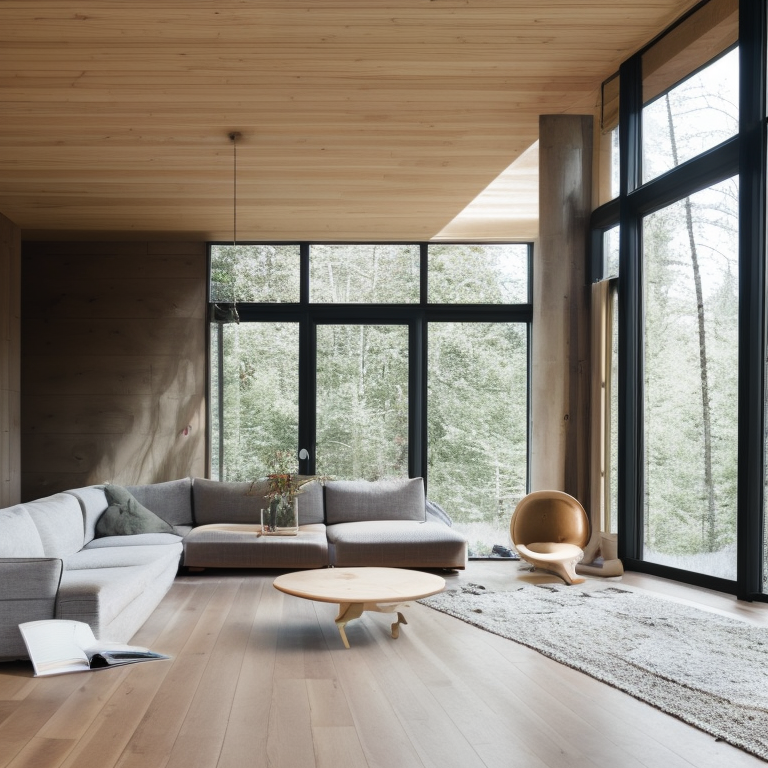}
            $\varrho=5$
        \end{minipage}
        \begin{minipage}{0.19\linewidth}
            \centering
            \includegraphics[width=\linewidth]{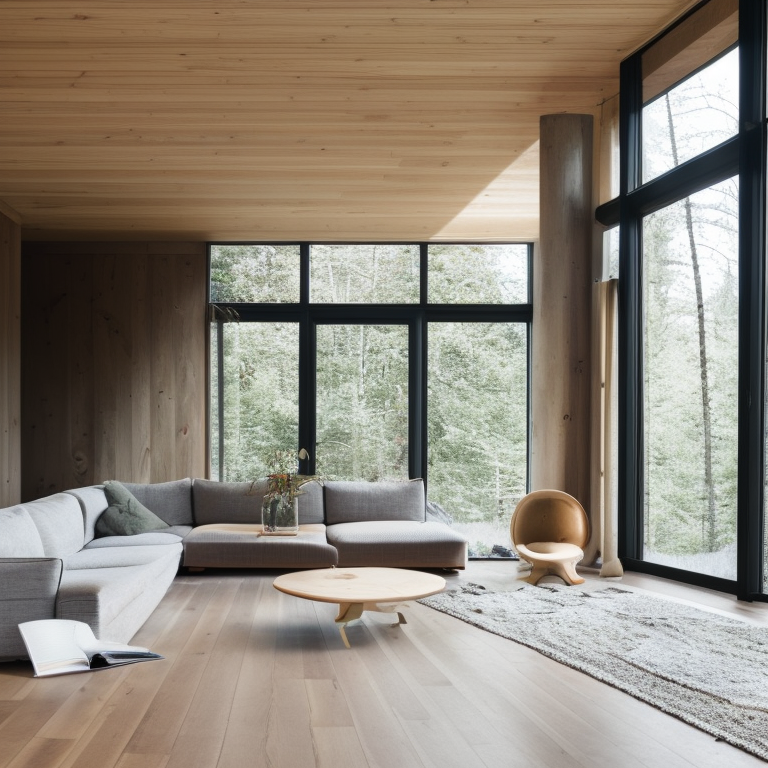}
            $\varrho=10$
        \end{minipage}
        \begin{minipage}{0.19\linewidth}
            \centering
            \includegraphics[width=\linewidth]{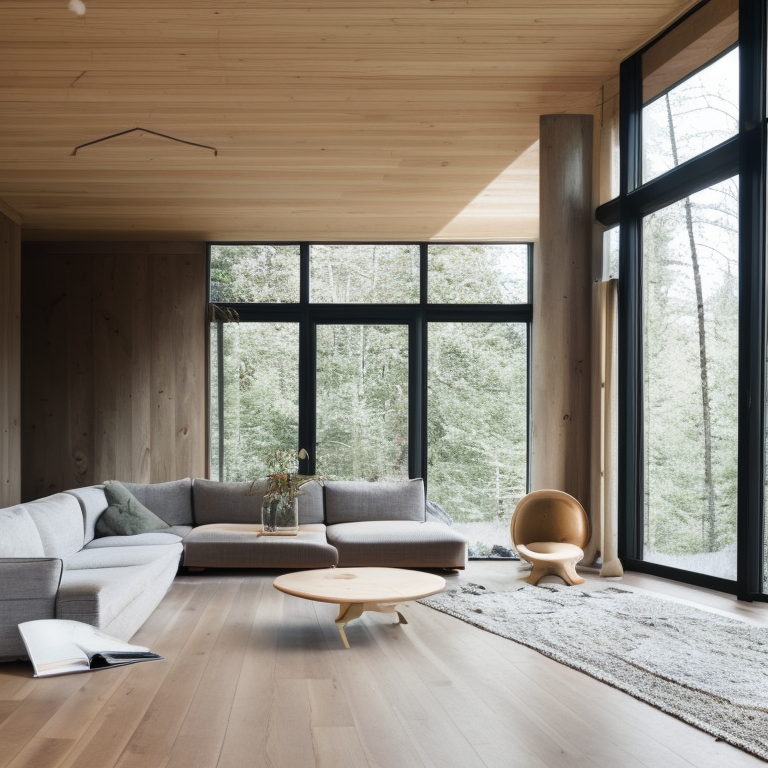}
             $\varrho=15$
        \end{minipage}
        \begin{minipage}{0.19\linewidth}
            \centering
            \includegraphics[width=\linewidth]{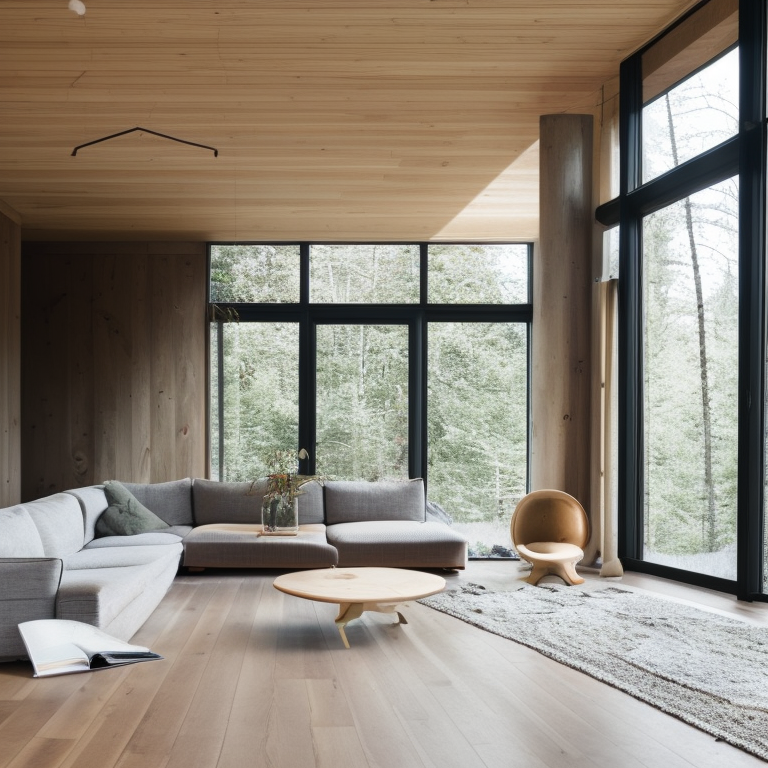}
            $\varrho=20$
        \end{minipage}
        \begin{minipage}{0.19\linewidth}
            \centering
            \includegraphics[width=\linewidth]{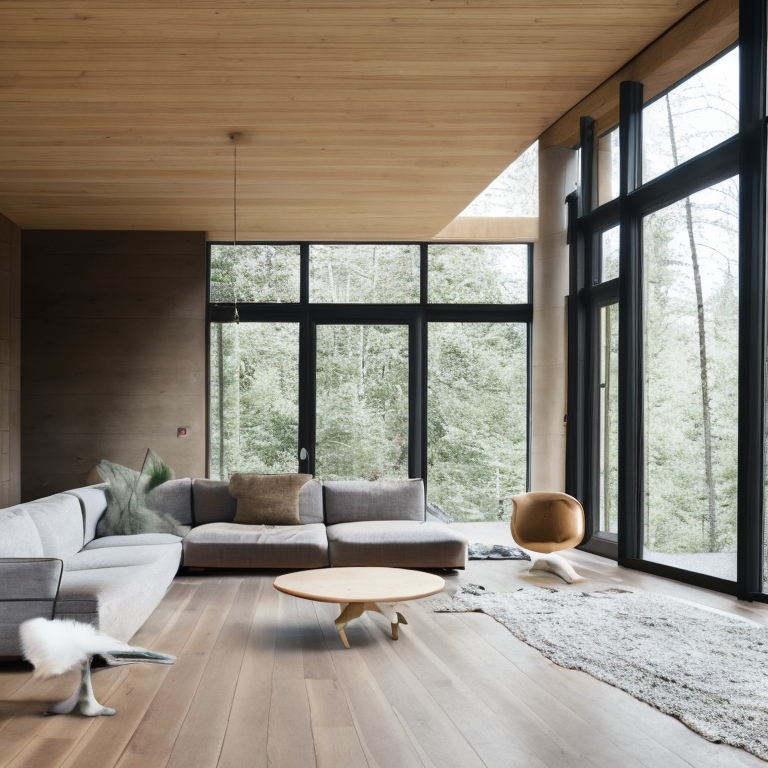}
            DDIM
        \end{minipage}
    \end{minipage}
    \vspace{0.1cm}

    \begin{minipage}{\linewidth}
        \centering
        \begin{minipage}{0.19\linewidth}
            \centering
            \includegraphics[width=\linewidth]{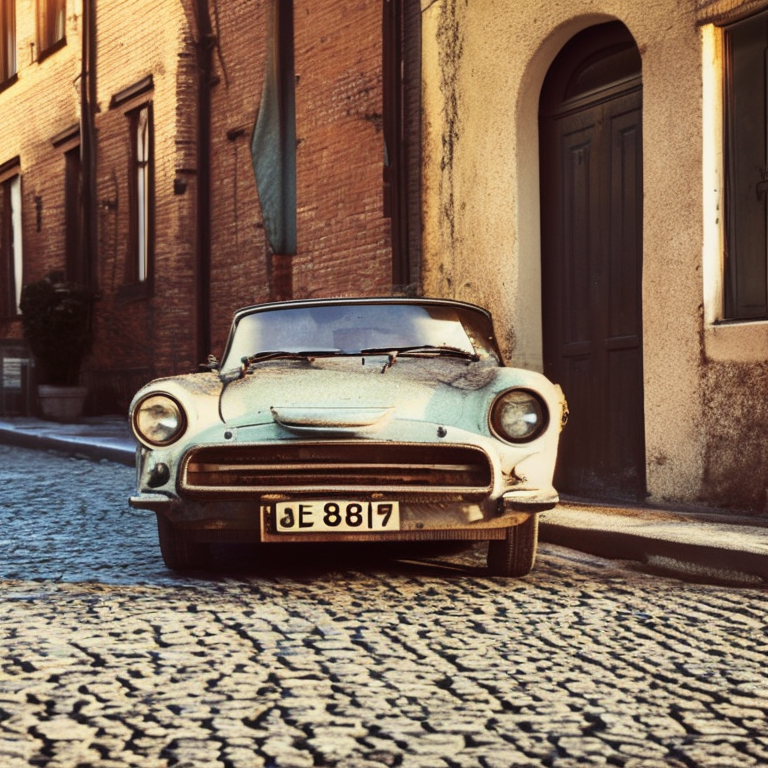}
            $\beta=0.9$
        \end{minipage}
        \begin{minipage}{0.19\linewidth}
            \centering
            \includegraphics[width=\linewidth]{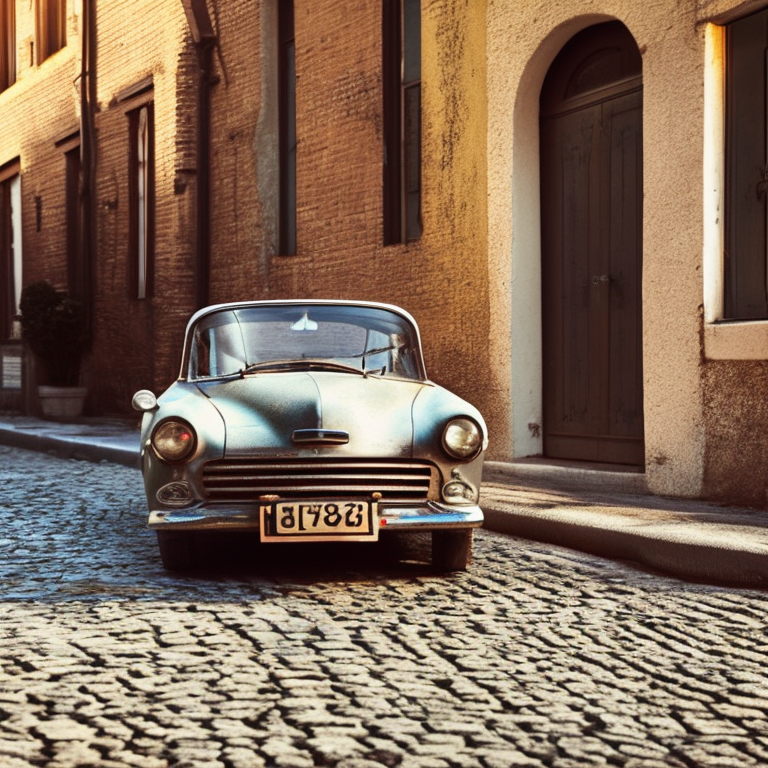}
            $\beta=0.99$
        \end{minipage}
        \begin{minipage}{0.19\linewidth}
            \centering
            \includegraphics[width=\linewidth]{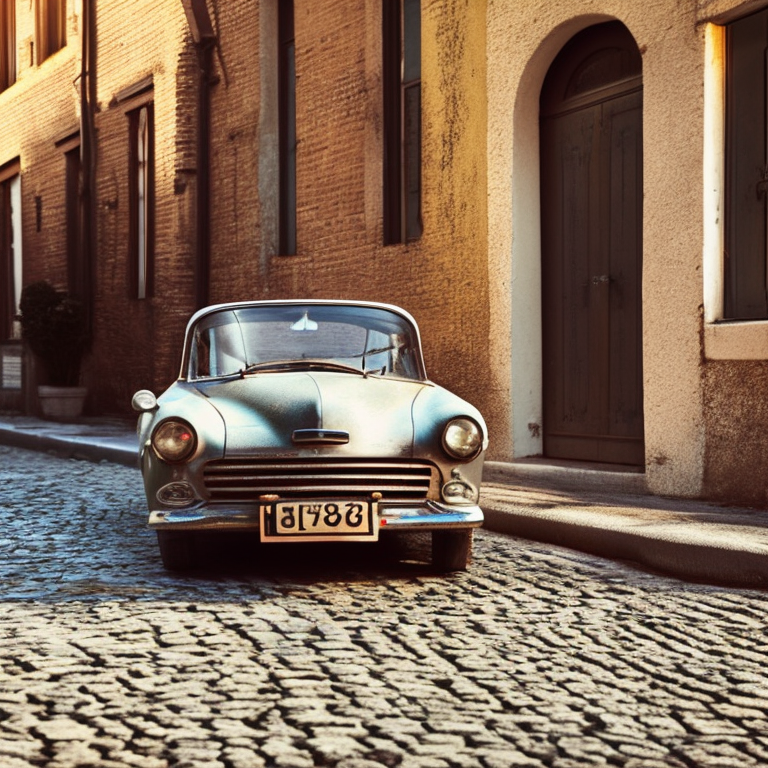}
            $\beta=0.999$
        \end{minipage}
        \begin{minipage}{0.19\linewidth}
            \centering
            \includegraphics[width=\linewidth]{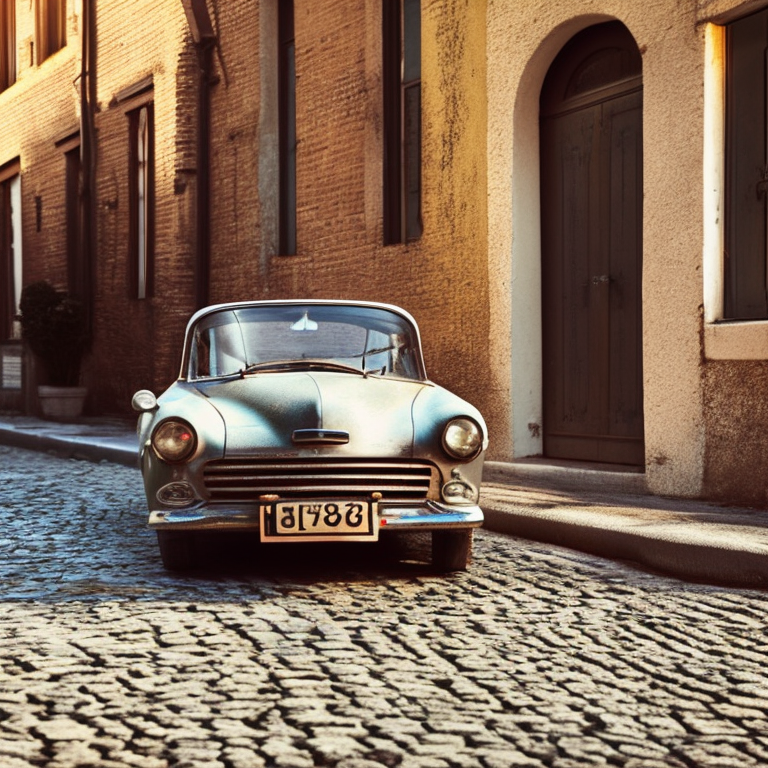}
            $\beta=1$
        \end{minipage}
        \begin{minipage}{0.19\linewidth}
            \centering
            \includegraphics[width=\linewidth]{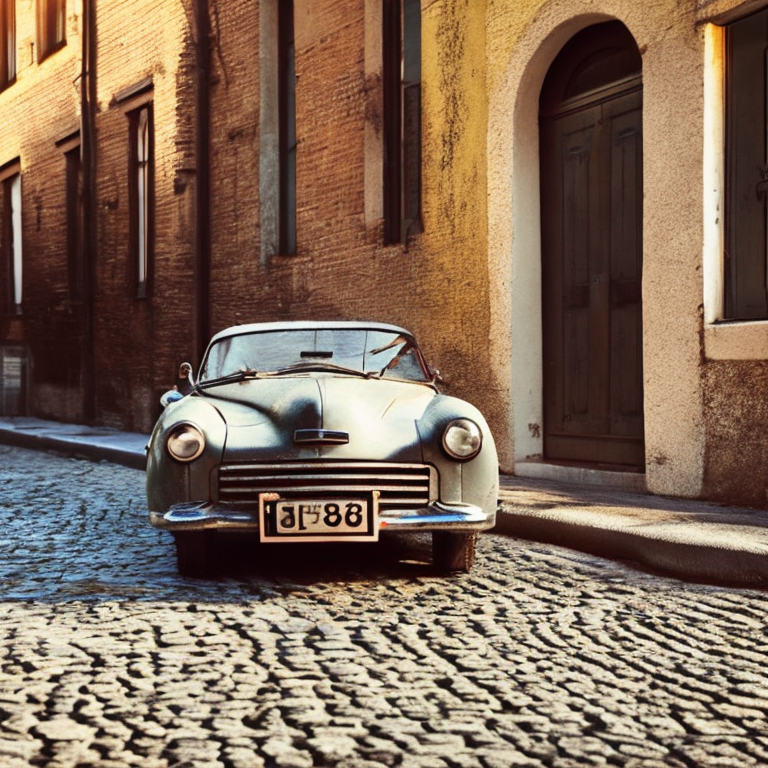}
            DDIM
        \end{minipage}
    \end{minipage}
    \vspace{0.1cm}

    \begin{minipage}{\linewidth}
        \centering
        \begin{minipage}{0.19\linewidth}
            \centering
            \includegraphics[width=\linewidth]{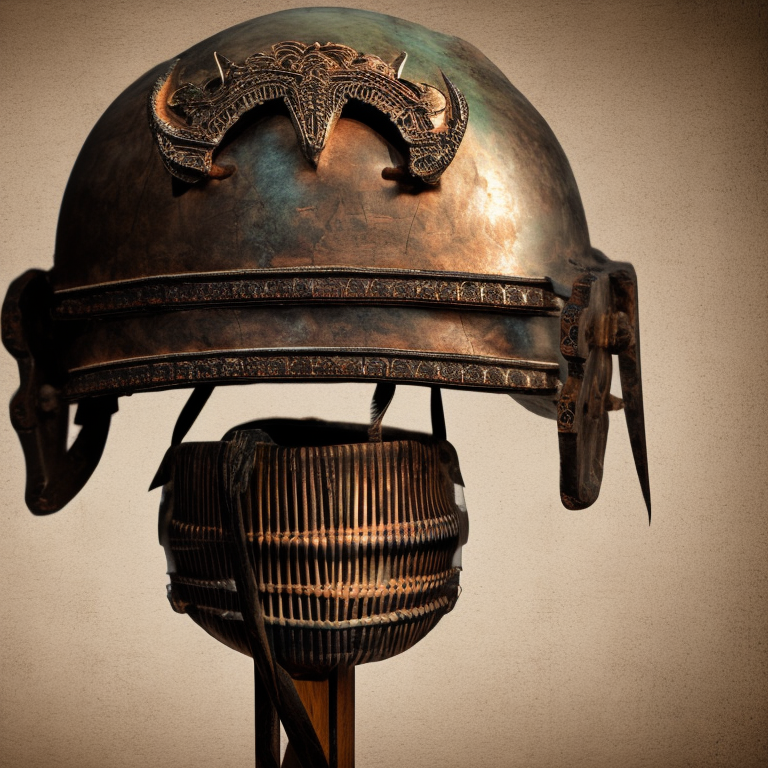}
            $\omega=1$
        \end{minipage}
        \begin{minipage}{0.19\linewidth}
            \centering
            \includegraphics[width=\linewidth]{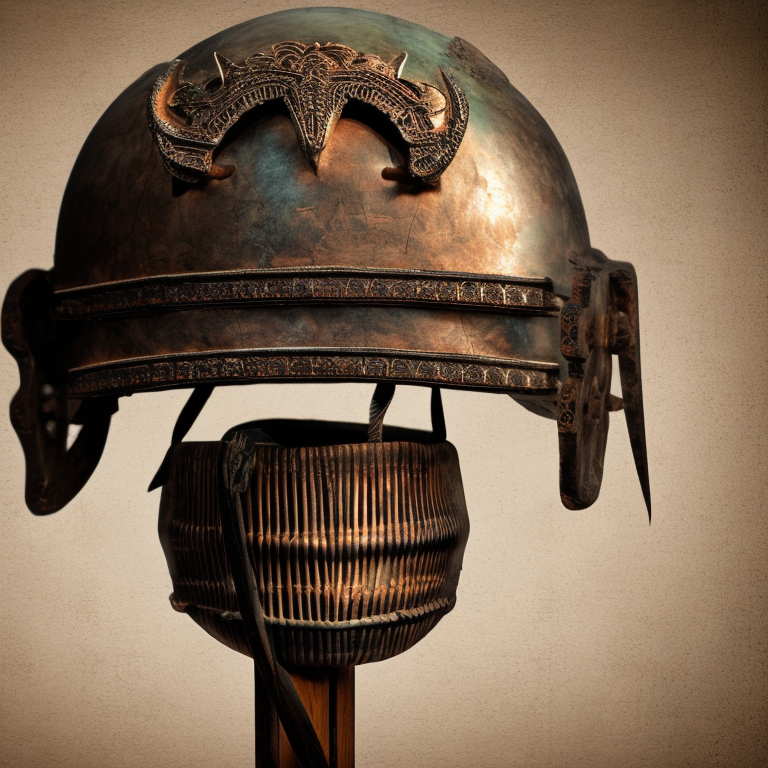}
            $\omega=1.02$
        \end{minipage}
        \begin{minipage}{0.19\linewidth}
            \centering
            \includegraphics[width=\linewidth]{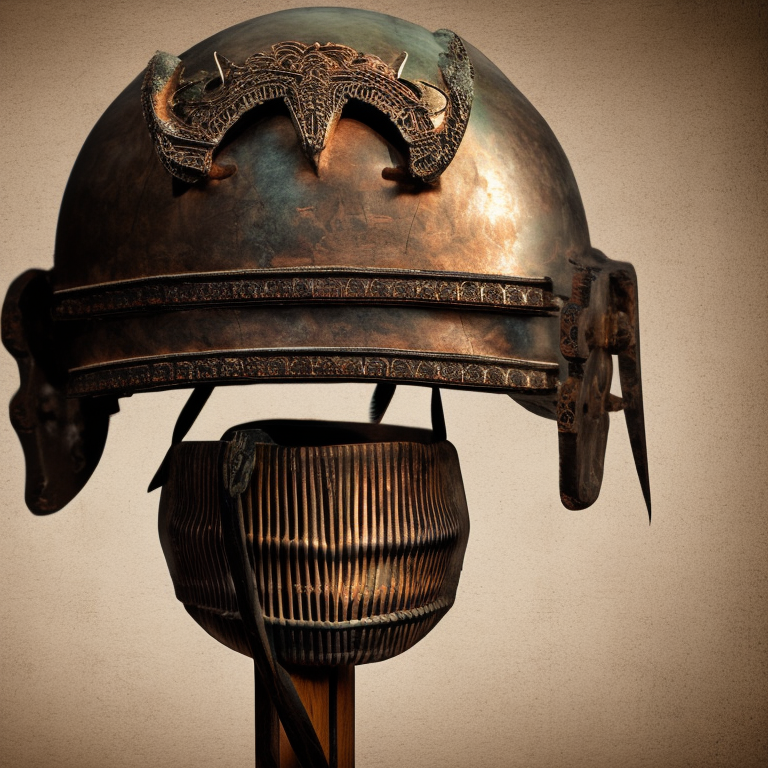}
            $\omega=1.04$
        \end{minipage}
        \begin{minipage}{0.19\linewidth}
            \centering
            \includegraphics[width=\linewidth]{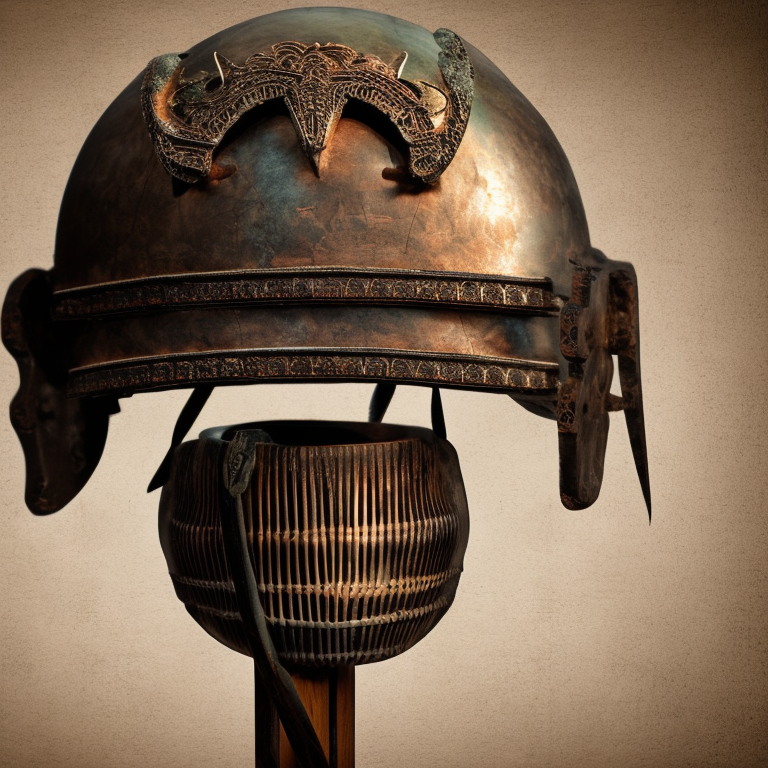}
            $\omega=1.06$
        \end{minipage}
        \begin{minipage}{0.19\linewidth}
            \centering
            \includegraphics[width=\linewidth]{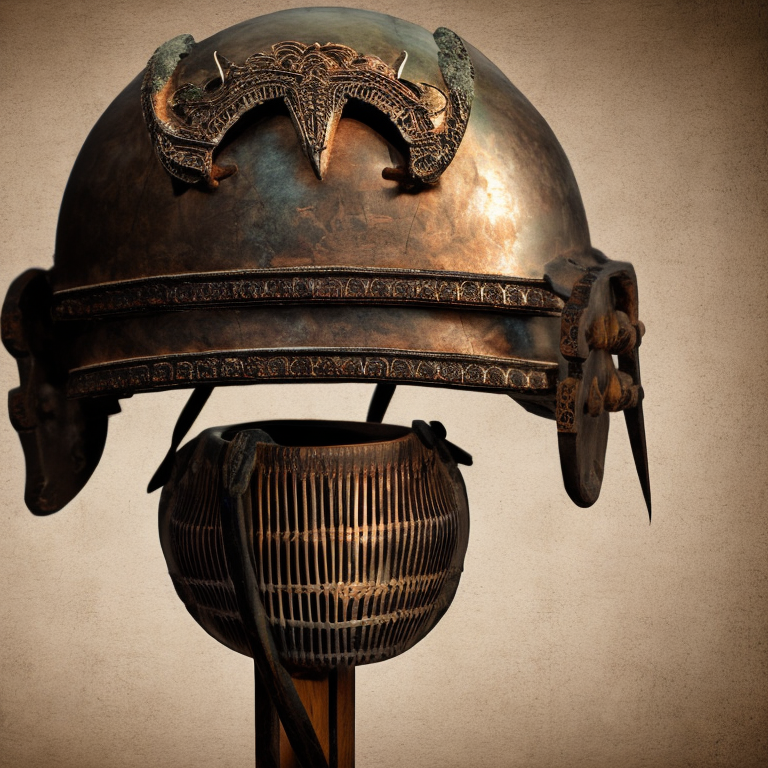}
            DDIM
        \end{minipage}
    \end{minipage}

     \caption{\label{fig:param_comparison}Ablation study on parameters.
     The first row  shows the results for different values of $\varrho$.
     The second row shows the results for different values of $\beta$.
     The third row  shows the results for different values of $\omega$.
     The last image in each row represents the baseline comparison.}
\end{figure}
\FloatBarrier
\section{Conclusion}\label{sec:conclusion}
We systematized training-free diffusion acceleration by the metrics that certify its quality-preservation claims, exposing heavy reliance on FID/CLIP-style evidence. Our generation-quality audit shows that representative few-step solvers can remain FID-close yet lose substantial semantic-composition, instruction-following, and long-text-rendering quality. To accelerate without shortening the trajectory, ParaLin++ strengthens ParaLin's partial-linear solver with adaptive nonlinearization, successive over-relaxation, sliding windows, and adaptive stopping, supported by stability and uniqueness guarantees. It preserves the full-step quality profile while reaching $1.68\times$ on FLUX.1-dev, $1.21\times$ on Ming-Flash-Omni 2.0, and $2.6\times$--$4.8\times$ on Stable Diffusion-v2/LSUN Church. Acceleration claims should therefore report generation-quality metrics alongside FID/CLIP and separate discrete steps from solver iterations.

\bibliographystyle{IEEEtranS}
\bibliography{paralin_plus}

\begin{IEEEbiography}[{\includegraphics[width=1in,height=1.25in,clip,keepaspectratio]{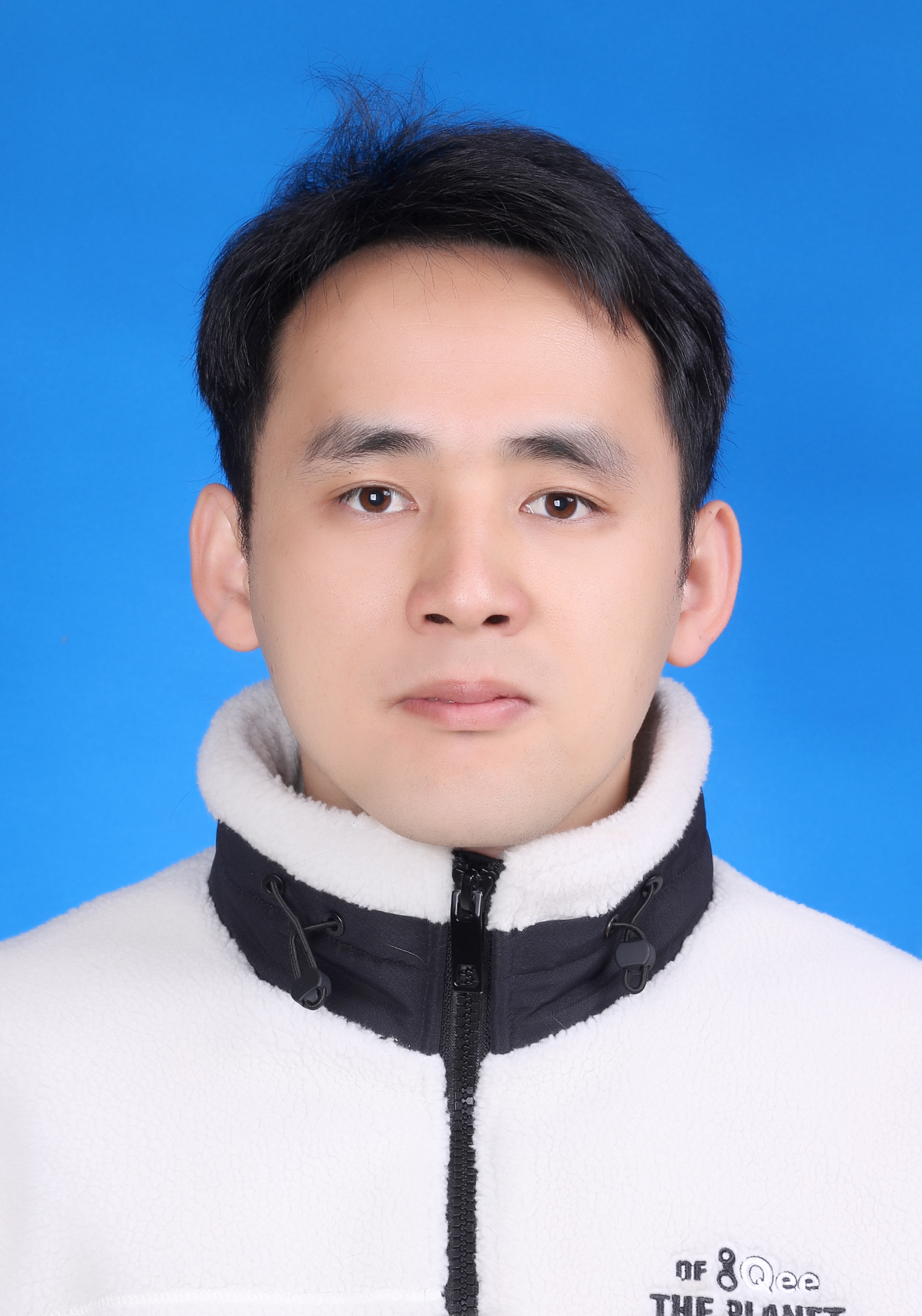}}]{Jianrong Lu} received the M.S degree in cyberspace security from Huazhong University of Science and Technology, Wuhan, China, in 2023. He is currently a Research Assistant at City University of Hong Kong. His research interests include diffusion model and stochastic optimization.
\end{IEEEbiography}

\vspace{-1.6cm}
\begin{IEEEbiography}[{\includegraphics[width=1in,height=1.25in,clip,keepaspectratio]{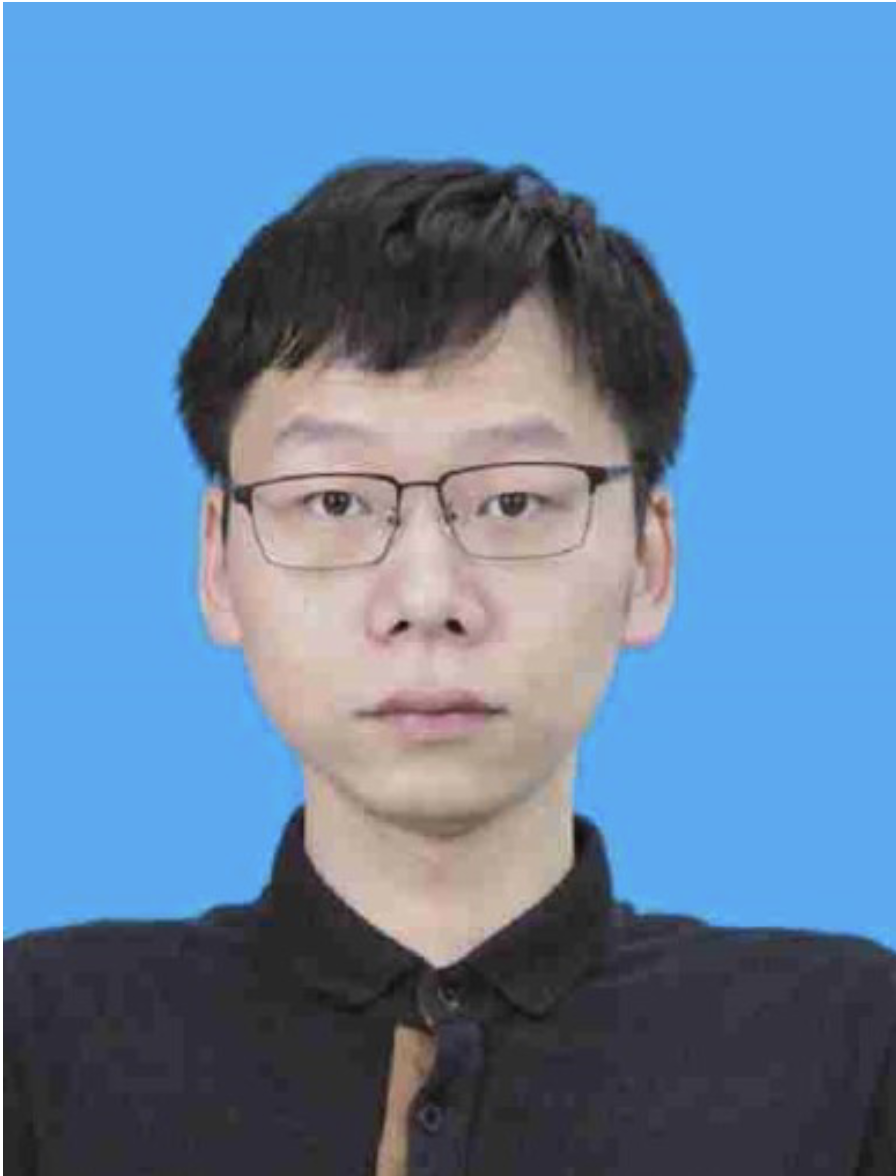}}]{Zhiyu Zhu}  is an Assistant Professor with City University of Hong Kong (Dongguan). He received the B.E. and M.E. degrees in Mechatronic Engineering, both from Harbin Institute of Technology, in 2017 and 2019, respectively, and a Ph.D. degree in Computer Science from City University of Hong Kong in 2023. 
His research interests include generative models and computer vision.
\end{IEEEbiography}

\vspace{-1.6cm}
\begin{IEEEbiography}[{\includegraphics[width=1in,height=1.25in,clip,keepaspectratio]{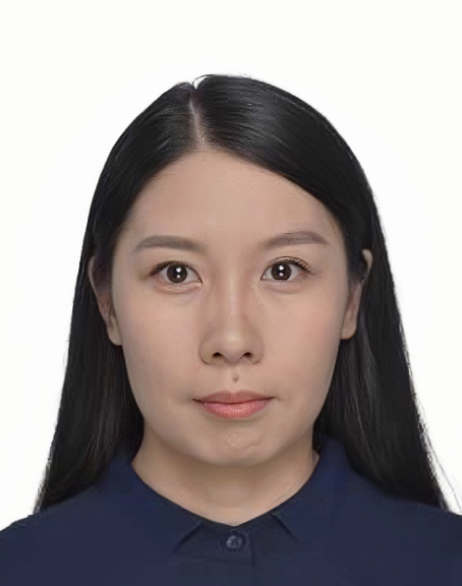}}]{Hui Liu} is currently an Assistant Professor with the School of Computing and Information Sciences at Saint Francis University, Hong Kong SAR, China. She received the B.Sc. degree in Communication Engineering from Central South University, Changsha, China, the M.Eng. degree in Computer Science from Nanyang Technological University, Singapore, and the Ph.D. degree from the Department of Computer Science, City University of Hong Kong, Hong Kong. Her research interests include image processing and machine learning.
\end{IEEEbiography}

\vspace{-1.6cm}
\begin{IEEEbiography}[{\includegraphics[width=1in,height=1.25in,clip,keepaspectratio]{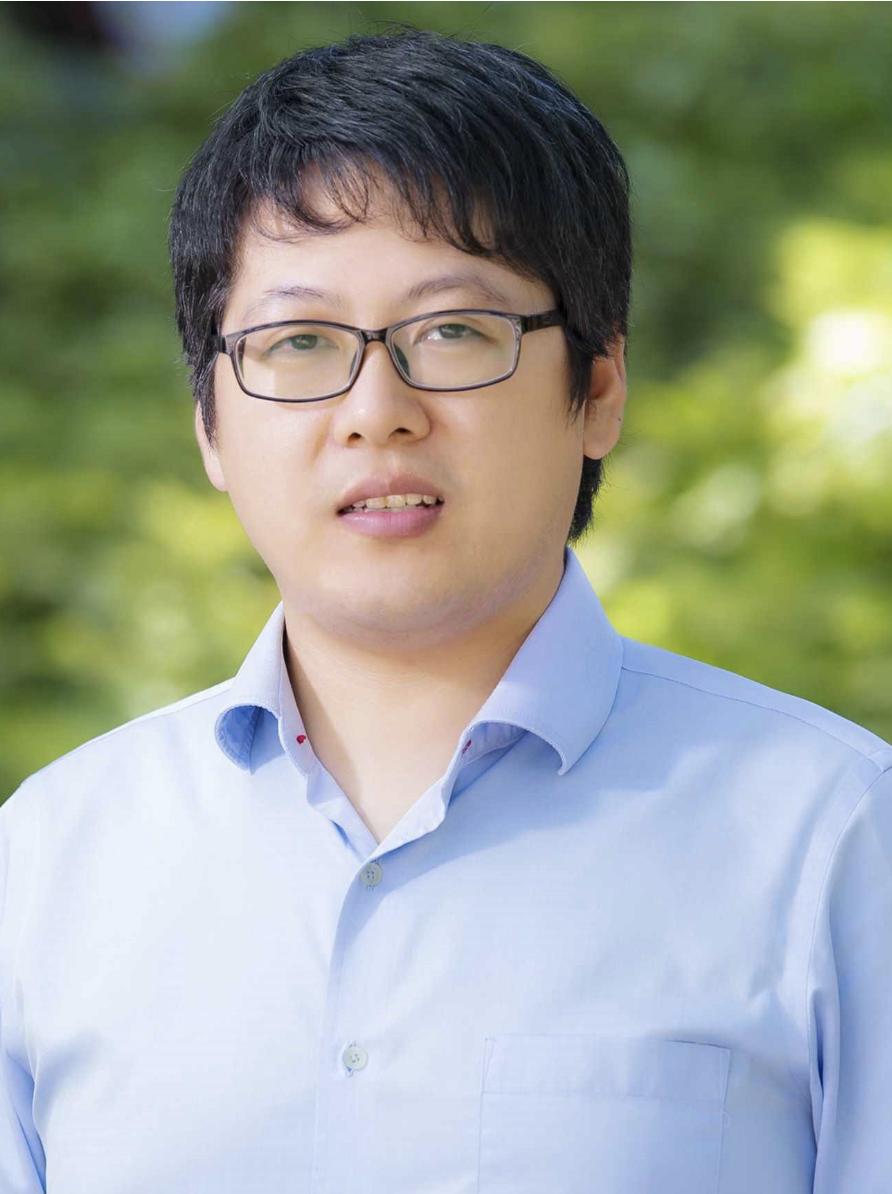}}]{Junhui Hou}
(Senior Member, IEEE)  is a Professor with the Department of Computer Science, City University of Hong Kong (CityUHK). His research interests include multidimensional visual computing and machine learning. 

He received the Early Career Award and the Research Fellow Award from the Hong Kong Research Grants Council, the Excellent Young Scientists Fund from the National Natural Science Foundation of China, the IEEE TIP Best Paper Award, and the CityUHK Presidential Research Excellence Award for Junior Faculty. He is serving as a Senior Area Editor for IEEE TIP and an Associate Editor for IEEE TVCG and TMM. He served as an Associate Editor for IEEE TIP and TCSVT.
\end{IEEEbiography}

\end{document}


\title{ParaLin++: Accelerating Parallel Diffusion Integrator via
Intrinsic Partially Linear Structure\\ (\textit{Supplementary Material})}



\markboth{Manuscript under Review}%
{ParaLin++: Accelerating Parallel Diffusion Integrator via Intrinsic Partially Linear Structure}

\maketitle



\section{Experimental Details}
\label{app_sec:setting}
This section records the experimental protocols used in the main manuscript. The purpose is to make the efficiency measurements and the generation-quality audit reproducible, while also clarifying the distinction between shortening a diffusion trajectory and parallelizing the solver iterations for a fixed trajectory.

\subsection{Evaluation Metrics}
\textbf{Efficiency.}
We report discrete diffusion steps, solver iterations (``Iters''), neural function evaluations (NFEs), wall-clock time, and speedup. For sequential samplers, one solver iteration corresponds to one discrete step; thus ``Steps'' and ``Iters'' coincide. For ParaLin++, the reference discretization is retained, and ``Iters'' denotes the average number of measured parallel solver iterations used to solve the coupled integration system. In the conventional timing experiments, wall-clock time is measured by \textit{torch.cuda.Event} from PyTorch~\cite{paszke2019pytorch}. In the large-scale generation-quality experiments, speedup is computed from the sampler/model forward-path time of the corresponding GenEval generation pass and excludes model loading and benchmark scoring.

\textbf{Quality.}
The conventional experiments report FID~\cite{fid} and CLIP score~\cite{hessel2021clipscore}. FID measures the distance between the generated distribution and a reference distribution, while CLIP score measures coarse text--image embedding similarity. These are useful distributional and alignment metrics, but they do not directly measure whether a generated sample satisfies semantic composition, dense instructions, or in-image text requirements. The generation-quality audit therefore additionally uses GenEval~\cite{ghosh2024geneval}, DPG-Bench~\cite{hu2024ella}, and LongTextBench~\cite{xomni2025}, which respectively assess semantic composition, dense instruction following, and long in-image text rendering. We treat these three benchmarks as a fixed, reproducible quality profile rather than as an absolute definition of generation quality.

\subsection{Datasets, Models, and Compute}
The conventional evaluation considers the latent-space model Stable Diffusion-v2~\cite{SD} and the pixel-space LSUN Church model~\cite{LSUNAndBedroom}. Stable Diffusion-v2 generates images at a resolution of $768\times768$ with a latent tensor size of $4\times96\times96$ and is evaluated with classifier guidance $w=7.5$ on 5{,}000 random samples using ImageNet-1k~\cite{imagenet15russakovsky}. The LSUN Church model is the pre-trained pixel-space model available on Hugging Face~\cite{DDPM,diffusers}, operates at $256\times256$, and is likewise evaluated on 5{,}000 samples. FID and CLIP are computed with the standard Hugging Face evaluation API. The CIFAR-10 less-step study uses 50{,}000 samples. The VP, VE, and EDM process study follows the sampler and schedule configurations specified in the EDM design space~\cite{EDM}.

The generation-quality study centers on the 12B FLUX.1-dev flow-matching model~\cite{bfl2024flux} and uses Ming-Flash-Omni 2.0~\cite{ai2025ming} as a second large-scale reference generator. For FLUX, FID is computed on 5{,}000 FluxPrompting~\cite{FluxPrompting} samples. Multi-GPU experiments use eight NVIDIA RTX 3090 GPUs; the single-GPU study is reported separately in the main manuscript.

\subsection{Algorithms and Official Implementations}
The sequential samplers used in the experiments include DDPM~\cite{DDPM}, DDIM~\cite{DDIM}, DPM-Solver~\cite{dpmsolver}, DPM-Solver-v3~\cite{zheng2023dpm}, DPMSolver++~\cite{dpmsolverpp}, UniPC~\cite{zhao2023unipc}, and Euler. All compared baseline samplers are invoked through the official implementations integrated in the Diffusers library~\cite{diffusers}; we do not use custom re-implementations of the baseline solvers. Parallel wrappers, including ParaDiGMS~\cite{shih2024parallel}, ParaSolver~\cite{parasolver}, ParaLin~\cite{paralin2026}, and ParaLin++, are built on the corresponding official sequential implementations. ParaTAA~\cite{tang2024accelerating} is discussed as related work but is not included in the quantitative benchmark tables because it is not available as an official Diffusers sampler under this comparison protocol. In addition, DPM-Solver is excluded from the LSUN Church table because it is not integrated with the LSUN model in the Diffusers library.

\subsection{ParaLin++ Hyperparameters}
Unless otherwise stated, ParaLin++ uses nonlinear degree $\varrho=8$ and searches $\delta$, $\beta$, and $\omega$ independently for each sampler. For Stable Diffusion-v2, the DDPM setting uses 1{,}000 discrete sampling steps, while DDIM and DPM-Solver use the listed 25-, 50-, and 100-step configurations. For LSUN Church, we evaluate the listed DDIM configurations. Recommended or best-reported tolerances from the prior papers are used for ParaDiGMS and ParaSolver. On Stable Diffusion-v2, ParaDiGMS uses tolerances $0.5$, $0.01$, and $0.01$ for DDPM, DDIM, and DPM-Solver, respectively; ParaSolver uses $0.55$, $0.01$, and $0.01$. On LSUN Church, ParaDiGMS uses $0.001$ for DDIM, while ParaSolver uses $0.005$. In ParaLin++, tolerance values are selected by grid search: DDPM uses $\delta\in\{0.5,0.6\}$; Stable Diffusion-v2 DDIM and DPM-Solver use $\delta\in\{0.05,0.04,0.03,0.02,0.01\}$; and LSUN DDIM uses $\delta\in\{0.005,0.004,0.003,0.002,0.001\}$.

\begin{figure*}[htbp]
    \centering

    \begin{minipage}{\linewidth}
        \centering
        \parbox[b][0.094\linewidth][t]{0.05\linewidth}{\rotatebox{90}{\small\textbf{DDPM}}}%
\includegraphics[width=0.094\linewidth]{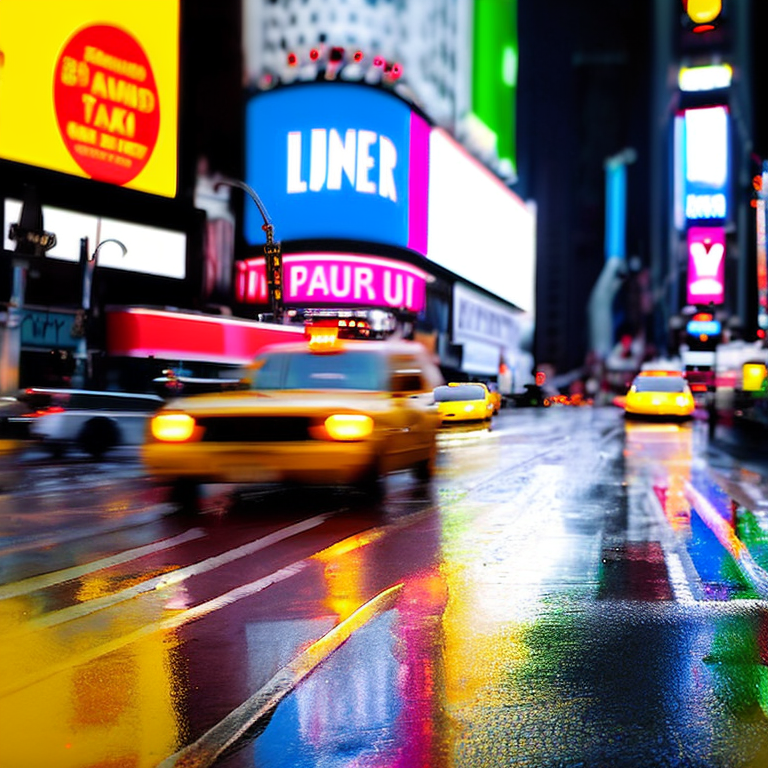}%
\includegraphics[width=0.094\linewidth]{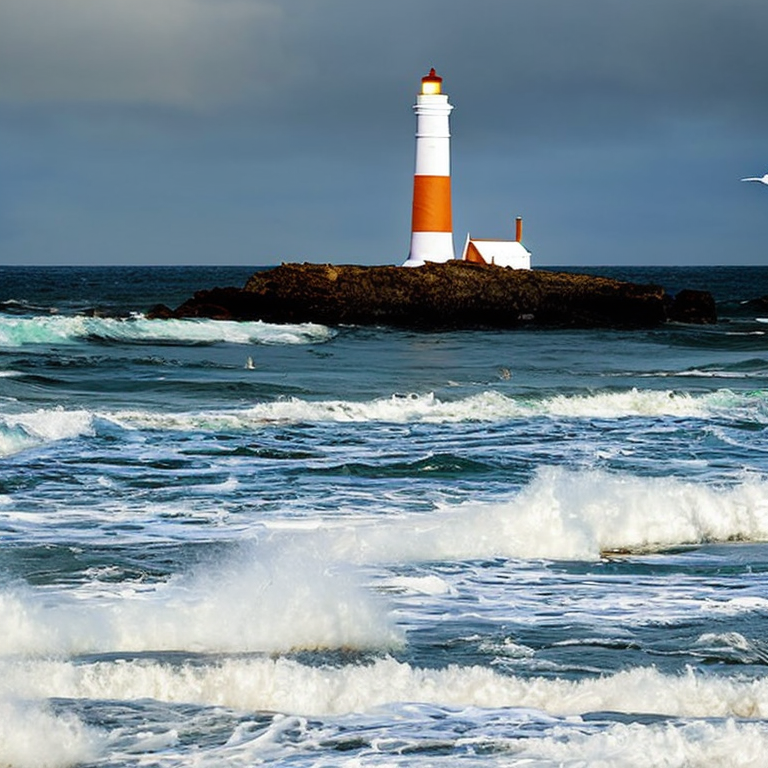}%
\includegraphics[width=0.094\linewidth]{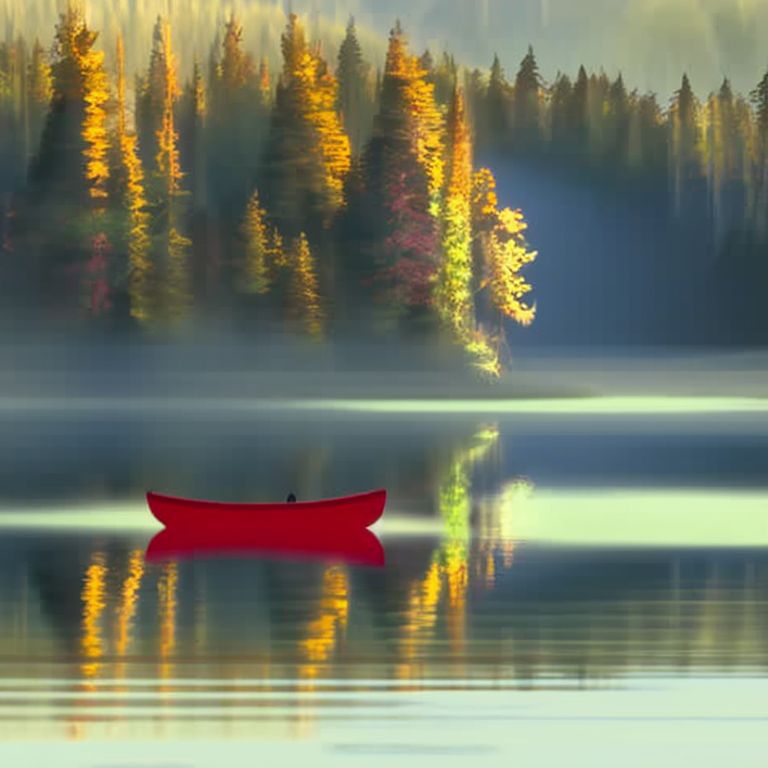}%
\includegraphics[width=0.094\linewidth]{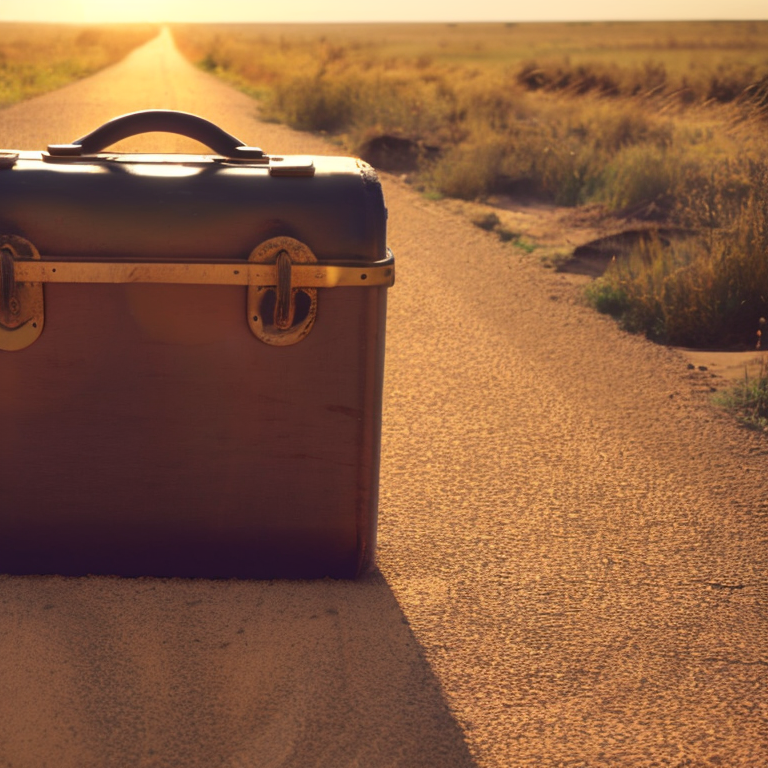}%
\includegraphics[width=0.094\linewidth]{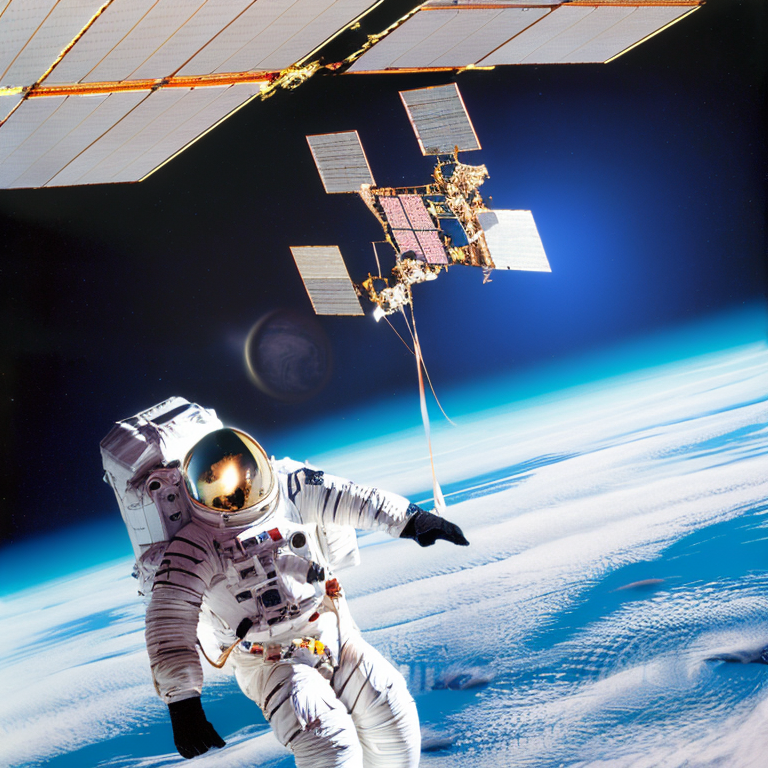}%
\includegraphics[width=0.094\linewidth]{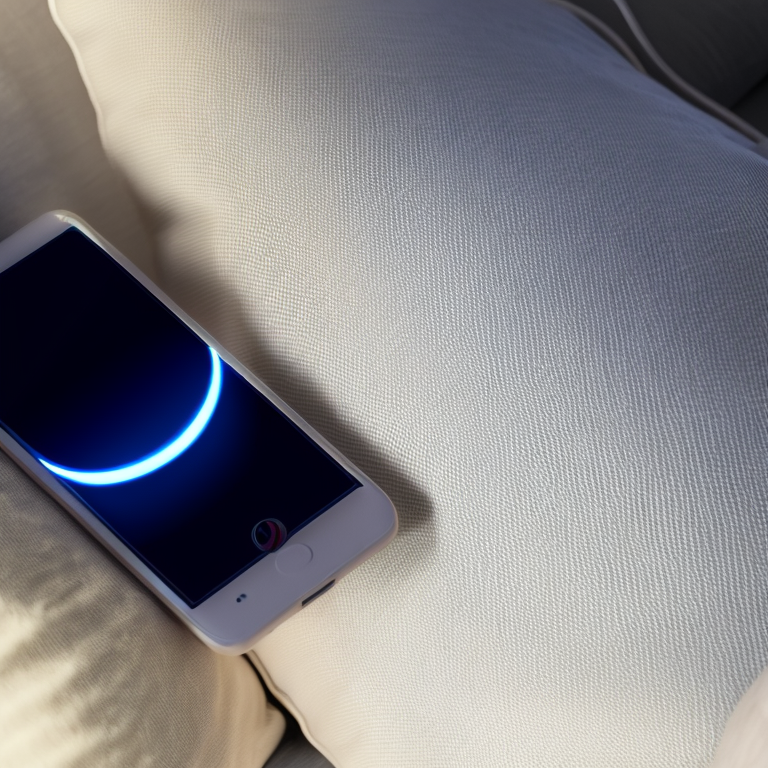}%
\includegraphics[width=0.094\linewidth]{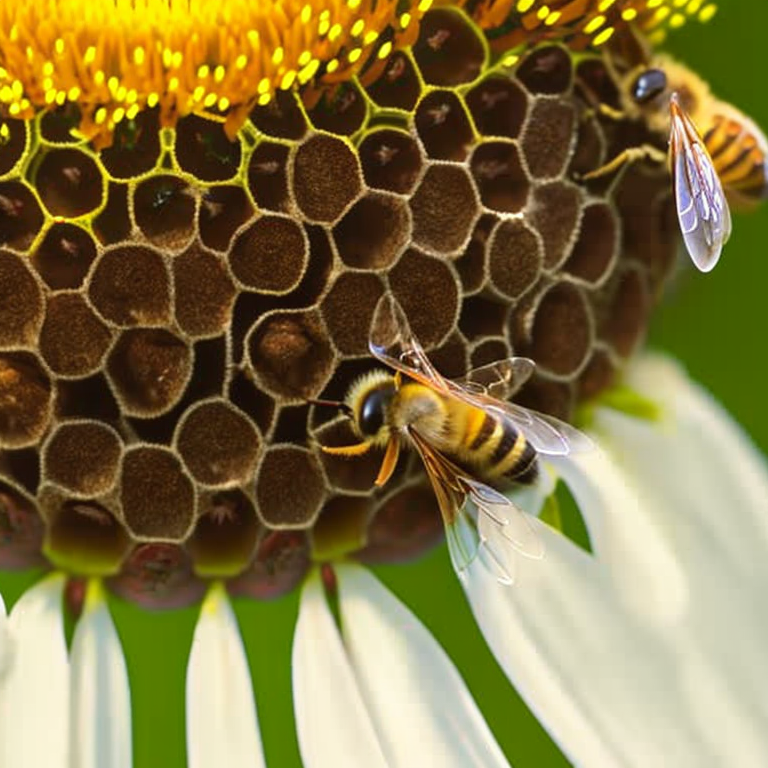}%
\includegraphics[width=0.094\linewidth]{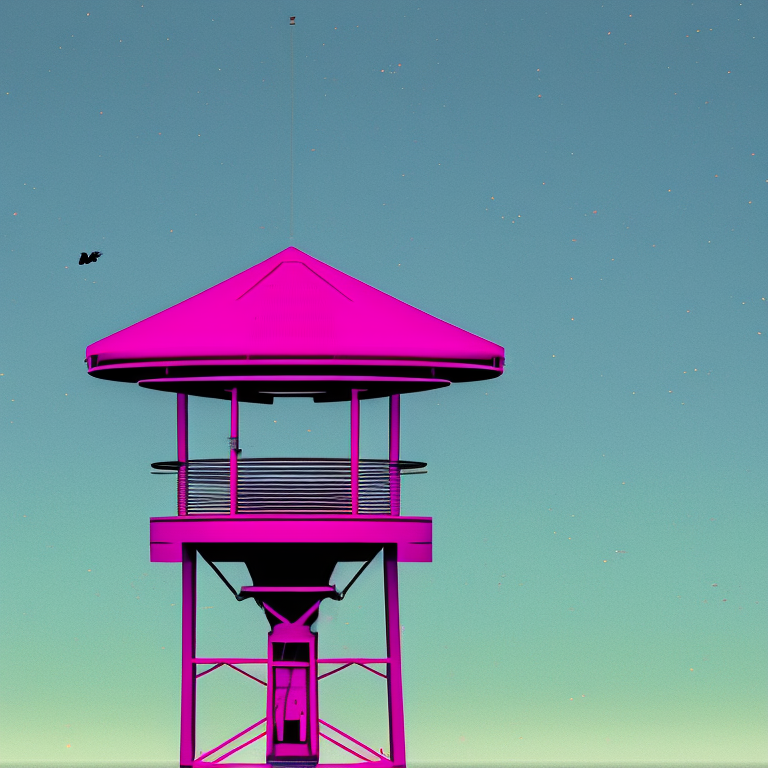}%
\includegraphics[width=0.094\linewidth]{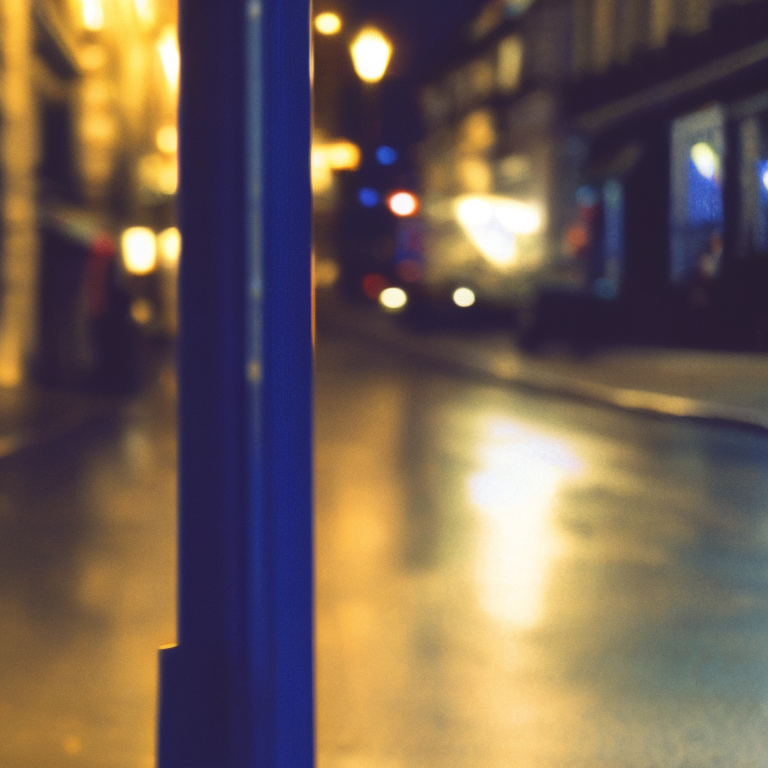}%
\includegraphics[width=0.094\linewidth]{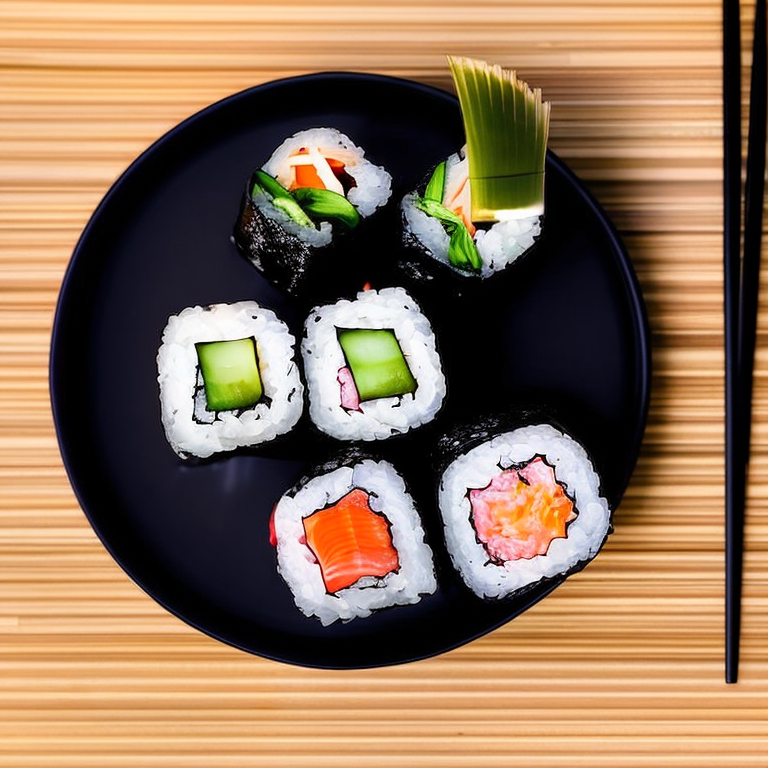}
    \end{minipage}
    \vspace{1mm} 

    \begin{minipage}{\linewidth}
        \centering
        \parbox[b][0.094\linewidth][t]{0.05\linewidth}{\rotatebox{90}{\small\textbf{ParaDiGMS}}}%
\includegraphics[width=0.094\linewidth]{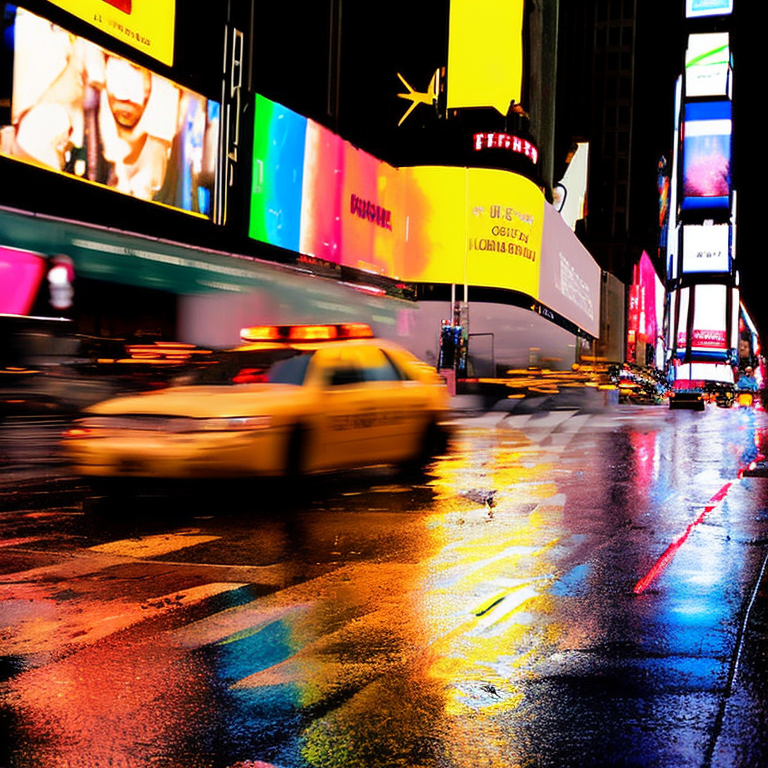}%
\includegraphics[width=0.094\linewidth]{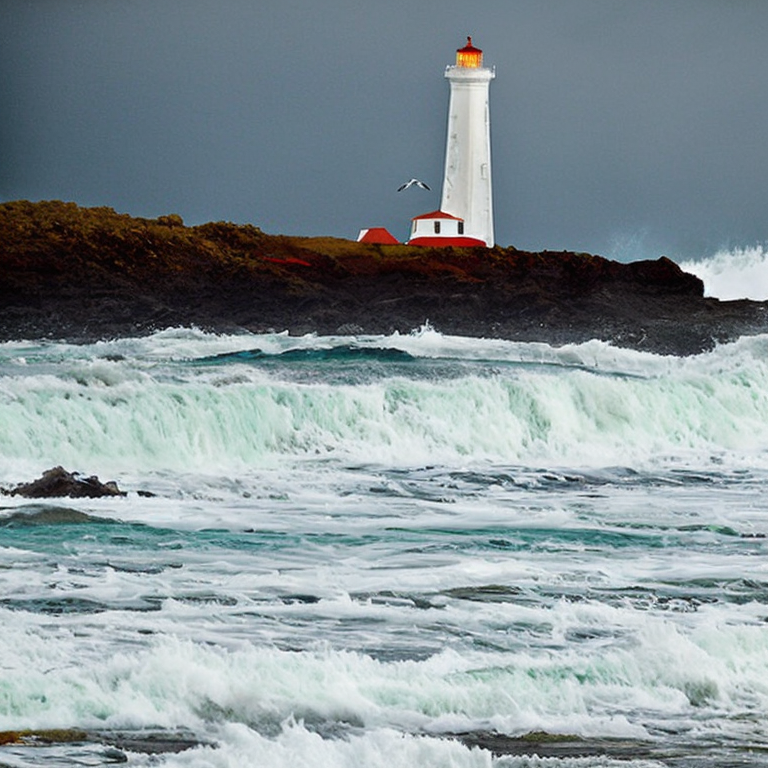}%
\includegraphics[width=0.094\linewidth]{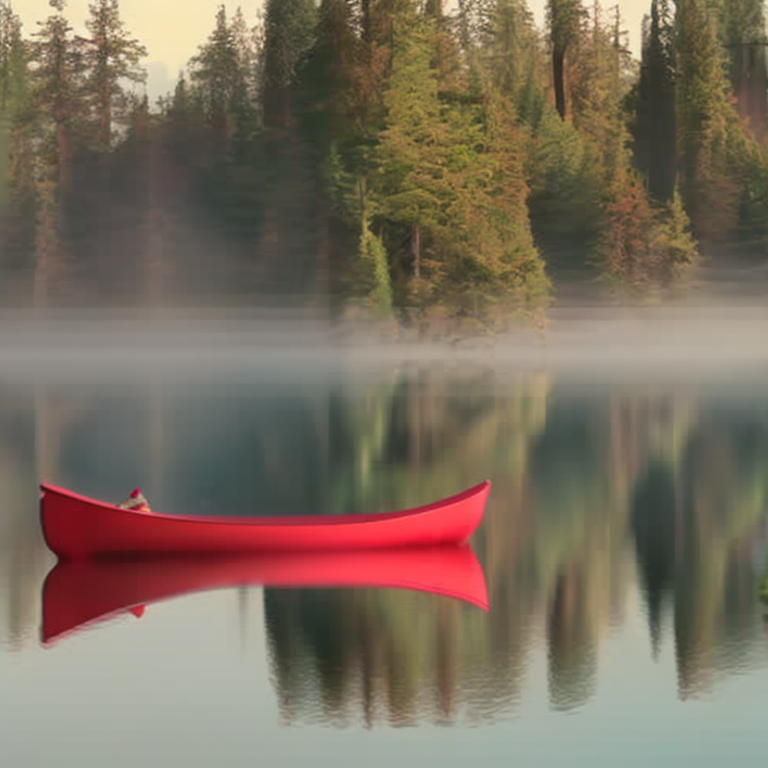}%
\includegraphics[width=0.094\linewidth]{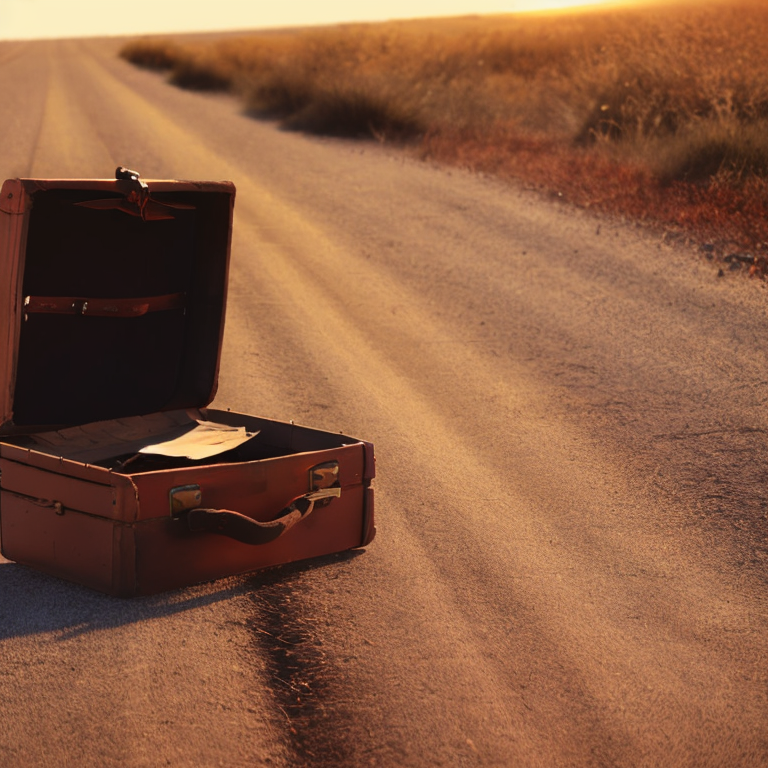}%
\includegraphics[width=0.094\linewidth]{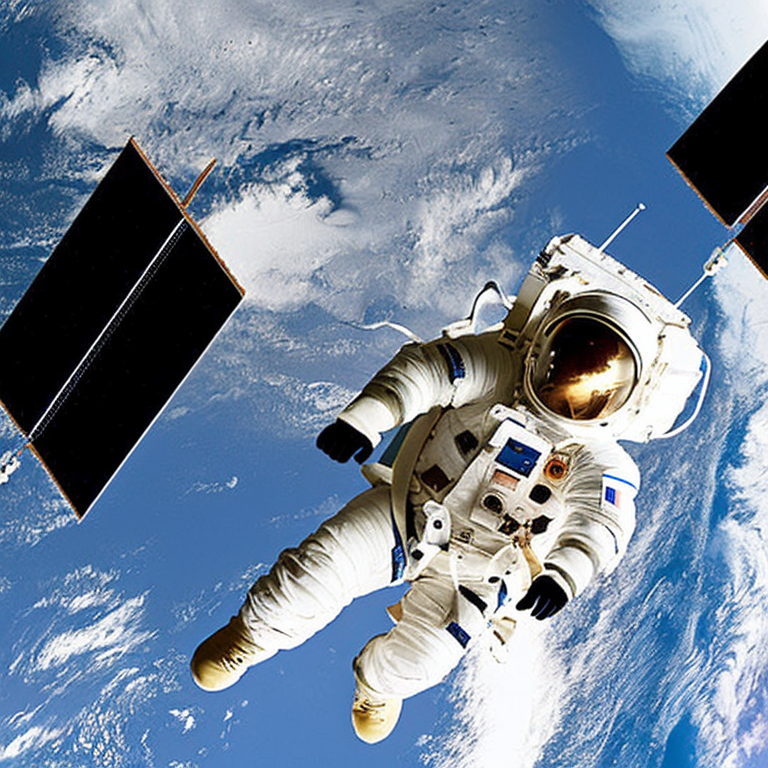}%
\includegraphics[width=0.094\linewidth]{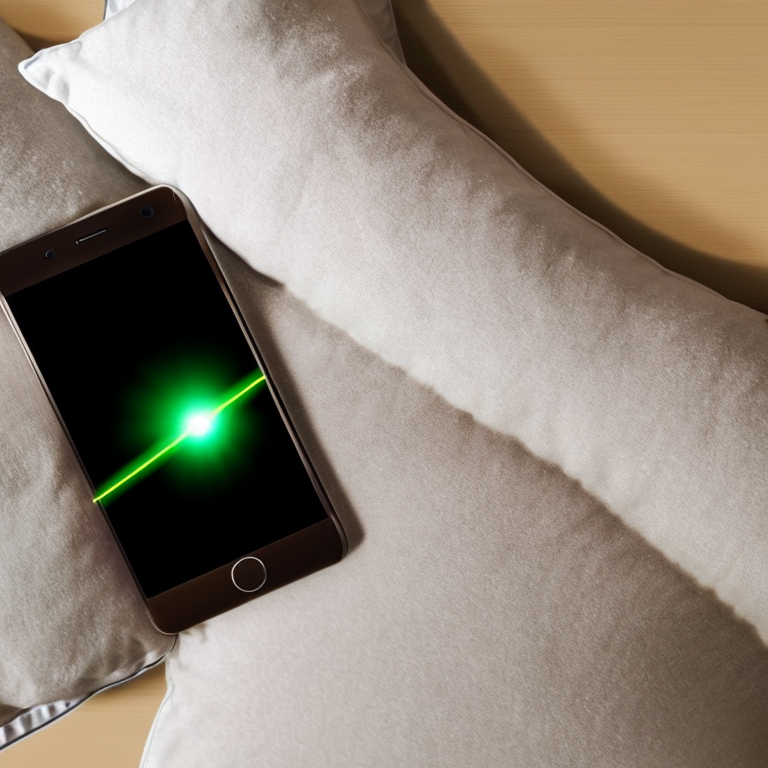}%
\includegraphics[width=0.094\linewidth]{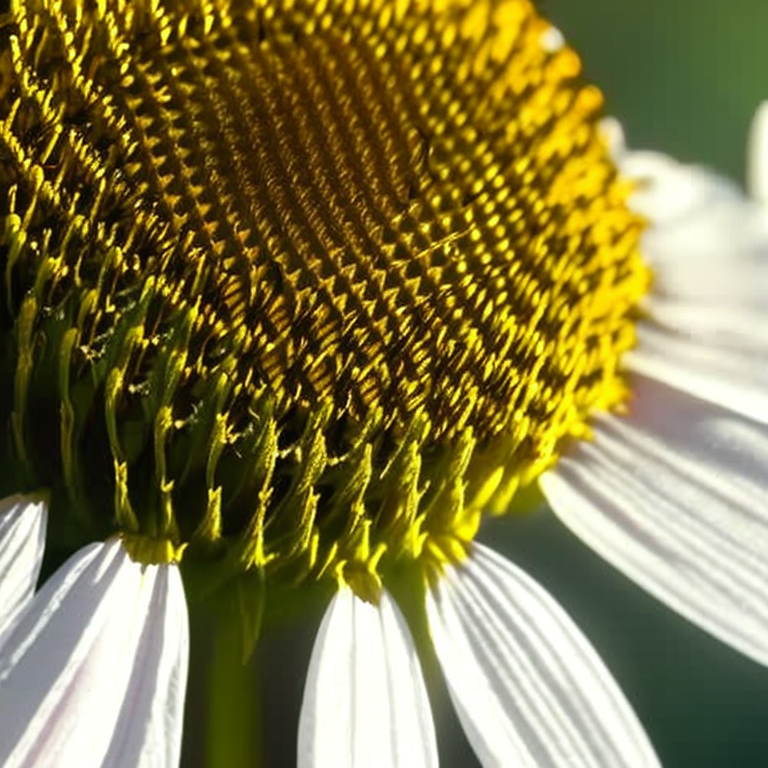}%
\includegraphics[width=0.094\linewidth]{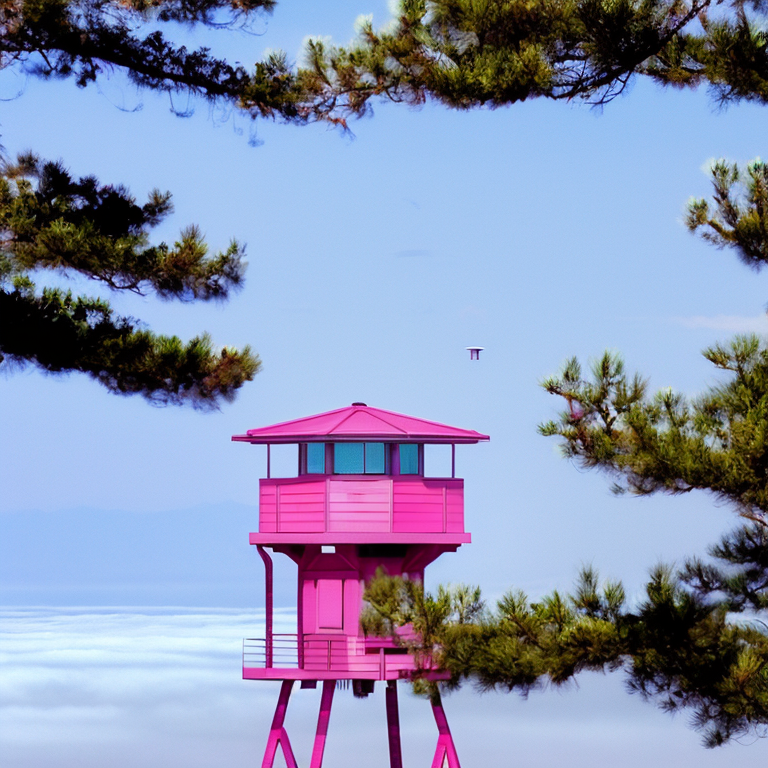}%
\includegraphics[width=0.094\linewidth]{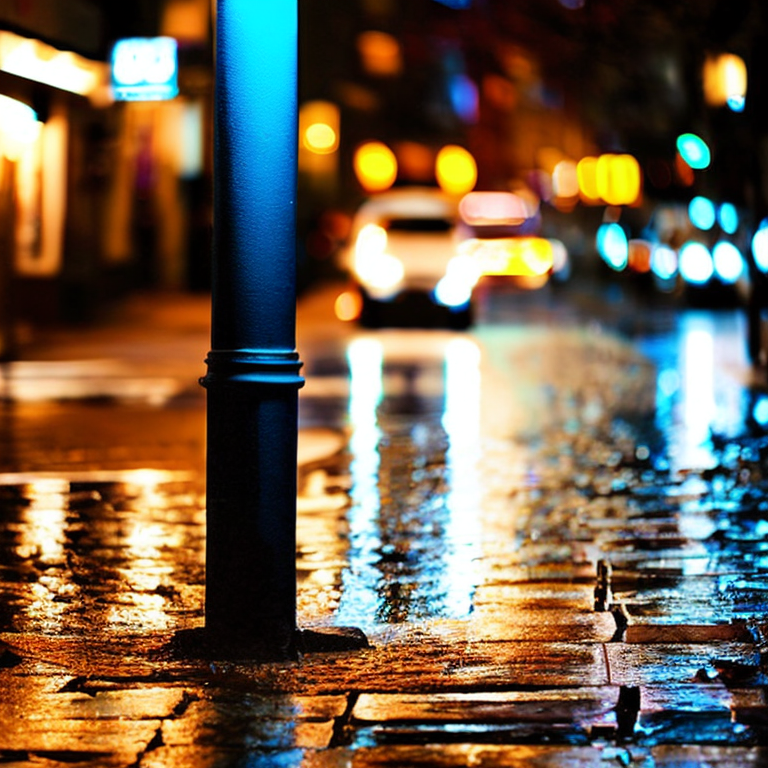}%
\includegraphics[width=0.094\linewidth]{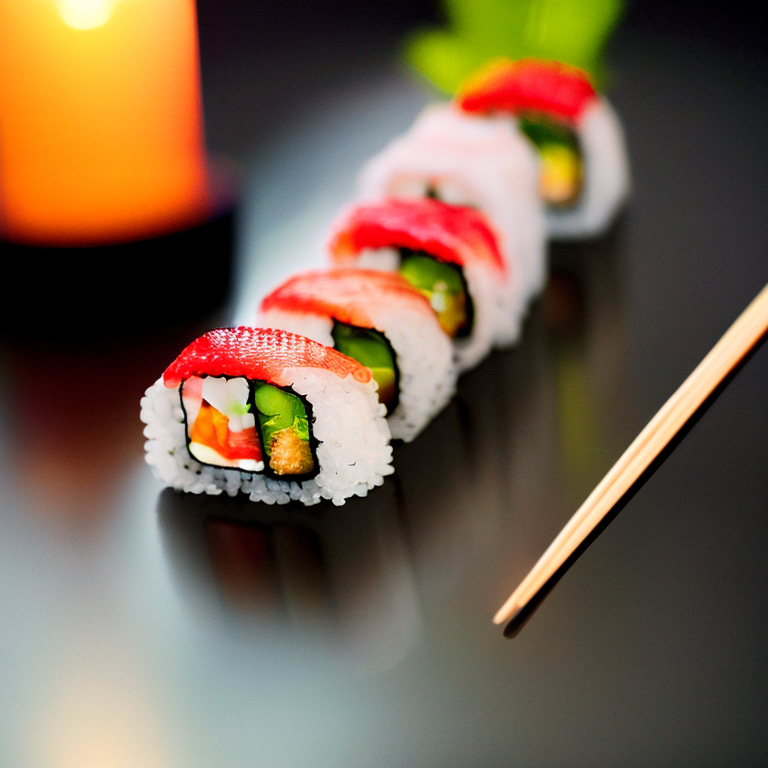}%
    \end{minipage}
    \vspace{1mm} 

    \begin{minipage}{\linewidth}
        \centering
        \parbox[b][0.094\linewidth][t]{0.05\linewidth}{\rotatebox{90}{\small\textbf{ParaSolver}}}%
\includegraphics[width=0.094\linewidth]{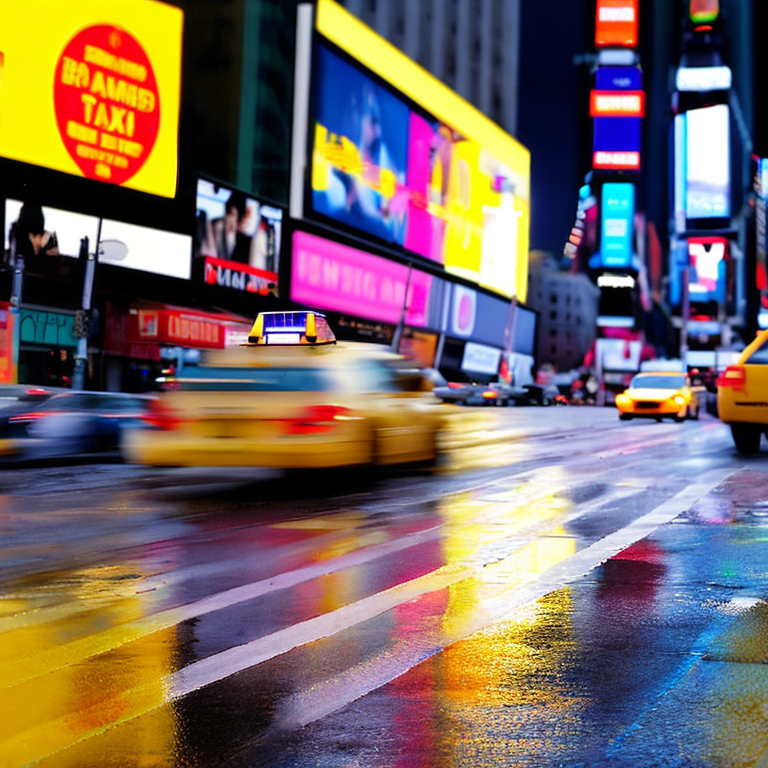}%
\includegraphics[width=0.094\linewidth]{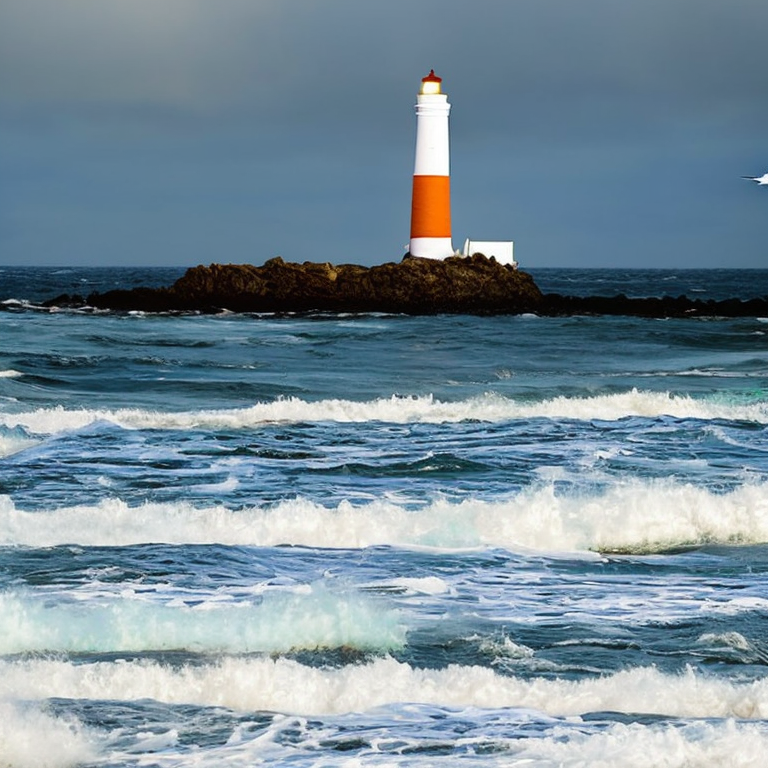}%
\includegraphics[width=0.094\linewidth]{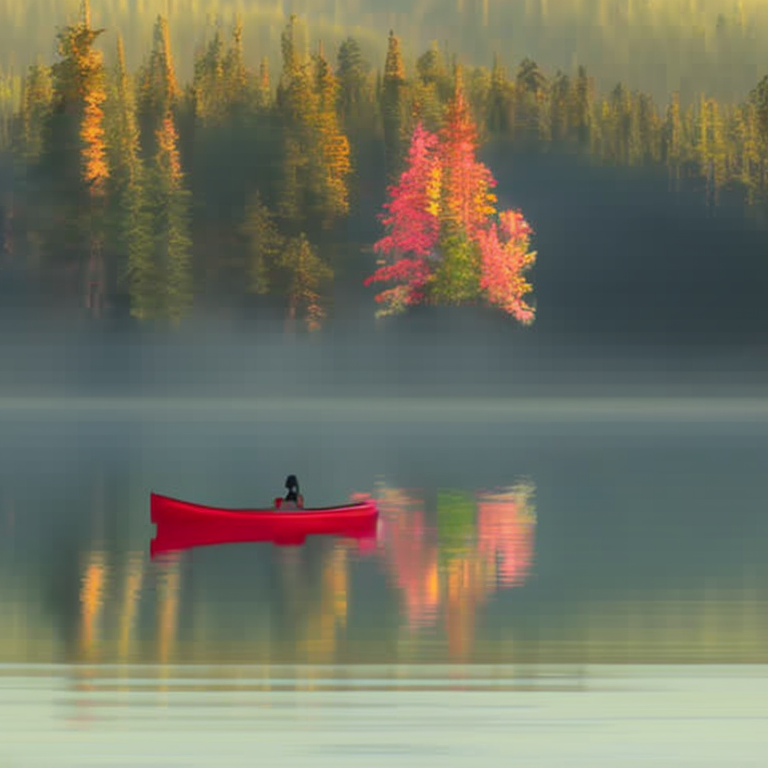}%
\includegraphics[width=0.094\linewidth]{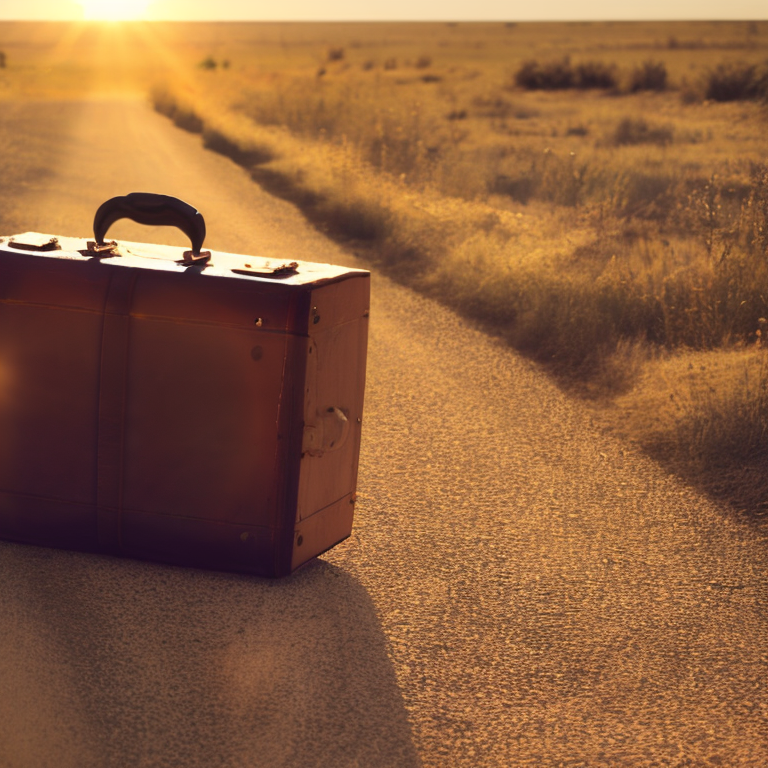}%
\includegraphics[width=0.094\linewidth]{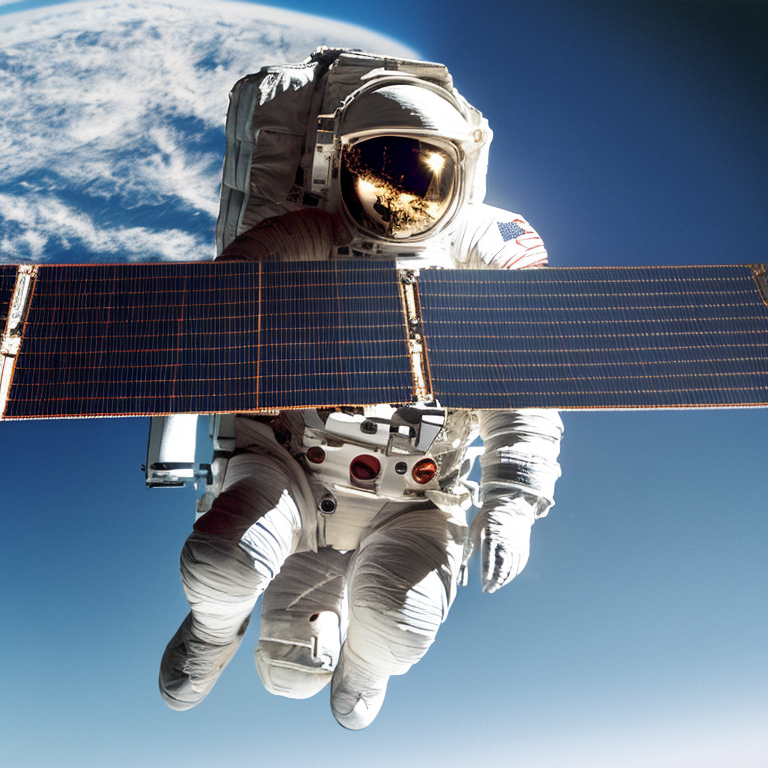}%
\includegraphics[width=0.094\linewidth]{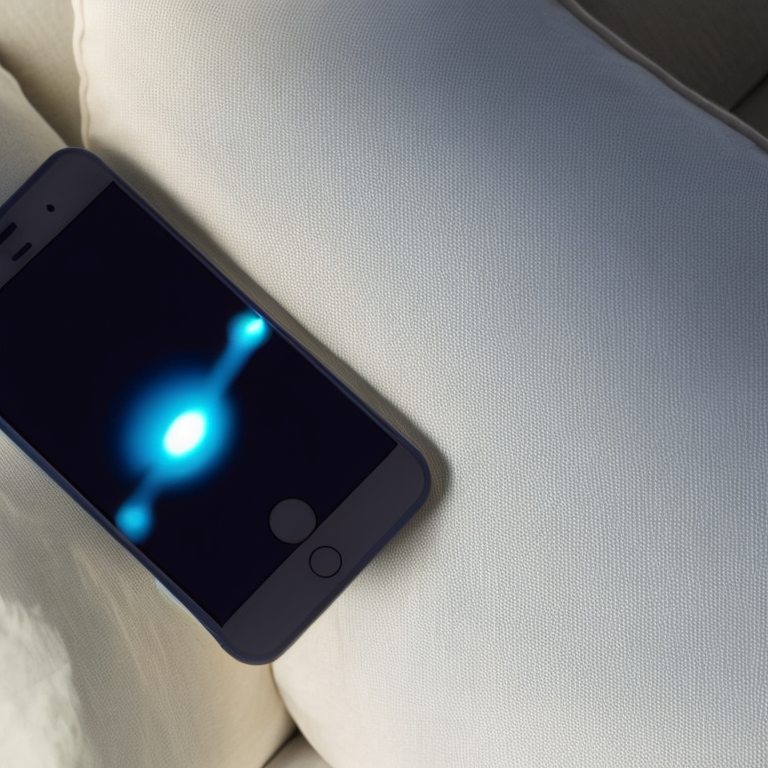}%
\includegraphics[width=0.094\linewidth]{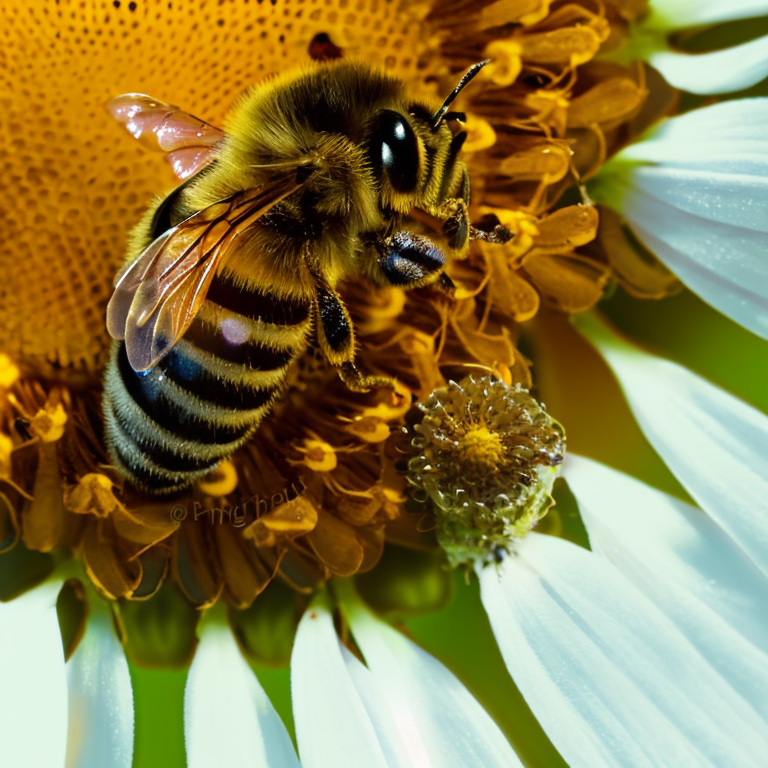}%
\includegraphics[width=0.094\linewidth]{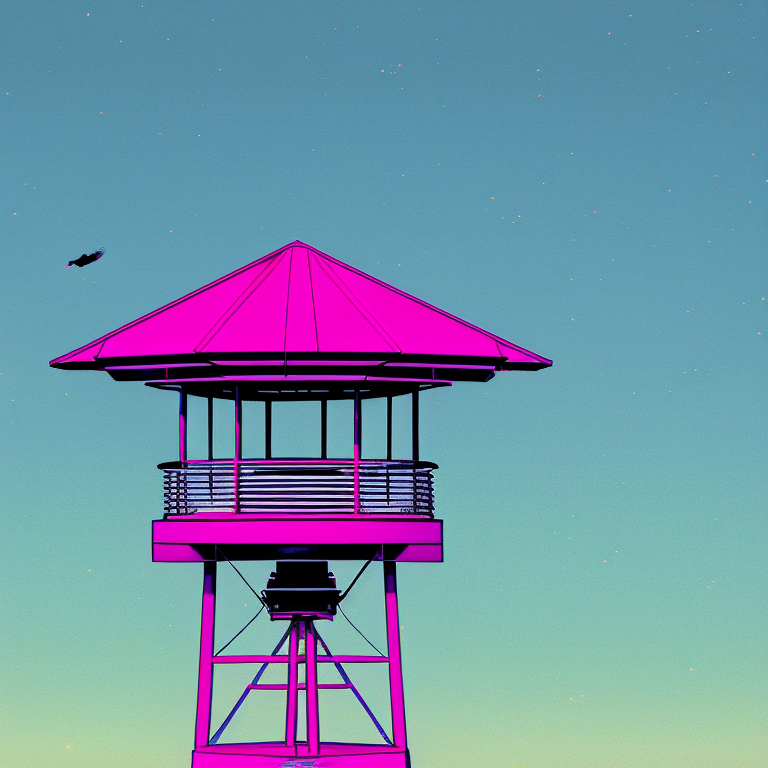}%
\includegraphics[width=0.094\linewidth]{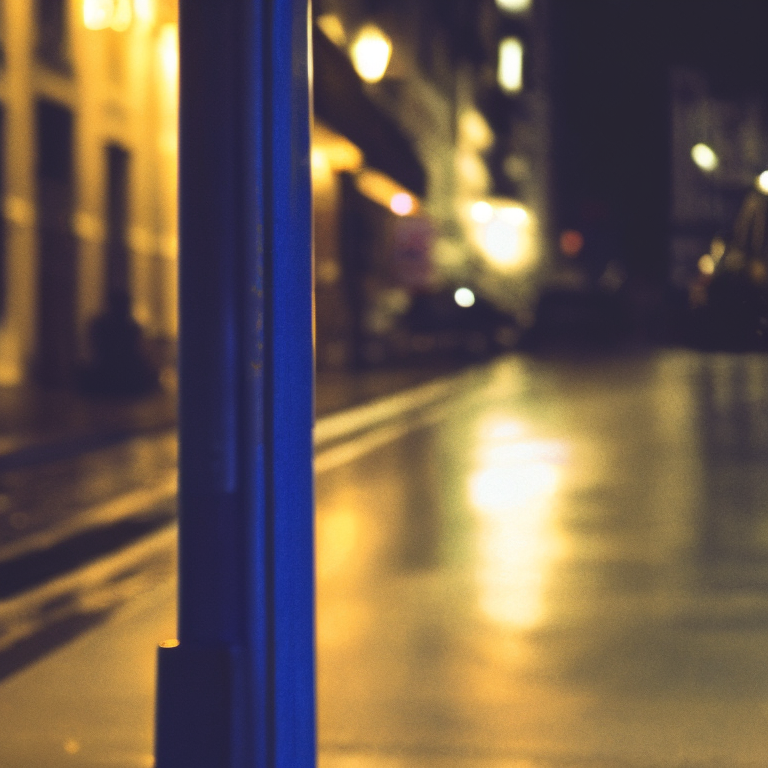}%
\includegraphics[width=0.094\linewidth]{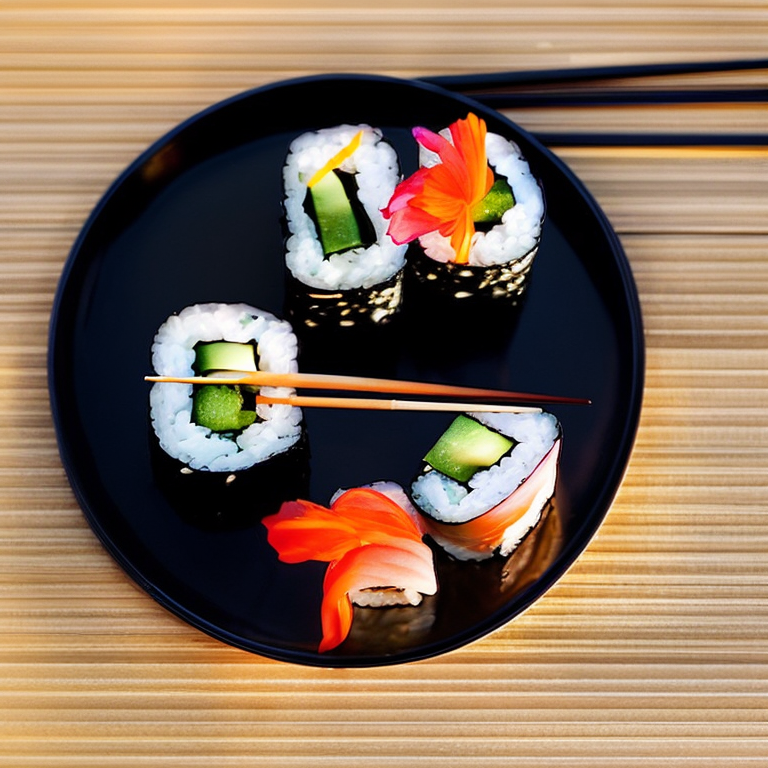}
    \end{minipage}
    \vspace{1mm} 

    \begin{minipage}{\linewidth}
        \centering
        \parbox[b][0.094\linewidth][t]{0.05\linewidth}{\rotatebox{90}{\small\textbf{ParaLin}}}%
\includegraphics[width=0.094\linewidth]{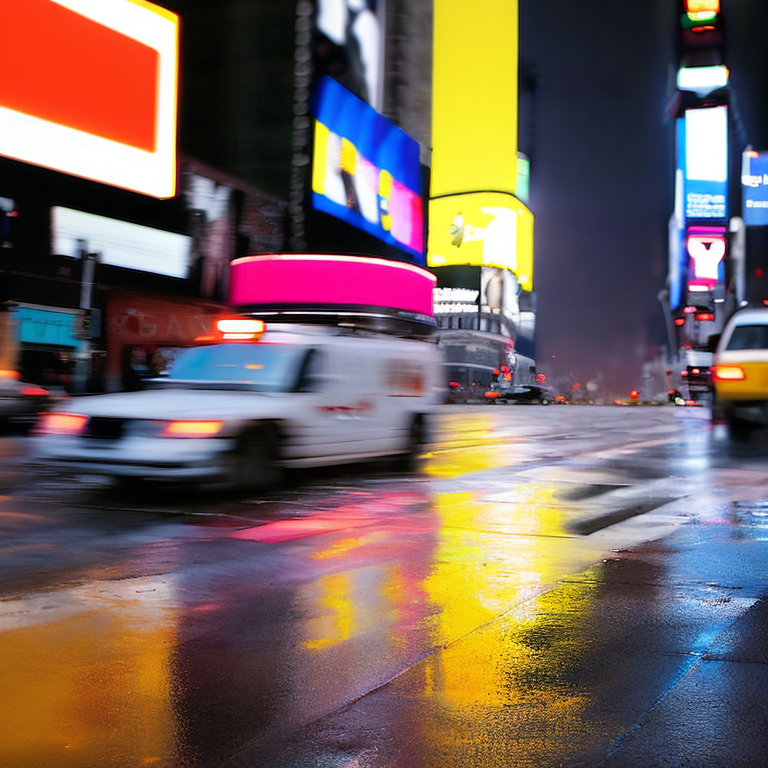}%
\includegraphics[width=0.094\linewidth]{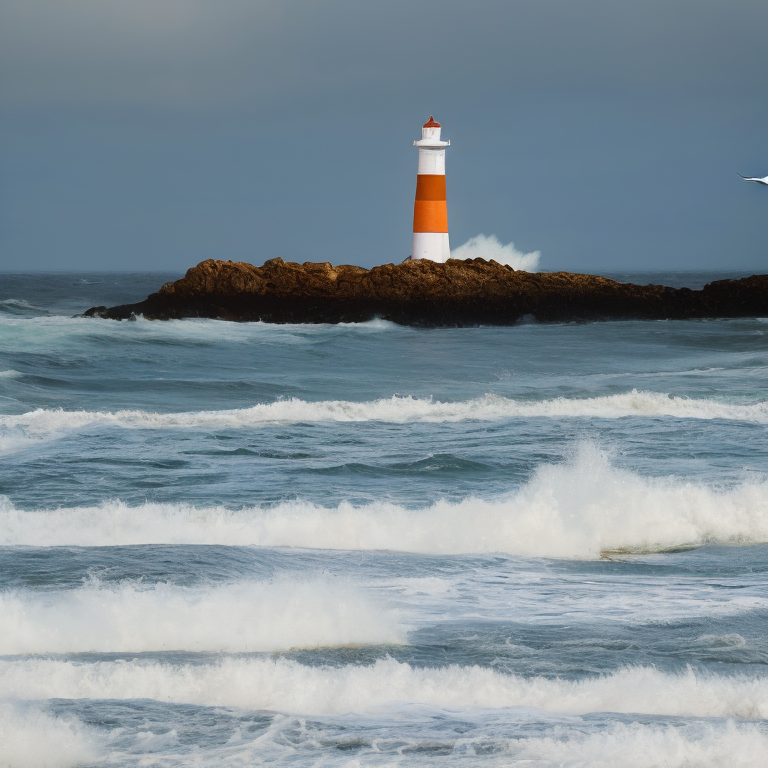}%
\includegraphics[width=0.094\linewidth]{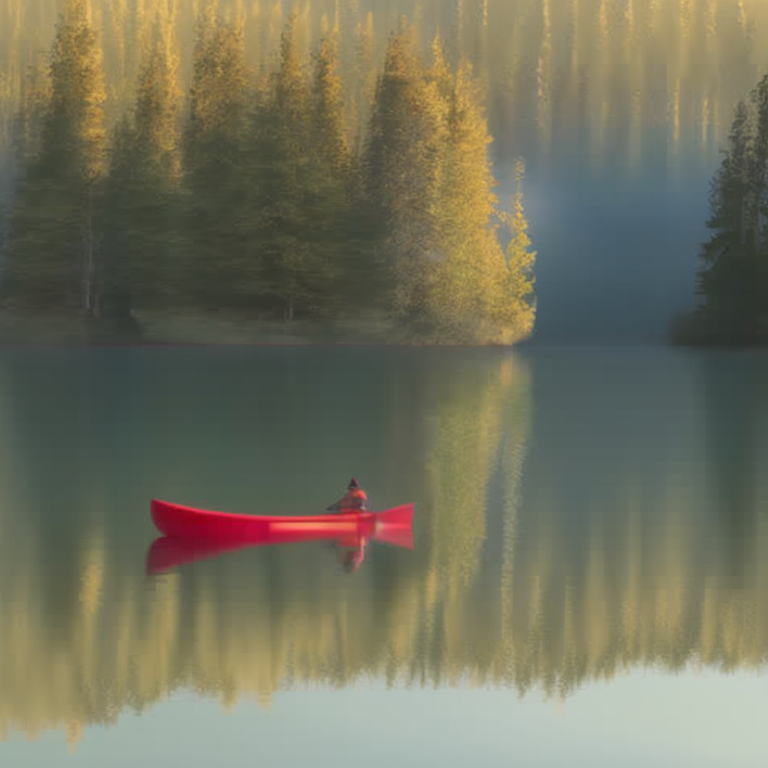}%
\includegraphics[width=0.094\linewidth]{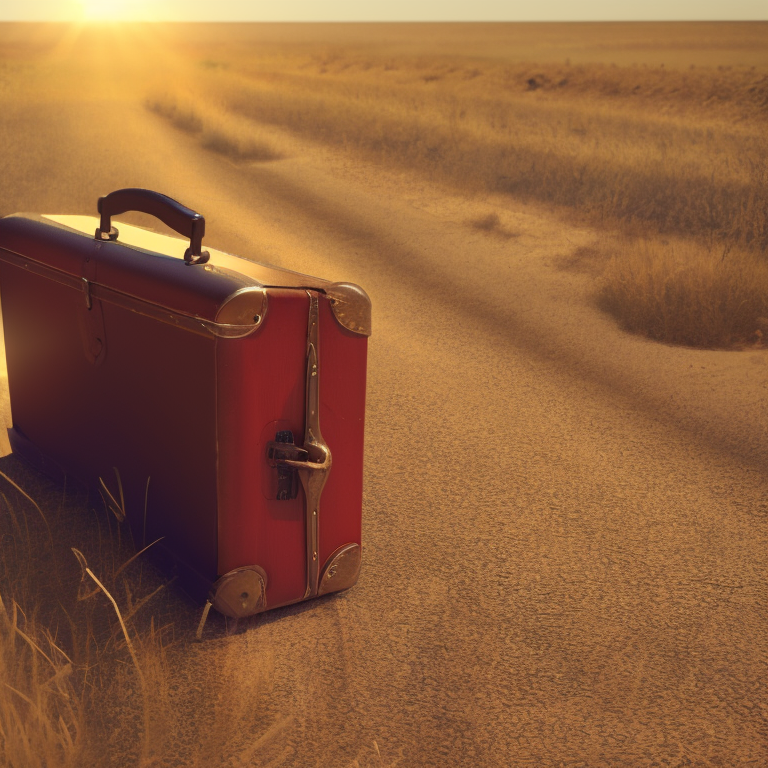}%
\includegraphics[width=0.094\linewidth]{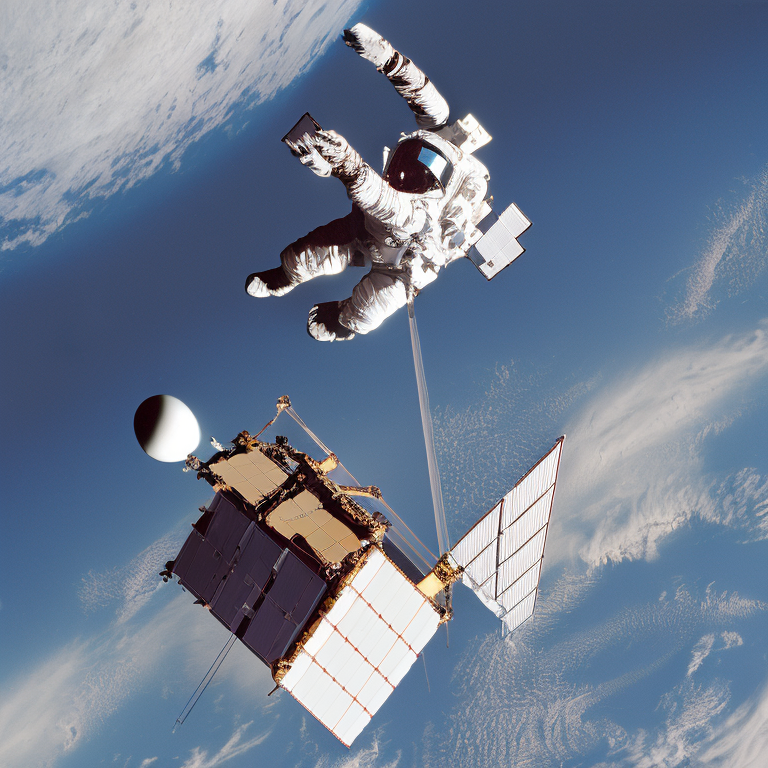}%
\includegraphics[width=0.094\linewidth]{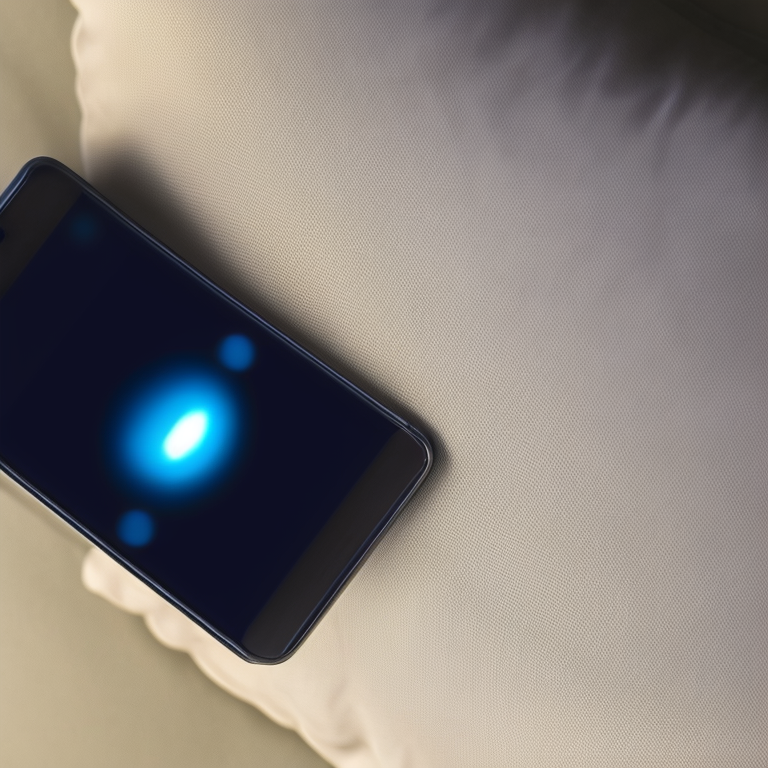}%
\includegraphics[width=0.094\linewidth]{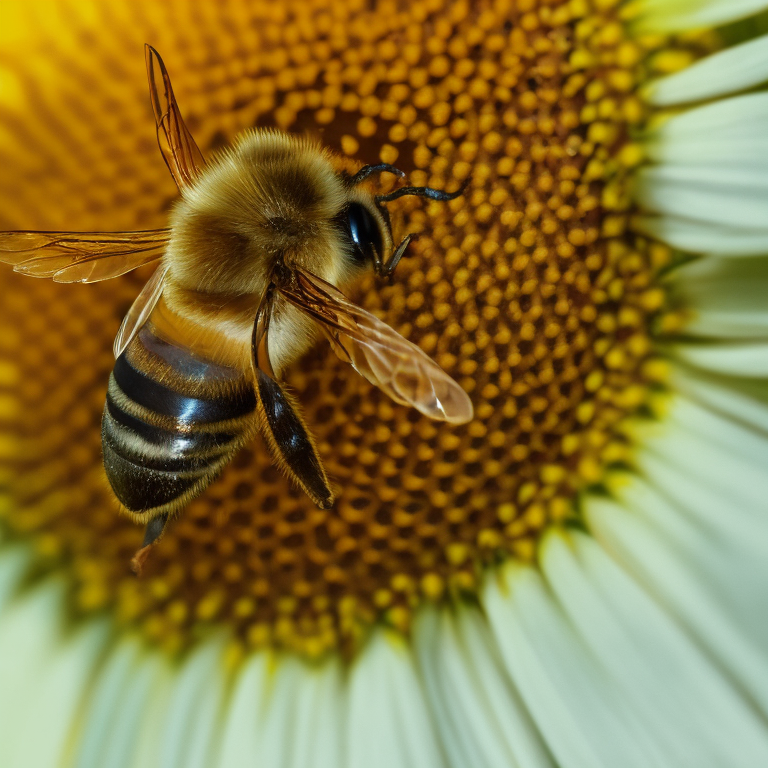}%
\includegraphics[width=0.094\linewidth]{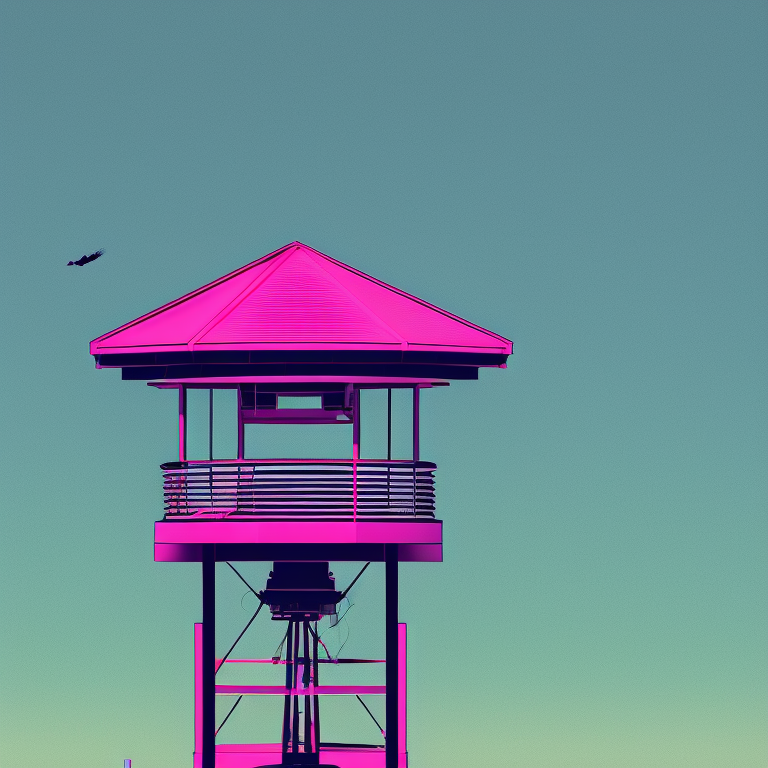}%
\includegraphics[width=0.094\linewidth]{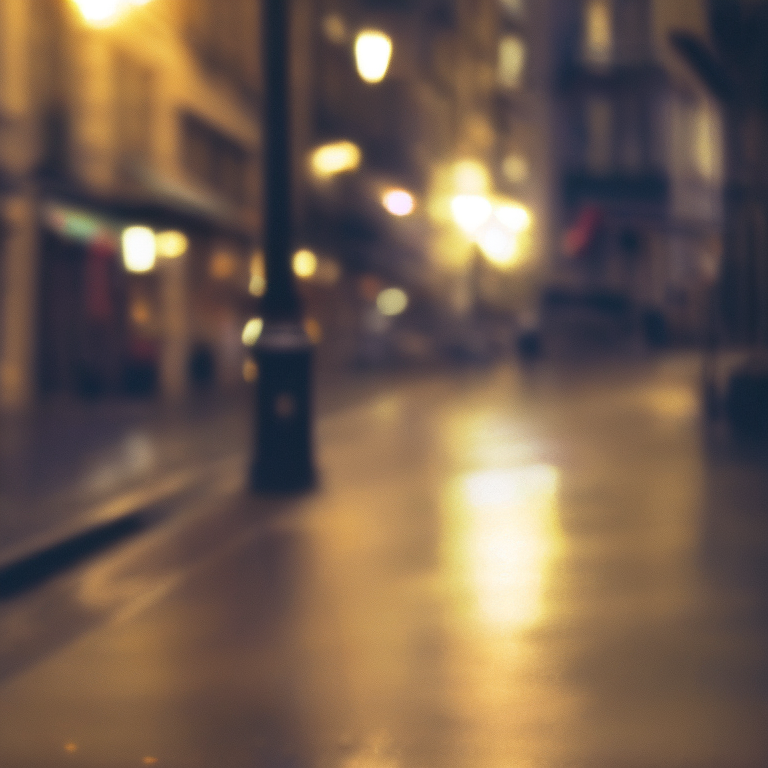}%
\includegraphics[width=0.094\linewidth]{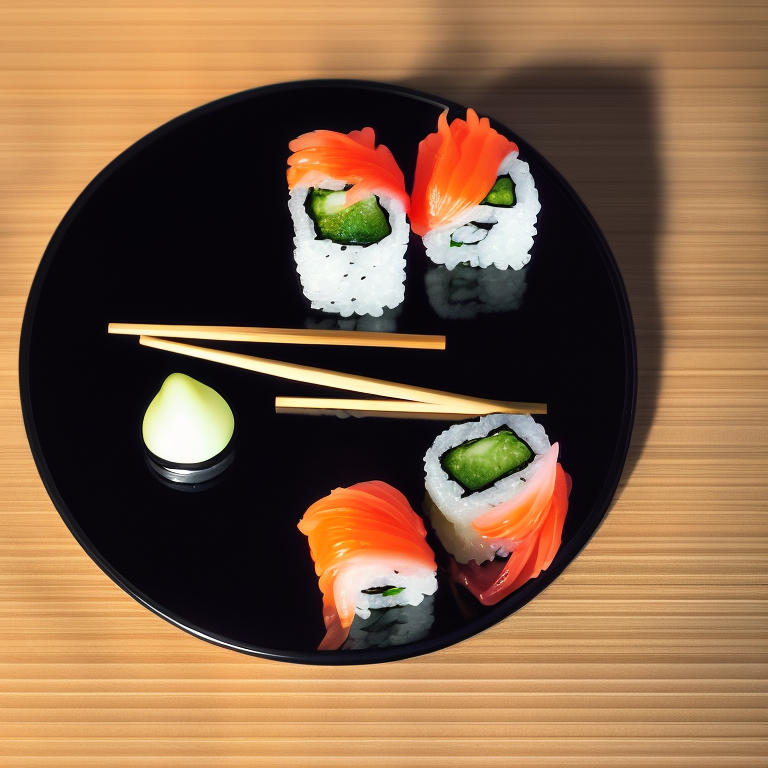}
    \end{minipage}
    \vspace{1mm} 

    \begin{minipage}{\linewidth}
        \centering
        \parbox[c][0.094\linewidth][t]{0.05\linewidth}{\rotatebox{90}{\small\textbf{Ours}}}%
\includegraphics[width=0.094\linewidth]{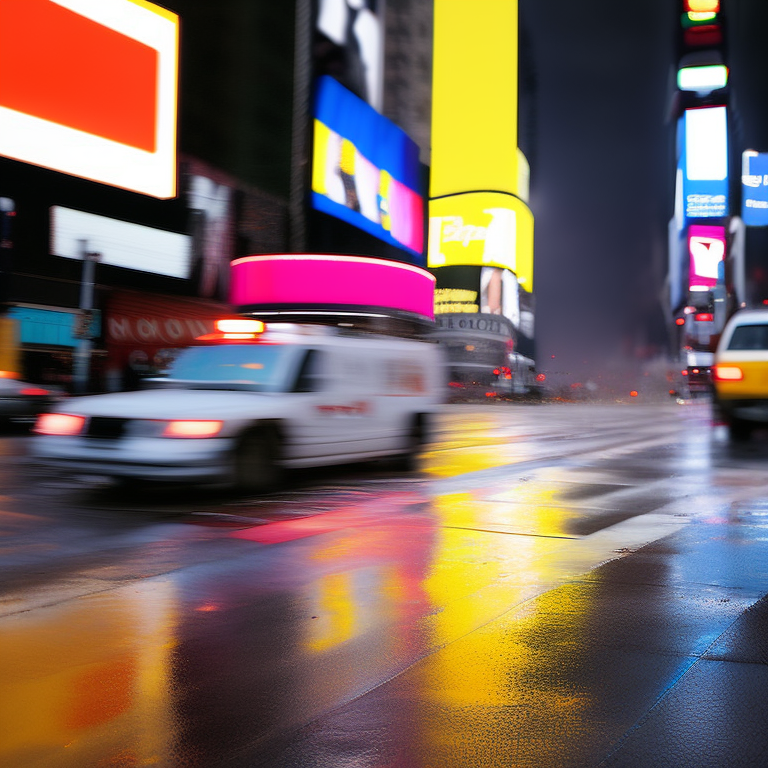}%
\includegraphics[width=0.094\linewidth]{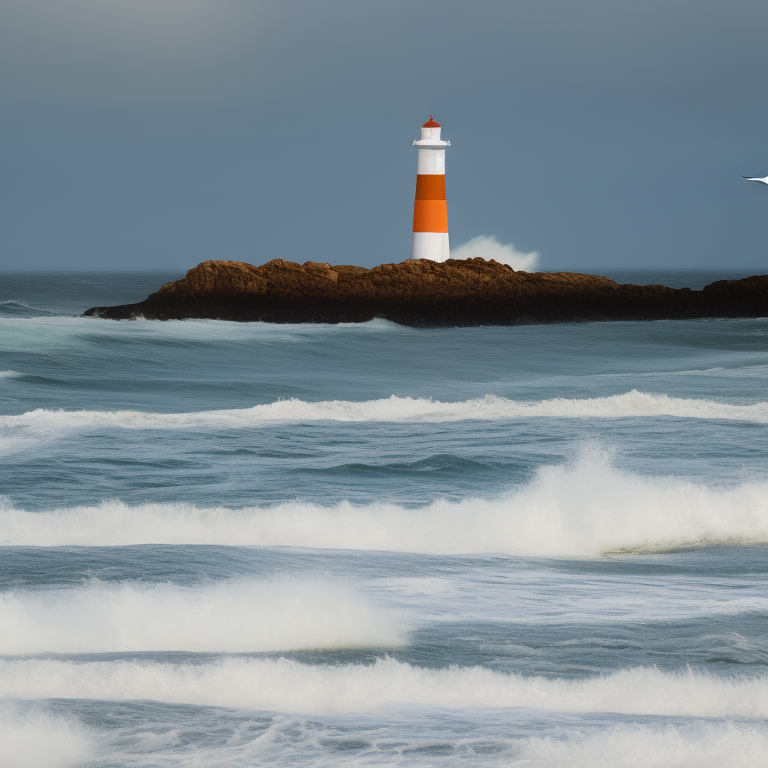}%
\includegraphics[width=0.094\linewidth]{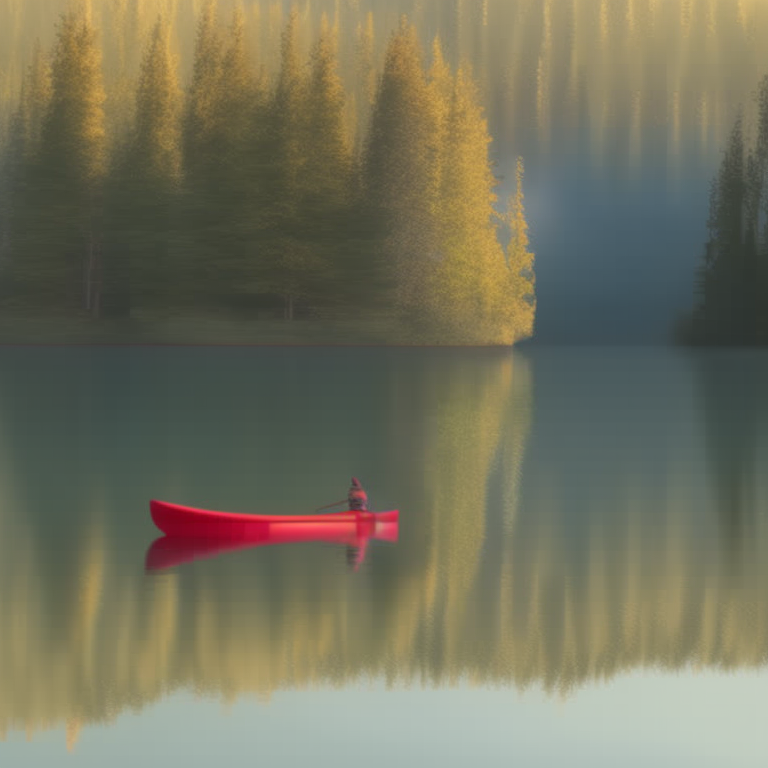}%
\includegraphics[width=0.094\linewidth]{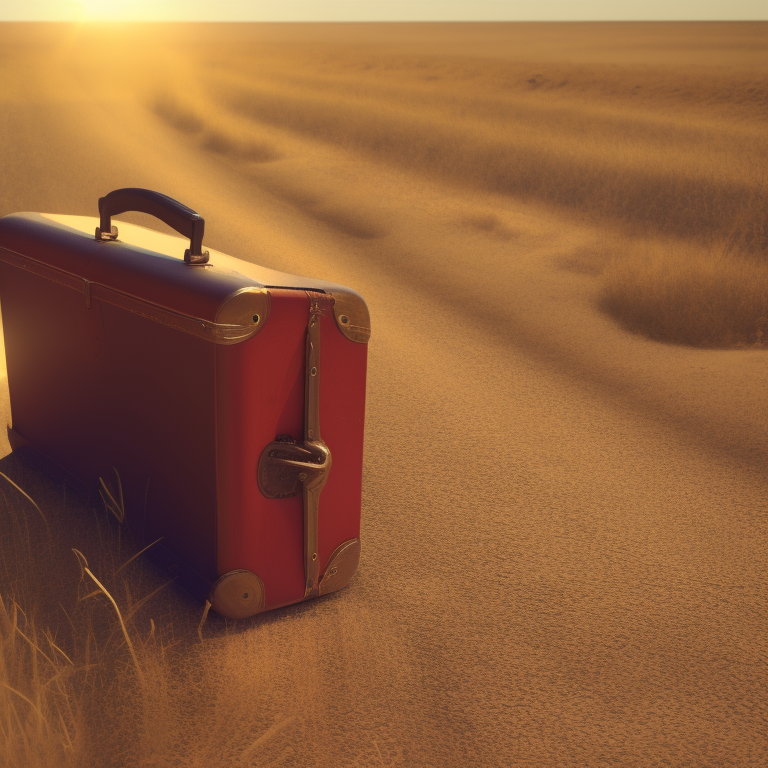}%
\includegraphics[width=0.094\linewidth]{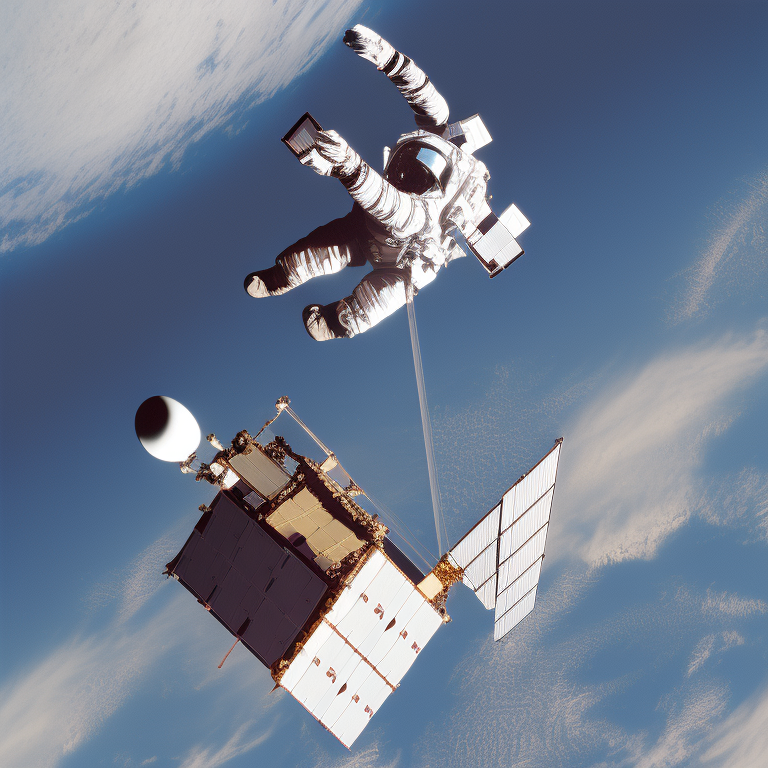}%
\includegraphics[width=0.094\linewidth]{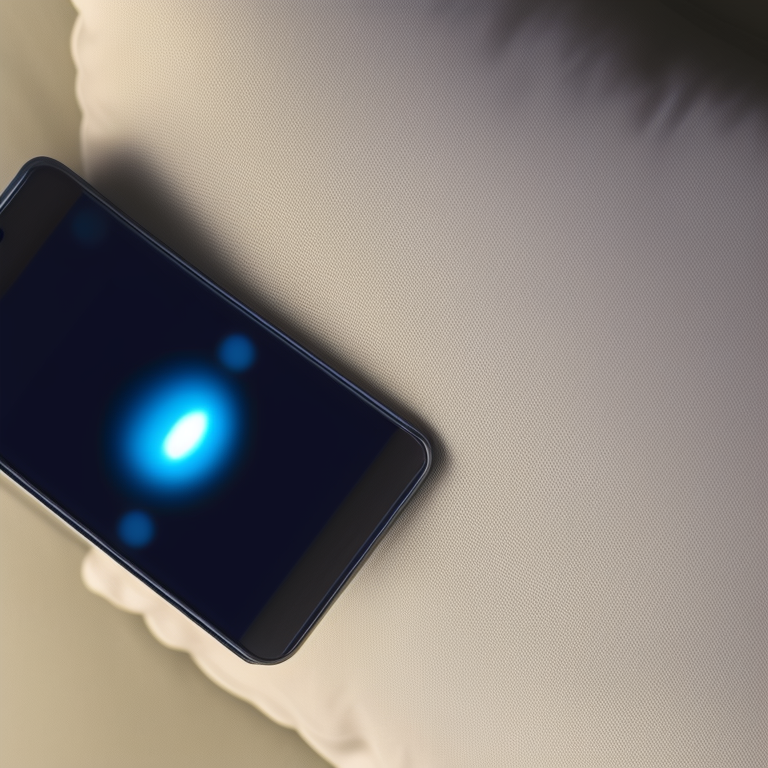}%
\includegraphics[width=0.094\linewidth]{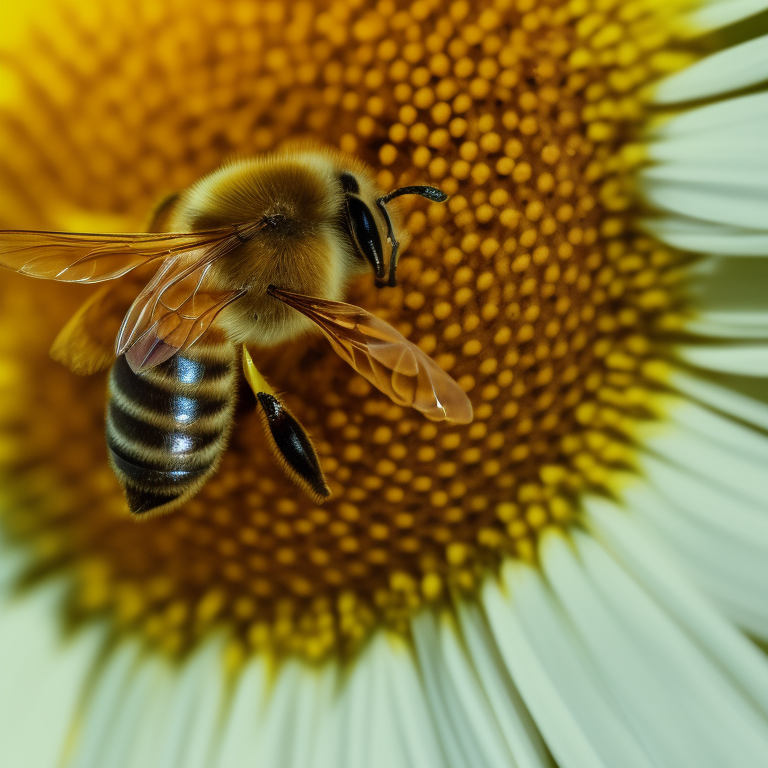}%
\includegraphics[width=0.094\linewidth]{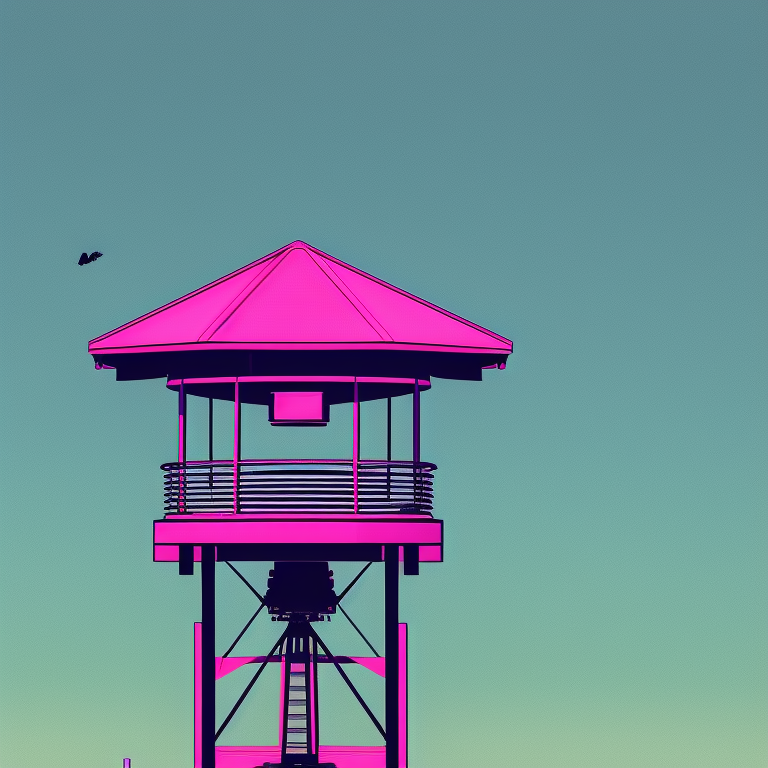}%
\includegraphics[width=0.094\linewidth]{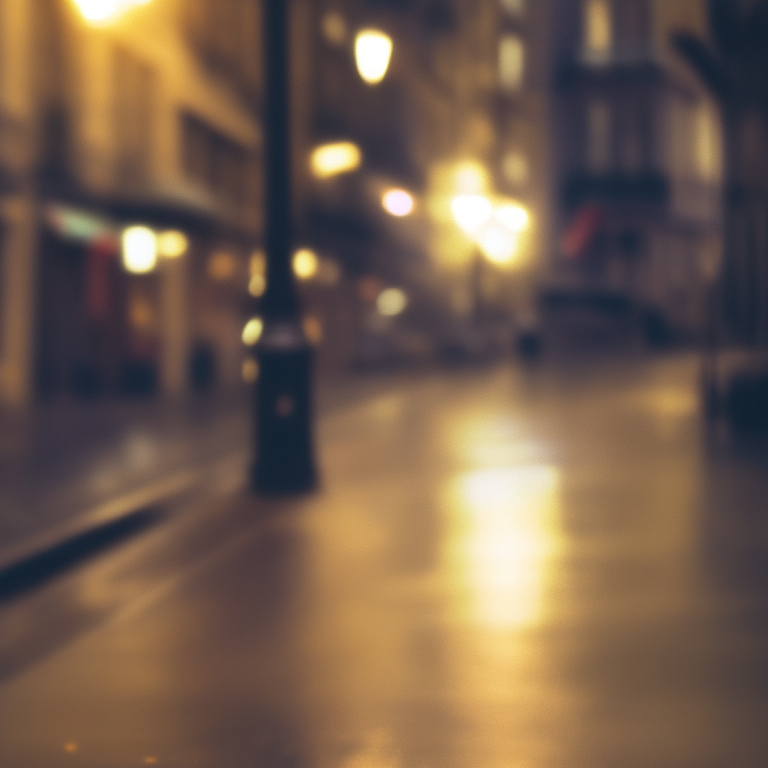}%
\includegraphics[width=0.094\linewidth]{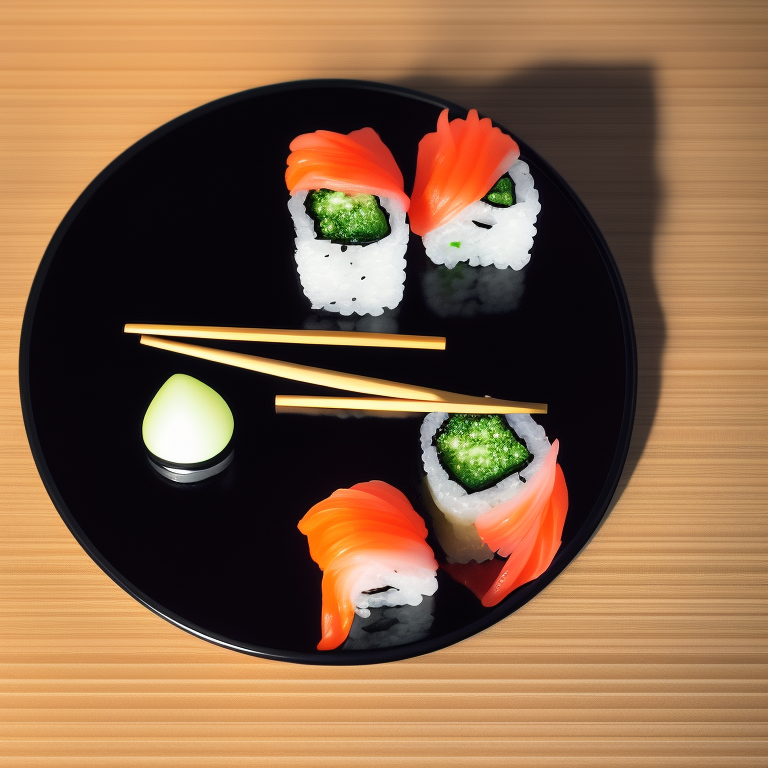}
    \end{minipage}
\vspace{-0.3cm}
    \caption{\label{fig:CMP_10_images_SD_DDPM}
    Visual comparison of final generation results on Stable Diffusion v2, using various samplers to expedite a 1000-step DDPM baseline. Each column represents the output for a different prompt, demonstrating fidelity and visual consistency across methods.
    \textbf{Rows (Top to Bottom)}: DDPM (Baseline), ParaDiGMS, ParaSolver, ParaLin, and ParaLin++.
    }
\end{figure*}

\subsection{Generation-Quality Benchmark Protocol}
\textbf{FLUX.1-dev.}
The FLUX study generates $1024\times1024$ images with guidance scale $3.5$ on eight GPUs. The reference is a 30-step Euler sampler. UniPC and DPMSolver++ are evaluated with their standard sequential few-step settings at 16, 20, 24, 26, and 30 discrete steps. ParaLin++ keeps all 30 discrete diffusion steps and uses tolerance $\delta=0.01$, nonlinear-transition parameter $\varrho=1$, parallel window size $p=100$, and a solver-step cap of 1{,}000. LongTextBench is the average of the English and Chinese text scores. The complete CFG sweep uses guidance values $0$, $3.5$, $7$, $10.5$, and $14$; speedup is always measured against the same-CFG 30-step Euler reference on the GenEval generation pass, excluding model loading and scoring.

\textbf{Ming-Flash-Omni 2.0.}
The Ming study generates $1024\times1024$ images with guidance scale $5.0$ on eight GPUs. Euler, UniPC, and DPMSolver++ use their standard sequential samplers at the listed step settings. ParaLin++ again keeps 30 discrete steps and uses $\delta=0.01$, $\varrho=1$, $p=100$, and a solver-step cap of 1{,}000. LongTextBench uses 320 prompts with four images per prompt, seed 42, hybrid OCR, and a bounding-box score threshold of $0.85$; the reported score averages the English and Chinese subsets. DPG-Bench uses 1{,}065 images with the mPLUG visual-question-answering model. GenEval uses 553 prompts with four images per prompt, bf16 generation, and eight evaluation workers. All speedup values use the CFG-5 GenEval generation pass against the 30-step Euler reference, excluding model loading and scoring.

\subsection{Consistency with the Main Findings}
The key distinction is between \emph{few-step} step-shortening and \emph{step-parity} parallelization. Few-step UniPC and DPMSolver++ change the trajectory discretization. In the FLUX audit, their 16-step configurations lose $21\%$--$24\%$ of the relative GenEval score, $6.5$--$8.2$ DPG-Bench points, and $67\%$--$70\%$ of the LongTextBench score relative to the 30-step Euler reference. ParaLin++ keeps the same 30 discrete steps, reports $16.37$ average parallel solver iterations, and retains an Euler-level quality profile while achieving a $1.68\times$ speedup. On Ming-Flash-Omni 2.0, ParaLin++ uses $24.69$ average solver iterations and obtains a $1.21\times$ speedup with GenEval, DPG-Bench, and LongTextBench scores matching the 30-step Euler reference within the reported measurement setting. The Stable Diffusion-v2 and LSUN Church experiments show comparable FID/CLIP scores with speedups of $2.6\times$--$4.8\times$ across the main multi-GPU settings. These supplementary details therefore use the same terminology and protocol as the main manuscript: discrete steps, solver iterations, distributional/alignment metrics, and generation-quality metrics are reported separately rather than conflated.

\section{More Visual Results}
This section presents additional qualitative outputs for ParaLin++ wrappers of DDPM, DDIM, and DPM-Solver on Stable Diffusion-v2. We randomly generated 10 prompts with different styles by DeepSeek-v3 for each method. These conventional Stable Diffusion-v2 comparisons complement the benchmark-specific generation-quality visual comparisons in the main manuscript.
The results for DDPM, DDIM and DPMSolver are shown in
Figures~\ref{fig:CMP_10_images_SD_DDPM},~\ref{fig:CMP_10_images_SD_DDIM}, and~\ref{fig:CMP_10_images_SD_DPMSolver}, respectively.

\begin{figure*}[htbp]
    \centering

    \begin{minipage}{\linewidth}
        \centering
        \parbox[b][0.094\linewidth][t]{0.05\linewidth}{\rotatebox{90}{\small\textbf{DDIM}}}%
        \includegraphics[width=0.094\linewidth]{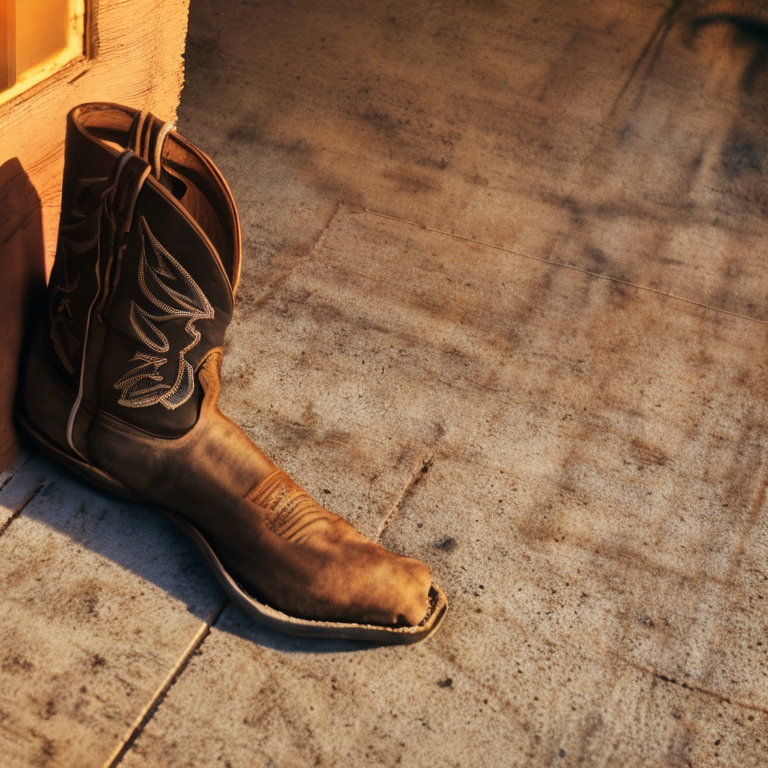}
        \includegraphics[width=0.094\linewidth]{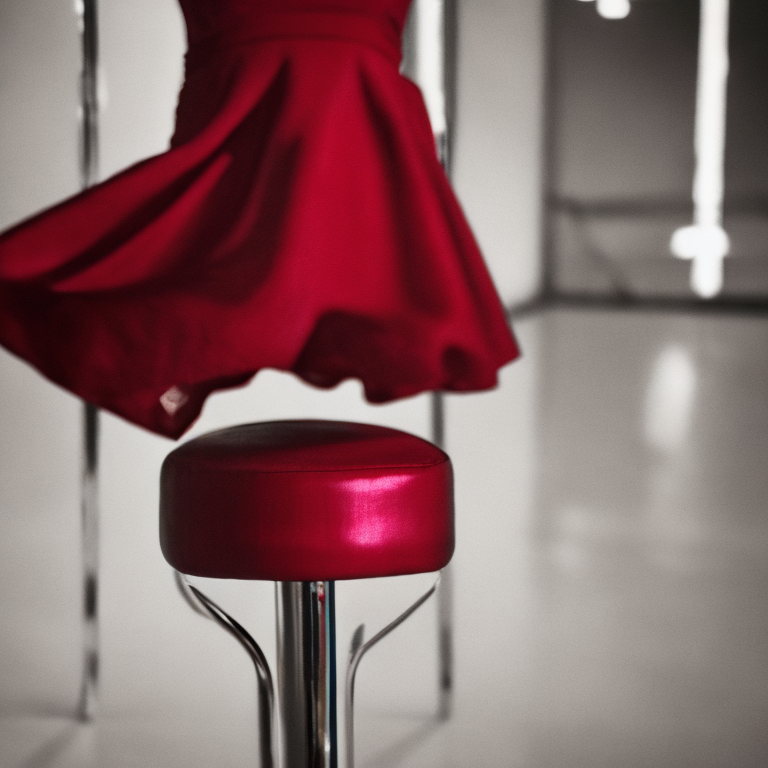}%
        \includegraphics[width=0.094\linewidth]{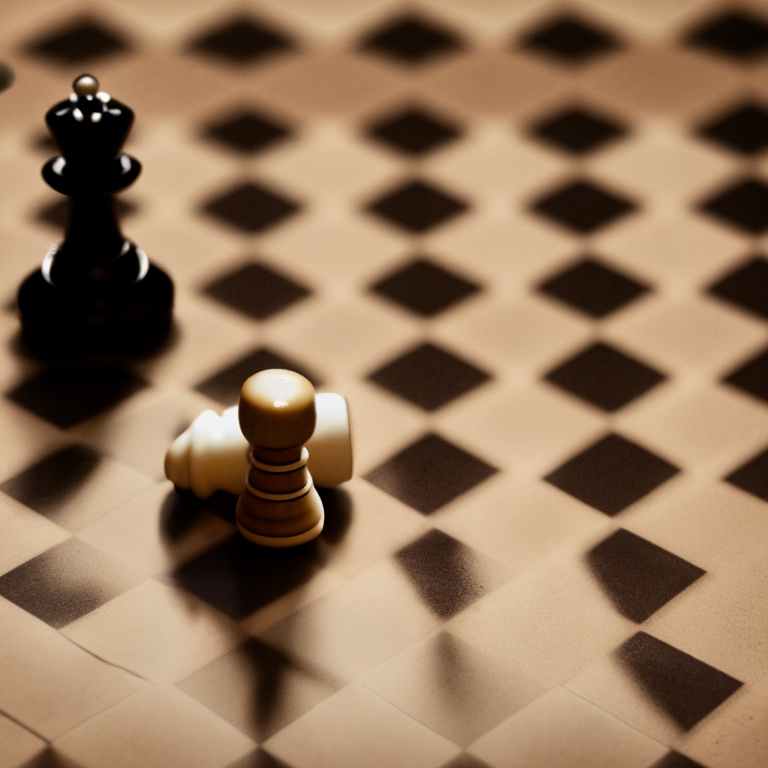}%
        \includegraphics[width=0.094\linewidth]{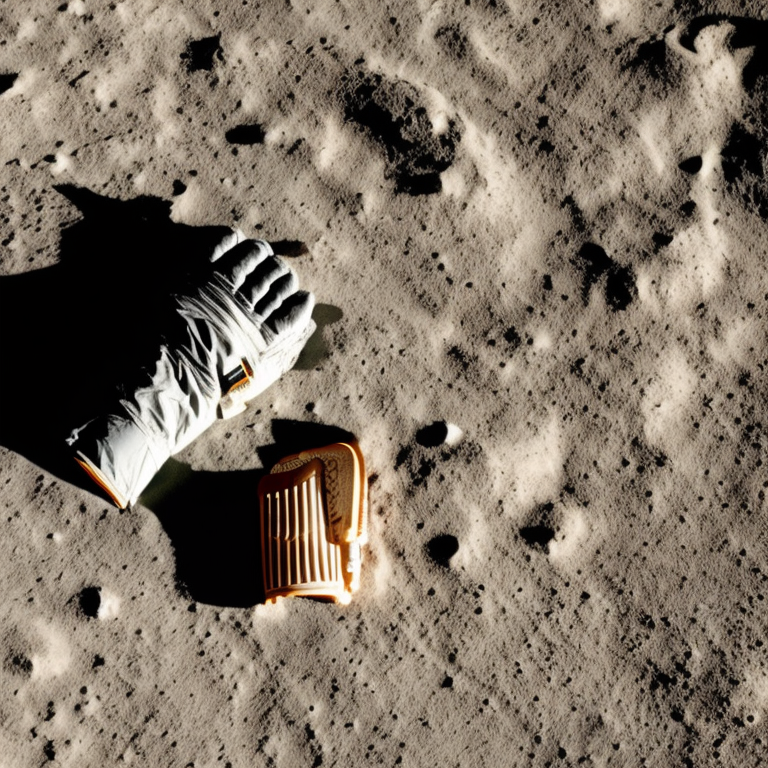}%
        \includegraphics[width=0.094\linewidth]{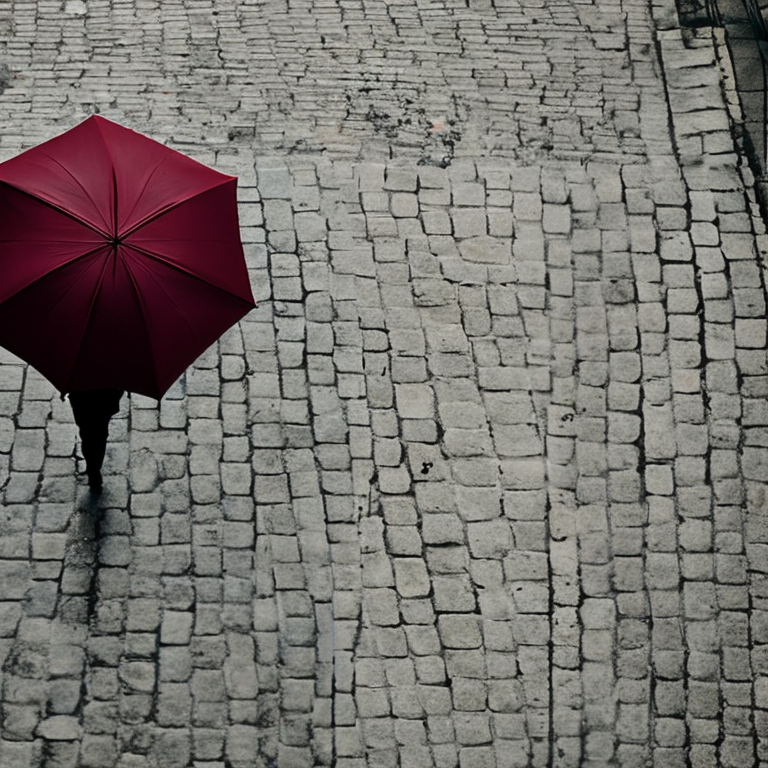}%
        \includegraphics[width=0.094\linewidth]{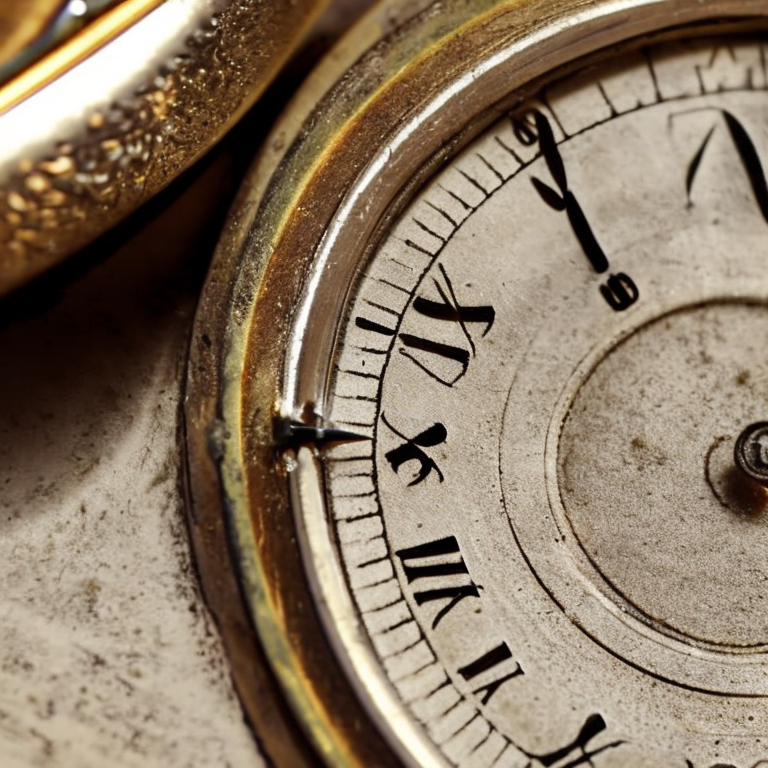}%
        \includegraphics[width=0.094\linewidth]{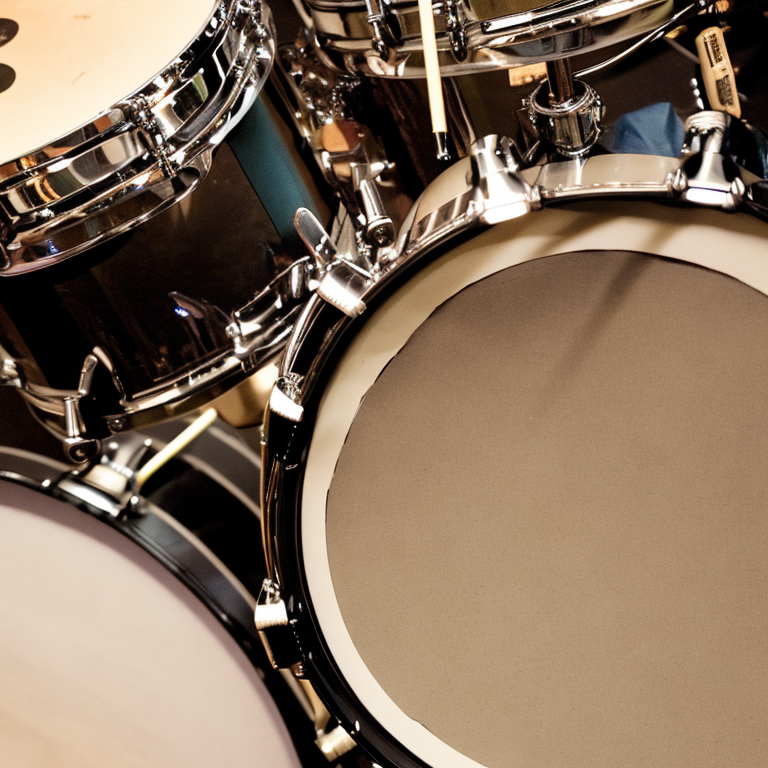}%
        \includegraphics[width=0.094\linewidth]{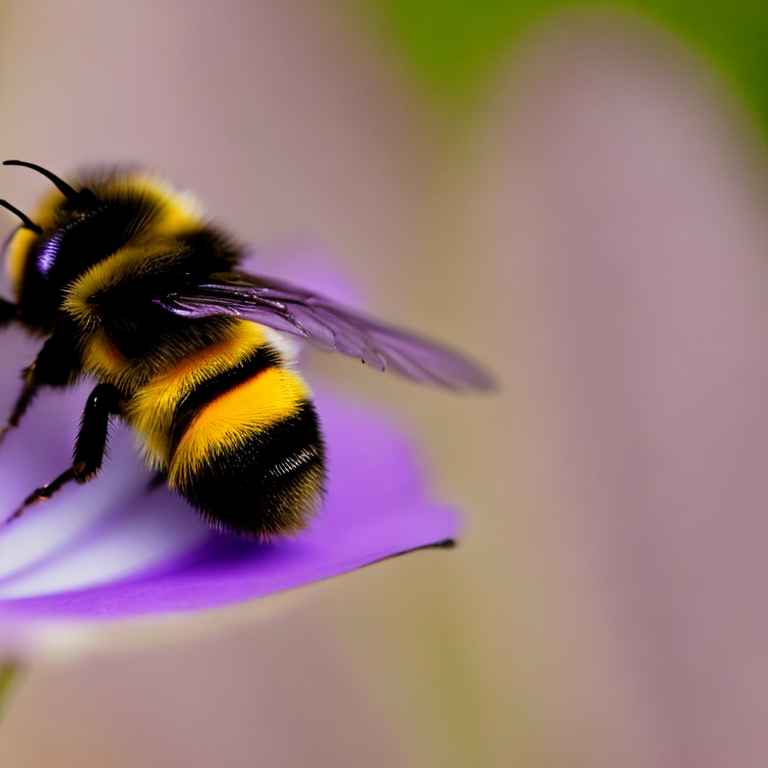}%
        \includegraphics[width=0.094\linewidth]{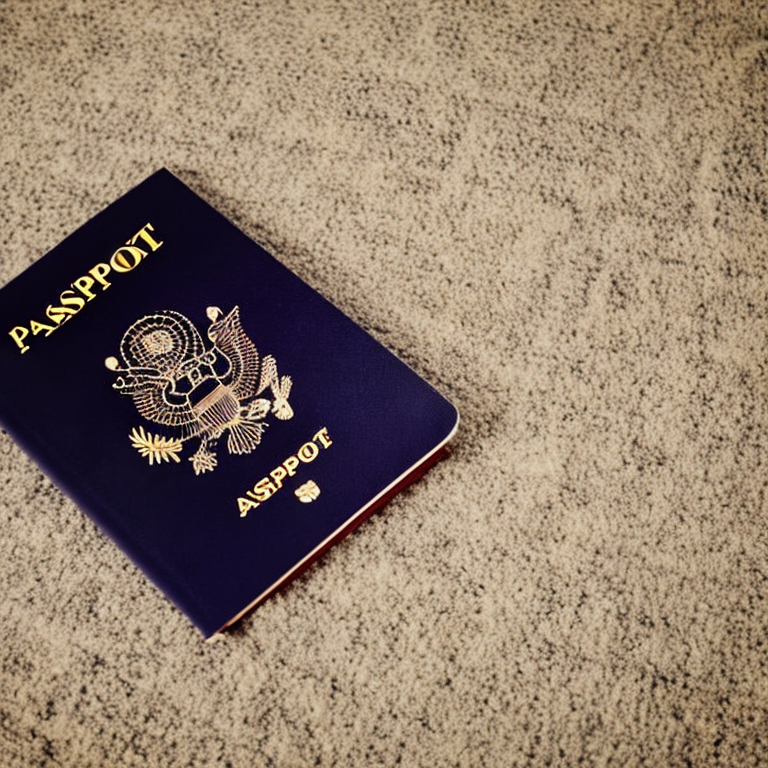}%
        \includegraphics[width=0.094\linewidth]{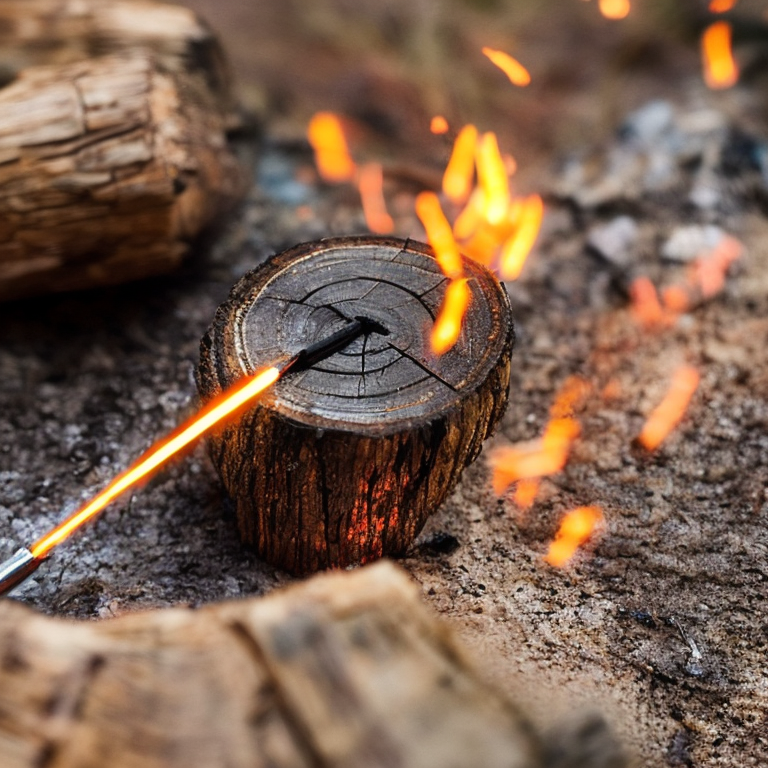}
    \end{minipage}
    \vspace{1mm} 

    \begin{minipage}{\linewidth}
        \centering
        \parbox[b][0.094\linewidth][t]{0.05\linewidth}{\rotatebox{90}{\small\textbf{ParaDiGMS}}}%
        \includegraphics[width=0.094\linewidth]{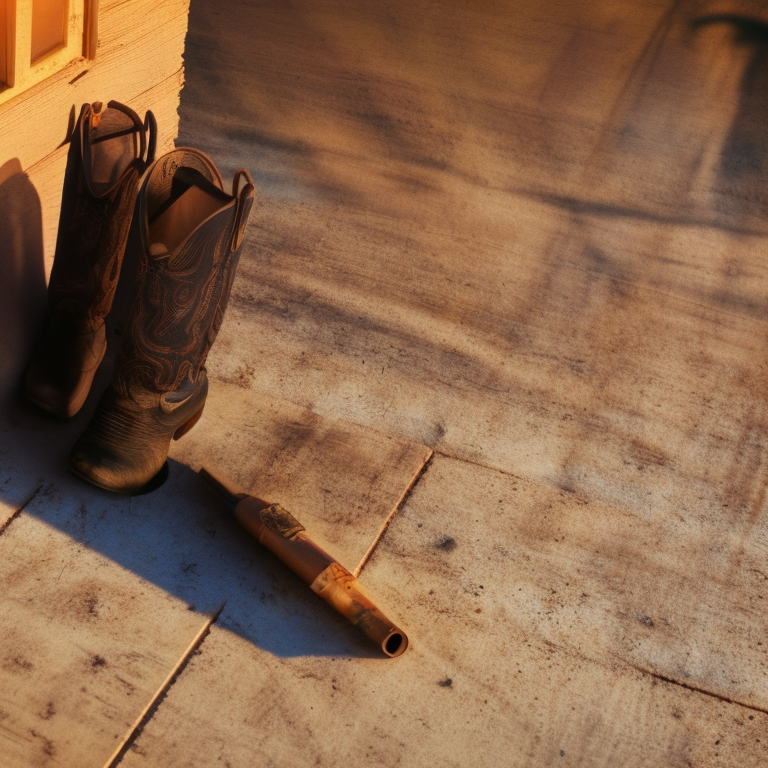}%
        \includegraphics[width=0.094\linewidth]{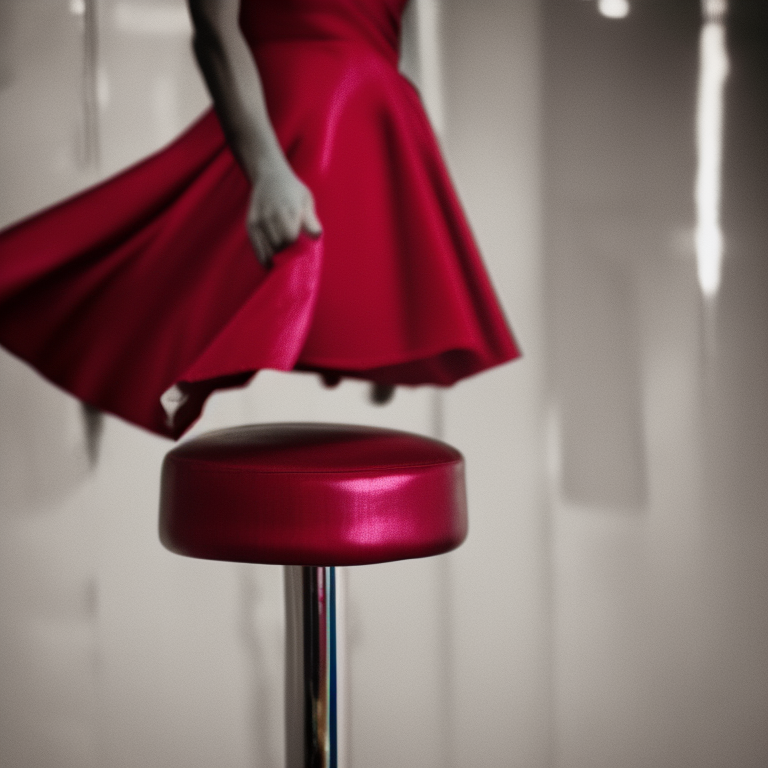}%
        \includegraphics[width=0.094\linewidth]{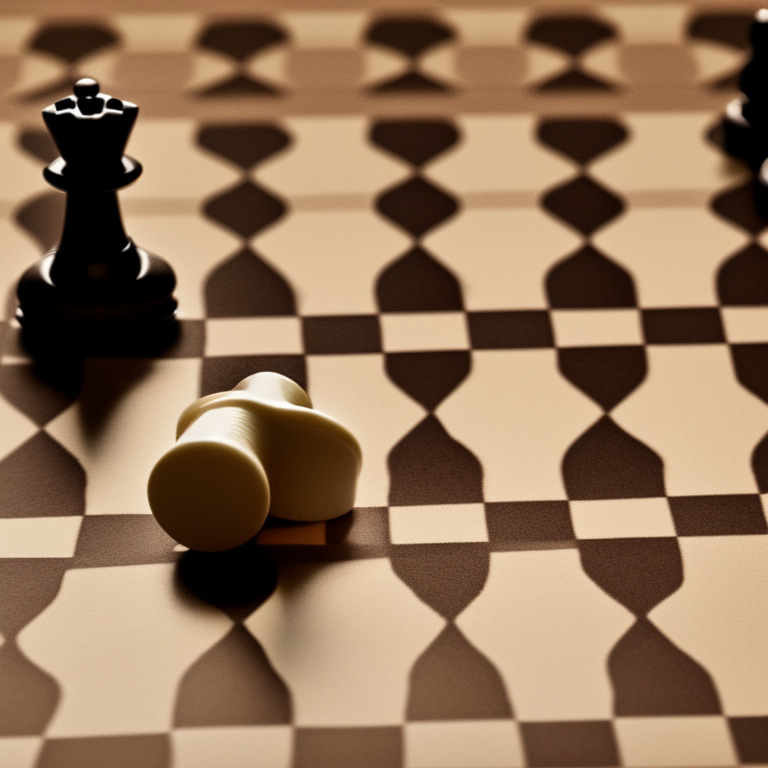}%
        \includegraphics[width=0.094\linewidth]{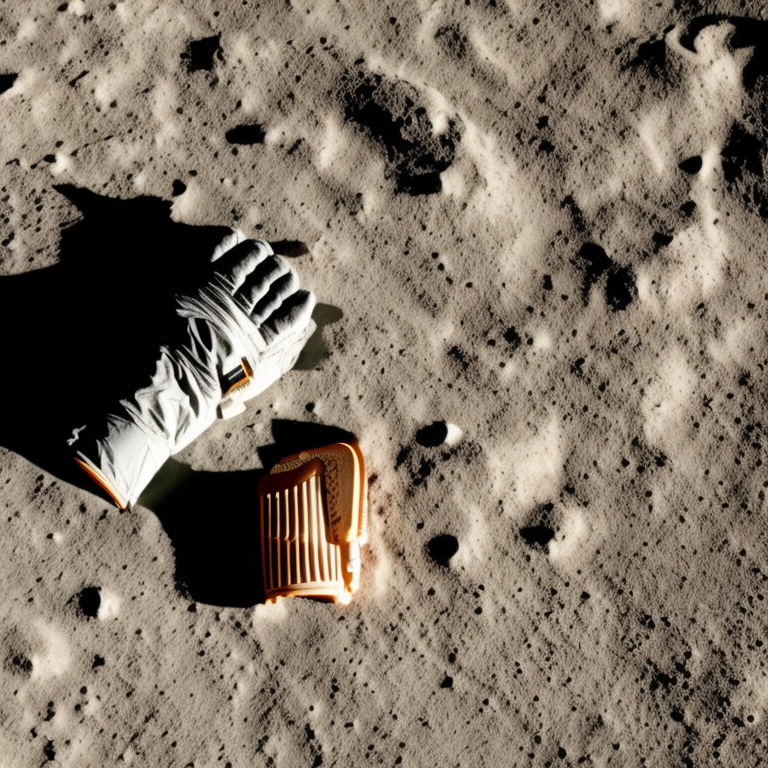}%
        \includegraphics[width=0.094\linewidth]{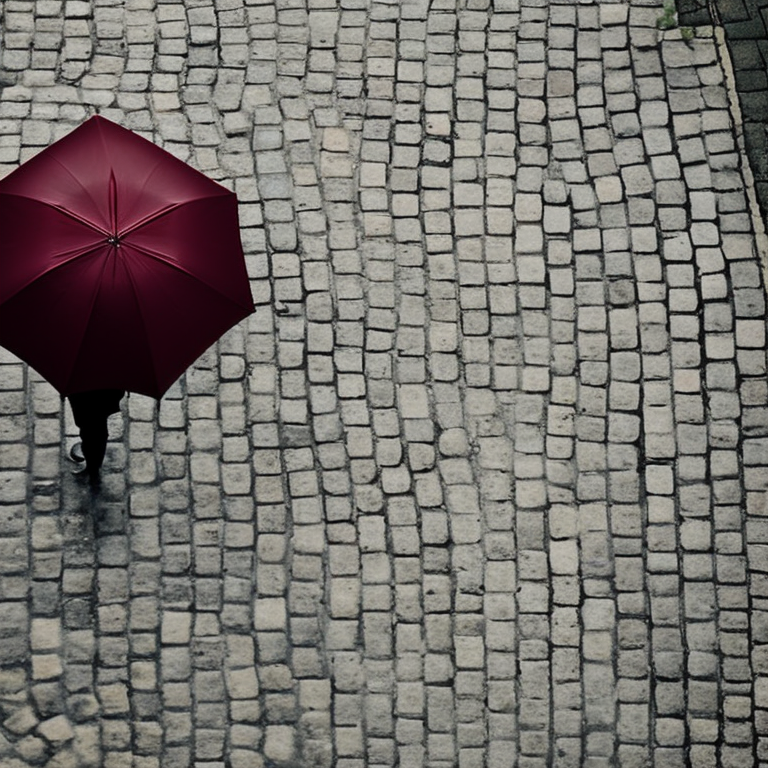}%
        \includegraphics[width=0.094\linewidth]{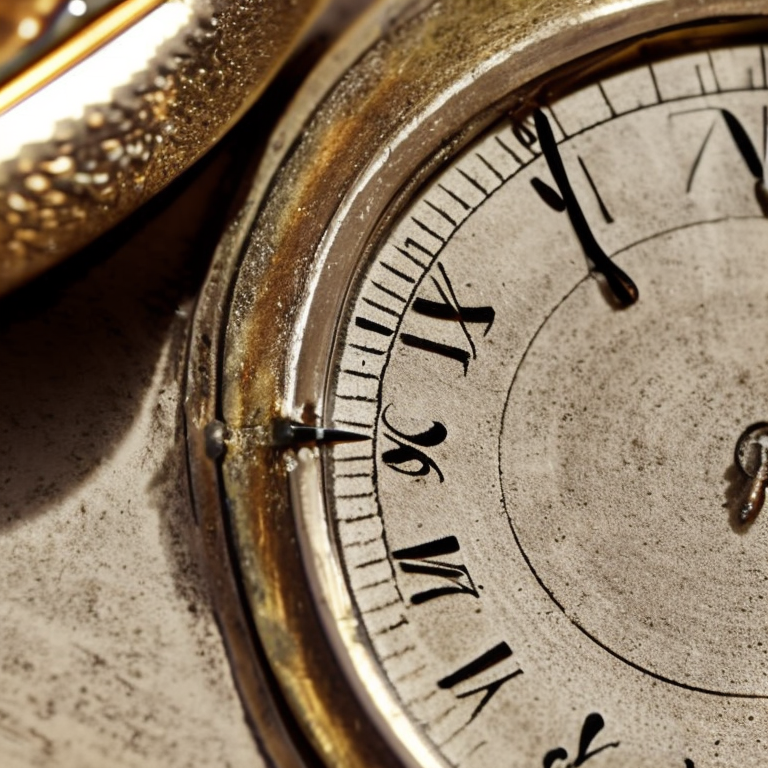}%
        \includegraphics[width=0.094\linewidth]{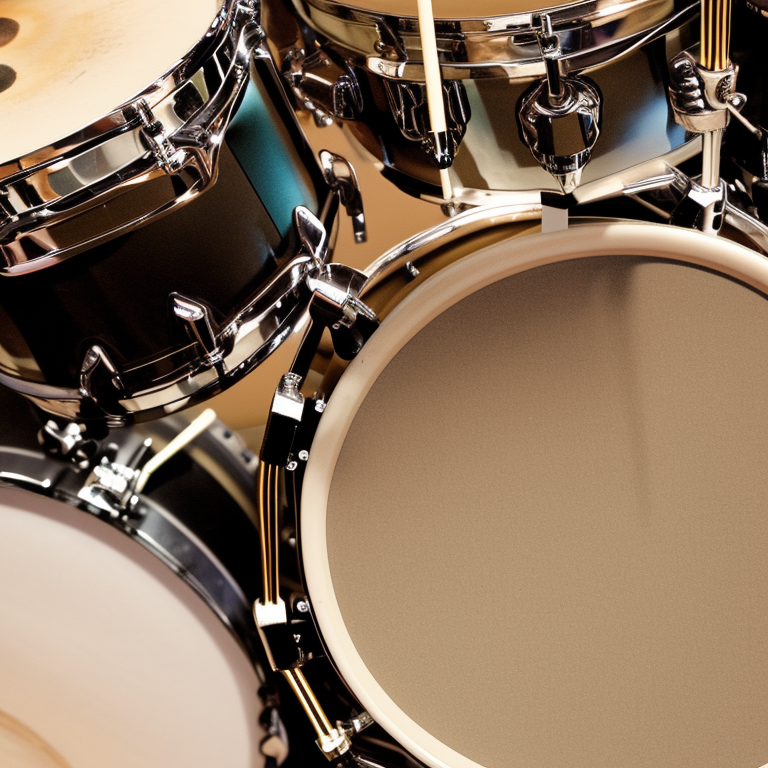}%
        \includegraphics[width=0.094\linewidth]{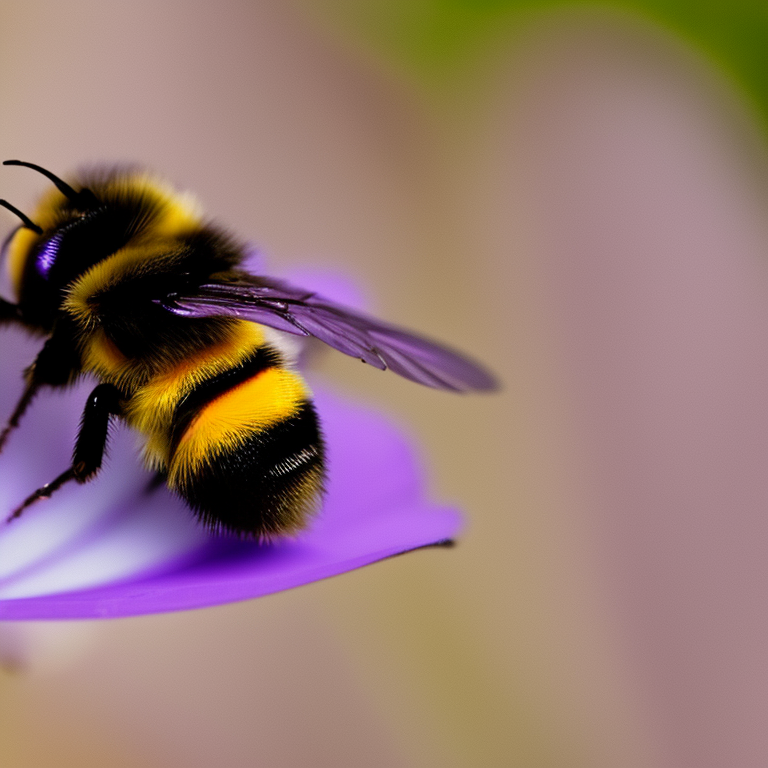}%
        \includegraphics[width=0.094\linewidth]{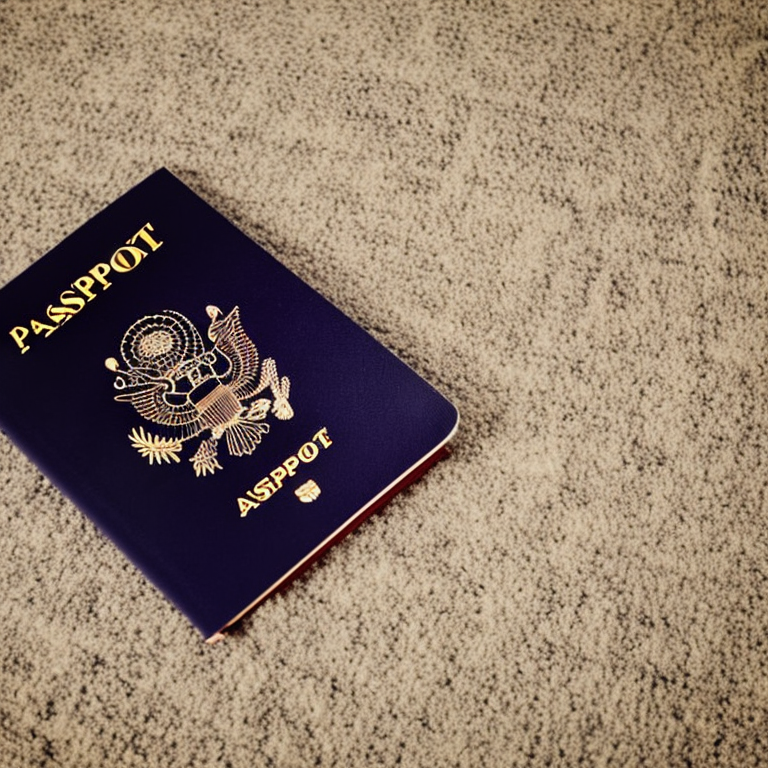}%
        \includegraphics[width=0.094\linewidth]{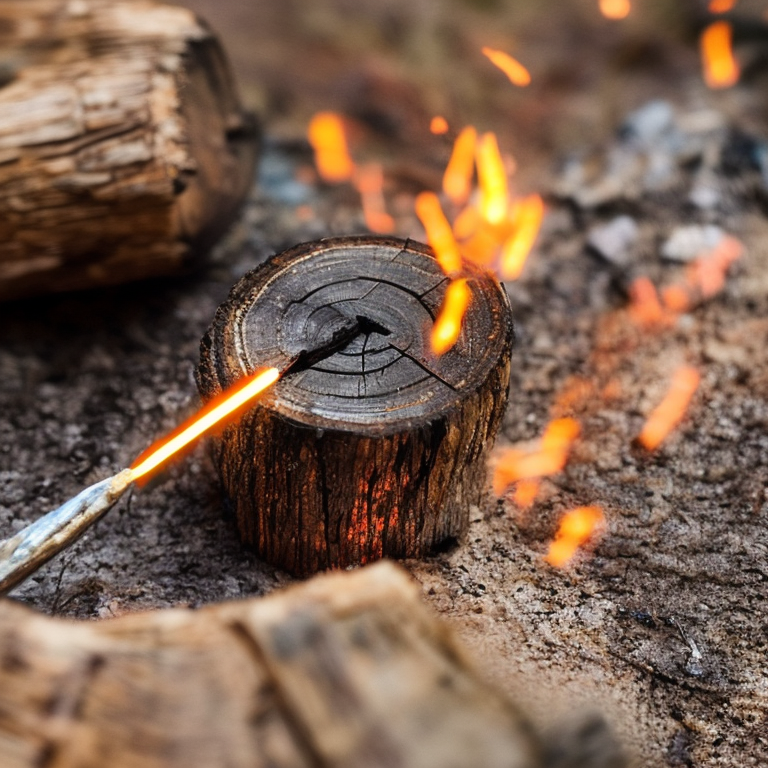}
    \end{minipage}
    \vspace{1mm} 

    \begin{minipage}{\linewidth}
        \centering
        \parbox[b][0.094\linewidth][t]{0.05\linewidth}{\rotatebox{90}{\small\textbf{ParaSolver}}}%
        \includegraphics[width=0.094\linewidth]{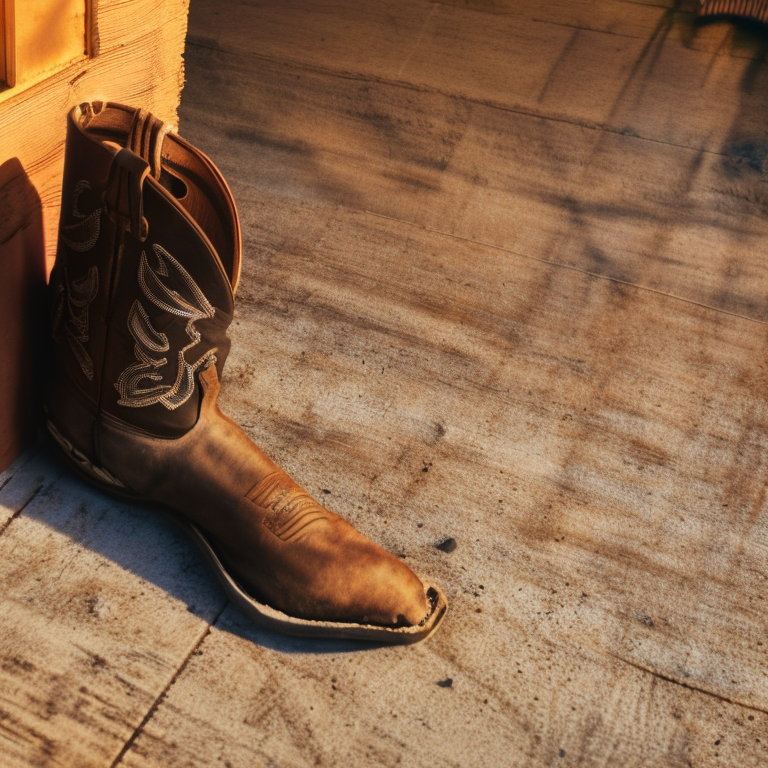}%
        \includegraphics[width=0.094\linewidth]{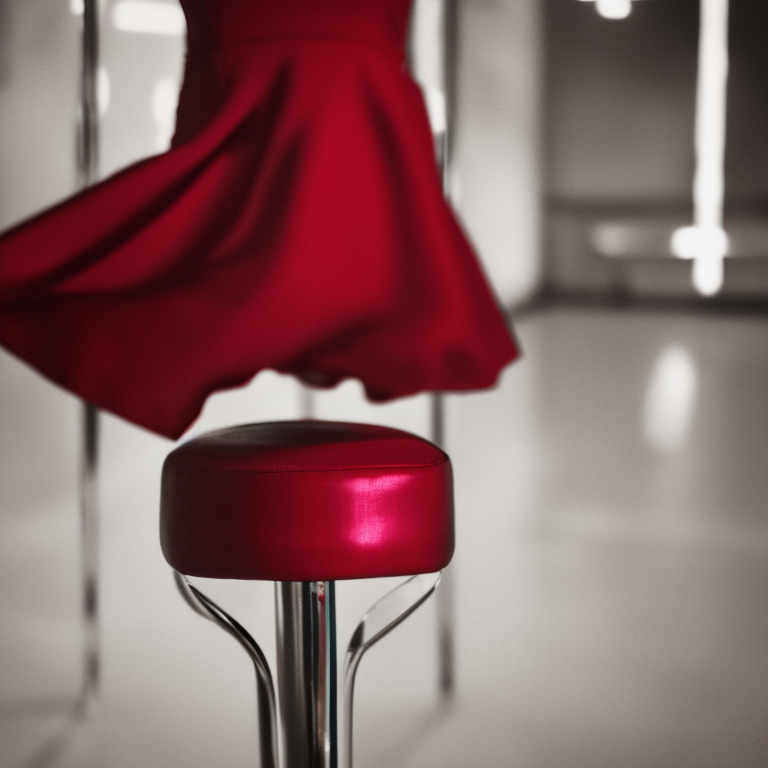}%
        \includegraphics[width=0.094\linewidth]{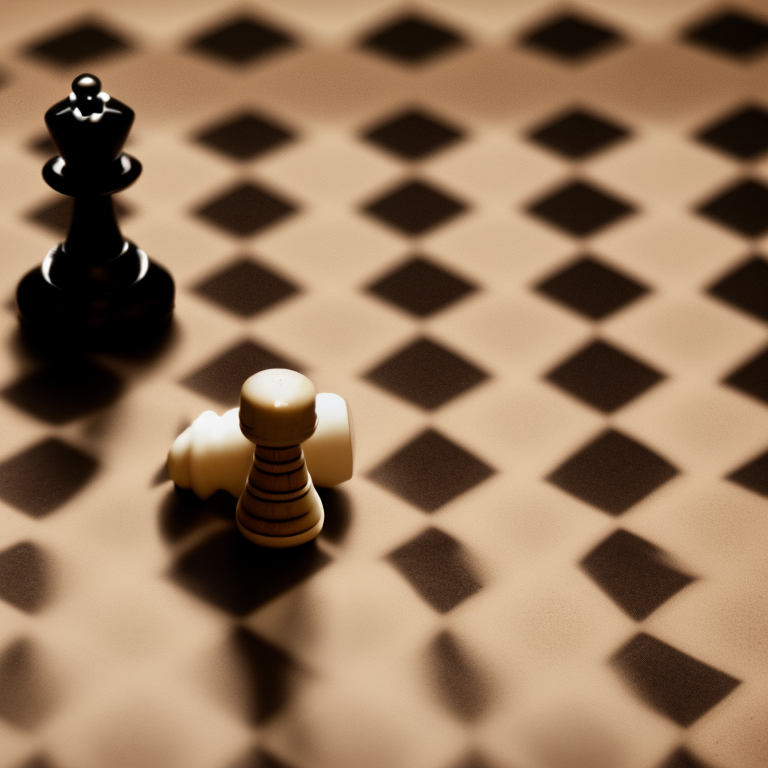}%
        \includegraphics[width=0.094\linewidth]{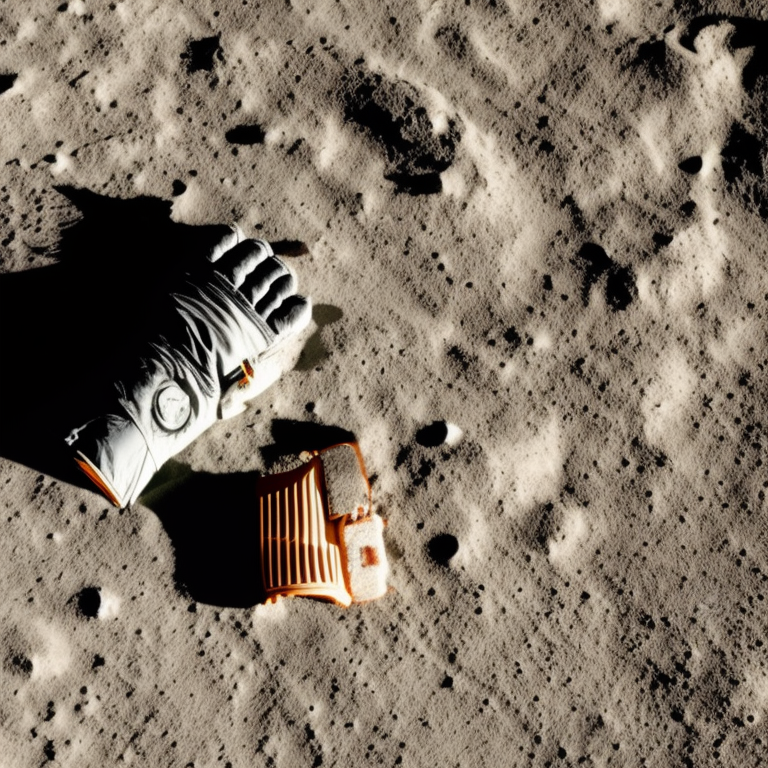}%
        \includegraphics[width=0.094\linewidth]{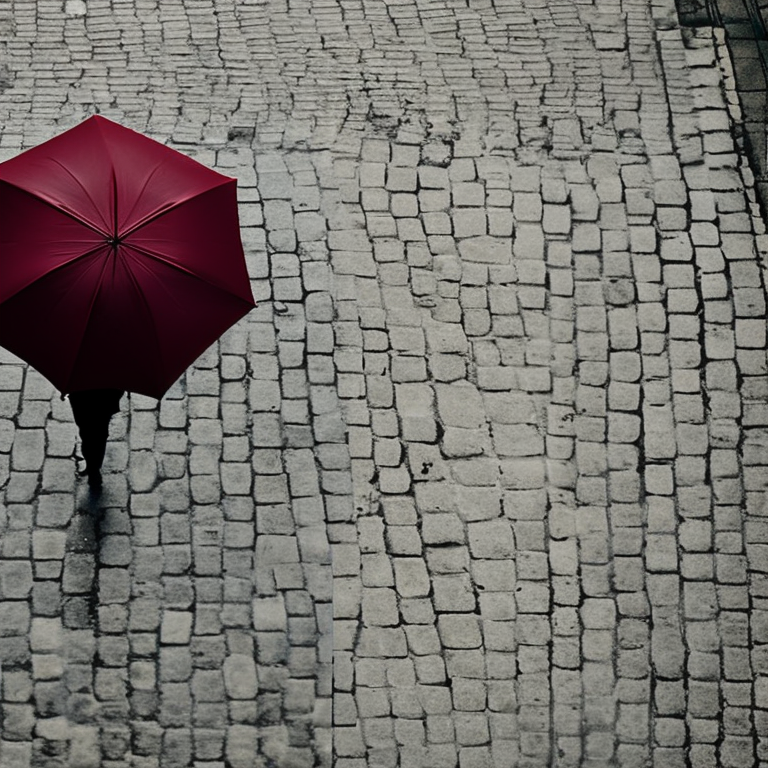}%
        \includegraphics[width=0.094\linewidth]{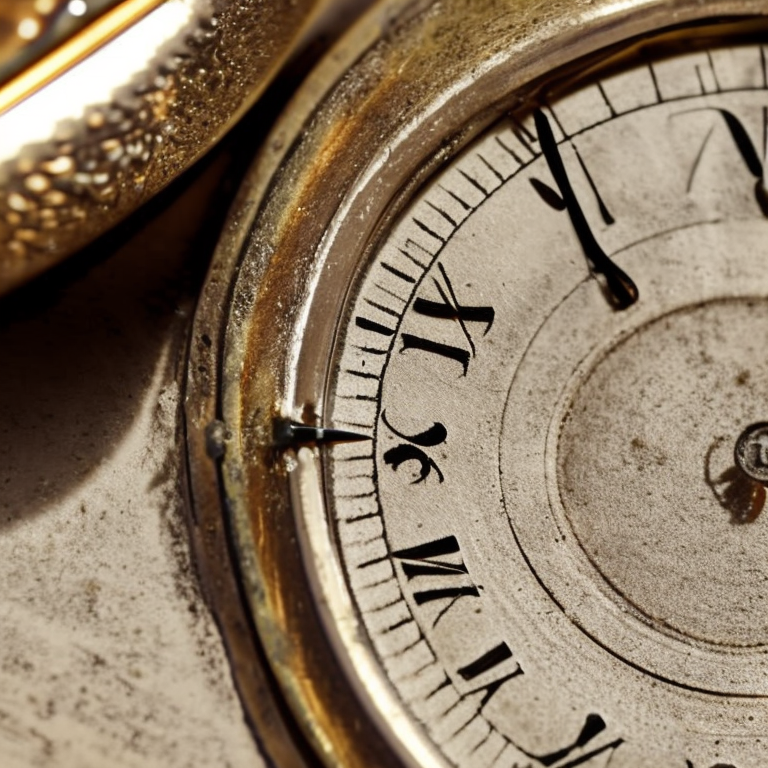}%
        \includegraphics[width=0.094\linewidth]{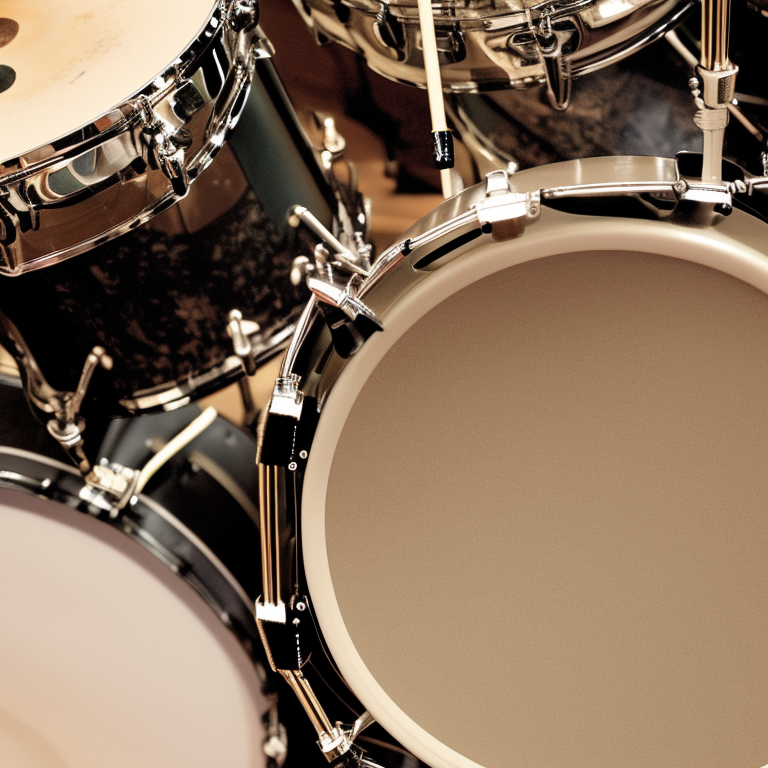}%
        \includegraphics[width=0.094\linewidth]{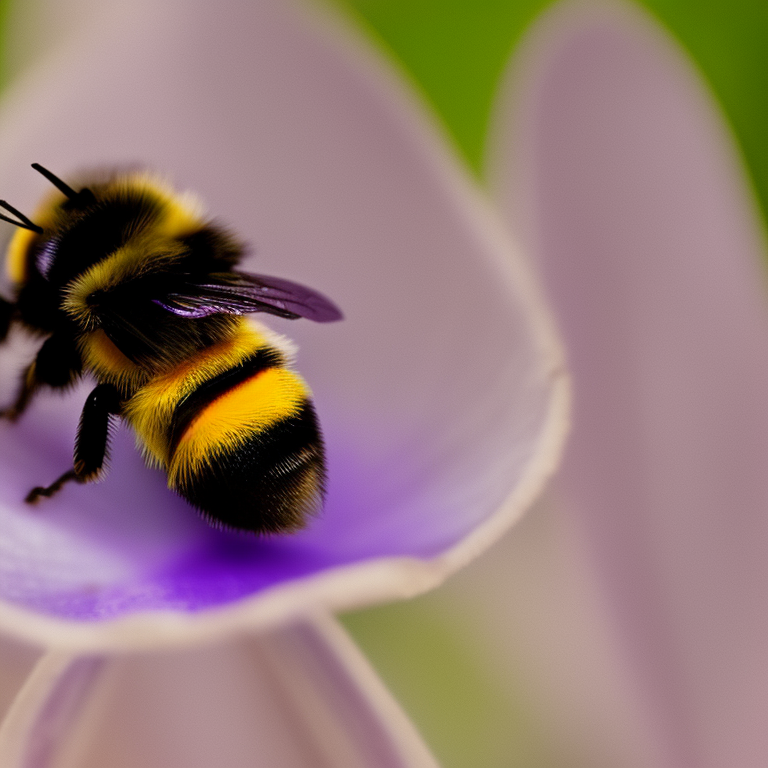}%
        \includegraphics[width=0.094\linewidth]{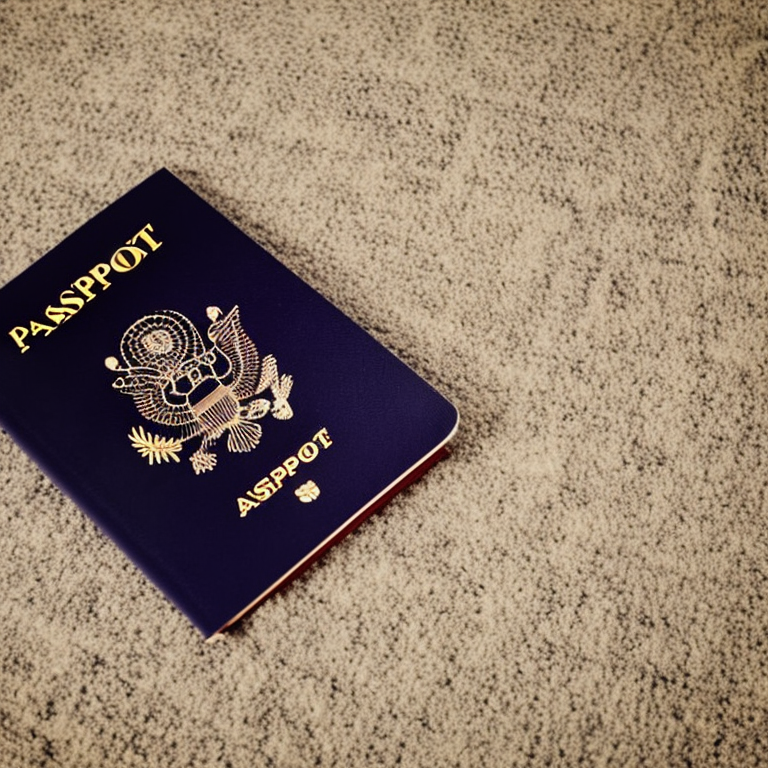}%
        \includegraphics[width=0.094\linewidth]{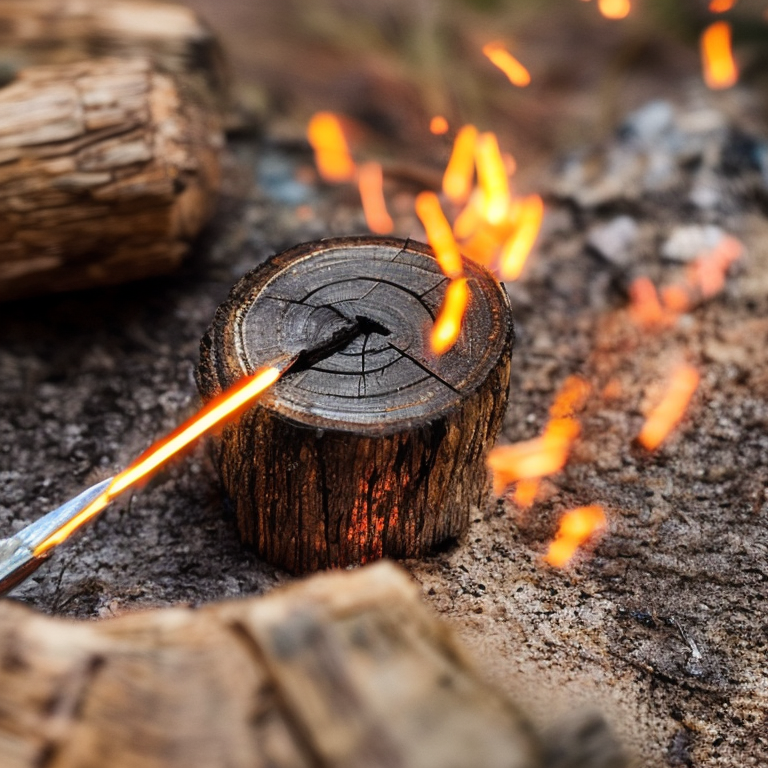}
    \end{minipage}
    \vspace{1mm} 

    \begin{minipage}{\linewidth}
        \centering
        \parbox[b][0.094\linewidth][t]{0.05\linewidth}{\rotatebox{90}{\small\textbf{ParaLin}}}%
        \includegraphics[width=0.094\linewidth]{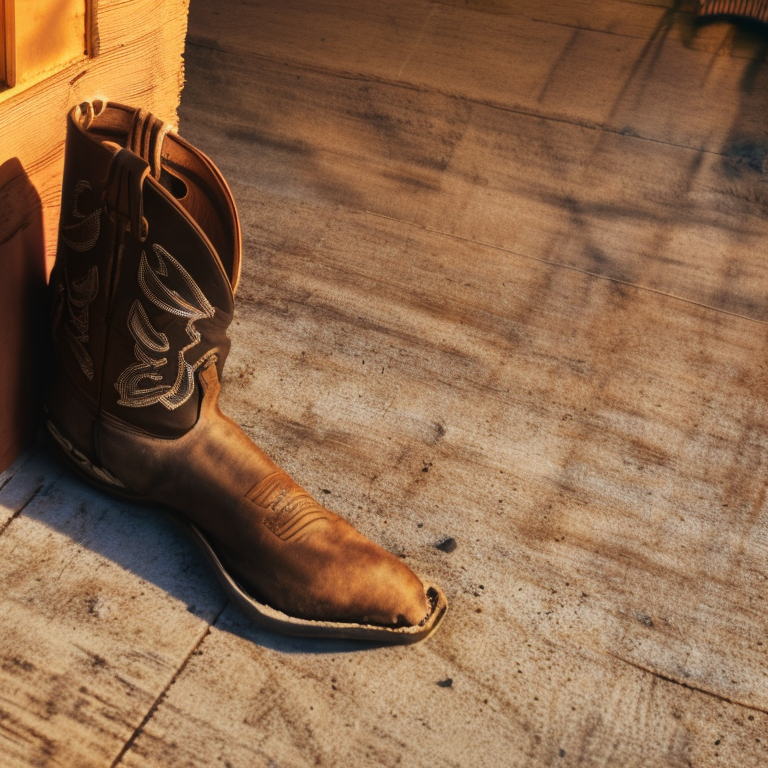}%
        \includegraphics[width=0.094\linewidth]{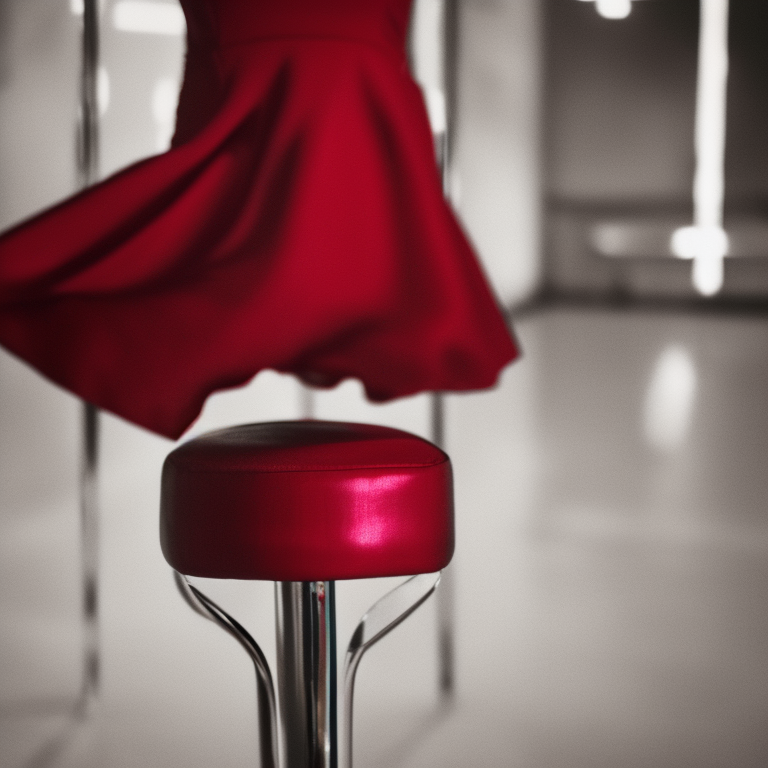}%
        \includegraphics[width=0.094\linewidth]{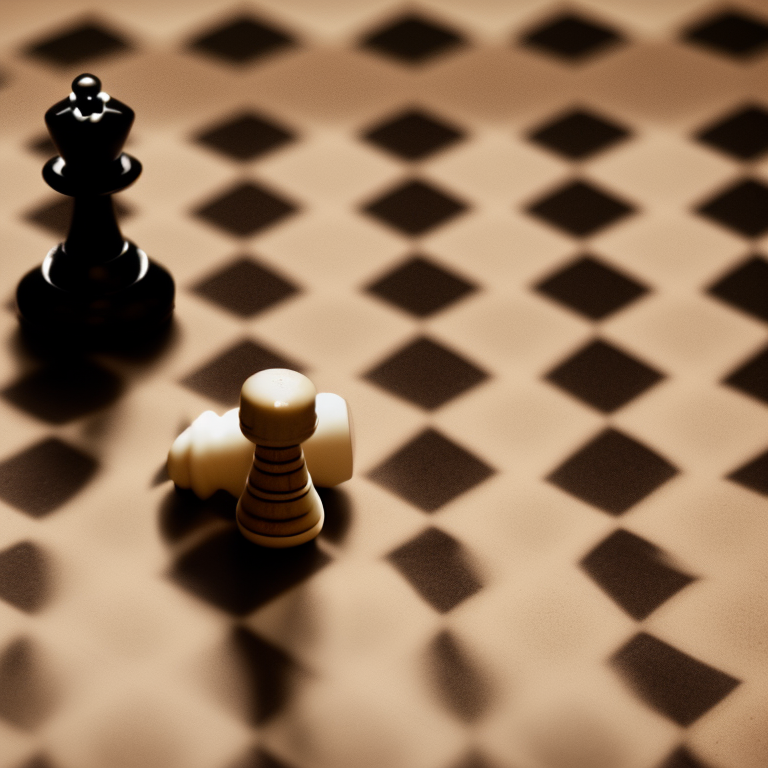}%
        \includegraphics[width=0.094\linewidth]{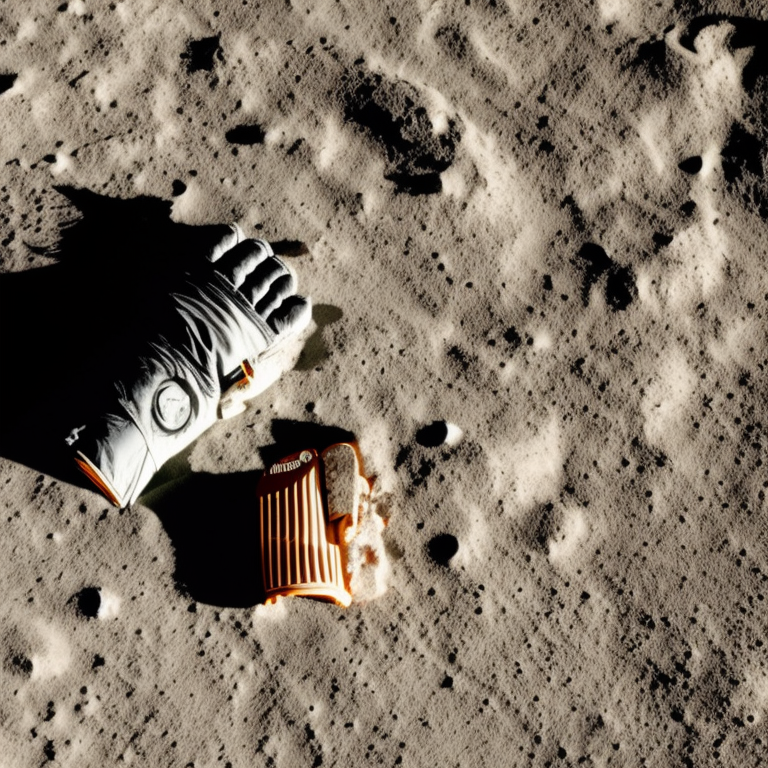}%
        \includegraphics[width=0.094\linewidth]{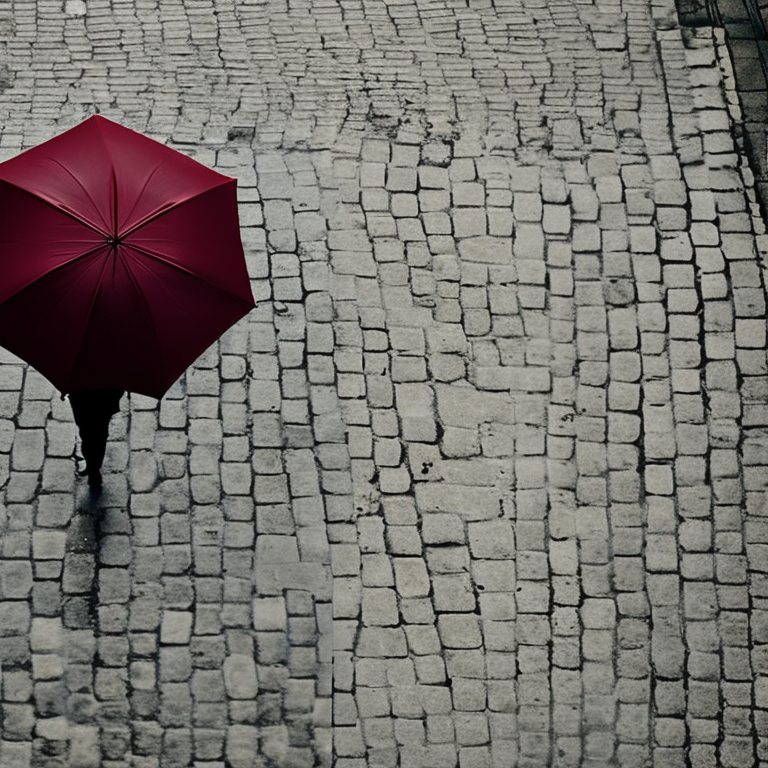}%
        \includegraphics[width=0.094\linewidth]{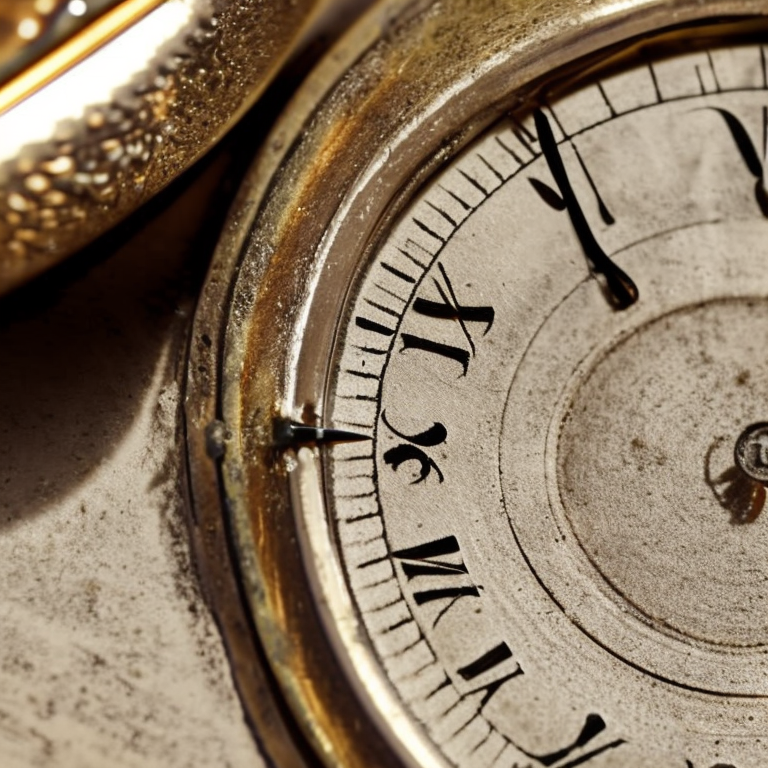}%
        \includegraphics[width=0.094\linewidth]{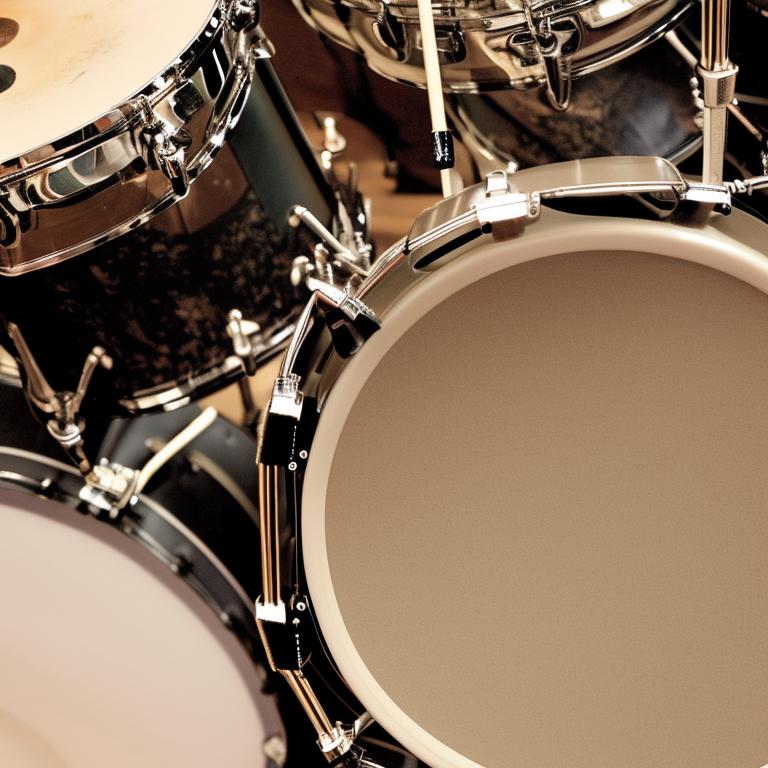}%
        \includegraphics[width=0.094\linewidth]{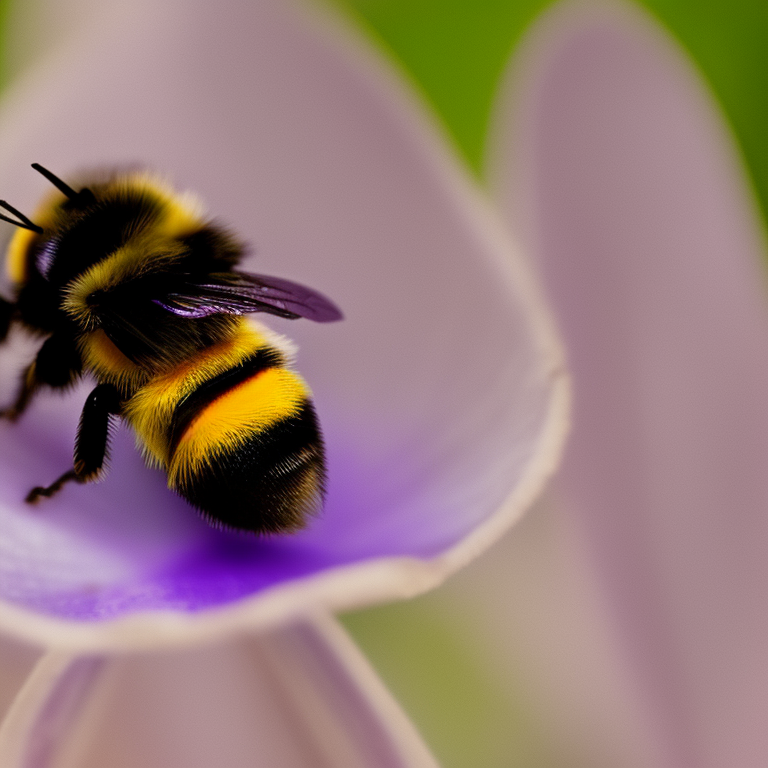}%
        \includegraphics[width=0.094\linewidth]{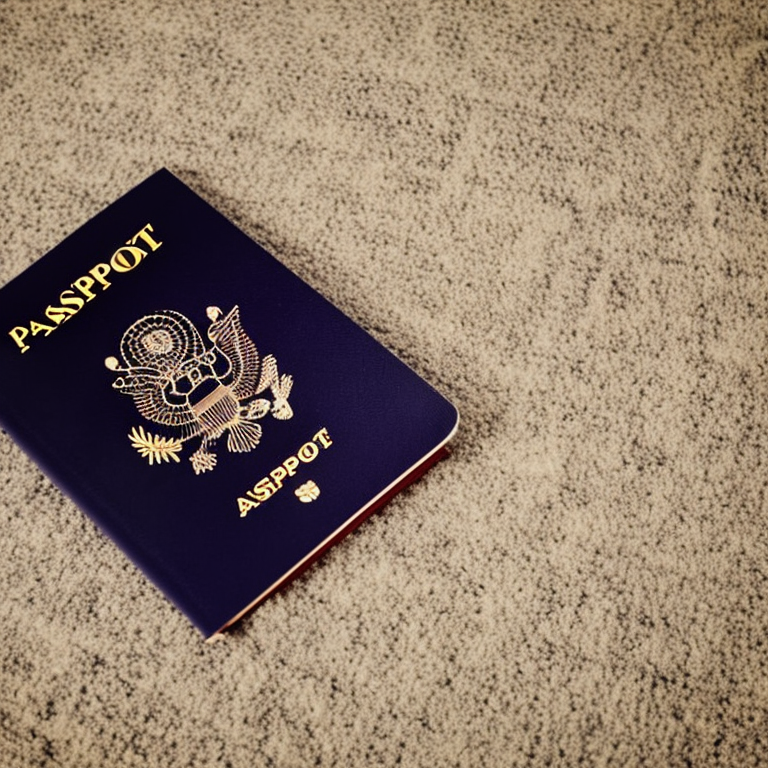}%
        \includegraphics[width=0.094\linewidth]{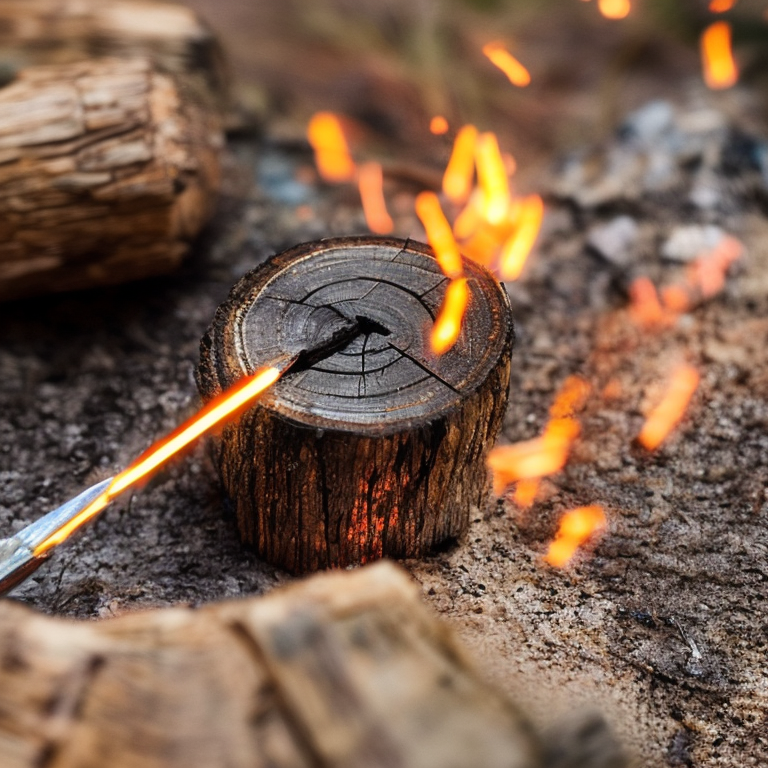}
    \end{minipage}
    \vspace{1mm} 

    \begin{minipage}{\linewidth}
        \centering
        \parbox[c][0.094\linewidth][t]{0.05\linewidth}{\rotatebox{90}{\small\textbf{Ours}}}%
        \includegraphics[width=0.094\linewidth]{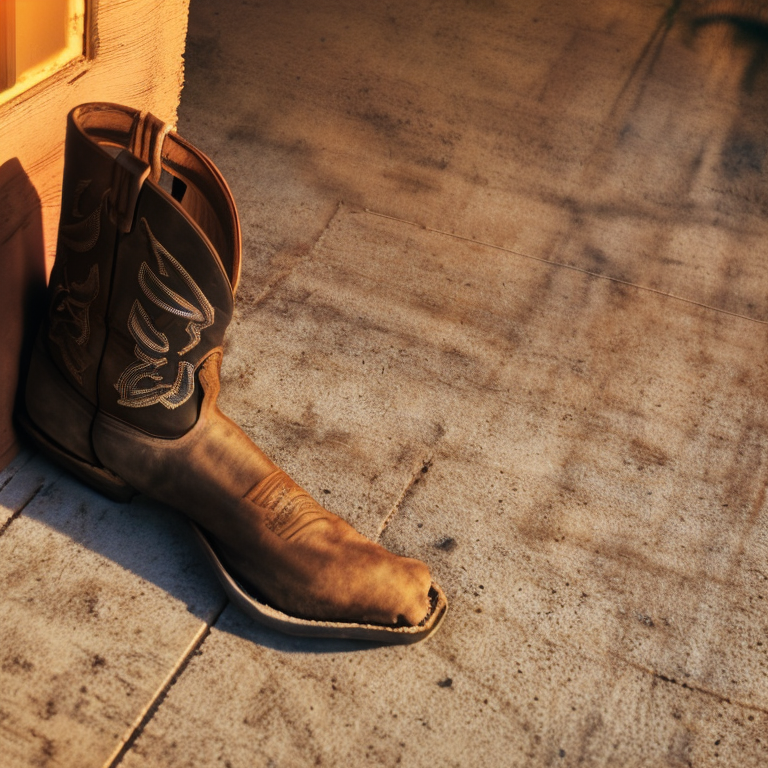}%
        \includegraphics[width=0.094\linewidth]{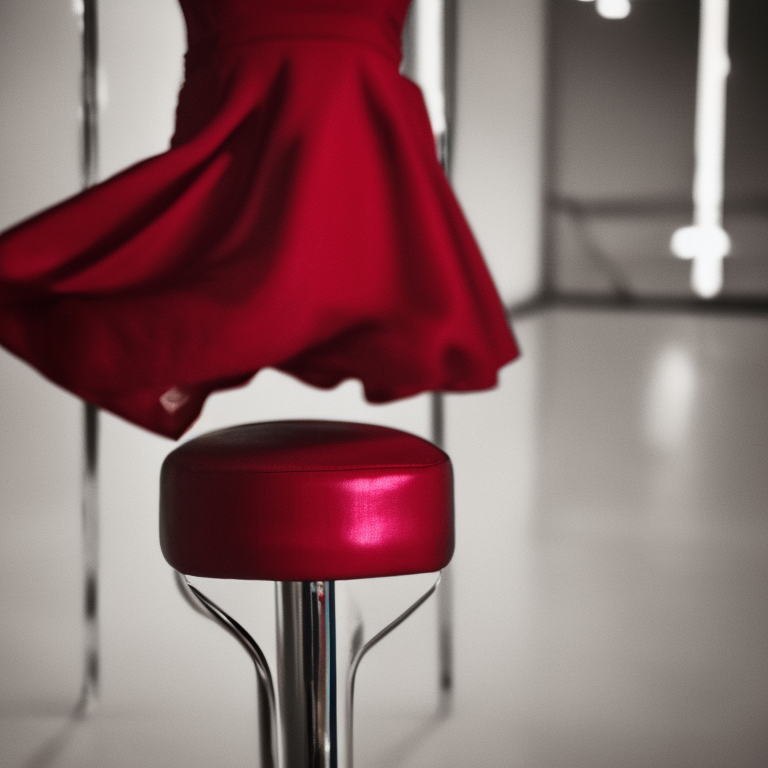}%
        \includegraphics[width=0.094\linewidth]{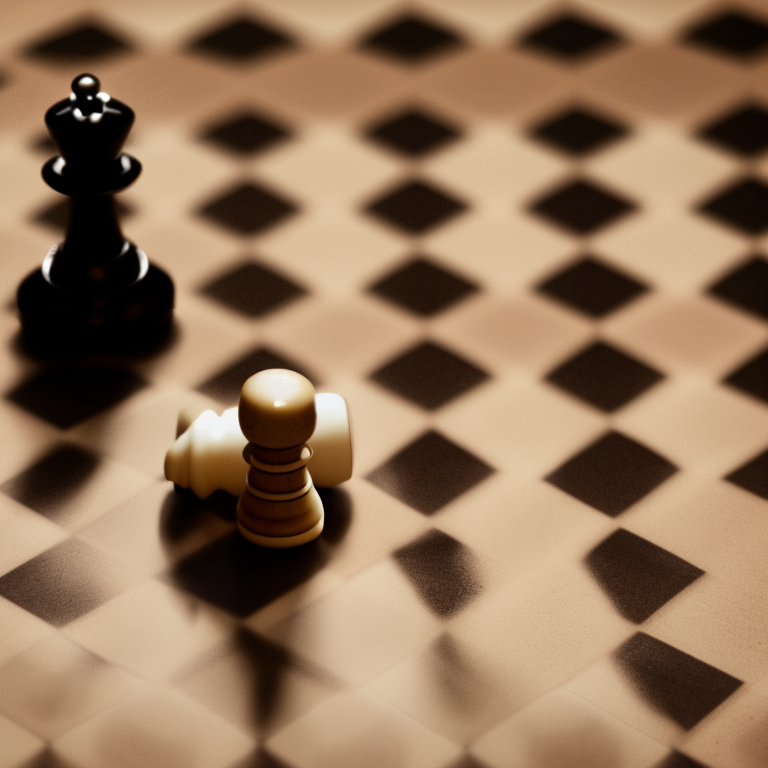}%
        \includegraphics[width=0.094\linewidth]{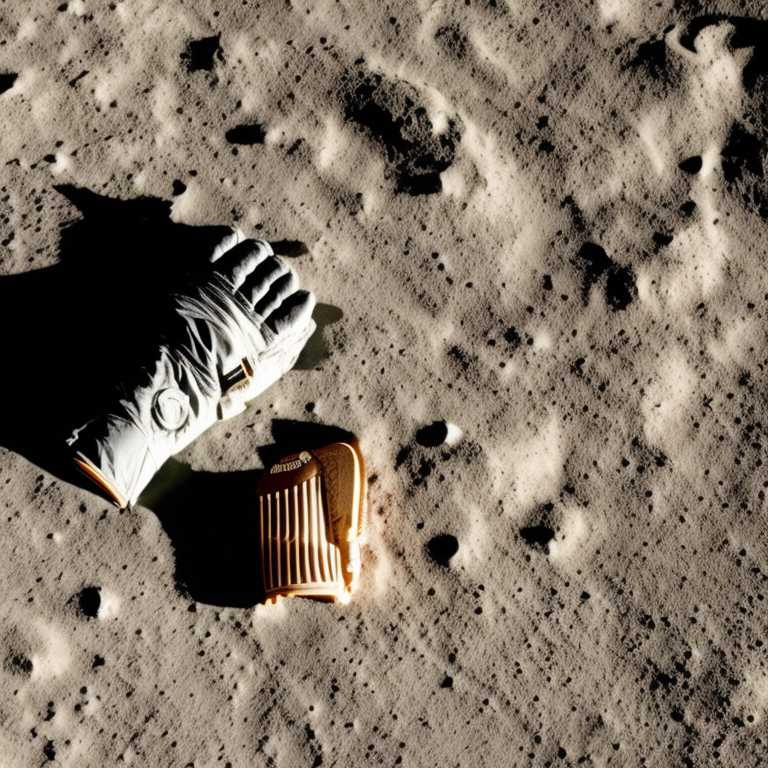}%
        \includegraphics[width=0.094\linewidth]{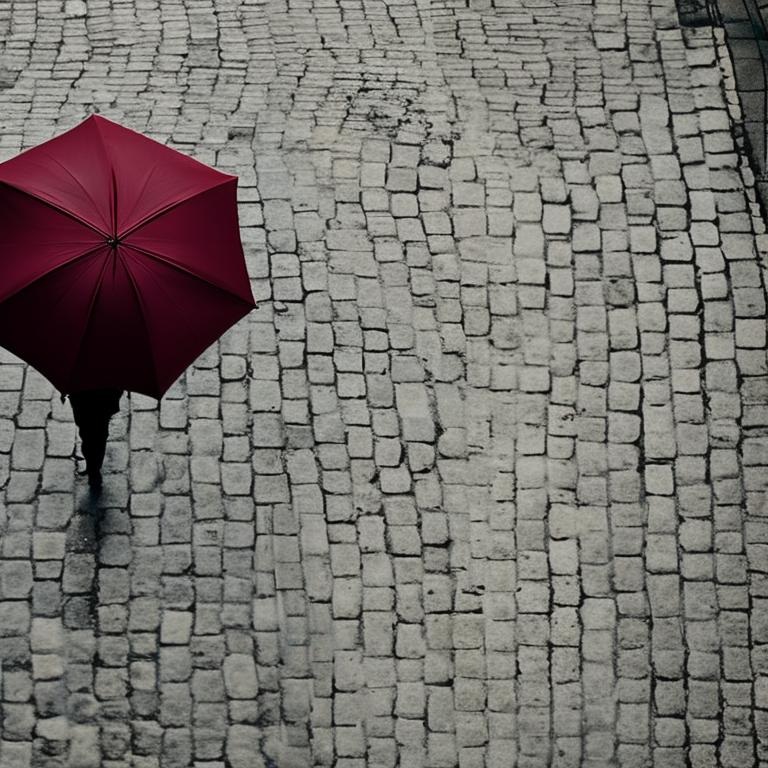}%
        \includegraphics[width=0.094\linewidth]{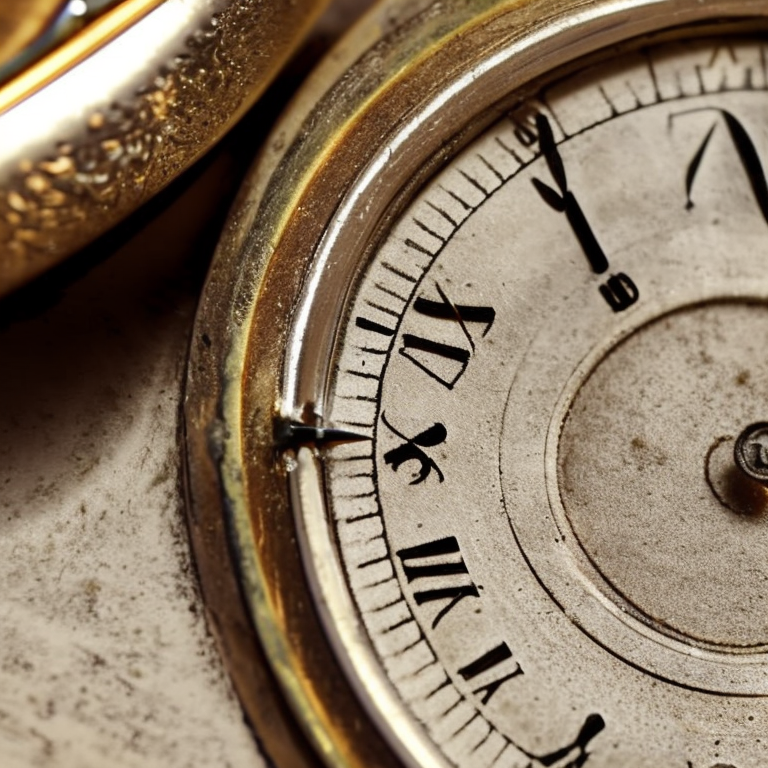}%
        \includegraphics[width=0.094\linewidth]{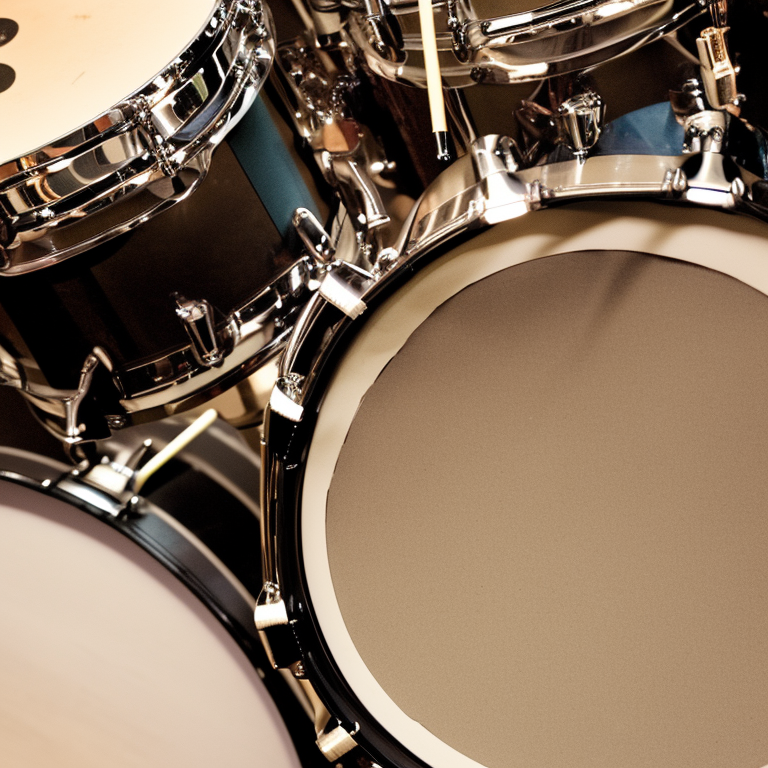}%
        \includegraphics[width=0.094\linewidth]{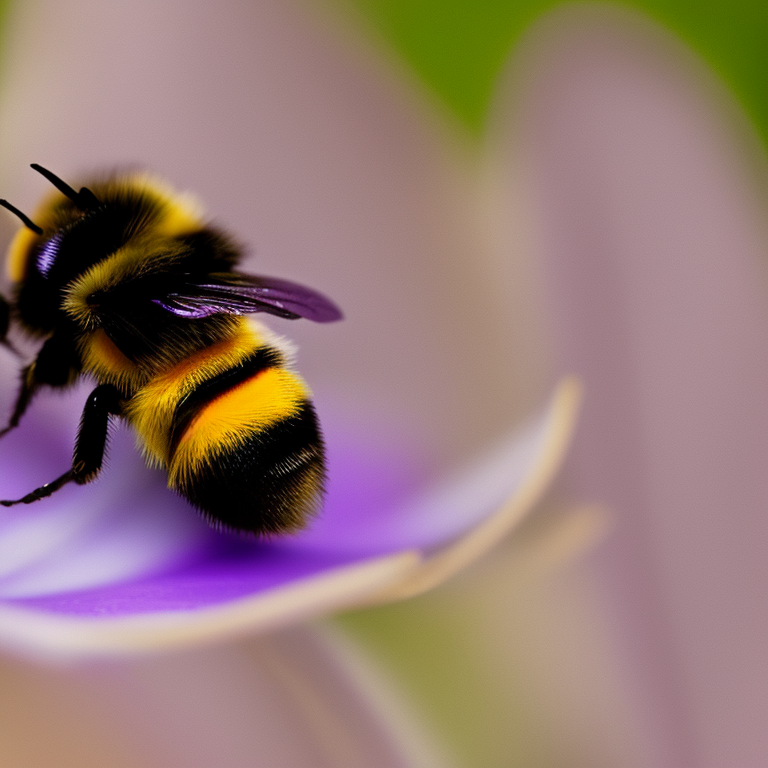}%
        \includegraphics[width=0.094\linewidth]{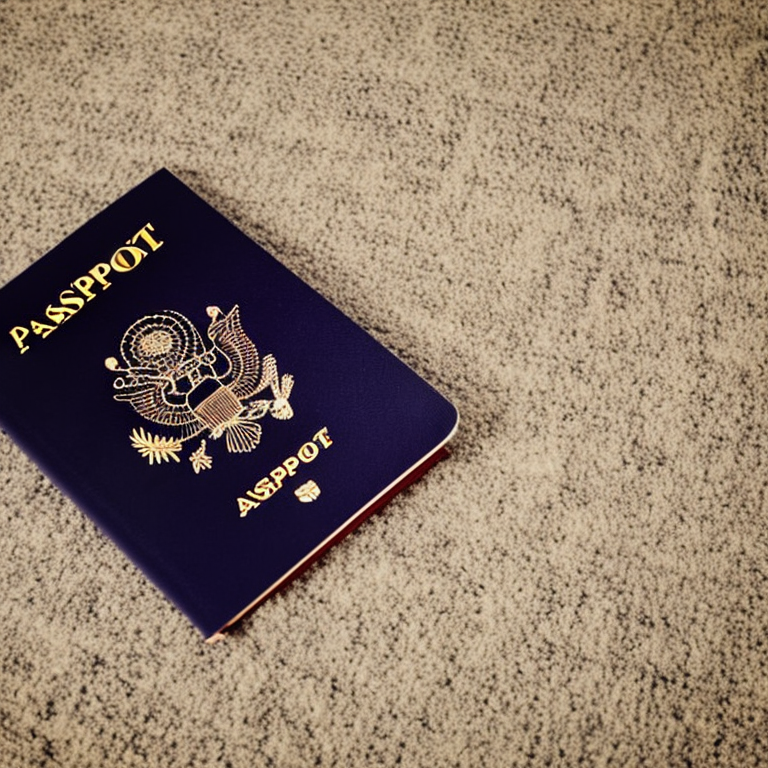}%
        \includegraphics[width=0.094\linewidth]{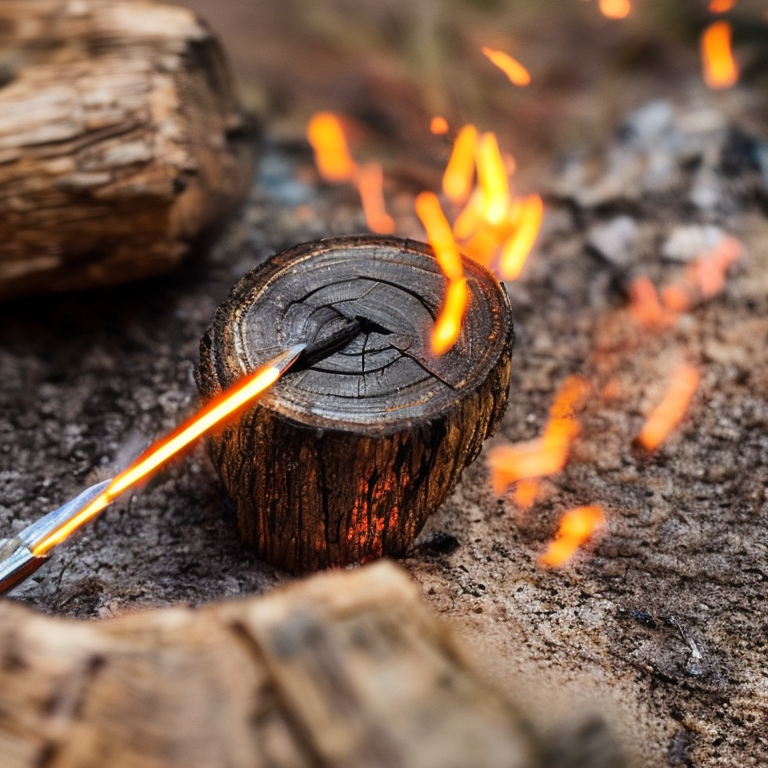}
    \end{minipage}
\vspace{-0.3cm}
    \caption{\label{fig:CMP_10_images_SD_DDIM}
    Visual comparison of final generation results on Stable Diffusion v2, using various samplers to expedite a 50-step DDIM baseline. Each column represents the output for a different prompt, demonstrating fidelity and visual consistency across methods.
    \textbf{Rows (Top to Bottom)}: DDIM (Baseline), ParaDiGMS, ParaSolver, ParaLin, and ParaLin++.
    }
\end{figure*}

\begin{figure*}[t]
    \centering

    \begin{minipage}{\linewidth}
        \centering
        \parbox[b][0.094\linewidth][t]{0.05\linewidth}{\rotatebox{90}{\small\textbf{DPMSolver}}}%
        \includegraphics[width=0.094\linewidth]{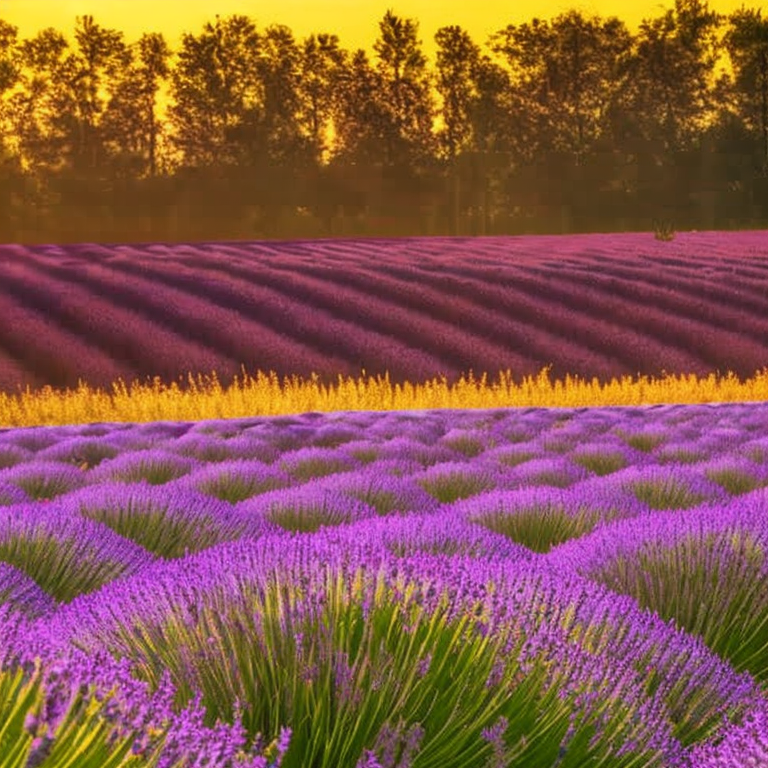}%
        \includegraphics[width=0.094\linewidth]{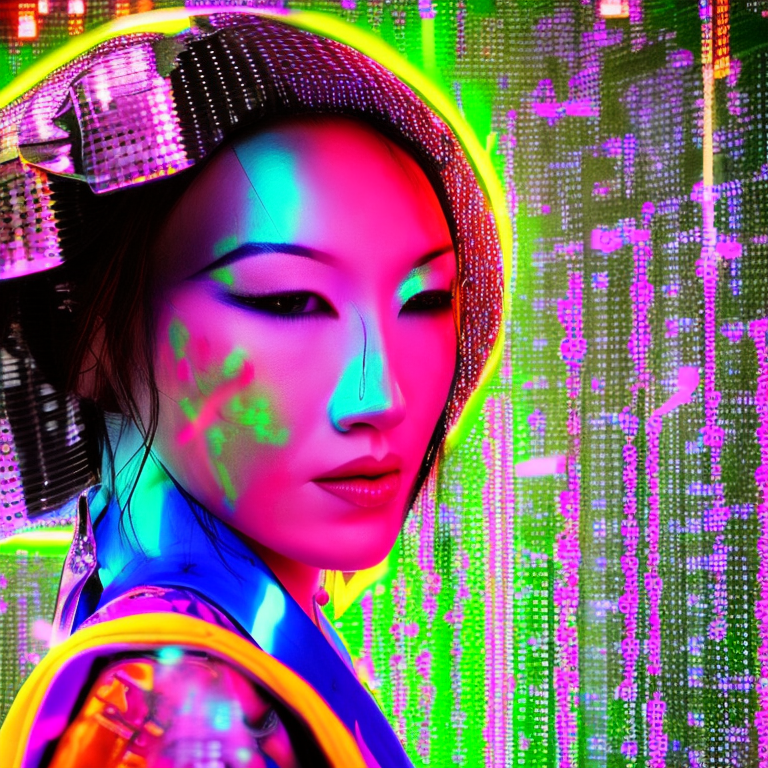}%
        \includegraphics[width=0.094\linewidth]{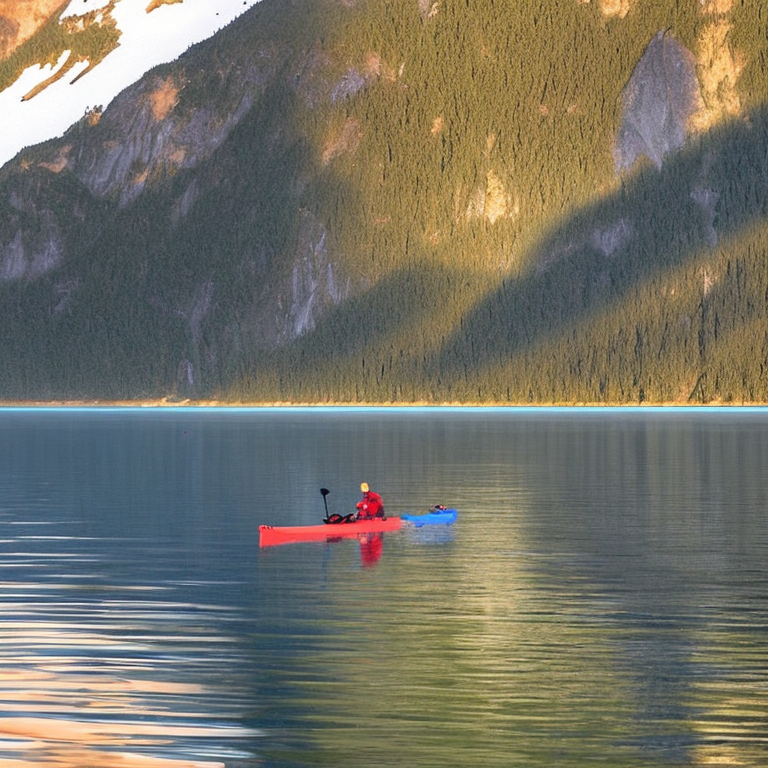}%
        \includegraphics[width=0.094\linewidth]{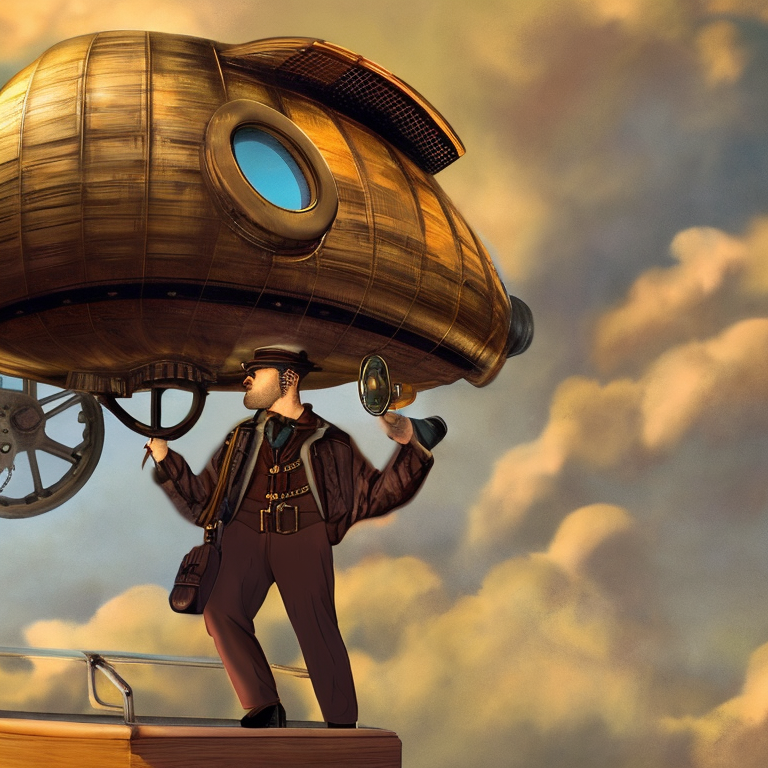}%
        \includegraphics[width=0.094\linewidth]{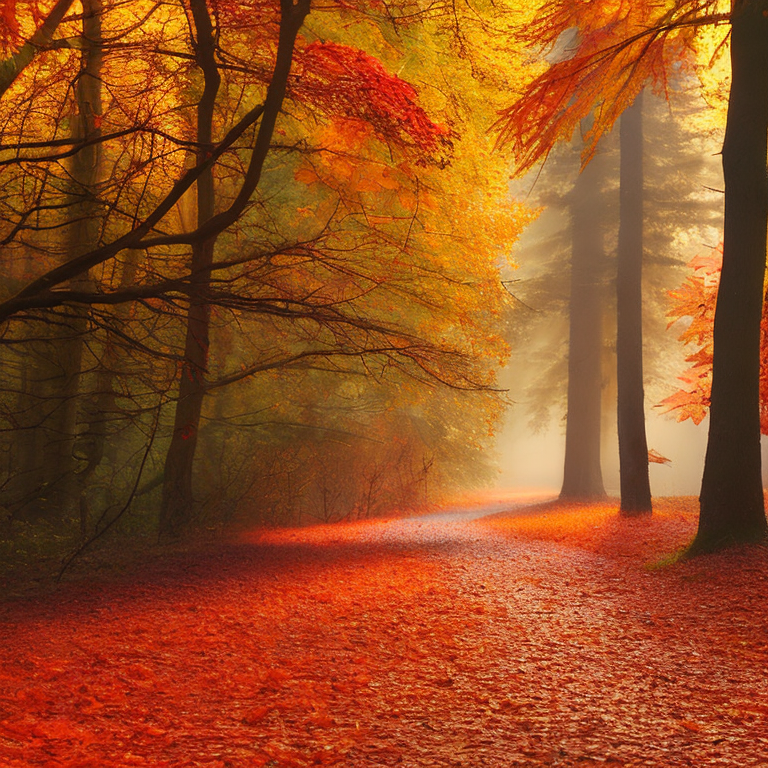}%
        \includegraphics[width=0.094\linewidth]{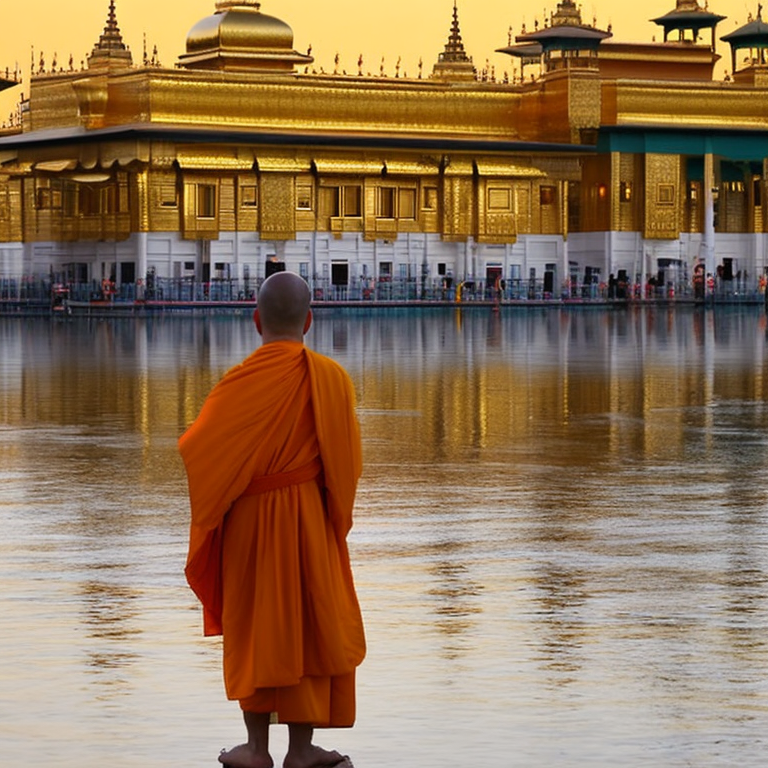}%
        \includegraphics[width=0.094\linewidth]{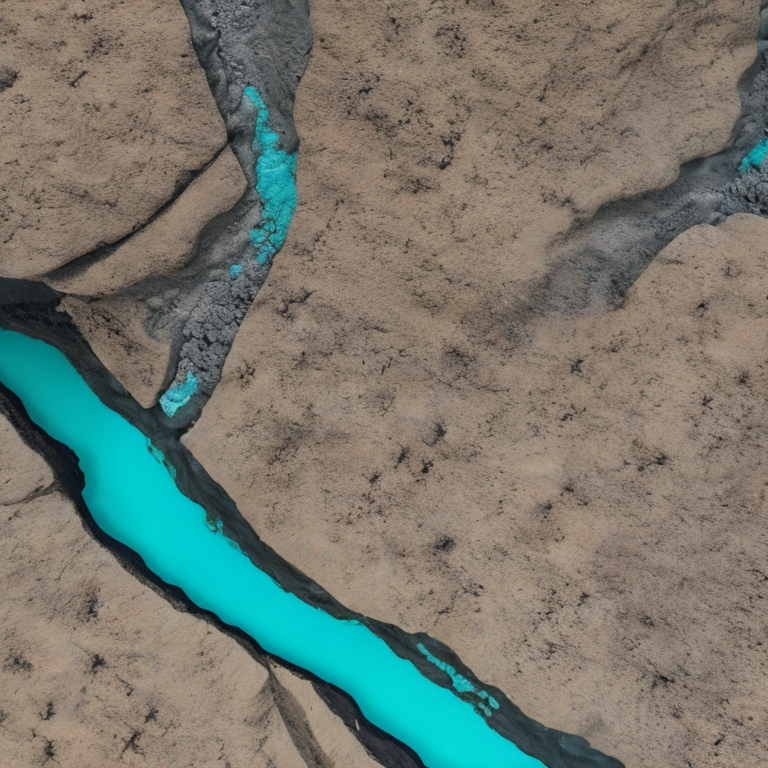}%
        \includegraphics[width=0.094\linewidth]{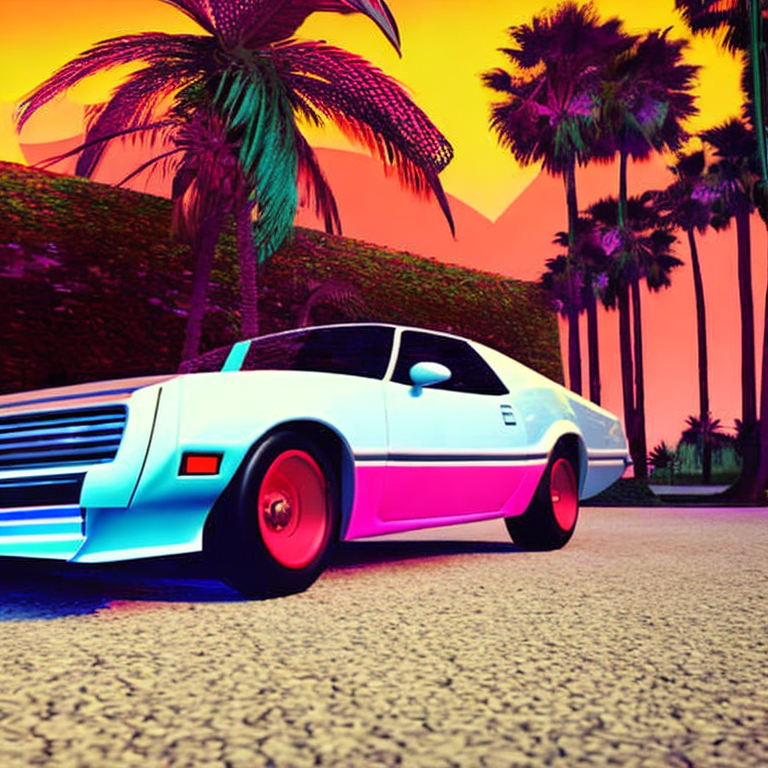}%
        \includegraphics[width=0.094\linewidth]{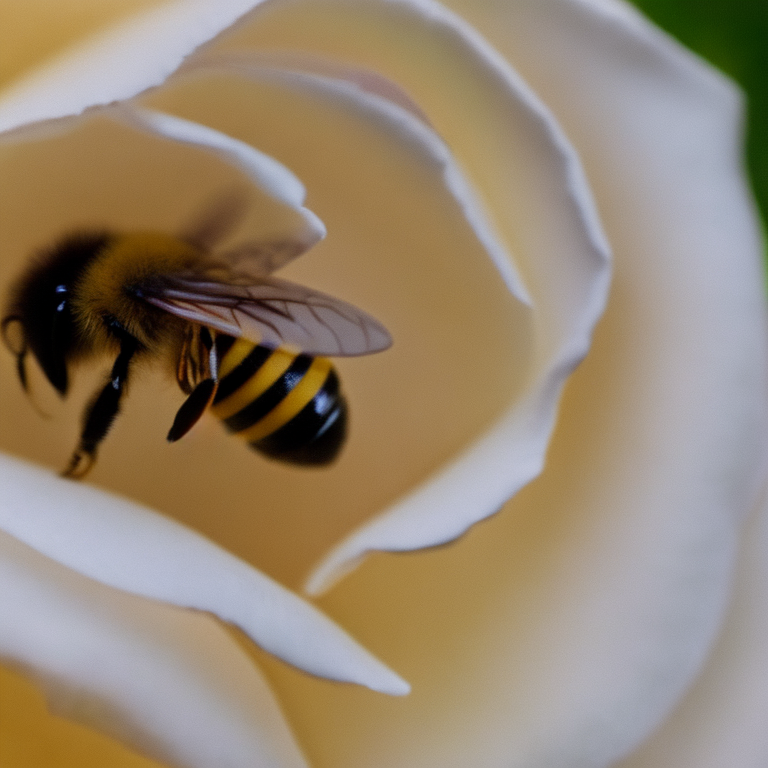}%
        \includegraphics[width=0.094\linewidth]{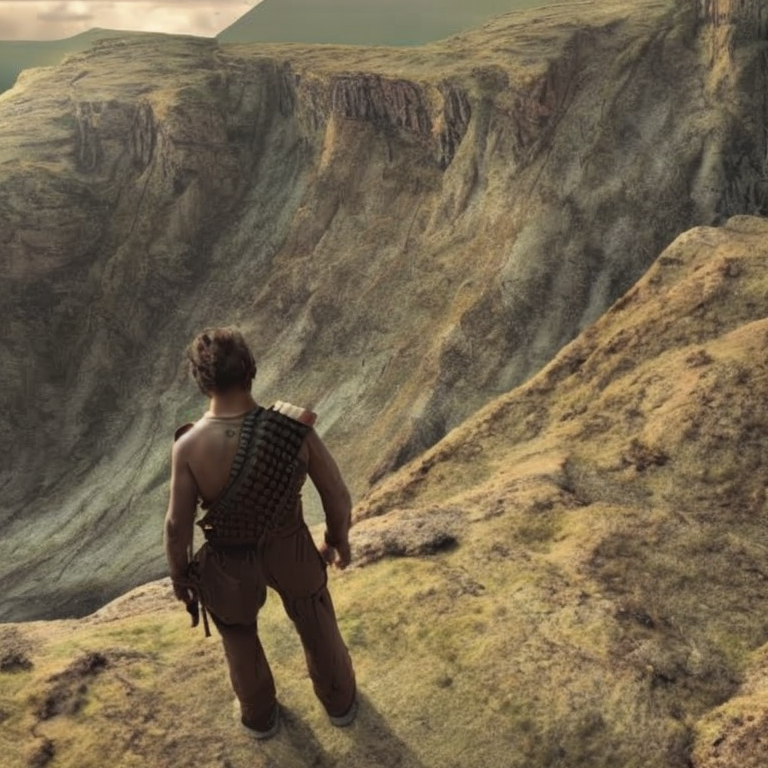}
    \end{minipage}
    \vspace{1mm} 

    \begin{minipage}{\linewidth}
        \centering
        \parbox[b][0.094\linewidth][t]{0.05\linewidth}{\rotatebox{90}{\small\textbf{ParaDiGMS}}}%
\includegraphics[width=0.094\linewidth]{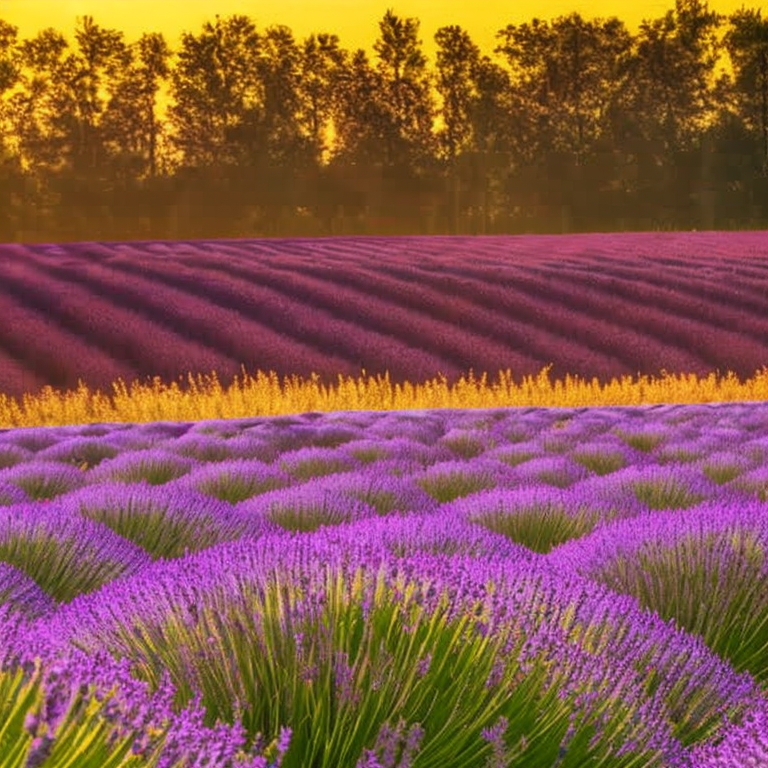}%
\includegraphics[width=0.094\linewidth]{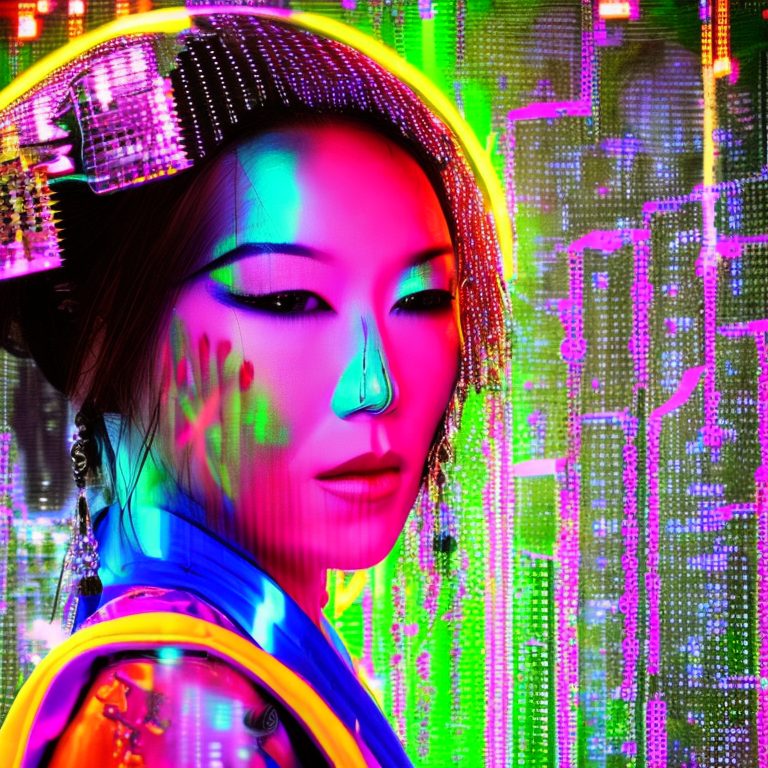}%
\includegraphics[width=0.094\linewidth]{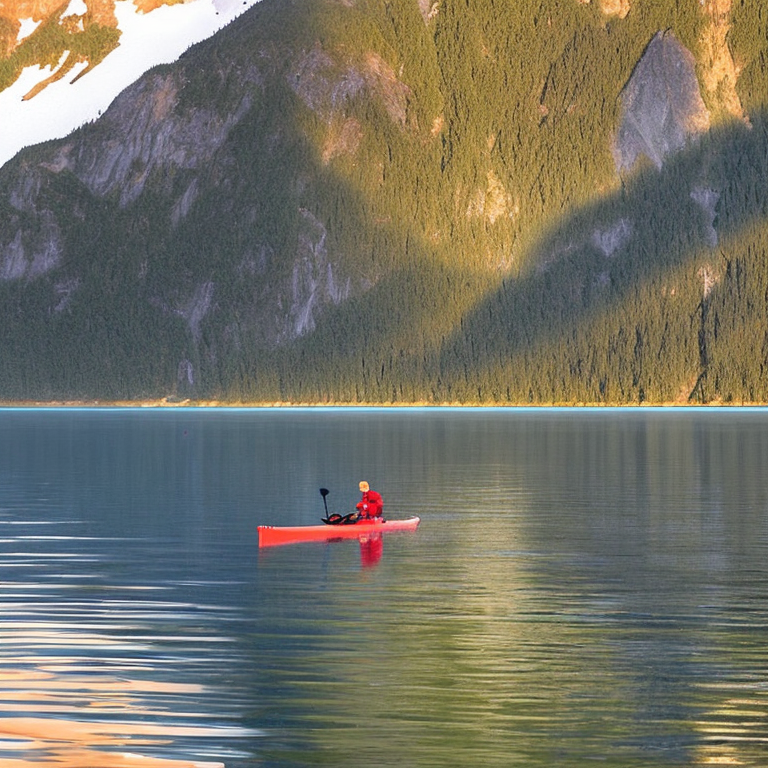}%
\includegraphics[width=0.094\linewidth]{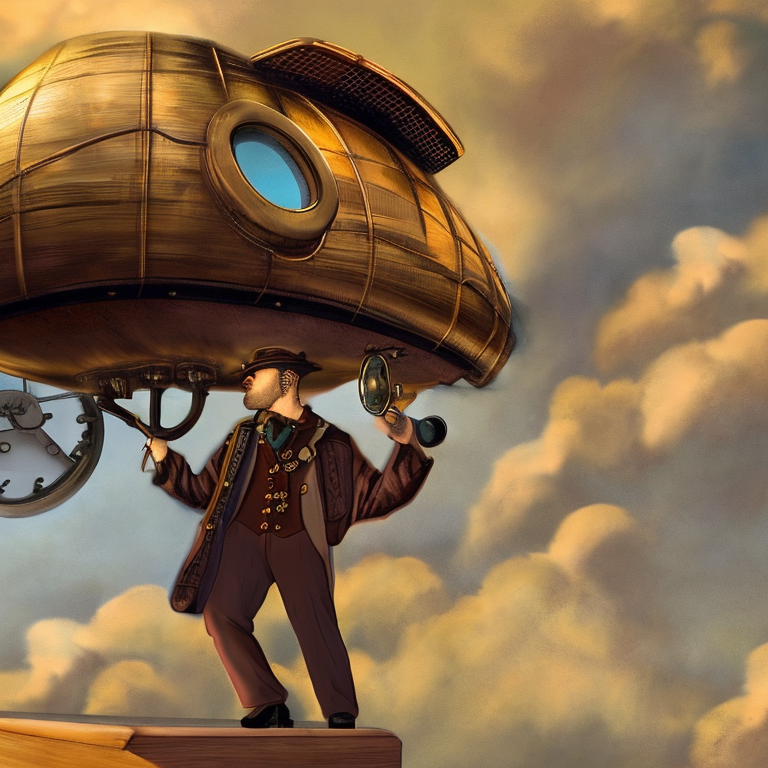}%
\includegraphics[width=0.094\linewidth]{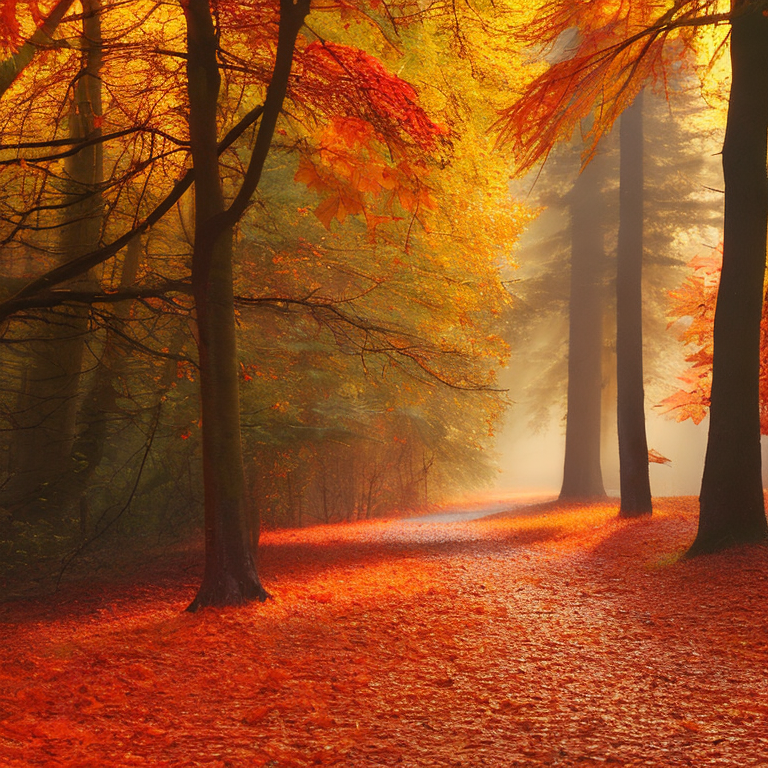}%
\includegraphics[width=0.094\linewidth]{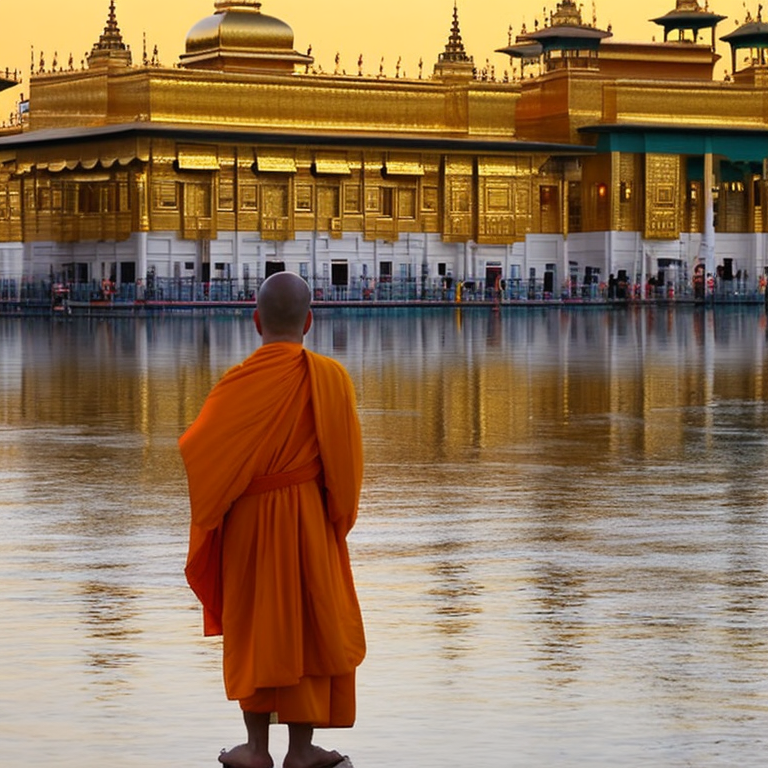}%
\includegraphics[width=0.094\linewidth]{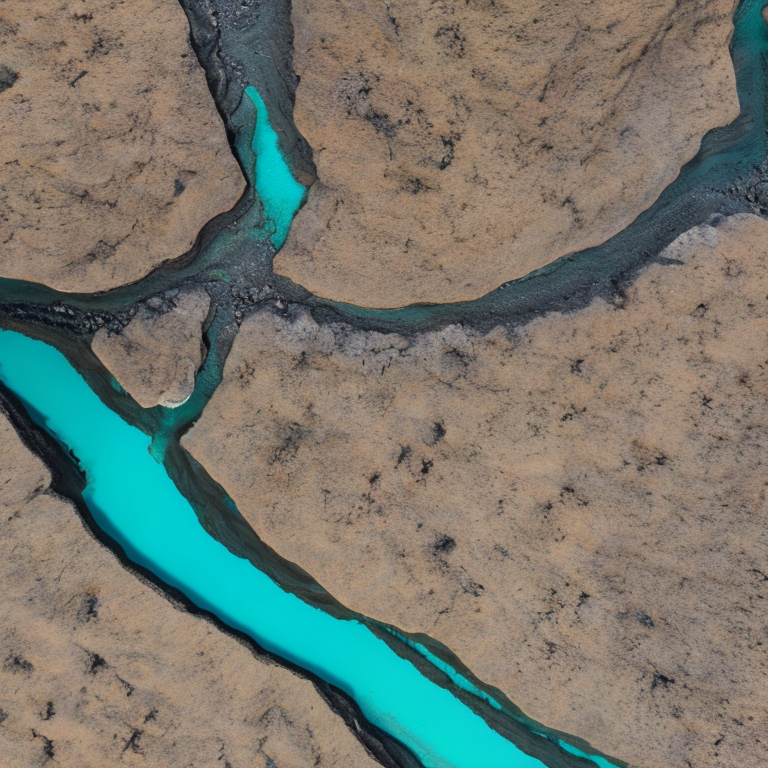}%
\includegraphics[width=0.094\linewidth]{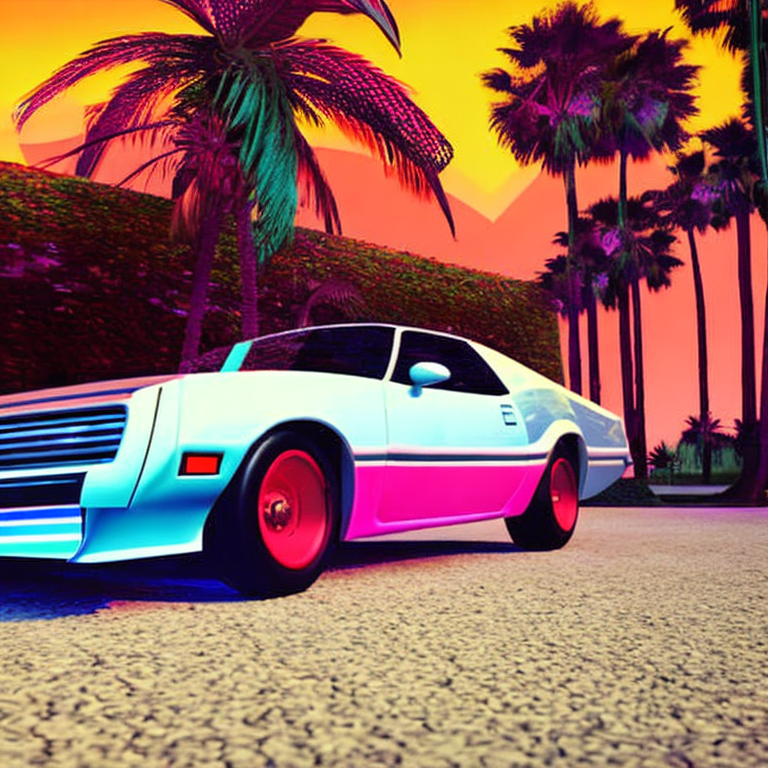}%
\includegraphics[width=0.094\linewidth]{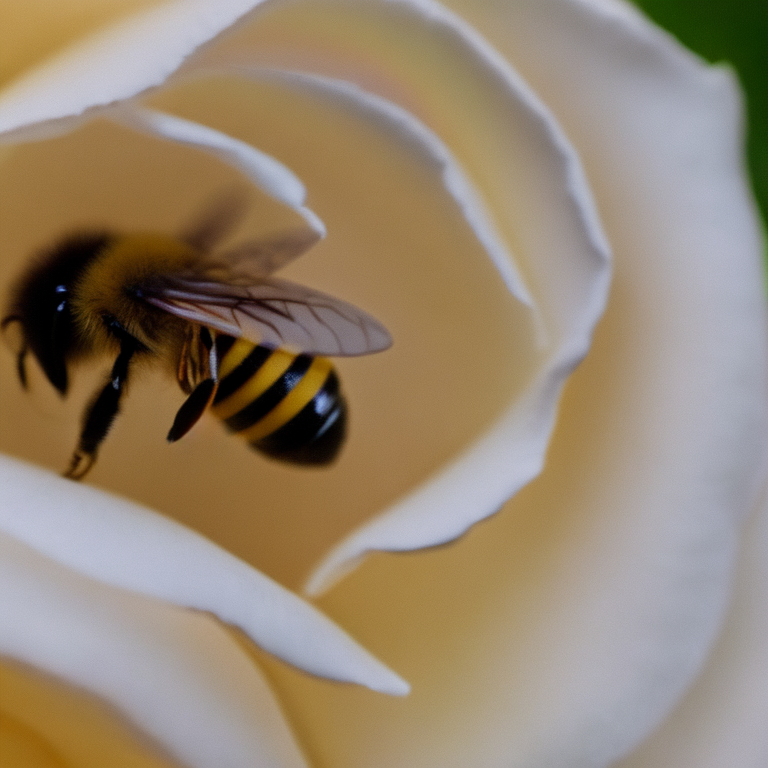}%
\includegraphics[width=0.094\linewidth]{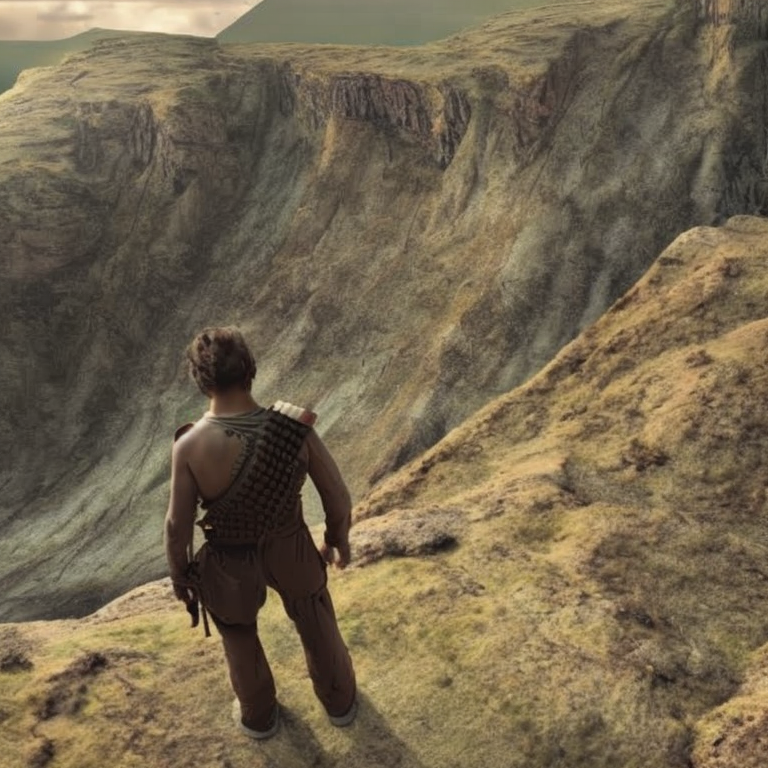}
    \end{minipage}
    \vspace{1mm} 

    \begin{minipage}{\linewidth}
        \centering
        \parbox[b][0.094\linewidth][t]{0.05\linewidth}{\rotatebox{90}{\small\textbf{ParaSolver}}}%
\includegraphics[width=0.094\linewidth]{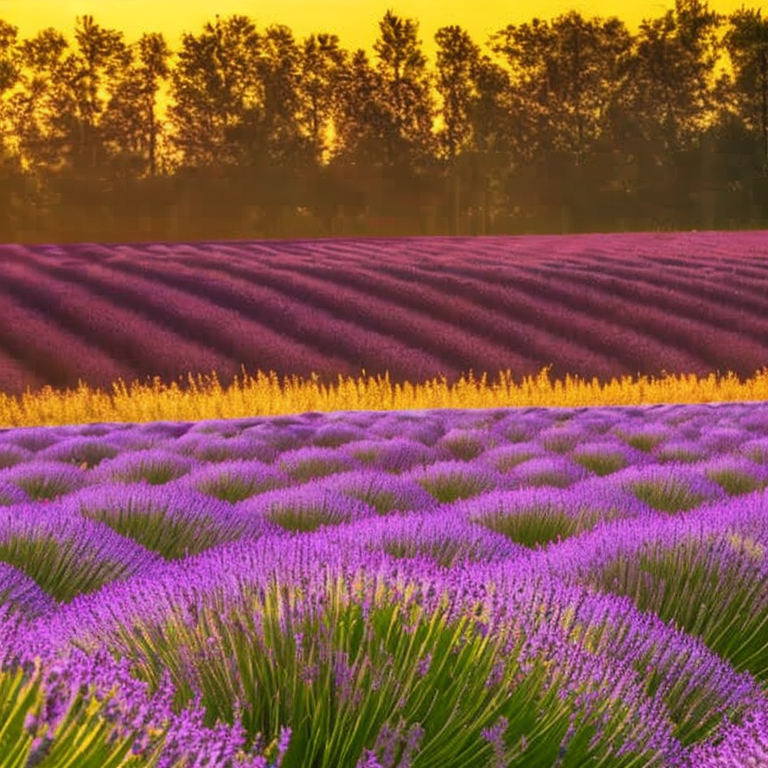}%
\includegraphics[width=0.094\linewidth]{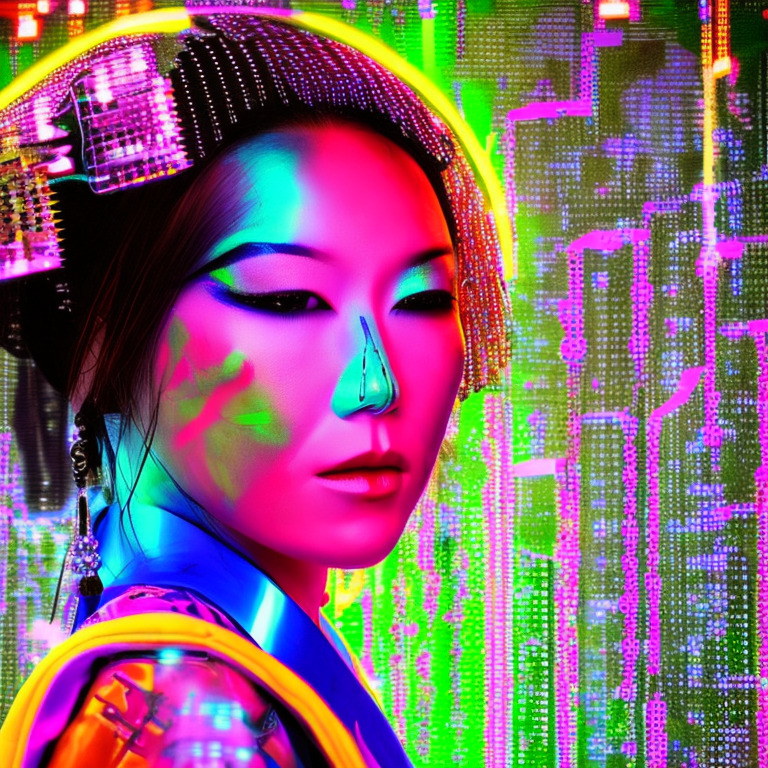}%
\includegraphics[width=0.094\linewidth]{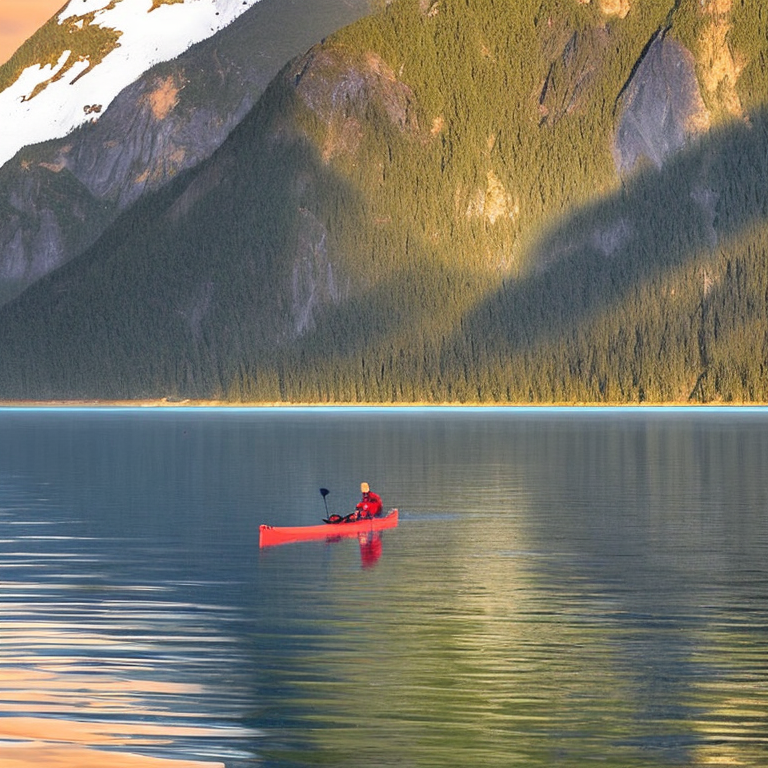}%
\includegraphics[width=0.094\linewidth]{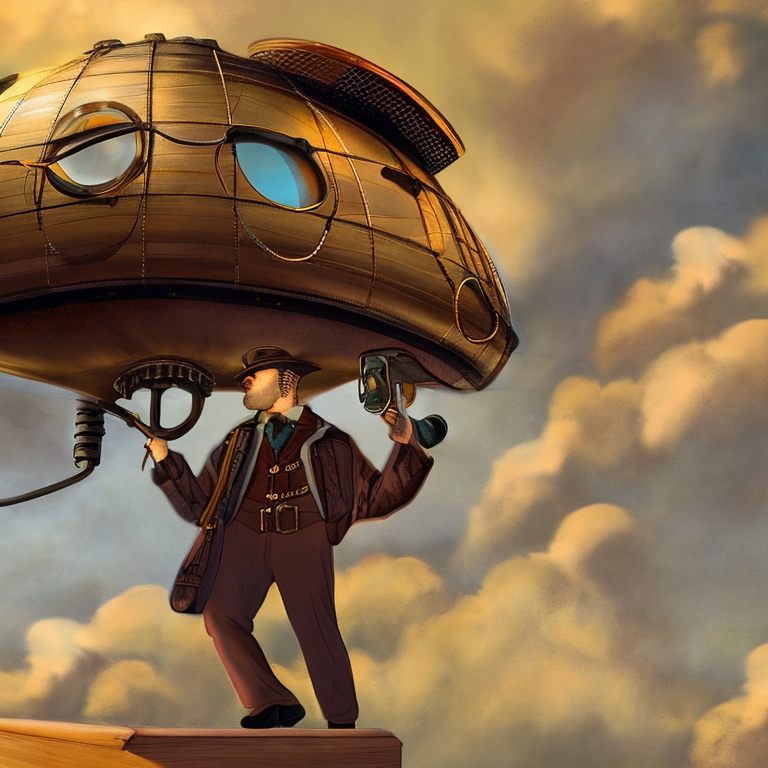}%
\includegraphics[width=0.094\linewidth]{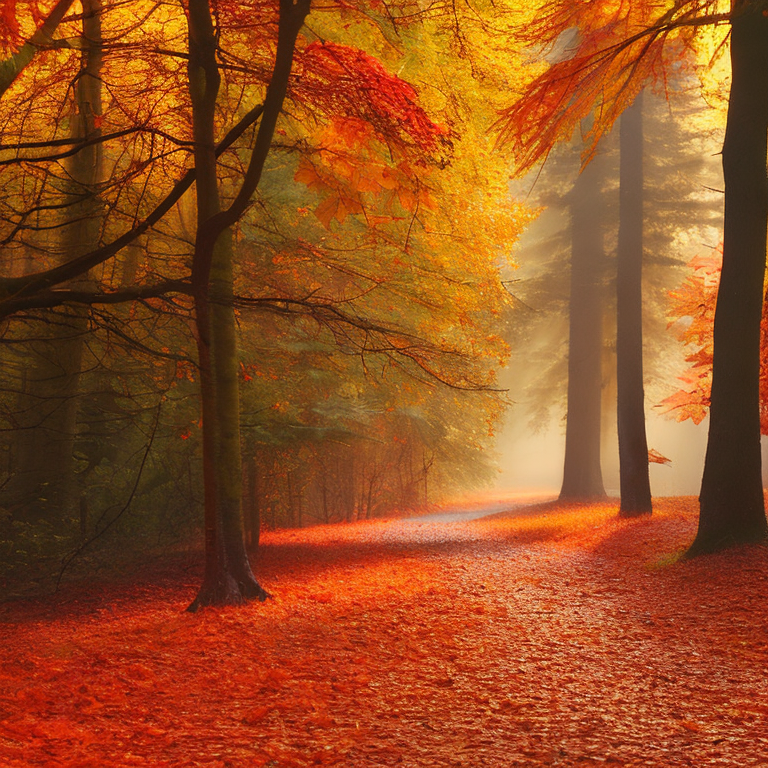}%
\includegraphics[width=0.094\linewidth]{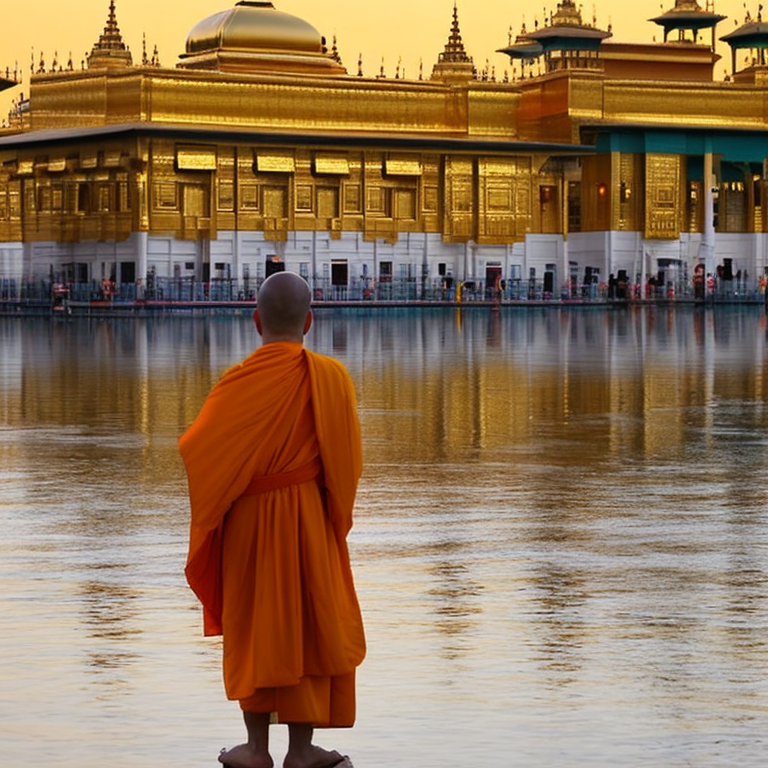}%
\includegraphics[width=0.094\linewidth]{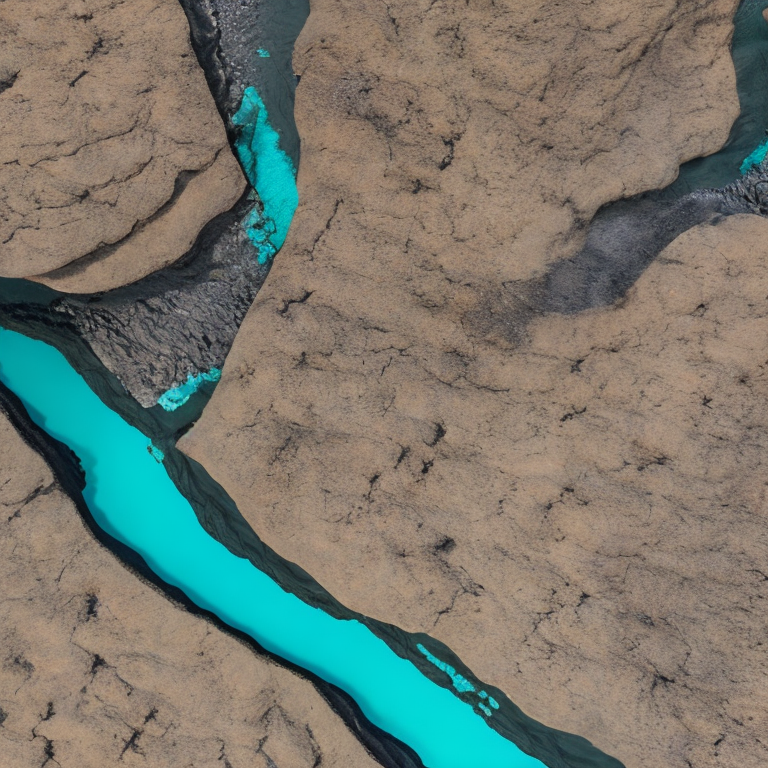}%
\includegraphics[width=0.094\linewidth]{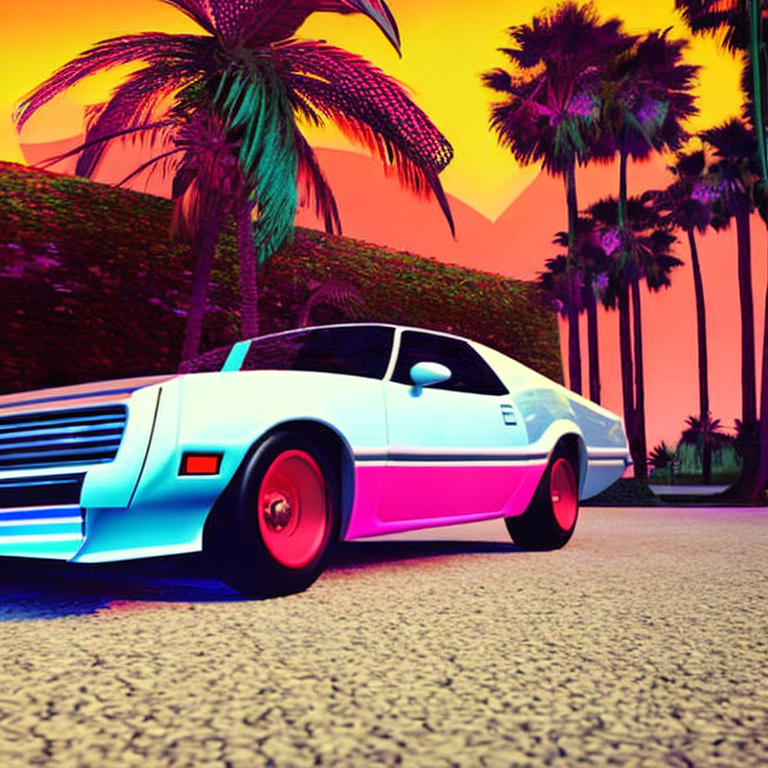}%
\includegraphics[width=0.094\linewidth]{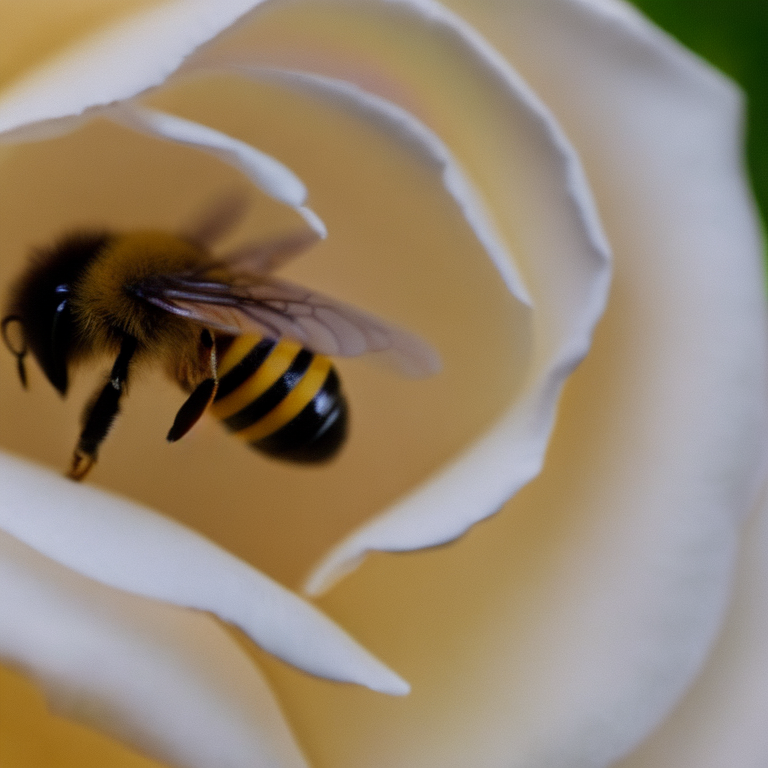}%
\includegraphics[width=0.094\linewidth]{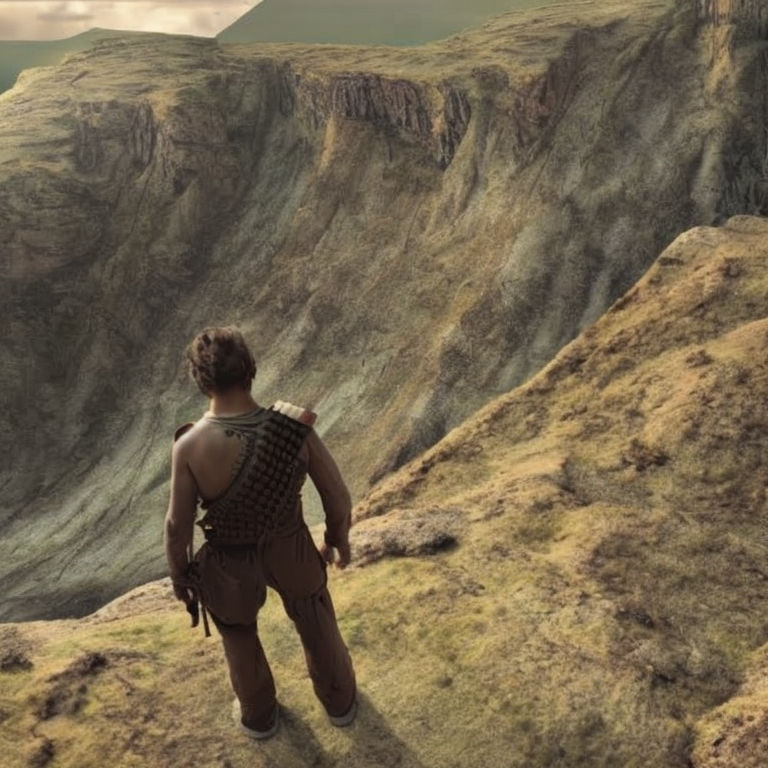}
    \end{minipage}
    \vspace{1mm} 

    \begin{minipage}{\linewidth}
        \centering
        \parbox[b][0.094\linewidth][t]{0.05\linewidth}{\rotatebox{90}{\small\textbf{ParaLin}}}%
\includegraphics[width=0.094\linewidth]{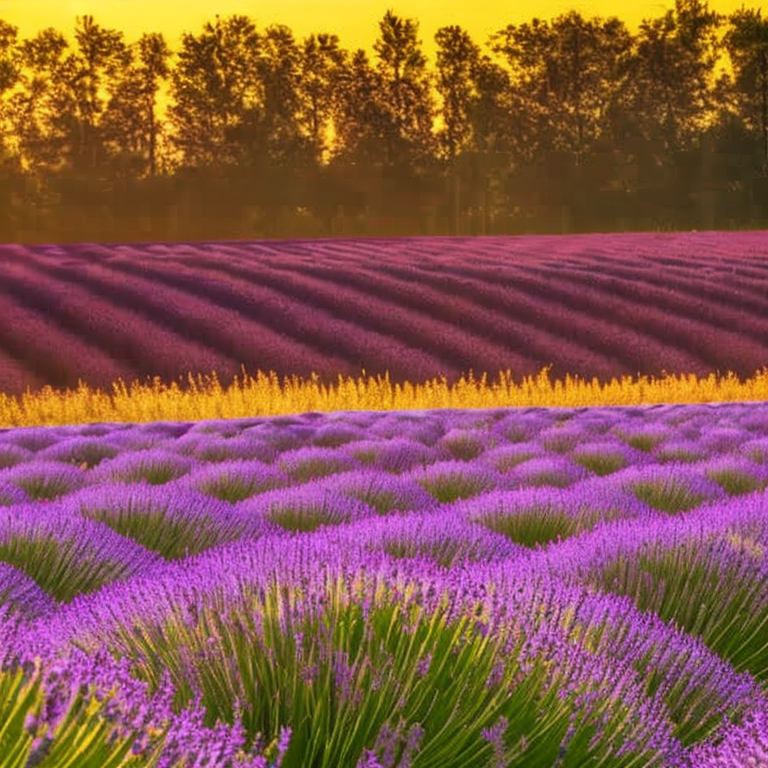}%
\includegraphics[width=0.094\linewidth]{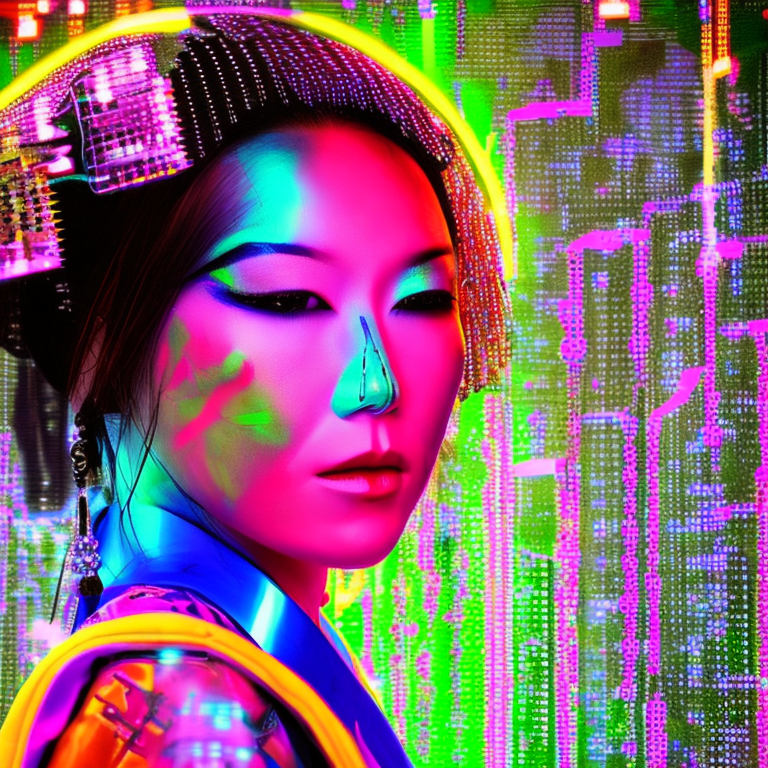}%
\includegraphics[width=0.094\linewidth]{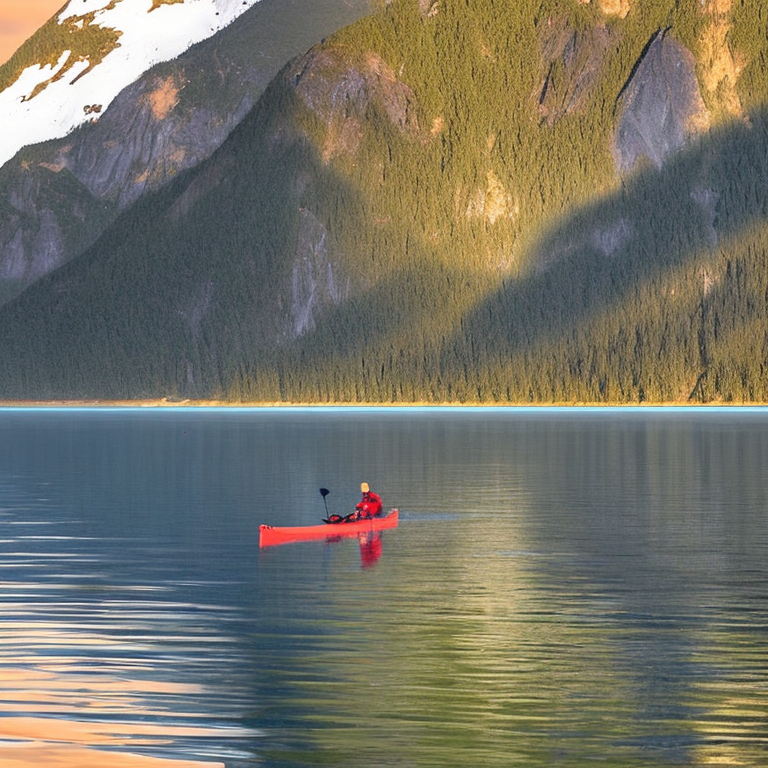}%
\includegraphics[width=0.094\linewidth]{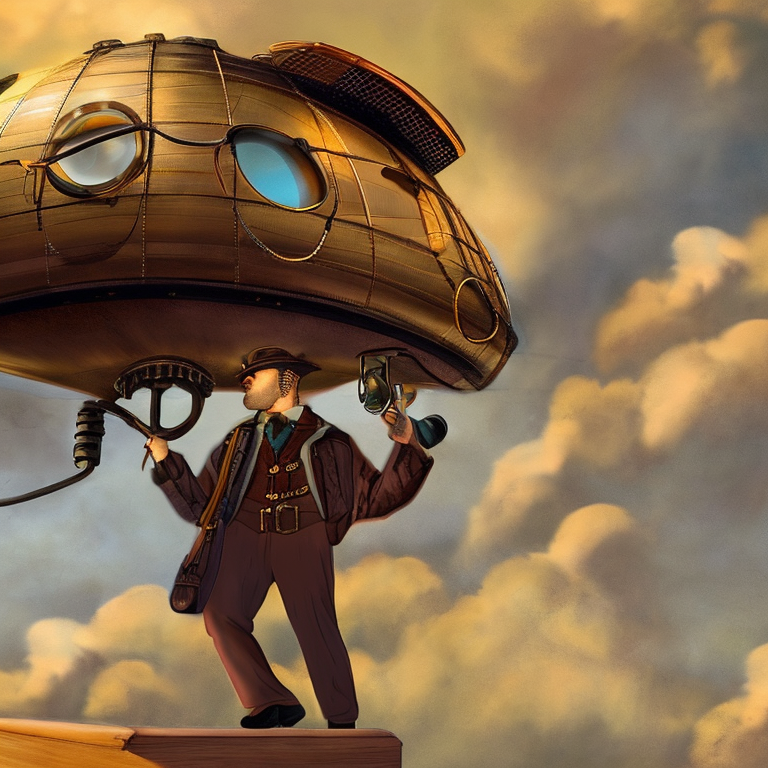}%
\includegraphics[width=0.094\linewidth]{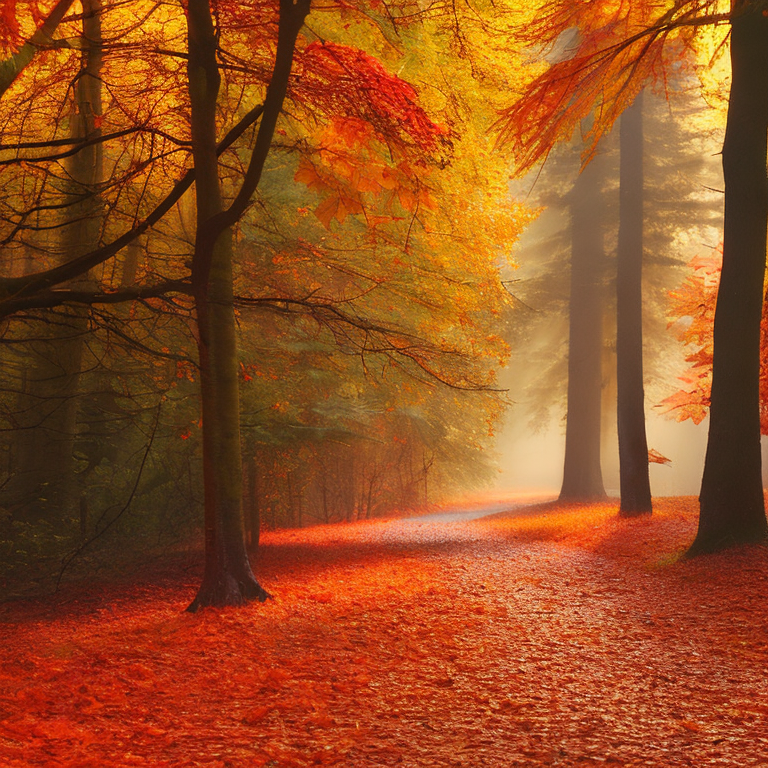}%
\includegraphics[width=0.094\linewidth]{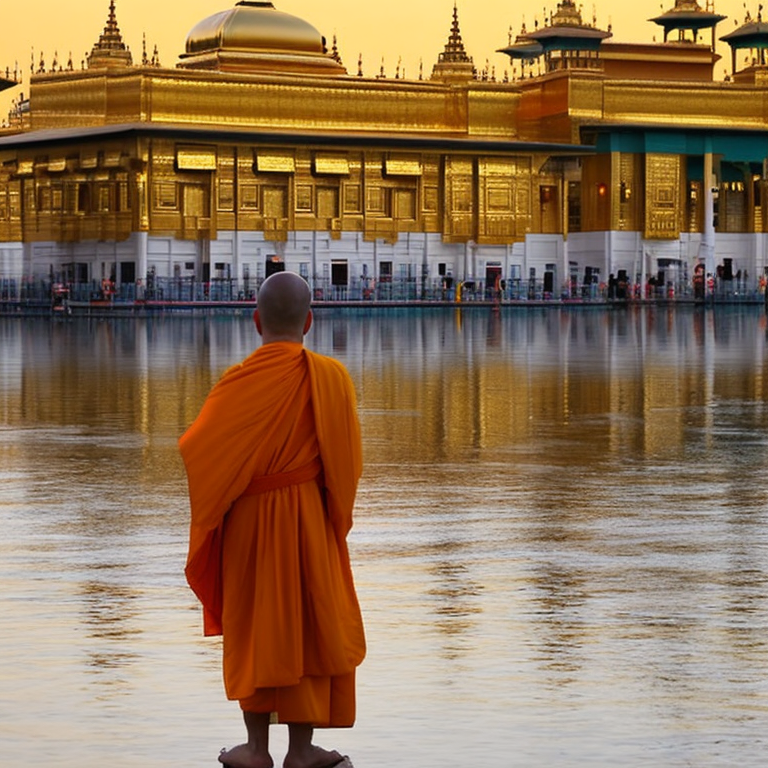}%
\includegraphics[width=0.094\linewidth]{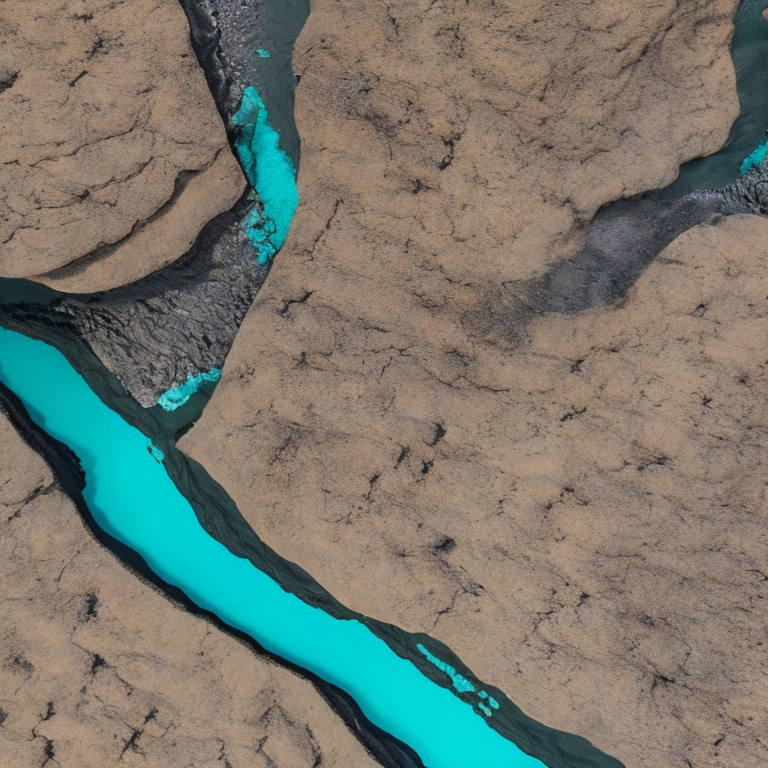}%
\includegraphics[width=0.094\linewidth]{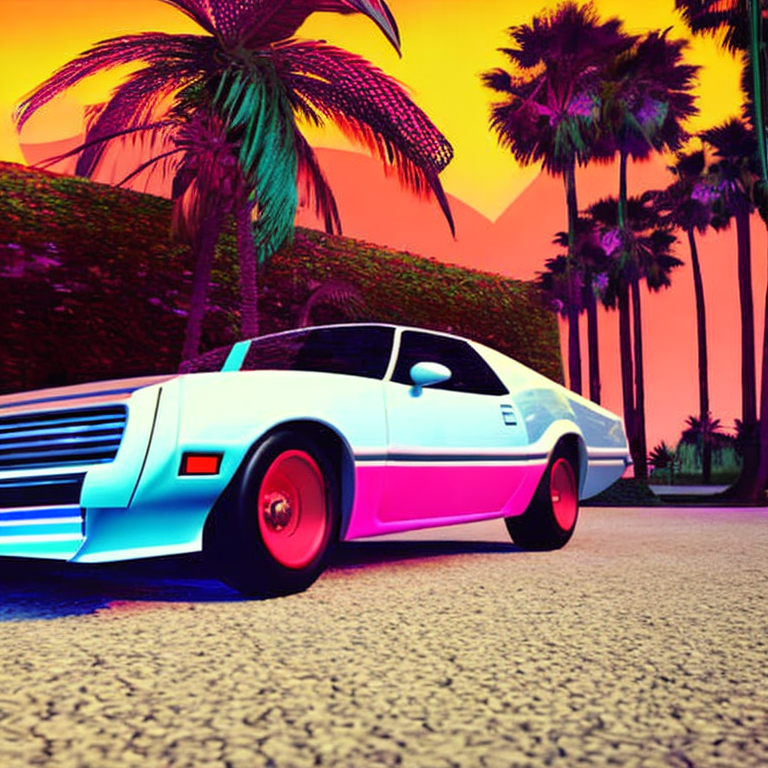}%
\includegraphics[width=0.094\linewidth]{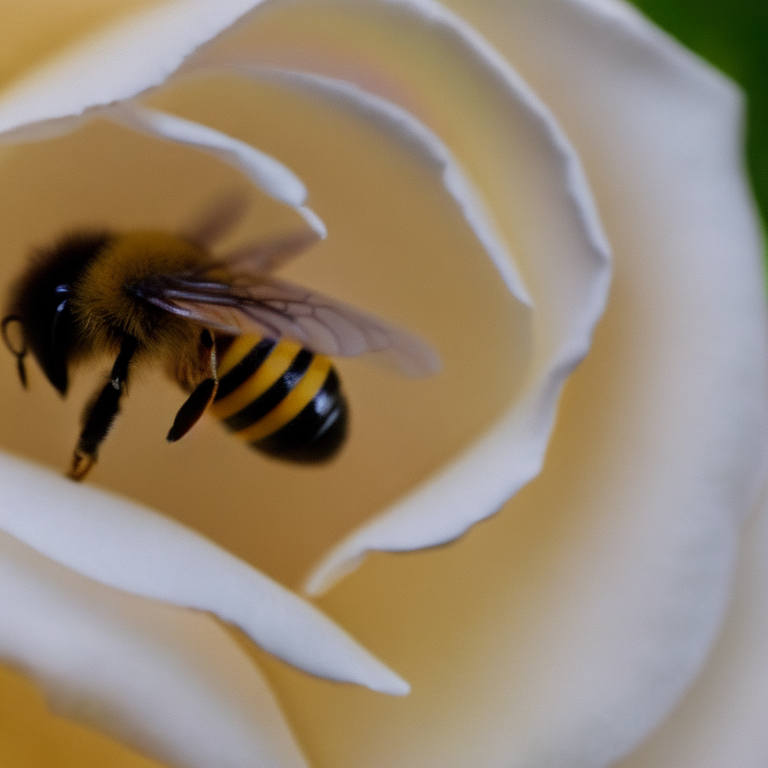}%
\includegraphics[width=0.094\linewidth]{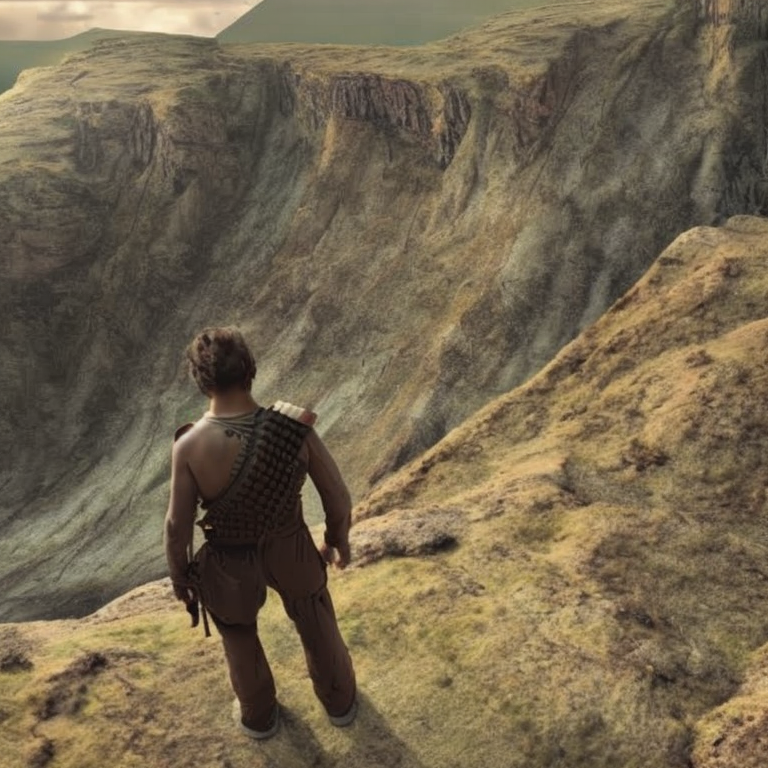}%
    \end{minipage}
    \vspace{1mm} 

    \begin{minipage}{\linewidth}
        \centering
        \parbox[c][0.094\linewidth][t]{0.05\linewidth}{\rotatebox{90}{\small\textbf{Ours}}}%
\includegraphics[width=0.094\linewidth]{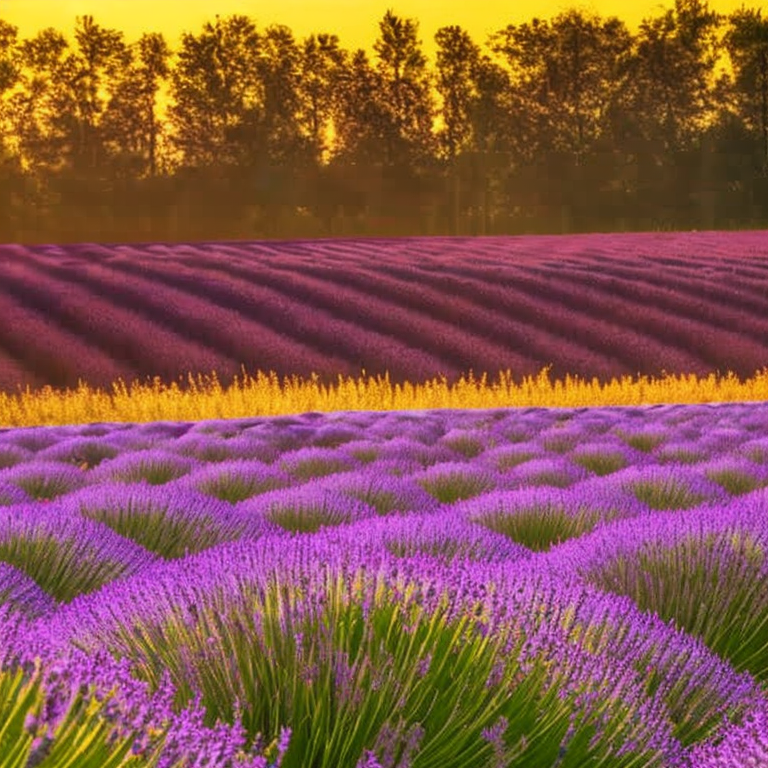}%
\includegraphics[width=0.094\linewidth]{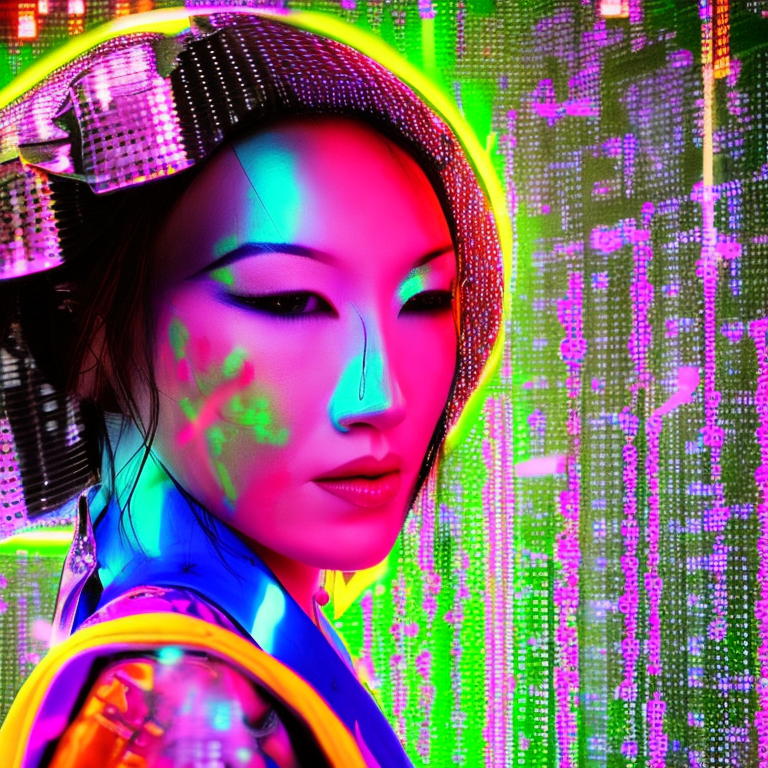}%
\includegraphics[width=0.094\linewidth]{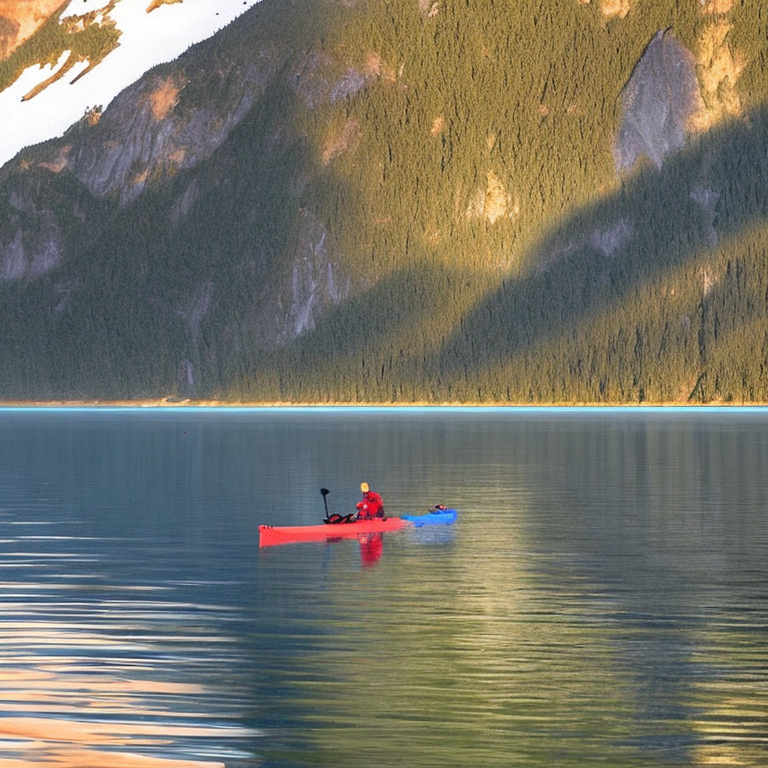}%
\includegraphics[width=0.094\linewidth]{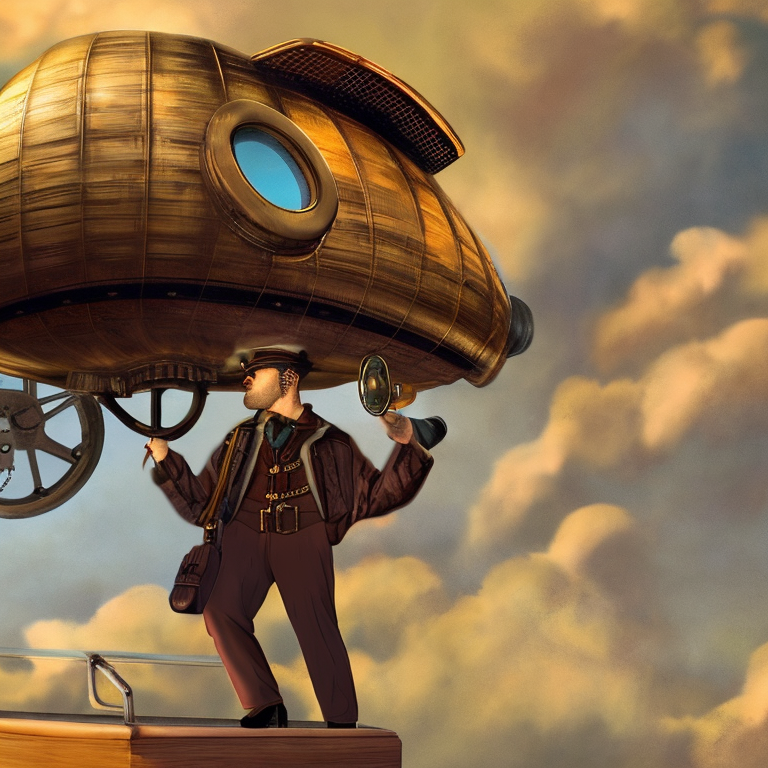}%
\includegraphics[width=0.094\linewidth]{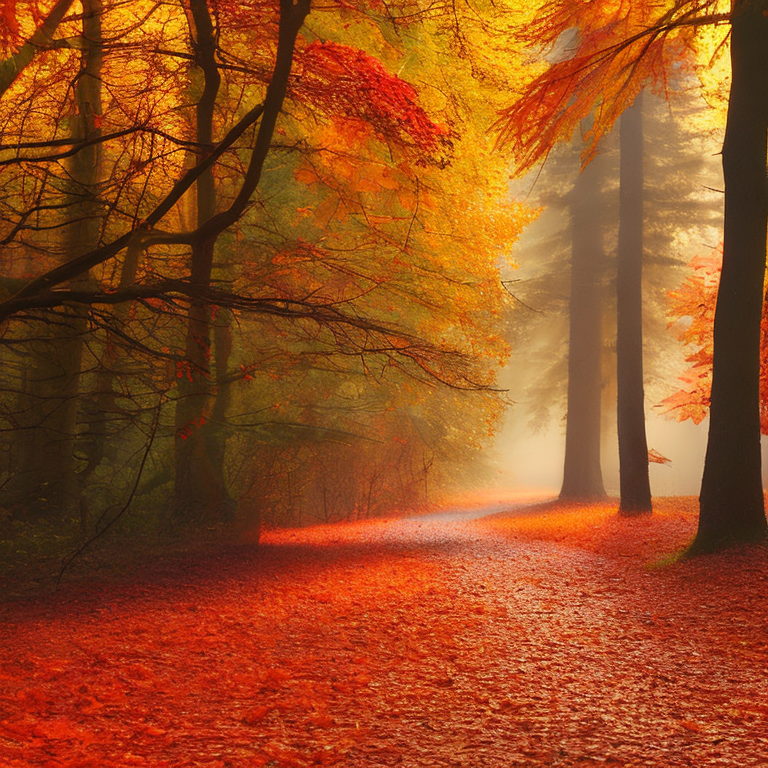}%
\includegraphics[width=0.094\linewidth]{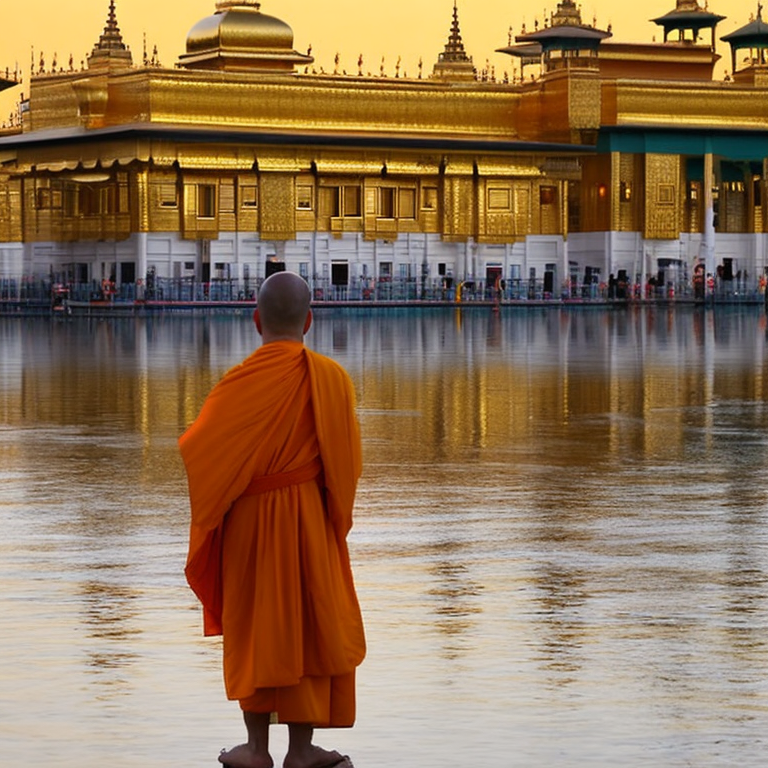}%
\includegraphics[width=0.094\linewidth]{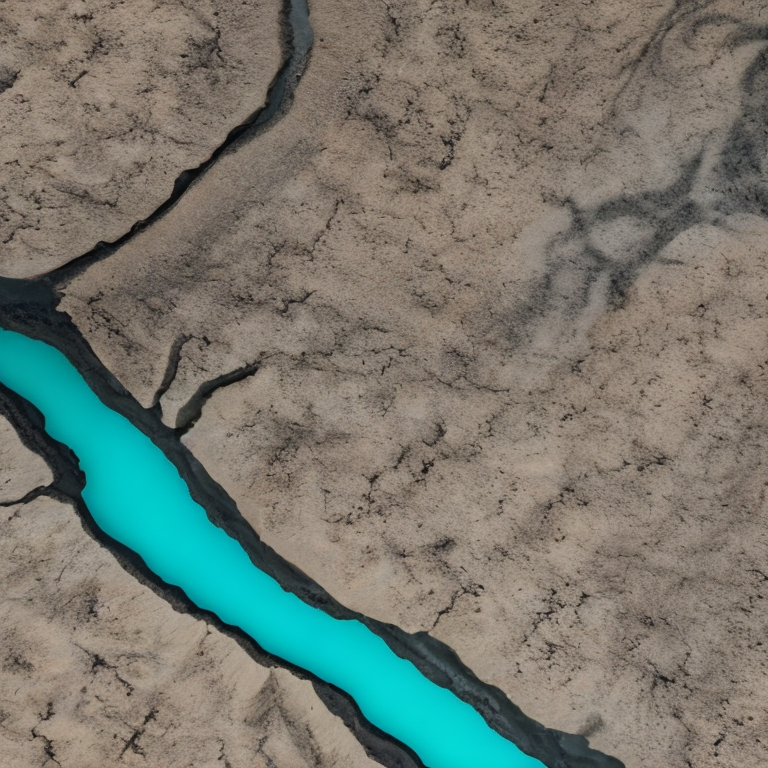}%
\includegraphics[width=0.094\linewidth]{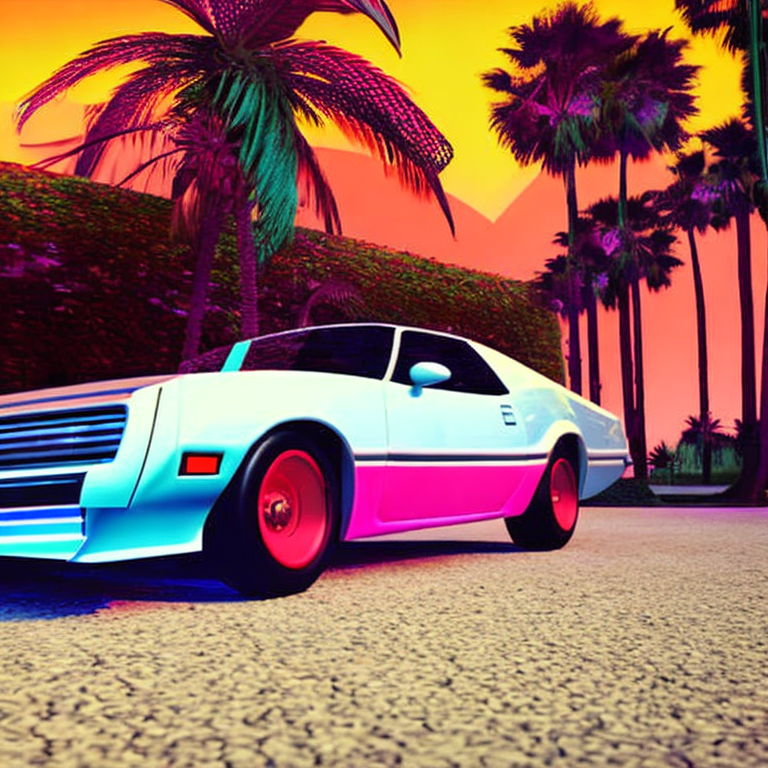}%
\includegraphics[width=0.094\linewidth]{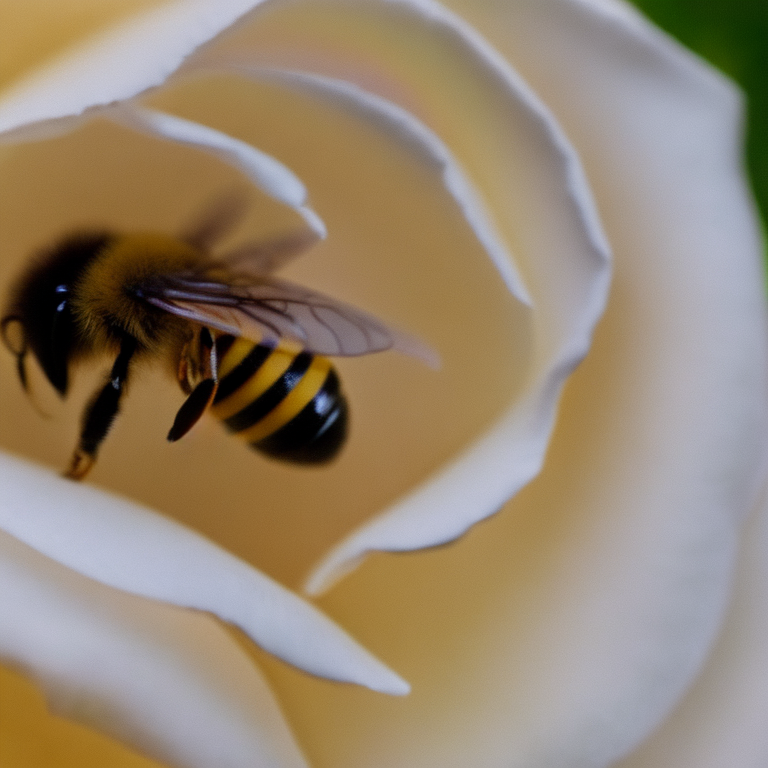}%
\includegraphics[width=0.094\linewidth]{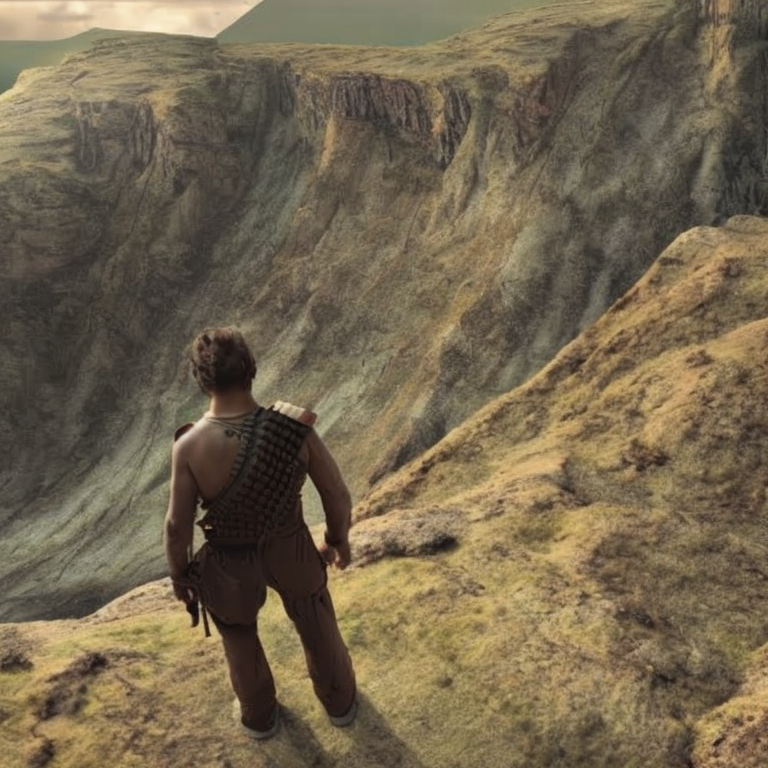}
    \end{minipage}
\vspace{-0.3cm}
    \caption{\label{fig:CMP_10_images_SD_DPMSolver}
    Visual comparison of final generation results on Stable Diffusion v2, using various samplers to expedite a 50-step DPMSolver baseline. Each column represents the output for a different prompt, demonstrating fidelity and visual consistency across methods.
    \textbf{Rows (Top to Bottom)}: DPMSolver (Baseline), ParaDiGMS, ParaSolver, ParaLin, and ParaLin++.
    }
\end{figure*}

\clearpage
\appendices
\numberwithin{equation}{section}

\noindent\emph{Equations numbered A.1, B.1, etc.\ are local to these appendices. References with plain numbers, such as Eq.~(7), as well as references to definitions and propositions, refer to the main manuscript.}

\section*{Appendix Contents}
\noindent
\begin{tabular}{@{}p{0.14\linewidth}p{0.76\linewidth}@{}}
\toprule
\textbf{Appendix} & \textbf{Contents and page}\\
\midrule
Appendix~\ref{apx:proof_Linear_SQ_Solver} & Proof of Proposition~1, p.~\pageref{apx:proof_Linear_SQ_Solver}\\
Appendix~\ref{apx:proof_unique_solution} & Proof of Proposition~2, p.~\pageref{apx:proof_unique_solution}\\
Appendix~\ref{apx:proof_convergence} & Proof of Proposition~3, p.~\pageref{apx:proof_convergence}\\
Appendix~\ref{apx:proof_Optimal_Relaxation_Factor} & Proof of Proposition~4, p.~\pageref{apx:proof_Optimal_Relaxation_Factor}\\
\bottomrule
\end{tabular}

\section{Proof of Proposition~1}
\label{apx:proof_Linear_SQ_Solver}
\subsection{Derivation of the Error Bound}
Before proceeding with the proof, we first state the following practical assumption regarding the denoising model.

\begin{assumption}\label{asp:well_trained_model}
The denoising network $\vec{S}_{\theta^*}(\vec{X}_t, t)$ is a well-trained noise predictor. Specifically, for any noisy sample $\vec{X}_t = \alpha_t\vec{X}_{T} + \sigma_t \vec{\epsilon}_t$ (where $\vec{X}_T$ is the clean sample and $\vec{\epsilon}_t \sim \mathcal{N}(0, \vec{I})$ is the specific noise instance), there exists a  upper bound $\zeta > 0$ such that the model's prediction error is bounded:
$$
\norm{\vec{S}_{\theta^*}(\vec{X}_t, t) - \vec{\epsilon}_t} < \zeta
$$
\end{assumption}

We recall the forward process $p(\vec{X}_t|\vec{X}_T) = \mathcal{N} (\alpha_{t}\vec{X}_T, \sigma_{t}^2\vec{I})$, where $\alpha_t$ and $\sigma_t$ are the noise schedule coefficients. We define the log-SNR as $\lambda_t = \log(\alpha_t / \sigma_t)$. It follows that $e^{-\lambda_t} = \sigma_t / \alpha_t$ and  $\lambda_t$ is a strictly decreasing function of $t$.

Following Proposition 3.1 in DPM-Solver~\cite{dpmsolver}, the exact solution of the corresponding SDE from time $s$ to $t$  is given by:
\begin{small}
\begin{equation}
\label{eq:ExactIntegral}
\begin{aligned}
\Phi(t,s,\vec{X}_s)
&=  \frac{\alpha_{t}}{\alpha_{s}}\vec{X}_s + \alpha_{t}\int_{\lambda_s}^{\lambda_t}e^{-\lambda}\sigma_\lambda\nabla_{\vec{X}_\lambda} \log p(\vec{X}_\lambda) d\lambda\\
&\quad + \int_{\lambda_s}^{\lambda_t} g(\lambda) d\vec{W},
\end{aligned}
\end{equation}
\end{small}
where $p(\vec{X}_\lambda)$ is the marginal distribution. In a practical solver, the true score $\nabla_{\vec{X}_\lambda} \log p(\vec{X}_\lambda) = -\mathbb{E}[\vec{\epsilon} | \vec{X}_\lambda] / \sigma_\lambda$ is unavailable and is replaced by the model's prediction. Specifically, the solver approximates the score using the noise prediction model $\vec{S}_{\theta^*}$:
$$
\hat{\nabla}_{\vec{X}_\lambda} \log p(\vec{X}_\lambda) = -\frac{\vec{S}_{\theta^*}(\vec{X}_\lambda, \lambda)}{\sigma_\lambda}
$$
By substituting this approximation into~\eqref{eq:ExactIntegral}, we define the \textbf{practical sequential solver} $\hat{\Phi}$ as:
\begin{small}
\begin{equation}
\label{eq:PracticalIntegral}
\begin{aligned}
\hat{\Phi}(t,s,\vec{X}_s)
&=  \frac{\alpha_{t}}{\alpha_{s}}\vec{X}_s + \alpha_{t}\int_{\lambda_s}^{\lambda_t}e^{-\lambda}\sigma_\lambda\left( -\frac{\vec{S}_{\theta^*}(\vec{X}_\lambda, \lambda)}{\sigma_\lambda} \right) d\lambda\\
&\quad + \int_{\lambda_s}^{\lambda_t} g(\lambda) d\vec{W} \\
&=  \frac{\alpha_{t}}{\alpha_{s}}\vec{X}_s - \alpha_{t}\int_{\lambda_s}^{\lambda_t}e^{-\lambda}\vec{S}_{\theta^*}(\vec{X}_\lambda, \lambda) d\lambda + \int_{\lambda_s}^{\lambda_t} g(\lambda) d\vec{W}.
\end{aligned}
\end{equation}
\end{small}
Our goal is to show that this practical solver $\hat{\Phi}$ can be precisely approximated by a linear solver. We use Assumption~\ref{asp:well_trained_model} to decompose the model's prediction $\vec{S}_{\theta^*}$ into the ground-truth noise $\vec{\epsilon}_\lambda$ and a bounded error term $\vec{\delta}(\vec{X}_\lambda, \lambda)$:
$$
\vec{S}_{\theta^*}(\vec{X}_\lambda, \lambda) = \vec{\epsilon}_\lambda + \vec{\delta}(\vec{X}_\lambda, \lambda),
$$
where $\norm{\vec{\delta}(\vec{X}_\lambda, \lambda)} < \zeta$.

Substituting this decomposition back into our practical solver~\eqref{eq:PracticalIntegral}:
\begin{small}
\begin{equation}
\begin{aligned}
\hat{\Phi}(t,s,\vec{X}_s)
&= \frac{\alpha_{t}}{\alpha_{s}}\vec{X}_s - \alpha_{t}\int_{\lambda_s}^{\lambda_t}e^{-\lambda} \left( \vec{\epsilon}_\lambda + \vec{\delta}(\vec{X}_\lambda, \lambda) \right) d\lambda \\
&\quad + \int_{\lambda_s}^{\lambda_t} g(\lambda) d\vec{W} \\
&= \underbrace{\frac{\alpha_{t}}{\alpha_{s}}\vec{X}_s - \alpha_{t}\int_{\lambda_s}^{\lambda_t}e^{-\lambda} \vec{\epsilon}_\lambda d\lambda + \int_{\lambda_s}^{\lambda_t} g(\lambda) d\vec{W}}_{\mathcal{H}(t,s,\vec{X}_s,\vec{X}_T)} \\
&\quad \underbrace{- \alpha_{t}\int_{\lambda_s}^{\lambda_t}e^{-\lambda} \vec{\delta}(\vec{X}_\lambda, \lambda) d\lambda}_{\mathcal{Z}}.
\end{aligned}
\end{equation}
\end{small}
We can identify the first part, $\mathcal{H}$, as the ideal linear solver. By substituting the definition of the ground-truth noise, $\vec{\epsilon}_\lambda = (\vec{X}_\lambda - \alpha_\lambda \vec{X}_T(\vec{X}_\lambda)) / \sigma_\lambda$, this term becomes:
\begin{small}
\begin{equation}
\label{eq:linear_solver_ideal}
\begin{aligned}
\mathcal{H}(t,s,\vec{X}_s,\vec{X}_T)
&=  \frac{\alpha_{t}}{\alpha_{s}}\vec{X}_s -
\alpha_{t}\int_{\lambda_s}^{\lambda_t}e^{-\lambda}
\frac{\vec{X}_\lambda - \alpha_\lambda\vec{X}_T(\vec{X}_\lambda)}{\sigma_\lambda}
d\lambda \\
&\quad + \int_{\lambda_s}^{\lambda_t} g(\lambda) d\vec{W},
\end{aligned}
\end{equation}
\end{small}
which is the linear solver structure presented in the proposition.

The second part, $\mathcal{Z}$, represents the cumulative error introduced by the non-ideal denoising model:
$$
\mathcal{Z} = - \alpha_{t}\int_{\lambda_s}^{\lambda_t}e^{-\lambda} \vec{\delta}(\vec{X}_\lambda, \lambda) d\lambda
$$
We can bound the norm of this error term using the triangle inequality for integrals and Assumption~\ref{asp:well_trained_model}:
\begin{small}
\begin{equation}
\begin{aligned}
\label{eq:error_term_final}
\norm{\mathcal{Z}} &=  \norm{- \alpha_{t}\int_{\lambda_s}^{\lambda_t}e^{-\lambda} \vec{\delta}(\vec{X}_\lambda, \lambda) d\lambda} \\
&\leq \alpha_{t}\int_{\lambda_s}^{\lambda_t} e^{-\lambda} \norm{\vec{\delta}(\vec{X}_\lambda, \lambda)} d\lambda \\
&\leq \alpha_{t} \zeta \int_{\lambda_s}^{\lambda_t} e^{-\lambda} d\lambda \\
&= \alpha_{t} \zeta \left[ -e^{-\lambda} \right]_{\lambda_s}^{\lambda_t} \\
&= \alpha_{t} \zeta \left( -e^{-\lambda_t} - (-e^{-\lambda_s}) \right) \\
&= \alpha_{t} \zeta \left( e^{-\lambda_s} - e^{-\lambda_t} \right).
\end{aligned}
\end{equation}
\end{small}
By substituting the definition $e^{-\lambda} = \sigma / \alpha$, we get the final bound:
\begin{equation}
\label{eq:VE_bound_raw}
\norm{\mathcal{Z}} \leq \alpha_{t} \zeta \left( \frac{\sigma_s}{\alpha_s} - \frac{\sigma_t}{\alpha_t} \right) = \zeta\sigma_t \left( \frac{\alpha_t \sigma_s}{\alpha_s\sigma_t} - 1 \right).
\end{equation}

\subsection{Analysis of Convergence and Step Size}

To demonstrate that the term $\left( \frac{\sigma_s}{\alpha_s} - \frac{\sigma_t}{\alpha_t} \right)$ is negligible in practice, we rely on the necessary conditions for high-quality sampling in diffusion probabilistic models.

\textbf{First}, as established in \cite{chen2023sampling} (Theorem 2), the convergence of the sampling process is strictly dependent on the discretization step size $h = s-t$. Specifically, the total variation distance between the sampled distribution and the true distribution is strictly bounded only when the maximum step size satisfies:
\begin{equation}
h \leq 1.
\end{equation}
 If $h$ is to exceed $1$, the local linearization errors inherent in the discretization (e.g., Euler-Maruyama or DPM-Solver) would dominate, rendering the solver mathematically unstable regardless of the model's accuracy. Therefore, the  practical deployment of a solver operates within the regime $h \in (0, 1]$.


\textbf{Second}, in standard large-scale diffusion models (e.g., Stable Diffusion, IDDPM), the continuous forward process is strictly defined over a normalized time interval $t \in [0, 1]$.
\begin{itemize}
    \item For high-fidelity generation, it is standard practice to use $N \in [20, 50]$ discretization steps with advanced solvers (e.g., DPM-Solver, UniPC). The average step size corresponds to $h \approx \frac{1}{N}$, resulting in $h \in [0.02, 0.05]$.
    \item Even in extreme few-step sampling scenarios (e.g.,  LCM methods with $N \approx 4$), the step size $h \approx 0.25$ remains well within the theoretical convergence bound of $1$.
\end{itemize}
Consequently, the operational regime $h \in (0, 1)$ covers the vast majority of practical use cases.

Under this condition ($h$ is small), we apply the Mean Value Theorem to the noise-to-signal ratio function $f(\tau) = \frac{\sigma_\tau}{\alpha_\tau}$. There exists a $\xi \in [s, t]$ such that:
\begin{equation}
\frac{\sigma_s}{\alpha_s} - \frac{\sigma_t}{\alpha_t} = f(s) - f(t) = f'(\xi) \cdot (s - t).
\end{equation}
Taking the absolute value and substituting $h = |s-t|$:
\begin{equation}
\left| \frac{\sigma_s}{\alpha_s} - \frac{\sigma_t}{\alpha_t} \right| \approx \left| \frac{d}{dt}\left(\frac{\sigma_t}{\alpha_t}\right) \right| h.
\end{equation}
Since the noise schedules $\sigma_t$ and $\alpha_t$ are smooth monotonic functions in practical SDEs, their derivative is bounded. Consequently, as $h \to 0$, the difference term approaches 0 linearly. Thus, the error bound in~\eqref{eq:VE_bound_raw} is minimized solely by enforcing the standard high-quality sampling constraint $h \leq 1$.

\subsection{Derivation of the Specific Error Bounds}

Following the fundamental error bound derived in~\eqref{eq:VE_bound_raw}:
\begin{equation}
\norm{\mathcal{Z}_s} \leq \alpha_{t} \zeta \left( \frac{\sigma_s}{\alpha_s} - \frac{\sigma_t}{\alpha_t} \right)
\end{equation}
we substitute the specific parameters for $\alpha_t$ (signal scaling) and $\sigma_t$ (noise schedule) from Table 1 of the EDM framework. This leads to Table~\ref{tab:error_comparison_edm}.

\begin{table*}[h]
\centering
\caption{Practical error bounds for diffusion frameworks defined in the EDM design space. Parameters are mapped from Table 1 of \cite{EDM} where $\beta_{\min} =0.1,\beta_{d}=19.9,$}.
\label{tab:error_comparison_edm}
\renewcommand{\arraystretch}{2.2}
\begin{tabular}{l l c c}
\toprule
\textbf{Method} & \textbf{Schedule Settings from Table 1} & \textbf{Substitution} & \textbf{Final Error Bound} \\
\midrule
\textbf{EDM~\cite{EDM}} & $\begin{aligned}\alpha_t &= 1 \\ \sigma_t &= \alpha_t \cdot t \end{aligned}$ & $\frac{\sigma_t}{\alpha_t} = t$ & $\norm{\mathcal{Z}_s} \leq \zeta (s - t) = \zeta h$ \\
\midrule
\textbf{VE~\cite{ScoreBase}} & $\begin{aligned}\alpha_t &= 1 \\ \sigma_t &=  \alpha_t \cdot\sqrt{t} \end{aligned}$ & $\frac{\sigma_t}{\alpha_t} = \sqrt{t}$ & $\norm{\mathcal{Z}_s} \leq \zeta (\sqrt{s} - \sqrt{t})$ \\
\midrule
\textbf{VP~\cite{ScoreBase}} & $\begin{aligned} \alpha_t &= 1/\sqrt{e^{\frac{1}{2}\beta_d t^2 +\beta_{\min}t}} \\ \sigma_t &=  \alpha_t \cdot \sqrt{e^{\frac{1}{2}\beta_d t^2 + \beta_{\min}t}-1}\end{aligned}$ & $\frac{\sigma_t}{\alpha_t} = \sqrt{e^{\frac{1}{2}\beta_d t^2 + \beta_{\min}t}-1}$ & $\norm{\mathcal{Z}_s} \leq \zeta \alpha_t \left( \frac{\sigma_s}{\alpha_s} - \frac{\sigma_t}{\alpha_t} \right)$ \\
\bottomrule
\end{tabular}
\end{table*}

\section{Proof of Proposition~2}
\label{apx:proof_unique_solution}


\begin{proof}
Let $\{\mathbf{X}_{t_n}\}_{n=0}^N$ be the exact trajectory generated by the sequential integral solver. From Proposition~1, the step transition is given by:
\begin{equation}
    \label{eq:proof_true_dynamics}
    \mathbf{X}_{t_{n+1}} = \mathcal{H}(t_{n+1}, t_n, \mathbf{X}_{t_n}, \mathbf{X}_T) + \mathcal{Z}_n,
\end{equation}
where $\mathbf{X}_T$ is the ground-truth data and $\|\mathcal{Z}_n\| \approx 0$ is the linearization error. Let $\{\hat{\mathbf{X}}_{t_n}\}_{n=0}^N$ be the solution to the $\varrho$-nonlinear system with the update rule:
\begin{equation}
    \label{eq:proof_approx_dynamics}
    \hat{\mathbf{X}}_{t_{n+1}} = \mathcal{H}(t_{n+1}, t_n, \hat{\mathbf{X}}_{t_n}, \mathbf{E}_{t_n}).
\end{equation}
Here, $\mathbf{E}_{t_n}$ denotes the estimated clean sample. Based on Definition~2, this estimator is explicitly defined as:
\begin{equation}
    \label{eq:estimator_def}
    \mathbf{E}_{t_n} =
    \begin{cases}
        \hat{\mathbf{X}}_{t_N}(\hat{\mathbf{X}}_{t_n}, \theta) & \text{if } n < \varrho, \\
        \hat{\mathbf{X}}_{t_N}(\hat{\mathbf{X}}_{t_{\varrho-1}}, \theta) & \text{if } n \geq \varrho.
    \end{cases}
\end{equation}

We define the cumulative state error as $\boldsymbol{\Delta}_{t_n} = \hat{\mathbf{X}}_{t_n} - \mathbf{X}_{t_n}$ and the clean sample estimation error as $\boldsymbol{\xi}_n = \mathbf{E}_{t_n} - \mathbf{X}_T$. To derive the error recurrence, we subtract~\eqref{eq:proof_true_dynamics} from~\eqref{eq:proof_approx_dynamics}. Using the integral definition of $\mathcal{H}$ from~\textup{Eq.~(7)}, we obtain:
\begin{equation}
\label{eq:delta_simplified}
\begin{aligned}
    \boldsymbol{\Delta}_{t_{n+1}}
    &= \frac{\alpha_{t_{n+1}}}{\alpha_{t_n}} \boldsymbol{\Delta}_{t_n}
       - \alpha_{t_{n+1}}
       \int_{\lambda_{t_n}}^{\lambda_{t_{n+1}}} e^{-\lambda}
       \mathcal{I}_{\lambda,n}\, d\lambda - \mathcal{Z}_n,
\end{aligned}
\end{equation}
where the integrand is
\begin{equation}
    \mathcal{I}_{\lambda,n}
       = \frac{\hat{\mathbf{X}}_\lambda - \alpha_\lambda \mathbf{E}_{t_n}}
              {\alpha_\lambda}
       - \frac{\mathbf{X}_\lambda - \alpha_\lambda \mathbf{X}_T}{\alpha_\lambda}.
\end{equation}
Using the Log-SNR property that $e^{-\lambda} = \sigma_\lambda / \alpha_\lambda$, this integrand simplifies as follows:
\begin{equation}
\begin{aligned}
    \mathcal{I}_{\lambda,n}
    &= \frac{\sigma_\lambda}{\alpha_\lambda}
       \left(
         \frac{\hat{\mathbf{X}}_\lambda - \mathbf{X}_\lambda}{\sigma_\lambda}
       - \frac{\alpha_\lambda(\mathbf{E}_{t_n} - \mathbf{X}_T)}{\sigma_\lambda}
       \right) \\
    &= \frac{\boldsymbol{\Delta}_\lambda}{\alpha_\lambda} - \boldsymbol{\xi}_n.
\end{aligned}
\end{equation}

To evaluate the intermediate error $\boldsymbol{\Delta}_\lambda$ inside the integral, we invoke the \textbf{Superposition Principle} for linear differential equations~\cite{boyce2021elementary,chen1999linear}. The total error $\boldsymbol{\Delta}_\lambda$ can be rigorously decomposed into the sum of the \textit{Zero-Input Response} (Homogeneous) and the \textit{Zero-State Response} (Forced):
\begin{itemize}
    \item \textbf{Homogeneous Response}: This component represents the natural evolution of the initial state error $\boldsymbol{\Delta}_{t_n}$ through the system dynamics, independent of the external estimation error. Due to the signal scaling inherent in the diffusion process $\mathcal{H}$, this is given by:
    \begin{equation}
        \boldsymbol{\Delta}_\lambda^{\text{homo}} = \frac{\alpha_\lambda}{\alpha_{t_n}} \boldsymbol{\Delta}_{t_n}.
    \end{equation}
    \item \textbf{Forced Response}: This component captures the error accumulated over the interval $[\lambda_{t_n}, \lambda]$ solely due to the  input source $\boldsymbol{\xi}_n$. Based on the integral structure of the linear solver in~\textup{Eq.~(7)}, the contribution accumulates as:
    \begin{equation}
        \boldsymbol{\Delta}_\lambda^{\text{force}} = \alpha_\lambda \left[ \int_{\lambda_{t_n}}^{\lambda} e^{-\tau} \frac{\alpha_\tau}{\sigma_\tau} d\tau \right] \boldsymbol{\xi}_n.
    \end{equation}
    Substituting $e^{-\tau} = \sigma_\tau / \alpha_\tau$, the integrand simplifies perfectly to unity:
    \begin{equation}
        \boldsymbol{\Delta}_\lambda^{\text{force}} = \alpha_\lambda \left[ \int_{\lambda_{t_n}}^{\lambda} 1 d\tau \right] \boldsymbol{\xi}_n = \alpha_\lambda (\lambda - \lambda_{t_n}) \boldsymbol{\xi}_n.
    \end{equation}
\end{itemize}
Combining these components via superposition, we obtain the explicit expression for the intermediate error:
\begin{equation}
    \boldsymbol{\Delta}_\lambda = \underbrace{\frac{\alpha_\lambda}{\alpha_{t_n}} \boldsymbol{\Delta}_{t_n}}_{\text{Homogeneous Response}} + \underbrace{\alpha_\lambda (\lambda - \lambda_{t_n}) \boldsymbol{\xi}_n}_{\text{Forced Response}}.
\end{equation}

We substitute this expression for $\boldsymbol{\Delta}_\lambda$ into the term $\frac{\boldsymbol{\Delta}_\lambda}{\alpha_\lambda}$. This leads to an algebraic cancellation of the time-dependent coefficient $\alpha_\lambda$ in the homogeneous term:
\begin{equation}
\begin{aligned}
\frac{\boldsymbol{\Delta}_\lambda}{\alpha_\lambda}
&= \frac{1}{\alpha_\lambda} \left( \frac{\alpha_\lambda}{\alpha_{t_n}} \boldsymbol{\Delta}_{t_n}
   + \alpha_\lambda (\lambda - \lambda_{t_n}) \boldsymbol{\xi}_n \right)\\
&= \frac{\boldsymbol{\Delta}_{t_n}}{\alpha_{t_n}}  + (\lambda - \lambda_{t_n}) \boldsymbol{\xi}_n.
\end{aligned}
\end{equation}
Now we can analytically evaluate the integral in~\eqref{eq:delta_simplified}:
\begin{equation}
    \begin{aligned}
    &\int_{\lambda_{t_n}}^{\lambda_{t_{n+1}}} \left( \frac{\boldsymbol{\Delta}_\lambda}{\alpha_\lambda} - \boldsymbol{\xi}_n \right) d\lambda \\
    &= \int_{\lambda_{t_n}}^{\lambda_{t_{n+1}}} \left( \frac{\boldsymbol{\Delta}_{t_n}}{\alpha_{t_n}} + (\lambda - \lambda_{t_n}) \boldsymbol{\xi}_n - \boldsymbol{\xi}_n \right) d\lambda \\
    &= \frac{\Delta \lambda_n}{\alpha_{t_n}} \boldsymbol{\Delta}_{t_n} + \left[ \frac{(\lambda - \lambda_{t_n})^2}{2} - (\lambda - \lambda_{t_n}) \right]_{\lambda_{t_n}}^{\lambda_{t_{n+1}}} \boldsymbol{\xi}_n \\
    &= \frac{\Delta \lambda_n}{\alpha_{t_n}} \boldsymbol{\Delta}_{t_n} + \left( \frac{(\Delta \lambda_n)^2}{2} - \Delta \lambda_n \right) \boldsymbol{\xi}_n,
    \end{aligned}
\end{equation}
where $\Delta \lambda_n = \lambda_{t_{n+1}} - \lambda_{t_n}$.

Plugging this integral result back into the expression for $\boldsymbol{\Delta}_{t_{n+1}}$ and grouping the terms by error source:
\begin{equation}
\begin{aligned}
    \label{eq:recurrence_raw}
    \boldsymbol{\Delta}_{t_{n+1}} =& \underbrace{\frac{\alpha_{t_{n+1}}}{\alpha_{t_n}} \Big( 1 - \Delta \lambda_n \Big)}_{\mathcal{A}_n} \boldsymbol{\Delta}_{t_n}\\
&+ \underbrace{\alpha_{t_{n+1}} \Delta \lambda_n \left( 1 - \frac{\Delta \lambda_n}{2} \right)}_{\mathcal{B}_n} \boldsymbol{\xi}_n\\
&- \mathcal{Z}_n.
\end{aligned}
\end{equation}

\textbf{Analysis of Coefficients}:
The convergence of the solver relies on the discretization step size $h = |t_{n+1} - t_n|$. As established in Theorem 2 of \cite{chen2023sampling}, high-quality generation strictly requires $h \leq 1$. Under this condition, we analyze the bounds of $\mathcal{A}_n$ and $\mathcal{B}_n$.

First, the coefficient $\mathcal{A}_n$ determines the propagation of historical errors. We evaluate its magnitude $|\mathcal{A}_n|$ across mainstream diffusion configurations:
\begin{itemize}
    \item \textbf{EDM and VE-SDE Schedules ($\alpha_t \equiv 1$):}
    The signal scale ratio is unity ($\frac{\alpha_{t_{n+1}}}{\alpha_{t_n}} = 1$). The coefficient simplifies to:
    \begin{equation}
        |\mathcal{A}_n| = |1 - \Delta \lambda_n|.
    \end{equation}
    Since $\Delta \lambda_n$ is a small positive value determined by the small step size, we have $|\mathcal{A}_n| < 1$. This indicates strictly contractive dynamics, where historical errors naturally decay over time.

    \item \textbf{VP-SDE Schedule ($\alpha_t$ varies):}
    Although $\alpha_t$ increases during the reverse process, the decay term $(1 - \Delta \lambda_n)$ counteracts this growth. Using the first-order approximation $1 - \Delta \lambda_n \approx e^{-\Delta \lambda_n}$ and the definition $e^{-\lambda} = \sigma/\alpha$:
    \begin{equation}
    \begin{aligned}
        \mathcal{A}_n
        &\approx \frac{\alpha_{t_{n+1}}}{\alpha_{t_n}}
                  e^{-(\lambda_{t_{n+1}} - \lambda_{t_n})} \\
        &= \frac{\alpha_{t_{n+1}}}{\alpha_{t_n}}
           \frac{\sigma_{t_{n+1}}/\alpha_{t_{n+1}}}{\sigma_{t_n}/\alpha_{t_n}} \\
        &= \frac{\sigma_{t_{n+1}}}{\sigma_{t_n}}.
    \end{aligned}
    \end{equation}
    Since the noise level $\sigma_t$ strictly decreases in the generative direction, the ratio satisfies $\frac{\sigma_{t_{n+1}}}{\sigma_{t_n}} < 1$. Thus, $|\mathcal{A}_n| < 1$ holds for VP schedules as well.
\end{itemize}
The error propagation is non-expansive ($|\mathcal{A}_n| \leq 1$), ensuring that prior errors $\boldsymbol{\Delta}_{t_n}$ do not explode during the sequential process.

Second, the coefficient $\mathcal{B}_n$ governs the sensitivity of the trajectory to the non-ideal estimator error $\boldsymbol{\xi}_n$.
Using the Mean Value Theorem, the log-SNR increment is proportional to the step size: $\Delta \lambda_n \approx |\dot{\lambda}| h$. Since the noise schedule derivatives are bounded, a small step size $h \leq 1$ implies $\Delta \lambda_n \ll 1$.

We bound $|\mathcal{B}_n|$ as follows:
\begin{itemize}
    \item Since $|\alpha_{t_{n+1}}| \leq 1$ (strictly 1 for EDM/VE, $\leq 1$ for VP) and $\Delta \lambda_n > 0$:
    \begin{equation}
        |\mathcal{B}_n| \leq \Delta \lambda_n \left| 1 - \frac{\Delta \lambda_n}{2} \right|.
    \end{equation}
    \item In the practical regime where $h \leq 1$, we maintain $\Delta \lambda_n < 2$, ensuring the term $(1 - \Delta \lambda_n/2)$ is strictly less than 1. Thus:
    \begin{equation}
        |\mathcal{B}_n| < \Delta \lambda_n \approx \mathcal{O}(h).
    \end{equation}
\end{itemize}
Even in few-step generation scenarios (e.g., $N=4, h \approx 0.25$), $|\mathcal{B}_n|$ remains a small fractional scalar. As $h \to 0$, $|\mathcal{B}_n|$ approaches zero linearly. This proves that the solver \textit{suppresses} the injection of the estimation error $\boldsymbol{\xi}_n$ rather than accumulating it linearly.

\textbf{Accumulated bound.} Combining the established bounds with \eqref{eq:recurrence_raw}, the per-step error inequality is
\begin{equation}
    \|\boldsymbol{\Delta}_{t_{n+1}}\| \leq \underbrace{|\mathcal{A}_n|}_{\leq 1} \|\boldsymbol{\Delta}_{t_n}\| + \underbrace{|\mathcal{B}_n|}_{\mathcal{O}(h)} \|\boldsymbol{\xi}_n\| + \underbrace{\|\mathcal{Z}_n\|}_{\zeta\,\mathcal{O}(h)}.
\end{equation}
Starting from the exact initial condition $\boldsymbol{\Delta}_{t_0}=\mathbf{0}$, unrolling the recurrence over the $N$ discretization steps yields the \emph{accumulated error bound}:
\begin{equation}
\label{eq:accumulated_bound}
\begin{aligned}
\|\boldsymbol{\Delta}_{t_N}\|
&\leq \sum_{n=0}^{N-1}\Bigl(\prod_{j=n+1}^{N-1}|\mathcal{A}_j|\Bigr)\bigl(|\mathcal{B}_n|\|\boldsymbol{\xi}_n\|+\|\mathcal{Z}_n\|\bigr)\\
&\leq \sum_{n=0}^{N-1}\bigl(|\mathcal{B}_n|\|\boldsymbol{\xi}_n\|+\|\mathcal{Z}_n\|\bigr) \leq N\,\mathcal{O}(h)\bigl(\|\boldsymbol{\xi}\|_{\max}+\zeta\bigr),
\end{aligned}
\end{equation}
where the second inequality uses $|\mathcal{A}_j|\leq 1$ and the third substitutes $|\mathcal{B}_n|=\mathcal{O}(h)$ and $\|\mathcal{Z}_n\|=\zeta\,\mathcal{O}(h)$ from Proposition~1. With $h=\mathcal{O}(1/N)$, this reduces to $\mathcal{O}(\|\boldsymbol{\xi}\|_{\max}+\zeta)$; for a 50-step schedule with $h=1/50$ and $\zeta\sim 10^{-2}$ the bound is on the order of $10^{-2}$, negligible relative to an $O(1)$ signal scale.

\textbf{Dependence on $\varrho$, SOR, and adaptive stopping.} The bound is \emph{independent of $\varrho$ in structure}: $\varrho$ only changes the source of the estimator $\mathbf{E}_{t_n}$ (and hence the magnitudes $\|\boldsymbol{\xi}_n\|$), not the propagation/sensitivity coefficients $\mathcal{A}_n$ and $\mathcal{B}_n$. The successive-over-relaxation update of ParaLin++ shares the same fixed point and is stable for $0<\omega<2/(1+L)$ (Proposition~4), so it obeys the same per-step recurrence and hence the same accumulated bound; SOR only accelerates the local contraction. The adaptive stopping criterion (Section~IV-B of the main manuscript) replaces the exact fixed point by a $\delta$-neighborhood, adding an additional $O(\delta)$ truncation, so the realized deviation satisfies
\begin{equation}
    \|\boldsymbol{\Delta}_{t_N}^{(\mathrm{ParaLin}++)}\| \leq N\,\mathcal{O}(h)\bigl(\|\boldsymbol{\xi}\|_{\max}+\zeta\bigr) + \mathcal{O}(\delta).
\end{equation}

\textbf{Conclusion.} Hence the $\varrho$-nonlinear surrogate trajectory $\{\hat{\mathbf{X}}_{t_n}\}$ deviates from the exact sequential trajectory $\{\mathbf{X}_{t_n}\}$ by a bounded, small quantity, \emph{not} by zero. ParaLin converges to the surrogate system rather than exactly to the nonlinear sequential sampler; this is the gap bounded in Proposition~2 and the corrected convergence target formalized in Proposition~3.
\end{proof}


\section{Proof of Convergence in Proposition~3}
\label{apx:proof_convergence}

\noindent Let $\{\hat{\mathbf{X}}^{\star}_{t_n}\}_{n=0}^N$ denote the unique solution of the $\varrho$-nonlinear surrogate system in~\textup{Eq.~(11)} (the partially linear surrogate trajectory), which by Proposition~2 deviates from the exact sequential trajectory $\{\mathbf{X}_{t_n}\}$ by the accumulated error bound~\eqref{eq:accumulated_bound}. We show that the iteration in~\textup{Eq.~(18)} converges to $\{\hat{\mathbf{X}}^{\star}_{t_n}\}$, \emph{not} to $\{\mathbf{X}_{t_n}\}$.

\noindent Assume by induction that for iteration $k$, the first $n$ components have converged to the surrogate, $\{\hat{\mathbf{X}}_{t_i}^{(k)}\}_{i=0}^{n-1}=\{\hat{\mathbf{X}}^{\star}_{t_i}\}_{i=0}^{n-1}$ (and, when $n\geq\varrho$, also $\hat{\mathbf{X}}_{t_{\varrho-1}}^{(k)}=\hat{\mathbf{X}}^{\star}_{t_{\varrho-1}}$, since $\varrho-1<n$). For the next component $\hat{\mathbf{X}}_{t_n}^{(k+1)}$, define the iterate and surrogate maps
\begin{equation}
\begin{aligned}
\mathcal{B}_n^{(k)}
  &= \mathcal{H}\Bigl(t_n, t_{n-1}, \hat{\mathbf{X}}_{t_{n-1}}^{(k)}, \mathbf{E}_{t_{n-1}}^{(k)}\Bigr),\\
\mathcal{B}_n^{\star}
  &= \mathcal{H}\Bigl(t_n, t_{n-1}, \hat{\mathbf{X}}^{\star}_{t_{n-1}}, \mathbf{E}^{\star}_{t_{n-1}}\Bigr),
\end{aligned}
\end{equation}
where $\mathbf{E}_{t_{n-1}}^{(k)}=\hat{\mathbf{X}}_{t_N}(\hat{\mathbf{X}}_{t_{n-1}}^{(k)},\theta)$ for $n<\varrho$ and $\mathbf{E}_{t_{n-1}}^{(k)}=\hat{\mathbf{X}}_{t_N}(\hat{\mathbf{X}}_{t_{\varrho-1}}^{(k)},\theta)$ for $n\geq\varrho$ (Definition~2), and $\mathbf{E}^{\star}_{t_{n-1}}$ is its surrogate analogue. By the surrogate equation in~\textup{Eq.~(11)}, $\hat{\mathbf{X}}^{\star}_{t_n}=\mathcal{B}_n^{\star}$. Therefore
\begin{equation}
\begin{aligned}
\hat{\mathbf{X}}_{t_n}^{(k+1)}
  &= \hat{\mathbf{X}}_{t_n}^{(k)} - \mathcal{G}^{(k)} \mathcal{R}_{t_n}^{(k)}
   = \hat{\mathbf{X}}_{t_n}^{(k)} - \mathcal{G}^{(k)}\bigl(\hat{\mathbf{X}}_{t_n}^{(k)} - \mathcal{B}_n^{(k)}\bigr)\\
  &\overset{(a)}{=} \hat{\mathbf{X}}_{t_n}^{(k)} - \mathcal{G}^{(k)}\bigl(\hat{\mathbf{X}}_{t_n}^{(k)} - \mathcal{B}_n^{\star}\bigr)\\
  &\overset{(b)}{=}\; \mathcal{B}_n^{\star} = \hat{\mathbf{X}}^{\star}_{t_n},
\end{aligned}
\end{equation}
with
\begin{itemize}[leftmargin=*]
\item Eq.~(a) from the induction hypothesis: $\hat{\mathbf{X}}_{t_{n-1}}^{(k)}=\hat{\mathbf{X}}^{\star}_{t_{n-1}}$ (and $\hat{\mathbf{X}}_{t_{\varrho-1}}^{(k)}=\hat{\mathbf{X}}^{\star}_{t_{\varrho-1}}$ for $n\geq\varrho$) gives $\mathcal{B}_n^{(k)}=\mathcal{B}_n^{\star}$;
\item Eq.~(b) since $\mathcal{G}^{(k)}=\mathbf{I}$ or $\mathcal{G}^{(k)}=\mathcal{J}^{-1}=\mathbf{I}$, where $\mathcal{J}=\partial(\hat{\mathbf{X}}_{t_n}^{(k)}-\mathcal{B}_n^{\star})/\partial\hat{\mathbf{X}}_{t_n}^{(k)}=\mathbf{I}$, because $\mathcal{B}_n^{\star}$ depends on $\hat{\mathbf{X}}_{t_n}^{(k)}$ only through the already-converged predecessors ($\hat{\mathbf{X}}_{t_{n-1}}^{\star}$ and, for $n\geq\varrho$, $\hat{\mathbf{X}}_{t_{\varrho-1}}^{\star}$), not through $\hat{\mathbf{X}}_{t_n}^{(k)}$ itself.
\end{itemize}
Thus each iteration advances the convergence front by at least one index, giving $K\leq N$ in the worst case (and $K\ll N$ in practice via immediate propagation of newly converged components).

\noindent\textbf{Convergence target.} The argument establishes convergence of the iterative solve to $\{\hat{\mathbf{X}}^{\star}_{t_n}\}$, i.e., to the $\varrho$-nonlinear surrogate, within $K\leq T$ iterations. It does \emph{not} reach the exact sequential trajectory: $\{\hat{\mathbf{X}}^{\star}_{t_n}\}\neq\{\mathbf{X}_{t_n}\}$ in general, and the residual gap $\|\hat{\mathbf{X}}^{\star}_{t_n}-\mathbf{X}_{t_n}\|$ is the accumulated error bounded in Proposition~2 (\eqref{eq:accumulated_bound}). The SOR update keeps this same fixed point (Proposition~4), and the adaptive stopping reintroduces only an $O(\delta)$ deviation.

\section{Proof of Proposition~4}
\label{apx:proof_Optimal_Relaxation_Factor}
\begin{proof}
Let the iterative solver described by the immediate update strategy in~\textup{Eq.~(18)} be represented by a  mapping $\mathcal{F}: \mathbb{R}^{d \times N} \to \mathbb{R}^{d \times N}$, where $d$ is the data dimension and $N$ is the number of discretization steps. The update rule can be written as:
$$
\hat{\mathbf{X}}^{(k+1)} = \mathcal{F}(\hat{\mathbf{X}}^{(k)}).
$$
For the iterative method to converge, the mapping $\mathcal{F}$ must be a contraction in a neighborhood of the unique solution $\mathbf{X}^*$. This means there exists a constant $L < 1$ such that for any two trajectories $\mathbf{X}$ and $\mathbf{Y}$ in this neighborhood, the following holds:
$$
\|\mathcal{F}(\mathbf{X}) - \mathcal{F}(\mathbf{Y})\| \le L \|\mathbf{X} - \mathbf{Y}\|.
$$
The SOR update modifies this iteration. Let the operator for the SOR-enhanced update be $T_{SOR}$. From~\textup{Eq.~(21)}, we can write the SOR update as:
$$
\hat{\mathbf{X}}^{(k+1)} = T_{SOR}(\hat{\mathbf{X}}^{(k)}) = (1-\omega)\hat{\mathbf{X}}^{(k)} + \omega \mathcal{F}(\hat{\mathbf{X}}^{(k)}).
$$
To analyze the convergence of the SOR method, we must show that $T_{SOR}$ is also a contraction mapping. We analyze its contraction factor, $L_{SOR}$:
\begin{align*}
\|T_{SOR}(\mathbf{X}) - T_{SOR}(\mathbf{Y})\|
  &= \bigl\| (1-\omega)(\mathbf{X} - \mathbf{Y})\\
  &\qquad + \omega (\mathcal{F}(\mathbf{X}) - \mathcal{F}(\mathbf{Y})) \bigr\|.
\end{align*}
Using the triangle inequality, we can bound this expression:
\begin{align*}
\|T_{SOR}(\mathbf{X}) - T_{SOR}(\mathbf{Y})\|
  &\le |1-\omega| \|\mathbf{X} - \mathbf{Y}\|\\
  &\qquad + \omega \|\mathcal{F}(\mathbf{X}) - \mathcal{F}(\mathbf{Y})\|\\
  &= \bigl(|1-\omega|+\omega L\bigr)\|\mathbf{X} - \mathbf{Y}\|.
\end{align*}
Since $\mathcal{F}$ is a contraction with factor $L$, we have:
$$
\|T_{SOR}(\mathbf{X}) - T_{SOR}(\mathbf{Y})\| \le (|1-\omega| + \omega L) \|\mathbf{X} - \mathbf{Y}\|.
$$
Thus, the contraction factor for the SOR iteration is $L_{SOR} = |1-\omega| + \omega L$. For convergence, we require $L_{SOR} < 1$. We analyze this condition in two cases based on the value of $\omega$.

\textbf{Case 1: Under-relaxation ($0 < \omega \le 1$)}
In this range, $|1-\omega| = 1-\omega$. The contraction factor becomes:
$$
L_{SOR} = (1-\omega) + \omega L = 1 - \omega(1-L).
$$
Since the base method is assumed to converge, we have $L < 1$, which implies $1-L > 0$. As $\omega > 0$, the term $\omega(1-L)$ is strictly positive. Therefore, $L_{SOR} = 1 - (\text{positive value}) < 1$.
Convergence is guaranteed for any $\omega \in (0, 1]$.

\textbf{Case 2: Over-relaxation ($\omega > 1$)}
In this range, $|1-\omega| = \omega-1$. The contraction factor becomes:
$$
L_{SOR} = (\omega-1) + \omega L = \omega(1+L) - 1.
$$
For convergence, we require $L_{SOR} < 1$:
$$
\omega(1+L) - 1 < 1 \implies \omega(1+L) < 2 \implies \omega < \frac{2}{1+L}.
$$
Combining the condition $\omega > 1$ with this derived upper bound, convergence for over-relaxation is guaranteed when $1 < \omega < \frac{2}{1+L}$.

Combining both cases, the SOR-enhanced solver is guaranteed to converge if the relaxation factor $\omega$ is chosen in the range:
$$
0 < \omega < \frac{2}{1+L}.
$$
Since $0 \le L < 1$, the upper bound $2/(1+L)$ is always in the range $(1, 2]$.
This analysis provides the effective range for $\omega$ and shows that over-relaxation can maintain convergence under a computable condition related to the contractiveness of the underlying solver.
\end{proof}

\vfill

\vfill

\bibliographystyle{IEEEtranS}
\bibliography{paralin_plus}